\documentclass{report}
\usepackage[online]{suthesis-2e}
\dept{Statistics}

\usepackage[T1]{fontenc}

\usepackage{adjustbox}

\usepackage{natbib}
\usepackage[table]{xcolor}
\usepackage{graphicx}
\usepackage{url}
\usepackage{amsmath}
\usepackage{amssymb}
\usepackage{bm}
\usepackage{enumitem}
\usepackage{booktabs}
\usepackage{multirow}
\usepackage{colortbl}
\usepackage{wrapfig}
\usepackage{caption}
\usepackage{tabularx}
\usepackage{tcolorbox}
\tcbuselibrary{skins,breakable,listings}
\usepackage{listings}

\newtcblisting{CodeBlock}[2][]{%
  listing only,
  breakable,
  colback=white,
  colframe=black,
  title=#2,
  listing options={
    basicstyle=\ttfamily\footnotesize,
    breaklines=true,
    numbers=left,
    numberstyle=\tiny,
    numbersep=8pt,
    xleftmargin=12pt,
    framexleftmargin=12pt,
    columns=fullflexible,
    keepspaces=true
  },
  #1
}

\lstdefinestyle{pseudocode}{
  language=Python,
  basicstyle=\ttfamily\small,
  keywordstyle=\bfseries,
  commentstyle=\color{thesisGrey}\itshape,
  stringstyle=\color{thesisGrey},
  breaklines=true,
  columns=fullflexible,
  keepspaces=true,
  showstringspaces=false,
  morekeywords={self, Callable, pass},
  deletekeywords={exec}
}
\usepackage{subfigure}

\DeclareMathOperator*{\argmax}{arg\,max}

\newcommand{\<}{\langle}
\renewcommand{\>}{\rangle}
\def\Dpre{{\mathcal{D}_\text{pretrain}}}
\def\Dst{{\mathcal{D}_\text{ST}}}
\def\thetast{{\theta_\text{ST}}}
\def\Spre{{\mathcal{S}_\text{pretrain}}}
\def\perf{{\mathrm{Perf}}}
\def\docone{{d_1}}
\def\doctwo{{d_2}}
\def\doc{{d}}
\definecolor{thesisBlack}{HTML}{000000}
\definecolor{thesisDarkGarnet}{HTML}{5C0E02}
\definecolor{thesisRichMahogany}{HTML}{2E0701}
\definecolor{thesisRosyTaupe}{HTML}{AE8781}
\definecolor{thesisGrey}{HTML}{808080}

\definecolor{appleLightGray}{HTML}{EEEEEE}

\usepackage{amsthm}
\usepackage{dsfont}
\usepackage{pifont}
\newtcolorbox{qualitativeBox}{
  enhanced,
  breakable,
  colback=gray!10,
  colframe=gray!30,
  boxrule=0.5pt,
  left=10pt,
  right=10pt,
  top=5pt,
  bottom=5pt,
  boxsep=5pt,
}
\renewcommand{\P}{\mathbb{P}}
\newcommand{\Ds}{\mathcal{D}_{\text{source}}}
\newcommand{\Qt}{\mathcal{Q}_{\text{test}}}
\newcommand{\Denti}{\mathcal{D}_{\text{EntiGraph}}}
\newcommand{\lmgen}{\mathsf{LM}_{\text{aug}}}
\newcommand{\lmsynth}{\mathsf{LM}_{\text{synth}}}

\newcommand{\cV}{\mathcal{V}}
\def\cD{{\mathcal D}}

\def\cR{{\mathcal R}}
\def\cS{{\mathcal S}}
\def\cL{{\mathcal L}}
\def\bM{{\boldsymbol M}}
\def\bI{{\boldsymbol I}}
\def\quality{QuALITY}
\def\Algsynth{\mathcal{A}_{\text{synth}}}
\def\Dsynth{\mathcal{D}_{\text{synth}}}
\newtheorem{theorem}{Theorem}
\newtheorem{lemma}{Lemma}[section]
\theoremstyle{definition}
\newtheorem{definition}{Definition}
\newlength{\bigfill}
\newlength{\smallfill}
\usepackage{algorithm}
\usepackage{algpseudocode}
\usepackage{float}
\usepackage{placeins}

\usepackage{makecell}
\usepackage{pdfpages}
\usepackage{hyperref}
\hypersetup{colorlinks=true, citecolor=thesisGrey, linkcolor=black, urlcolor=black, filecolor=black, pdftitle={Continually self-improving AI}, pdfauthor={Zitong Yang}}

\newcommand{\sonek}{\textbf{s1K}}
\newcommand{\sone}{\textbf{s1-32B}}

\definecolor{defaultblue}{HTML}{5C0E02}
\definecolor{defaultlightblue}{HTML}{AE8781}

\usepackage{textcomp}
\usepackage{array}
\newcolumntype{L}[1]{>{\raggedright\arraybackslash}p{#1}}

\makeatletter
\renewcommand\paragraph{\@startsection{paragraph}{4}{\z@}%
  {-3.25ex \@plus -1ex \@minus -.2ex}%
  {-1em}%
  {\normalfont\normalsize\bfseries}}
\makeatother

\makeatletter
\def\titlep{%
        \thispagestyle{empty}%
        \null\vskip1in%
        \begin{center}
                \large\uppercase\expandafter{\@title}
        \end{center}
        \vfill
        \begin{center}
                \large
                A DISSERTATION\\
                \uppercase\expandafter{\@whichsub}\\
                OF STANFORD UNIVERSITY\\
                IN PARTIAL FULFILLMENT OF THE REQUIREMENTS\\
                FOR THE DEGREE OF\\
                DOCTOR OF PHILOSOPHY
        \end{center}
        \vfill
        \begin{center}
                \rm
                Committee in charge:\\
                \vskip0.5em
                \if*\@coprincipaladviser
                        \@principaladviser, Principal \advis@r\\
                \else
                        \@principaladviser, Co-chair\\
                        \@coprincipaladviser, Co-chair\\
                \fi
                \if*\@firstreader \else
                        \@firstreader\\
                \fi
                \if*\@secondreader \else
                        \@secondreader\\
                \fi
        \end{center}
        \vfill
        \begin{center}
                \rm \@author\\
                \@submitdate\\
        \end{center}\vskip.5in\newpage}
\makeatother

\begin{document}
\title{Continually self-improving AI}
\author{Zitong Yang}
\principaladvisor{Emmanuel Candès}
\coprincipaladvisor{Tatsunori Hashimoto}
\firstreader{Percy Liang}
\secondreader{Ruoming Pang}
 
\beforepreface
\setcounter{page}{2}
\includepdf[pages=-]{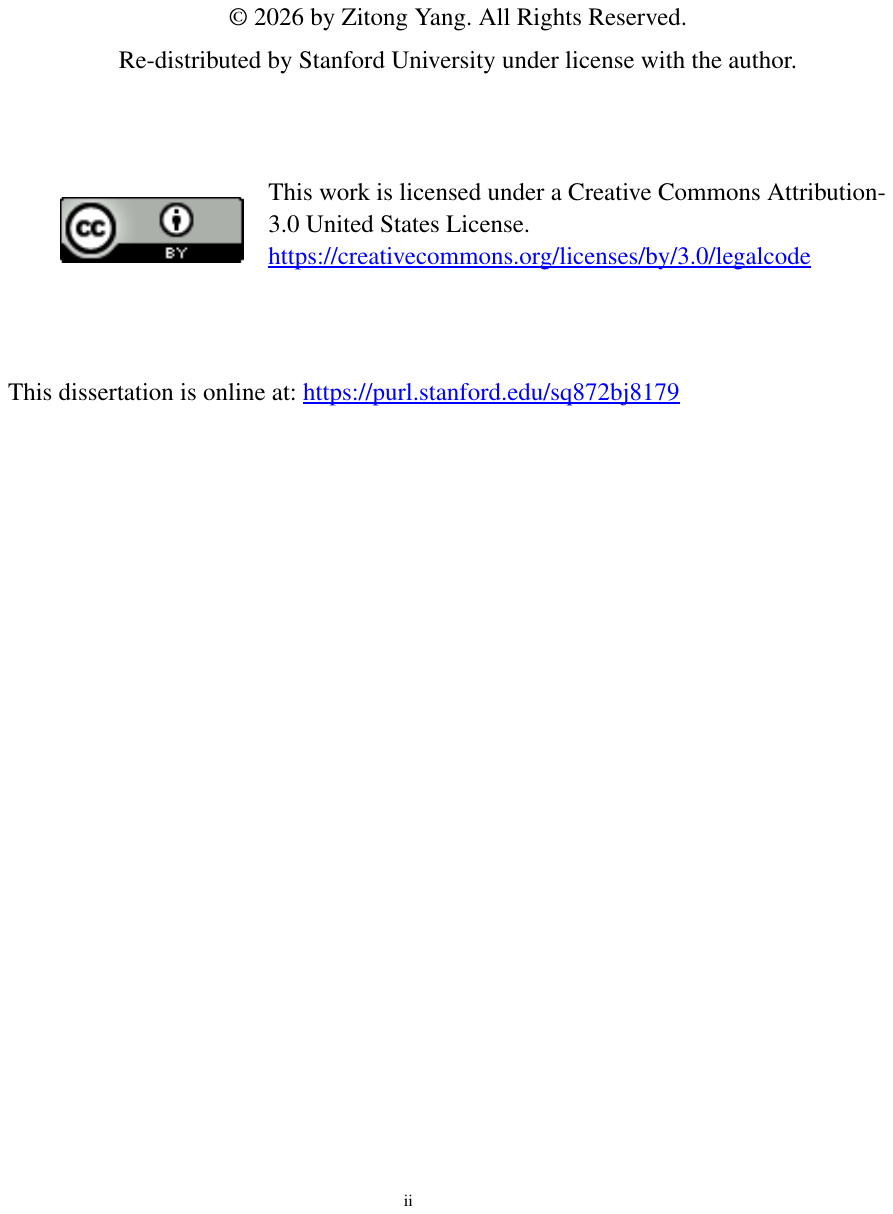}
\prefacesection{Abstract}

Modern language model-based AI systems are remarkably powerful, yet their capabilities remain fundamentally capped by their human creators in three key ways.
First, although a model's weights can be updated via fine-tuning, acquiring new knowledge from small, specialized corpora after pretraining remains highly data-inefficient.
Second, the training of these systems relies heavily on finite, human-generated data from across history.
Third, the pipelines used to train AI models are confined by the algorithms that human researchers can discover and explore.

This thesis takes a small step toward overcoming these inherent limitations, presenting three chapters aimed at breaking these dependencies to create continually self-improving AI.
First, to overcome this data-efficiency barrier in knowledge acquisition, we propose a synthetic data approach that diversifies and amplifies small corpora into rich knowledge representations, enabling a model to effectively update its parameters from limited source material.
Second, to reduce reliance on human data, we show that given a fixed amount of such data, the model can self-generate synthetic data to bootstrap its fundamental pretraining capabilities without distillation from any off-the-shelf, instruction-tuned LM.
Finally, to transcend human-engineered training paradigms, we demonstrate that by scaling search during test time over the space of algorithms, AI can search over a larger space of learning algorithm configurations than human researchers can explore manually.

\prefacesection{Acknowledgments}
Being a Ph.D. student in AI from 2022 to 2025 at Stanford, CA, is a once-in-a-lifetime experience.
The rate at which the surrounding environments change is exhilaratingly opportunistic yet mentally taxing.
Every few months, we see the valley's capital frenzy occupying the news headlines.
Every few weeks, we see a new model release with uncanny capability from a 2016-2020 perspective.
Every few days, we see a new paper announcement delivering similar progress to ours.

Amid an ever-changing appearance lies the unchanging reality.
For me, this reality is my wife, Angie, and my advisors, Emmanuel Candès and Tatsunori Hashimoto.

Thank you, Angie, for being my emotional anchor over the past 10 years, for guiding me, comforting me, soothing me, empowering me, and encouraging me to explore far into the unseen by offering me a tranquil harbor where I can introspect.

Thank you, Emmanuel, for shaping my approach to research and to life, for pushing me to pursue the problem I am truly excited about, for teaching me the history of science, for pushing me toward thoughtful work, and for showing me, by example, the unwavering will to explore.

Thank you, Tatsu, for welcoming me to the world of natural language processing, for instilling in me the value of introspection, for walking me through every obstacle, for planting in me the seed of rigor, and for building a research lab so vibrant that I couldn’t possibly hope for more.

I would like to thank Percy Liang and Ruoming Pang for being on my thesis committee.
Thank you, Percy, for pushing me to always form opinions by running my own experiments.
Thank you, Ruoming, for teaching me the perspective to view computer systems through the lens of hardware.

I would like to extend my gratitude to Berkeley, where my research started.
I would like to thank Yi Ma and Jacob Steinhardt for bringing me to the world of machine learning.
Thank Song Mei, Chong You, Yaodong Yu, and Yuexiang Zhai, for mentoring me hand-in-hand through my first few research projects.
Thank Jiabao Yang for all the discussions of physics we had over the years.
I am grateful to all the friends, collaborators, and mentors from my time at Berkeley: Christina Baek, Ryan Chan, Chih-Yuan Chiu, Xili Dai, Sara Fridovich-Keil, Zhen Guo, Zhiyue Hu, Jiantao Jiao, Druv Pai, Haozhi Qi, Shengbang Tong, Alex Wei, Eric Xia, Chiao-Yu Yang, Jiahao Yao, Chong You, Yue Zhang, Ruiqi Zhong, and Banghua Zhu.

Upon entering Stanford, Jacob Steinhardt encouraged me to explore both Sequoia and Gates.
I am grateful to have two buildings I can call home.

I would like to thank Sourav Chatterjee and Amir Dembo for enabling me to appreciate the wisdom of Kolmogorov.
Thank John Duchi for teaching me the way of information theory and optimization.
Thank Lihua Lei for mentoring my first paper at Stanford.
Thank Andrea Montanari for chairing the department over the past three years and helping me with various requests.
Thank Susie Ementon for working with me through all the unprecedented logistics hassles.
Thank all the fellow residents of Sequoia over the years: Kelly Buchanan, Chen Cheng, Gary Cheng, Zhaomeng Cheng, John Cherian, Noah Cowan, Brice Huang, Andrew Ilyas, Wenlong Ji, Ying Jin, Joon Lee, Jinzhou Li, Shuangping Li, Gennie Ma, Yash Nair, Michael Salerno, Ziang Song, Asher Spector, Zihao Wang, Yao Zhang, and Tijana Zrnic.

I would like to thank the wonderful staff of CS336 for their help in getting started with language model research.
Thank Neil Band for building the world's best reranker.
Thank Chenglei Si for guiding me to the world of automating AI research.
Thank Diyi Yang and Ludwig Schmidt for all the discussions about society and AI.
Thank all the fellow residents of Gates for being around over the years: Jiaao Chen, Ian Covert, Yann Dubois, Mingjian Jiang, Xiang Lisa Li, Xinhao Li, Xuechen Li, Hong Liu, Ken Liu, Niklas Muennighoff, Sam Park, Chenglei Si, Yu Sun, Rohan Taori, Tristan Thrush, and Tianyi Zhang.

I would also like to thank mentors and friends from the industry.
Thank Aonan Zhang and Dong Yin from Apple for hosting me and working with me through our research.
Thank Aditya Menon and Sanjiv Kumar from Google for welcoming me and showing the power of retrieval.
Thank Zonglin Li from Anthropic for all the AGI discussions over the years.
Thank John Schulman for encouraging me to work on unorthodox problems.

Finally, I would like to thank friends outside Stanford and work.
Thank you, Samy Jelassi, for working with me on hard problems and for our discussion over the years.
Thank you, Sam Buchanan, Tianle Cai, Danqi Chen, Tianzhe Chu, Weijie Su, Xuyang Tian, Hongtao Yao, Liang Yuan, and Mingxuan Zuo for your friendship.

Finally, thank my parents P.~W. and G.~Y. for everything.

\afterpreface

\chapter{Introduction}
\label{chap:intro}
\section{Defining continually self-improving AI}
\label{sec:intro-definition}

In a single sentence: \emph{a continually self-improving AI is one that, once created, can autonomously and continually improve itself better than its human creators can improve it.}
We state two assumptions that scope the definition to the class of AI systems studied in this thesis.

\begin{enumerate}[leftmargin=36pt, label=\textbf{(A\arabic*)}]
\item The AI system is based on one or more neural networks, so that its knowledge is encoded in a well-defined set of parametric weights.
\item There exists a resource-intensive pretraining phase during which the system is created:
\begin{equation}
\label{eqn:pretraining}
\texttt{ai\_system} = \texttt{learning\_algorithm}(\texttt{training\_signal}),
\end{equation}
where \texttt{training\_signal} is human knowledge (e.g., internet text), \texttt{learning\_algorithm} encompasses things like architecture (e.g., the Transformer) and optimizer (e.g., gradient descent), and \texttt{ai\_system} is the resulting model.
\end{enumerate}

These two assumptions clearly encompass the current large language model paradigm---Transformers trained via gradient descent on internet text---but they do not exclude non-Transformer architectures, non-gradient-descent optimizers, or non-textual training signals.
The definition captures any parametric system that undergoes an expensive creation phase followed by continued operation.

With these assumptions in place, we define a continually self-improving AI as one satisfying three properties.
Note that the pretraining formula already implies that improvement is data-driven---grounded in a learning signal rather than, say, hardware upgrades or manual weight surgery.

\begin{definition}[Continually self-improving AI]
\label{def:csiai}
Under Assumptions~(A1)--(A2), an AI system is \emph{continually self-improving} if it satisfies:
\begin{enumerate}[leftmargin=36pt, label=\textbf{(P\arabic*)}]
\item After the pretraining phase~\eqref{eqn:pretraining}, the system continues to acquire new knowledge built into its parametric weights without catastrophically forgetting existing capabilities.
\item The system generates its own \texttt{training\_signal}, and learning from these self-generated signals yields continued improvement beyond what existing human-generated signals provide.
\item The system autonomously determines what \texttt{learning\_algorithm} to use to learn from its training signals.
\end{enumerate}
\end{definition}

The assumptions are not arbitrary---each one makes a specific aspect of self-improvement well-defined.
Assumption~(A1) makes Property~(P1) meaningful: without weights, there is no substrate into which new knowledge can be written.
Assumption~(A2) establishes a \emph{pretraining phase} that creates the system, making ``improvement after the pretraining phase'' a precise concept---the system continues to update after this expensive initial phase.
The pretraining formula also makes explicit that three components exist---the model, the algorithm, and the data---each of which can be the target of improvement.

Each property corresponds to one chapter of this thesis:
\begin{itemize}[leftmargin=16pt]
\item \textbf{Chapter~\ref{chap:scp}} (Property~P1). Improving what the model knows, by synthesizing diverse representations of a small corpus for continued pretraining.
\item \textbf{Chapter~\ref{chap:sbp}} (Property~P2). Improving the system's fundamental pretraining capability, by exploiting inter-document correlations to strengthen pretraining itself.
\item \textbf{Chapter~\ref{chap:automated-ai-research}} (Property~P3). Improving the process by which models are trained, by scaling test-time search from the token level to the idea level---generating research ideas, executing them, and learning from the results.
\end{itemize}
Altogether, these three paths sketch a future where AI systems continually improve themselves.
While we do not claim that current AI has goals in a human sense, the drive toward self-improvement may be understood as an inherent goal for sufficiently capable systems---just as organisms evolve toward greater fitness without deliberate intent, AI systems that can improve their own training may represent a new kind of open-ended optimization.
We next summarize each chapter's contribution in turn.

\section{Continual knowledge acquisition}
\label{sec:intro-knowledge}

We first address Property~(P1) of Definition~\ref{def:csiai}: after pretraining, how can a language model continue learning from a small, specialized corpus?
This is a \emph{data-limited} problem: niche knowledge---proprietary datastores, specialized scientific domains, private corpora---inherently lacks the diverse internet representation that makes standard pretraining effective.
Several approaches address this problem: knowledge editing~\citep{rome, memit} modifies individual facts but does not scale to corpus-level knowledge; retrieval-augmented generation~\citep{rag} keeps knowledge external and is limited by the context window; and collecting more real data is often infeasible for proprietary or niche domains.
We pursue synthetic data generation because it operates at corpus scale, writes knowledge directly into the model's parametric weights, and remains applicable when real data cannot be obtained.
In Chapter~\ref{chap:scp}, we tackle two challenges---data efficiency and catastrophic forgetting---through \emph{synthetic continued pretraining}.
At a high level, we convert a small corpus into a large, diverse synthetic corpus using a knowledge graph--inspired augmentation algorithm called EntiGraph, and then continue pretraining on the expanded data while mixing in a fraction of the original pretraining distribution to prevent forgetting.

The approach is effective: a model trained on a synthetically augmented corpus acquires the knowledge of the original documents and demonstrates it across a range of downstream tasks.
However, the synthetic data generator we used was GPT-4---a model far more powerful than the student being trained.
We deliberately allowed this distillation because the scientific question in Chapter~\ref{chap:scp} is about \emph{data efficiency}---whether synthetic data can bridge the gap between a small corpus and the internet-like diverse representations that make learning effective---not about self-improvement.
But this raised an immediate follow-up: was the improvement genuine learning, or merely distillation from a stronger teacher?
This question motivated the next chapter.

\section{Bootstrapping pretraining capabilities}
\label{sec:intro-self-improvement}

The release of OpenAI o1 \citep{o1} pulled the field into reasoning models.
A natural question is: how much data is needed to elicit reasoning capabilities from a pretrained model?
In Chapter~\ref{chap:sbp}, \S\ref{sec:s1-sample-efficiency}, we show that the answer is strikingly little: training on just 1,000 carefully curated examples with reasoning traces suffices to build a competitive reasoning model.

The implication is clear.
A model cannot possibly acquire mathematical knowledge from 1,000 questions---the capability must already be latent in the pretrained weights.
Pretraining is therefore the foundation, and finetuning merely elicits what is already there.

Two threads now converge---the centrality of pretraining and the desire for genuine, not distilled, self-improvement---addressing Property~(P2) of Definition~\ref{def:csiai}.
Can a model trained on a fixed dataset generate synthetic data to train a \emph{better} model?
If so, this would constitute true self-improvement: no stronger teacher and no new information from the environment.
This question is set in a \emph{data-limited} regime as well, motivated by the approaching exhaustion of high-quality internet text~\citep{villalobos2024run}: we assume a fixed pool of unique documents and ask whether the model can extract more value from them than simple repetition provides.
The main alternatives are architectural changes, which are orthogonal and complementary to data-driven gains; retrieval-augmented pretraining~\citep{Borgeaud:2021, Khandelwal:2020}, which leverages related documents but keeps the additional signal external to the weights; and in-context pretraining~\citep{shi2024incontext}, which groups related documents into the same context window but is limited by context length.
We choose synthetic data because it creates new training signal from existing data and writes it directly into the model's weights via standard pretraining.

In Chapter~\ref{chap:sbp}, \S\ref{sec:sbp-introduction}, we show that the answer is yes.
Synthetic Bootstrapped Pretraining (SBP) trains a conditional data synthesizer that generates new training documents from existing ones---for instance, synthesizing a code tutorial from an arXiv paper, or a critical essay from a novel.
By training on these synthetic documents alongside the original corpus, SBP improves pretraining perplexity in compute-matched comparisons, closing up to 60\% of the gap to an oracle with access to unlimited unique data.

This result is qualitatively different from synthetic continued pretraining.
The defining constraint is that \emph{distillation is forbidden}: the data synthesizer is trained from the same pretraining corpus, not from a stronger external model.
Without this constraint, self-improvement would be trivially achievable by distilling from a more capable teacher.
SBP operates without any external teacher---the improvement stems from using the same data more efficiently by exploiting a weaker form of self-supervision latent in the pretraining corpus.
Because SBP improves perplexity---a fundamental quantity that correlates with all downstream tasks---the gain is not confined to niche benchmarks but reflects a genuine improvement in the model's core capability.

\section{Towards AI-designed AI via test-time search}
\label{sec:intro-algorithms}

AI research may be a domain where AI itself can deliver significant progress.
Consider the scientific method in the Popperian tradition.
Science proceeds in two steps: generating hypotheses and testing them with experiments to falsify them.
For mathematics, the execution step is special: the chain-of-thought that AI models produce implicitly carries out the verification, which helps explain the rapid progress AI has made in mathematical reasoning.
For AI research, execution materializes entirely as code---writing training scripts, launching experiments, logging metrics---and AI systems are already remarkably capable at code generation.
Meanwhile, idea generation takes place in natural language, which AI models handle fluently.
The natural design, then, is to connect an AI idea generator to an AI experiment executor end-to-end.
Established approaches to algorithmic improvement---Neural Architecture Search~\citep{Zoph2016NeuralAS, So2019TheET}, automated algorithm discovery~\citep{Real2020AutoMLZeroEM}, and learned optimizers~\citep{Chen2023SymbolicDO}---are effective within their respective domains but operate within constrained, predefined spaces or require end-to-end differentiable pipelines.
These are reasonable approaches, but they make it difficult to discover techniques outside the search space or scale to full training systems.
We pursue research automation because it operates in an unbounded action space---ideas expressed in natural language, validated via code execution---using capabilities that language models already possess.

A second observation from the reasoning work reinforces this direction.
In Chapter~\ref{chap:automated-ai-research}, \S\ref{sec:s1-budget-forcing}, we show that even a crude intervention---suppressing the end-of-thinking token to force longer reasoning, a technique we call \emph{budget forcing}---improves accuracy.
If brute-force thinking at the token level already helps, systematically scaling search at the \emph{idea} level---generating research ideas, executing them, and learning from the results---should yield further improvement.

In Chapter~\ref{chap:automated-ai-research}, we pursue this question by building an automated AI research system and applying test-time search at the idea level: generating research ideas, executing them automatically, and feeding results back to guide the next round of search.
This represents another flavor of self-improvement: not improving training data or model capability, but improving the training algorithm itself.

\section{Publications}
\label{sec:intro-publications}

This thesis is based on the following four publications, listed in reverse chronological order ($*$ denotes equal contribution):

\begin{enumerate}
\item \textbf{Towards Execution-Grounded Automated AI Research} \\
Chenglei Si$^*$, Zitong Yang$^*$, Yejin Choi, Emmanuel Cand\`es, Diyi Yang, Tatsunori Hashimoto. \\
\textit{arXiv preprint}, 2026.

\item \textbf{Synthetic Bootstrapped Pretraining} \\
Zitong Yang$^*$, Aonan Zhang$^*$, Hong Liu, Tatsunori Hashimoto, Emmanuel Cand\`es, Chong Wang, Ruoming Pang. \\
\textit{International Conference on Learning Representations (ICLR)}, 2026.

\item \textbf{S1: Simple Test-Time Scaling} \\
Niklas Muennighoff$^*$, Zitong Yang$^*$, Weijia Shi$^*$, Xiang Lisa Li$^*$, Li Fei-Fei, Hannaneh Hajishirzi, Luke Zettlemoyer, Percy Liang, Emmanuel Cand\`es, Tatsunori Hashimoto. \\
\textit{Empirical Methods in Natural Language Processing (EMNLP, Oral)}, 2025.

\item \textbf{Synthetic Continued Pretraining} \\
Zitong Yang$^*$, Neil Band$^*$, Shuangping Li, Emmanuel Cand\`es, Tatsunori Hashimoto. \\
\textit{International Conference on Learning Representations (ICLR, Oral)}, 2025.
\end{enumerate}
\noindent Table~\ref{tab:provenance} maps each thesis section to the publication it is based on.
Sections marked ``Original'' were written for this thesis and do not appear in any prior publication.

\begin{table}[ht]
\centering
\caption{Provenance of thesis sections. Publication numbers refer to the list above.}
\label{tab:provenance}
\small
\begin{tabular}{ll}
\toprule
\textbf{Section} & \textbf{Source} \\
\midrule
Chapter~\ref{chap:intro}: Introduction (\S\ref{sec:intro-definition}--\S\ref{sec:intro-publications}) & Original \\
\quad \S\ref{sec:intro-related-work} Related work & Adapted from Publications 1--4 \\
\midrule
Chapter~\ref{chap:scp}: Continual knowledge acquisition & Publication 4 \\
\midrule
Chapter~\ref{chap:sbp}: Bootstrapping pretraining capabilities & \\
\quad \S3.1 Prelude: sample-efficient reasoning & Publication 3 \\
\quad \S3.2--3.7 Synthetic bootstrapped pretraining & Publication 2 \\
\midrule
Chapter~\ref{chap:automated-ai-research}: Towards AI-designed AI & \\
\quad \S4.1--4.3 Automated AI research system & Publication 1 \\
\quad \S4.4 Test-time scaling via budget forcing & Publication 3 \\
\quad \S4.5--4.6 Evolutionary search \& discussion & Publication 1 \\
\midrule
Chapter~\ref{chap:conclusion}: Conclusion & Original \\
\midrule
Appendix A: Supplementary for Chapter~\ref{chap:scp} & Publication 4 \\
Appendix B: Supplementary for Chapter~\ref{chap:sbp} & \\
\quad \S B.1--B.3 & Publication 2 \\
\quad \S B.4 & Publication 3 \\
Appendix C: Supplementary for Chapter~\ref{chap:automated-ai-research} & \\
\quad \S C.1--C.2 & Publication 1 \\
\quad \S C.3 & Publication 3 \\
Appendix D: Supplementary for Chapter~\ref{chap:conclusion} & Original \\
\bottomrule
\end{tabular}
\end{table}
\section{Related work}
\label{sec:intro-related-work}

We consolidate the related work for all three chapters in this section, organized by the corresponding chapter topic.

\subsection{Continual knowledge acquisition}
\label{sec:related-work}

Several approaches address knowledge acquisition in language models---knowledge editing, retrieval-augmented generation, and collecting more real data---each with different trade-offs in scale, permanence, and applicability.
We pursue synthetic data generation because it operates at corpus scale, writes knowledge into parametric weights, and applies when real data cannot be obtained; the following review provides evidence for this choice.

\paragraph{Synthetic generation of pretraining data} Recent approaches synthesize \emph{pretraining} data using hierarchical prompting to promote dataset diversity.
\cite{eldan2023tinystories} prompt LLMs to generate stories containing sampled keywords and show that small LMs trained on their dataset generate fluent text.
\cite{gunasekar2023textbooksneed} synthesize textbooks and code exercises by conditioning on topic, target audience, and function names, later releasing strong LMs pretrained on synthetic data \citep{li2023textbooksneediiphi15, phi2, abdin2024phi3technicalreporthighly}; their datasets and prompts are not public.
\cite{wrap} prompt an LM to rephrase documents for pretraining, improving training efficiency.
In contrast, we focus on teaching a pretrained LLM the knowledge of a small corpus.
\cite{mecklenburg2024injectingnewknowledgelarge} consider task-specific finetuning and propose a fact-based synthetic QA generation procedure but do not show improvement on generic instruction following.
In contrast, we focus on teaching a model generally useful knowledge about a small corpus, untied to any particular downstream task.
\cite{ovadia2024finetuningretrievalcomparingknowledge} continually pretrain Llama 2--based LMs on synthetic paraphrases of Wikipedia articles but do not observe consistent improvements.
We adapt the approach of \cite{wrap} and \cite{mecklenburg2024injectingnewknowledgelarge} to our small corpus setting (``Rephrase baseline'' in Chapter~\ref{chap:scp}, \S\ref{sec:exp-main}).
Our graph-based augmentation algorithm outperforms this baseline; we postulate that the improvement stems from enforcing diversity through entity-based generation.

\paragraph{Continued pretraining} Continual or continued \emph{pretraining} methods \citep{gururangan2020dont} adapt pretrained LMs to broad target domains such as code, medicine, or mathematics by collecting massive datasets (often $>$100B tokens; see Table~\ref{tbl:cpt-prev-work} for a survey) and applying causal language modeling recipes \citep{gupta2023continualpretraininglargelanguage, ibrahim2024simplescalablestrategiescontinually, parmar2024reusedontretrainrecipe}.
We aim to extend continued pretraining to small, specialized domains such as proprietary datastores.
Because standard continued pretraining proves ineffective on small corpora, we propose a knowledge graph--inspired approach to synthesize a diverse, related corpus more amenable to learning.

\paragraph{Knowledge editing} A related line of work updates LMs with small units of factual knowledge---e.g., $(\text{subject, relation, object})$ tuples.
\cite{zhu2020modifying} study constrained fine-tuning to limit model complexity.
Later approaches localize where factual knowledge is stored in Transformers and update only those weights \citep{mend, rome, memit}, or maintain an external memory of edits and prepend them as context during generation \citep{mquake, rippleeffects}.
Most related to our work is \cite{akyurek-etal-2024-deductive}, which first deduces implications of a factual edit and then finetunes on those implications.
Unlike knowledge editing, which learns atomic, sentence-length facts, we aim to learn from a small corpus of documents.

\paragraph{Synthetic data generation} A rich literature uses neural networks to generate synthetic data.
Many approaches stem from semi-supervised learning: self-training and pseudo-labeling improve models by iteratively training them on their own predictions \citep{scudder1965probability, lee2013pseudolabel, yalniz2019billionscalesemisupervisedlearningimage, berthelot2019mixmatchholisticapproachsemisupervised, xie2020selftraining}, and co-training uses two models to supervise each other \citep{blum1998combining, balcan2004cotraining}.
Before language models rose to prominence, few approaches attempted to synthesize inputs; one exception is membership query synthesis \citep{Angluin1988QueriesAC, schumann-rehbein-2019-active}.
Contemporary works employ co-training \citep{lang2022cotraining} and self-training to improve language model performance, often on mathematical reasoning tasks \citep{huang2023large, gulcehre2023reinforcedselftrainingrestlanguage, zhang2024restmctsllmselftrainingprocess}, or synthesize input-output pairs for instruction tuning by conditioning on a curated seed set \citep{wang-etal-2023-self-instruct, honovich-etal-2023-unnatural, alpaca, peng2023instructiontuninggpt4, yuan2024selfrewardinglanguagemodels, li2024syntheticdataalmostscratch}.

\paragraph{Continual learning and pretraining} Continual learning stems from historical work on connectionist networks \citep{mccloskey1989catastrophic, ratcliff1990connectionist} and considers learning with tasks arriving online \citep{schlimmer1986case, grossberg2012studies}.
The central challenge is mitigating a neural network's ``catastrophic forgetting'' of previously encountered tasks \citep{robins1995catastrophic, goodfellow2015empiricalinvestigationcatastrophicforgetting, kemker2018measuring}.
Approaches include regularizing parameter updates to preserve important parameters \citep{nguyen2017variational, zenke2017continual, kirkpatrick2017overcoming}, dynamically modifying the architecture \citep{rusu2016progressive, golkar2019continual}, and recalling or replaying previous experiences \citep{rebuffi2017icarl, shin2017continual, lopez2017gradient}.
Modern continued pretraining methods mitigate catastrophic forgetting by scaling parameter count \citep{ramasesh2022effect} and mixing in updates on pretraining data \citep{instruct_gpt}.

\subsection{Bootstrapping pretraining capabilities}
\label{sec:sbp-related-work}

Given the self-improvement constraint---no external teacher---the main alternatives for improving pretraining are architectural changes and retrieval-augmented approaches.
Both are reasonable and complementary; we choose synthetic generation because it creates new training signal from existing data and writes it directly into the model's weights.
The following review situates this choice across three areas: LM pretraining, synthetic data for LMs, and retrieval-augmented LMs.

\paragraph{LM pretraining} Modern pretraining stems from a series of works including ELMo \citep{peters2018deepcontextualizedwordrepresentations}, ULMFiT \citep{howard2018universallanguagemodelfinetuning}, and BERT \citep{devlin-etal-2019-bert}, which pretrain a neural network via an unsupervised objective and subsequently finetune for a wide range of downstream tasks.
The GPT series \citep{gpt1, gpt2, gpt3, gpt4} cemented next-token prediction as the pretraining objective applied to large-scale crawled webpages, as opposed to task-specific datasets (e.g., English-to-French translation).
In recent years, the size of pretraining corpora has grown rapidly, driven by massive web-crawled datasets: BERT~\citep{devlin-etal-2019-bert, liu2020roberta}, GPT-2 WebText \citep{gpt2}, CommonCrawl \citep{commoncrawl}, CCNet~\citep{wenzek2019ccnet}, T5 C4~\citep{t5}, the Pile~\citep{gao2020pile}, Gopher Massive Text~\citep{rae2021scaling}, Llama series \citep{touvron2023llama2openfoundation, llama3}, RefinedWeb~\citep{penedo2023refinedweb}, Dolma~\citep{soldaini2024dolma}, DCLM-baseline~\citep{li2024datacomplm}, NemotronCC~\citep{su2024nemotron}, etc.
While pretraining has proven tremendously successful, the rapid depletion of available internet text motivates a shift from acquiring more data to using existing data more effectively---an opportunity that SBP directly exploits.

\paragraph{Synthetic data} One way to overcome scarce high-quality web data is to pretrain~\citep{gunasekar2023textbooksneed, phi2, abdin2024phi3technicalreporthighly, phi4, Bai2025KimiKO} or continually pretrain~\citep{ruan2025reasoning, zweiger2025selfadaptinglanguagemodels, nguyen2025recyclingwebmethodenhance} LMs on synthetic data.
Existing approaches rely on distillation from a powerful ``teacher'' LM that generates compressed knowledge for the ``student'' LM to learn \citep{hinton2015distillingknowledgeneuralnetwork}.
These teacher models must first undergo human alignment, which requires extensive annotations and preference data \citep{instruct_gpt}.
However, synthetic data from teacher LMs shows limited scaling: while such data can be as much as 7x more effective than real data \citep{datologyai2025beyondweblessonsscalingsynthetic}, performance quickly converges to that of the teacher LM \citep{busbridge2025distillation}.
In contrast, we consider the scenario where the sole source of world knowledge comes from a fixed set of pretraining documents (e.g., the internet) and algorithmically learn a data synthesizer with minimal human intervention (e.g., no generative teacher models or human writing prompts).
Our experimental setup therefore simulates a situation where LMs can self-boost their pretraining capability by refining their understanding of the fixed collection of pretraining documents.

\paragraph{Retrieval-augmented LMs} A natural class of methods that incorporates multiple documents together is retrieval-augmented generation (RAG) \citep{Lample:2019, rag}.
Originally a test-time technique for domain-specific downstream tasks \citep{li2022decoupled}, retrieval-augmented approaches have since been extended in scope:
\cite{Borgeaud:2021}, \cite{Khandelwal:2020}, and \cite{resmem} implement RAG at pretraining scale and show improved test perplexity;
\cite{Guu:2020} incorporates RAG at pretraining time by jointly training a retriever and the model itself for improved QA performance;
\cite{shi2024incontext} groups related documents into the same context window for improved long-context capability.
While RAG-based approaches enable the model to leverage rich inter-document correlations, they are fundamentally limited by the LM's context window.
In contrast, SBP encodes correlations into synthetic data that can be iteratively learned by the LM one document at a time.
Prior to embedding models that enable retrieving entire documents, \cite{guu2018generatingsentenceseditingprototypes} retrieve neighboring pairs of sentences using Jaccard similarity and model the conditional distribution between them---an objective that resembles our conditional data synthesizer---but they do not perform pretraining experiments.

\subsection{Towards AI-designed AI}
\label{sec:air-related-work}

As discussed in \S\ref{sec:intro-algorithms}, established approaches to algorithmic improvement---Neural Architecture Search, automated algorithm discovery, and learned optimizers---are effective within their respective domains but operate within constrained search spaces or require end-to-end differentiability.
The following review details these alternatives and the evidence for pursuing research automation in an unbounded action space instead.

\paragraph{Test-time scaling methods} As introduced in Chapter~\ref{chap:automated-ai-research}, \S\ref{sec:s1-bf-ttc}, we differentiate two approaches to scaling test-time compute: \textbf{parallel} and \textbf{sequential}.
Parallel methods generate multiple solution attempts independently and select the best outcome via specific criteria---choosing the most frequent response (majority voting) or the highest-scoring response under an external reward (Best-of-N)~\citep{brown2024largelanguagemonkeysscaling, irvine2023rewardingchatbotsrealworldengagement, levi2024simplemodelinferencescaling}.
Sequential methods instead let the model generate solution attempts one after another, refining each attempt based on previous outcomes~\citep{snell2024scalingllmtesttimecompute,hou2025advancinglanguagemodelreasoning,lee2025evolvingdeeperllmthinking}.
Tree-based search methods~\citep{gandhi2024streamsearchsoslearning, wu2024inference} offer a hybrid between sequential and parallel scaling---examples include Monte-Carlo Tree Search (MCTS)~\citep{liu2024dontthrowawayvalue, zhang2023planninglargelanguagemodels, zhou2024languageagenttreesearch, choi2023kcts} and guided beam search~\citep{xie2024self}.
\textsc{REBASE}~\citep{wu2024inference} uses a process reward model to balance exploitation and pruning during tree search, outperforming both sampling-based methods and MCTS.
Reward models~\cite{lightman2023letsverifystepstep, wang-etal-2024-math,wang2024helpsteer2opensourcedatasettraining} play a key role in these methods and come in two variants: outcome reward models~\cite{xin2024deepseekproveradvancingtheoremproving, ankner2024critiqueoutloudrewardmodels}, which assign a score to complete solutions and are useful in Best-of-N selection, and process reward models~\citep{lightman2023letsverifystepstep, wang-etal-2024-math, wu2024inference}, which assess individual reasoning steps and guide tree-based search.

\paragraph{AutoML} Our work connects to the AutoML literature.
Neural Architecture Search (NAS) defines a constrained set of neural network operators and optimizes architectures based on validation performance through reinforcement learning~\cite{Zoph2016NeuralAS,Zoph2017LearningTA} or search~\cite{Liu2017HierarchicalRF,So2019TheET}.
Recent works also use LMs directly to propose architecture variants and implement them for validation~\cite{Liu2025AlphaGoMF, Cheng2025LanguageMB}.
Beyond architectures, similar automatic optimizations have been applied to hyperparameter tuning~\cite{Zhang2023UsingLL}, discovering machine learning algorithms~\cite{Real2020AutoMLZeroEM}, improving post-training objectives~\cite{Lu2024DiscoveringPO}, discovering better neural network optimizers~\cite{Chen2023SymbolicDO}, and designing agent scaffolds~\cite{Hu2024AutomatedDO}.
In contrast, we tackle automated AI research in a fully open-ended setting without constraints on idea type.
Our goal is to improve idea generation effectiveness, where natural-language ideas represent a higher level of abstraction than specific architecture variants or code optimizations.

\paragraph{LM-based research agents} Recent works build LM-based research agents for accelerating scientific discovery, including AI research.
AI-Scientist~\cite{Lu2024TheAS,Yamada2025TheAS}, AI-Researcher~\cite{Tang2025AIResearcherAS}, and Agent Laboratory~\cite{Schmidgall2025AgentLU} are end-to-end research agents that use LMs to generate ideas and implement them through carefully designed agent scaffolds.
These systems address open-ended AI research as we do but do not study how to learn from execution feedback to improve idea effectiveness.
On more grounded benchmarks with clear performance metrics---MLE-Bench~\cite{Chan2024MLEbenchEM}, RE-Bench~\cite{Wijk2024REBenchEF}, and ML-Gym~\cite{Nathani2025MLGymAN}---various works explore learning from execution feedback through search~\cite{Toledo2025AIRA,Jiang2025AIDEAE} or RL~\cite{Yang2025ReinforcementLF} to optimize performance on targeted ML engineering tasks.
While we also study algorithms for learning from execution feedback, we tackle open-ended research problems like pretraining and post-training rather than ML engineering tasks that depend heavily on feature engineering and hyperparameter tuning.

\paragraph{AI for research} Apart from fully end-to-end automated AI research, many works study how to use LMs for specific components of the scientific research pipeline: literature review~\cite{Asai2024OpenScholarSS,Lala2023PaperQARG}, idea generation~\cite{Si2024CanLG,Wang2023SciMONSI}, data analysis~\cite{Majumder2024DiscoveryBenchTD,Mitchener2025KosmosAA}, experiment plan generation~\cite{Goel2025TrainingAC}, research code execution~\cite{Starace2025PaperBenchEA,Hua2025ResearchCodeBenchBL,Tian2024SciCodeAR}, and paper reviewing~\cite{Liang2023CanLL,Zhu2025DeepReviewIL}.
Our work focuses on automated idea execution and learning from execution feedback, complementing the above efforts that improve other aspects of the research pipeline.

\paragraph{Execution grounding for code} Learning from execution feedback has been explored in code generation: \cite{Zheng2024OpenCodeInterpreterIC} curate data and train models to refine code from human or execution feedback; \cite{Gehring2024RLEFGC} use end-to-end RL to teach models to improve code based on execution feedback; \cite{Lavon2025ExecutionGL} directly guide code generation with execution signals during inference.
In contrast, we explore execution grounding for idea generation, where verification is more complicated and expensive.

\chapter{Continual knowledge acquisition}
\label{chap:scp}
As established in Chapter \ref{chap:intro}, pretraining creates powerful models by absorbing knowledge from large-scale internet text.
But what happens when the knowledge we need is not on the internet?
Proprietary corpora, niche scientific domains, and private datastores contain valuable knowledge that appears rarely---or never---in the pretraining distribution.
Standard continued pretraining on such small corpora is ineffective: the model cannot generalize from a compressed representation of knowledge, and data-inefficient phenomena such as the ``reversal curse''~\citep{berglund2023reversal} and the requirement of thousands of examples per fact~\citep{allenzhu2024physicslanguagemodels32} make direct memorization unreliable.

Several approaches address knowledge acquisition in language models (see \S\ref{sec:related-work} for a detailed review).
Knowledge editing methods~\citep{rome, memit} effectively update atomic facts---e.g., $(\text{subject, relation, object})$ tuples---but do not scale to corpus-level knowledge.
Retrieval-augmented generation~\citep{rag} keeps knowledge external to the model and is fundamentally limited by the context window.
Collecting more real data is often infeasible by definition for proprietary or niche domains.
We pursue synthetic data generation because it operates at corpus scale, writes knowledge directly into the model's parametric weights, and applies precisely when real data cannot be obtained.
The open question is how to synthesize effectively---which is what this chapter addresses.

\section{Synthetic continued pretraining}
\label{sec:scp-intro}

We propose to address this problem of acquiring knowledge from small corpora with \emph{synthetic continued pretraining}.
To illustrate, consider the problem of teaching an LM a new area of mathematics, succinctly documented by a small set of textbooks.
Directly training the model on those textbooks is unlikely to be effective because of the limited volume of text (e.g., tens of thousands of words), and the model will struggle to generalize from this compressed representation of knowledge.
In contrast, learning established mathematical areas like linear algebra is straightforward because a large-scale corpus with diverse knowledge representations is accessible: for example, online lecture notes, Stack Exchange discussions, or Python implementations of the singular value decomposition.
Synthetic continued pretraining bridges this gap by first converting a small, data-constrained domain into a synthetic corpus with diverse knowledge representations, and then continuing pretraining on it.

One basic approach is to simply paraphrase or rewrite the source documents in multiple ways.
However, we find that generic rephrasing does not bridge the gap in knowledge representation diversity.
We repeatedly rephrase a small corpus and find that the value of incremental synthetic data quickly decreases, with downstream model performance scaling poorly.
This failure stems from the lack of diversity in paraphrasing alone.
In the linear algebra example, online lecture notes and Stack Exchange discussions go beyond a simple rewrite of any textbook---they provide deeper analysis and application of the underlying concepts and techniques.

\begin{figure}[t]
\centering
\includegraphics[width=\textwidth]{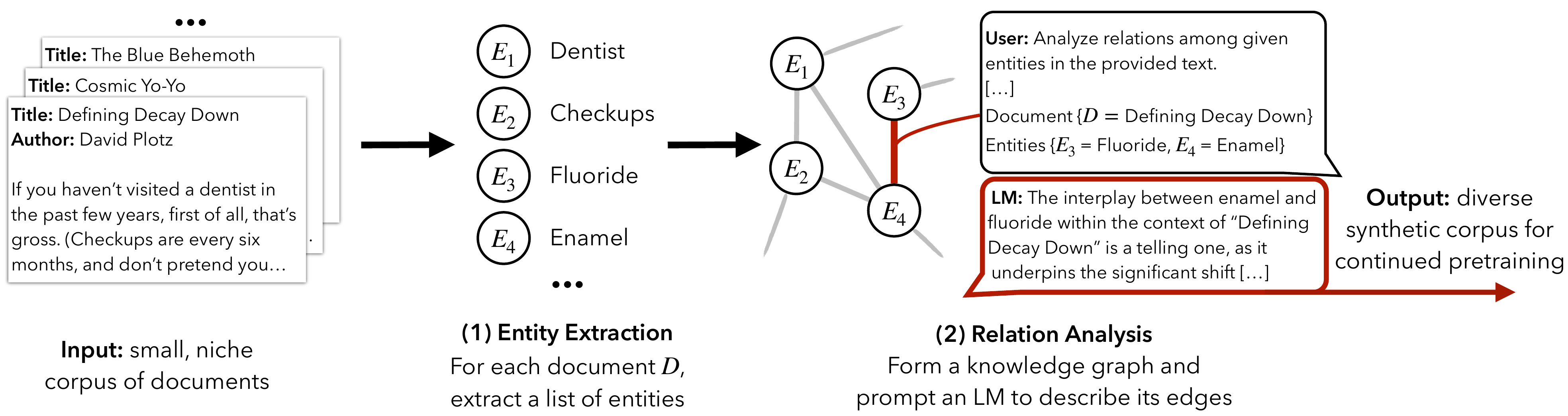}
\caption{\textbf{Synthetic continued pretraining (synthetic CPT)} converts a small source corpus into a large synthetic corpus that is amenable to learning via standard continued pretraining. 
We instantiate synthetic CPT using a synthetic data augmentation algorithm called \textbf{EntiGraph}, which forms a knowledge graph over entities extracted from documents, and then prompts an LM to synthesize a text-based representation of the graph.
}
\label{fig:entigraph_cartoon}
\end{figure}

We address this shortcoming with EntiGraph, an entity-centric augmentation algorithm.
EntiGraph breaks down a text corpus into a list of entities and then uses an LM to describe relations among entities, iteratively ``filling in'' the knowledge graph underlying the corpus (Figure~\ref{fig:entigraph_cartoon}).

Our work operates in a \emph{data-limited} regime.
Niche knowledge---proprietary corpora, specialized scientific domains, private datastores---by definition does not enjoy the diverse internet representation that makes standard pretraining effective: there are no Stack Exchange threads, no blog posts, no alternative expositions.
The \emph{goal} is to study data efficiency: can synthetic data bridge the gap between a small niche corpus and the diverse representations that enable effective learning?
We deliberately use a stronger model (GPT-4) as the synthetic data generator, because the scientific question here is about data efficiency, not self-sufficiency.
The question of whether a model can bootstrap its own pretraining \emph{without} an external teacher is deferred to Chapter~\ref{chap:sbp}.

To concretely instantiate this goal, we propose an experimental setting based on QuALITY~\citep{quality}, a reading comprehension dataset.
This setup enables evaluating synthetic data generation methods for data-efficient learning without the high compute costs of pretraining from scratch.
Specifically, we assume access to a collection of 265 books totaling 1.3M tokens.
Our task is to synthesize a corpus such that continued pretraining on it enables a model to answer queries (e.g., multiple-choice QA or user instructions related to the book content) \emph{without} access to the source texts.

In our main experiments (\S\ref{sec:exp-open-book}), we use EntiGraph to generate 455M synthetic tokens from 1.3M real tokens using GPT-4 \citep{gpt4}.
We then continually pretrain Llama 3 8B \citep{llama3} on these synthetic tokens and evaluate QA accuracy on the QuALITY questions.
We observe log-linear scaling in accuracy as synthetic token count increases, up to 455M (\S\ref{sec:exp-qa-result}).
At the endpoint, synthetic continued pretraining with 455M EntiGraph tokens provides 80\% of the accuracy gain of having source documents available at inference time (\S\ref{sec:exp-open-book}).
Beyond QA, we also instruction tune the continually pretrained model and find that it can follow open-ended instructions (e.g., summarization) related to the QuALITY books (\S\ref{sec:exp-instruct-result}).

To summarize, our key contributions are as follows:
\begin{itemize}[topsep=0pt,leftmargin=4mm,itemsep=0pt,partopsep=0pt,parsep=2pt]
    \setlength\itemsep{0.08mm}
    \item We propose to learn from small corpora with \textbf{synthetic continued pretraining}---converting the small corpus into a large, diverse, synthetic corpus and continuing pretraining on it---and instantiate this approach using the \textbf{EntiGraph} synthetic data augmentation algorithm (\S\ref{sec:entigraph-method}).
    \item We demonstrate that continued pretraining on the EntiGraph-synthesized corpus yields a QA accuracy scaling trend that is log-linear in the synthetic token count, outperforming continued pretraining on the source documents or paraphrases (\S\ref{sec:exp-qa-result}).
    Furthermore, we show that instruction tuning the EntiGraph continually pretrained model enables it to follow more diverse queries related to the source documents (\S\ref{sec:exp-instruct-result}).
    \item We complement the main experiments with an open-book setup (\S\ref{sec:exp-open-book}), providing the model with access to the source documents when answering queries.
    We demonstrate that the knowledge acquired through synthetic continued pretraining with EntiGraph is \emph{complementary} to the knowledge accessed through retrieval-augmented generation (RAG, \cite{rag})---RAG with the EntiGraph continually pretrained model outperforms RAG with the base model.
    \item Lastly, we build a mathematical model that captures the intuition behind EntiGraph.
    We analyze it to obtain a parametric formula for the scaling trend of a continually pretrained model's accuracy with respect to EntiGraph synthetic tokens, closely matching our empirical observations (\S\ref{sec:entigraph-scaling}).
\end{itemize}

Practically, synthetic continued pretraining with EntiGraph enables pretrained LMs to adapt to specialized domains by acquiring \emph{parametric} knowledge, rather than non-parametric knowledge accessed through retrieval.
At a higher level, our approach points toward a family of synthetic data generation algorithms that convert compute into data efficiency for continued pretraining.

\section{Our method}
\label{sec:method}

We focus on learning parametric knowledge from a small corpus of documents.
Our goal is to continually pretrain an LM to acquire the knowledge of a niche corpus.
Because simple continued pretraining proves ineffective (\S\ref{sec:exp-main}), we propose synthetic continued pretraining: first synthesizing a larger corpus from the small one, then continuing pretraining on the synthetic corpus.
We first outline this problem setting and our evaluation approach (\S\ref{sec:setup}), then provide a concrete instantiation using a data augmentation algorithm called EntiGraph (\S\ref{sec:entigraph-method}).

\subsection{Problem setup}
\label{sec:setup}

\begin{table}[t]
\centering
\resizebox{\textwidth}{!}{%
\begin{tabular}{lccc}
\toprule
Study & Domain & Model Parameter Count & Total Unique CPT Tokens \\
\midrule
Minerva \citep{minerva} & STEM & 8B, 62B, 540B & 26B-38.5B \\
MediTron \citep{chen2023meditron70bscalingmedicalpretraining} & Medicine & 7B, 70B & 46.7B \\
Code Llama \citep{rozière2024codellamaopenfoundation} & Code & 7B, 13B, 34B & 520B-620B \\
Llemma \citep{azerbayev2024llemma} & Math & 7B, 34B & 50B-55B \\
DeepSeekMath \citep{shao2024deepseekmathpushinglimitsmathematical} & Math & 7B & 500B \\
SaulLM-7B \citep{colombo2024saullm7bpioneeringlargelanguage} & Law & 7B & 30B \\
SaulLM-\{54, 141\}B \citep{colombo2024saullm54bsaullm141bscaling} & Law & 54B, 141B & 520B \\
HEAL \citep{yuan2024continuedpretrainedllmapproach} & Medicine & 13B & 14.9B \\
\midrule
Our setting & Articles \& Books & 8B & 1.3M \\
\bottomrule
\end{tabular}
}
\caption{Comparing the scale of modern continued pretraining (CPT) works with our small corpus setting.
Prior work adapts LMs to broad domains with diverse, large-scale corpora.
We aim to downscale CPT to small corpora; we use a corpus that is 10,000$\times$ smaller than the smallest modern corpus for domain-adaptive CPT.
}
\label{tbl:cpt-prev-work}
\end{table}

\paragraph{Continued pretraining on small corpora} We focus on approaches that continually pretrain an LM to teach it the knowledge of a small source corpus $\Ds$.
These approaches acquire ``parametric knowledge''---the knowledge of $\Ds$ is learned in the LM's parameters, as in pretraining.

\paragraph{Synthetic \underline{c}ontinued \underline{p}re\underline{t}raining (synthetic CPT)} First, we apply a synthetic data generation algorithm $\Algsynth$ to convert a small corpus $\Ds$ into a synthetic corpus $\Dsynth$:
\begin{equation}
\label{eqn:entigraph-op}
\Algsynth: \Ds \longmapsto \Dsynth.
\end{equation}
We then perform continued pretraining on $\Dsynth$ instead of on $\Ds$.
We implement $\Algsynth$ using a prompted LM.
A natural concern is that the LM may hallucinate, fabricating false knowledge.
We therefore consider \textbf{synthetic data augmentation} algorithms that condition generation on the source documents to improve faithfulness.

\paragraph{Evaluation with knowledge-intensive queries} We evaluate a synthetic data augmentation algorithm $\Algsynth$ by testing whether the downstream synthetic CPT model effectively acquires the knowledge of $\Ds$ in its parameters.
We curate test queries $\Qt$ that probe knowledge about $\Ds$.
For example, in the linear algebra setting, $\Qt$ could be held-out exam questions.
To test parametric knowledge, we do not allow the model to access the source documents $\Ds$ at test time.
The queries therefore cannot be ambiguous without access to $\Ds$---a reading comprehension question like ``Where was he born?'' is ambiguous without context.
Altogether, we evaluate data augmentation algorithms $\Algsynth$ for synthetic CPT using a paired source corpus and related test queries $(\Ds, \Qt)$.

\subsection{EntiGraph}
\label{sec:entigraph-method}

We next present EntiGraph, our instantiation of a synthetic data augmentation algorithm $\Algsynth$.
At a high level, EntiGraph generates diverse knowledge representations from a small corpus $\Ds$ by using a prompted LLM to synthesize a knowledge graph representation.
EntiGraph operates in two steps: extracting entities from the document and analyzing relations among arbitrary subsets of entities (Figure \ref{fig:entigraph_cartoon}).
This hierarchical prompting strategy \emph{externalizes} the problem of generating diverse synthetic text to a combinatorial structure---a graph relating entities appearing in the corpus documents.

\paragraph{Step 1: Entity extraction} First, EntiGraph extracts a list of salient entities $\{E_1, E_2, \dots, E_n\}$ from the document $\Ds$ using an \texttt{entity\_extraction} prompt (full prompt in Appendix \ref{sec:appendix-entigraph-prompts}): $$\{E_1, E_2, \dots, E_n\} \sim \lmgen\big(\texttt{entity\_extraction}(\Ds) \big)$$.
In the linear algebra example, $\Ds$ could be one specific linear algebra textbook.
We would expect to extract entities such as $\{E_1 = \texttt{Linear space},~ E_2 = \texttt{Vector},~ E_3 = \texttt{SVD}, \dots\}$.

\paragraph{Step 2: Relation analysis} Next, EntiGraph analyzes relations among subsets of entities.
The intuition is to explore edges of the knowledge graph underlying the source document $\Ds$, analogous to a student writing diverse notes about a linear algebra textbook.
We apply a \texttt{relation\_analysis} prompt (full prompt in Appendix \ref{sec:appendix-entigraph-prompts}) to describe how a subset of $k \leq n$ entities relate in the context of $\Ds$: $$\widetilde{D}_{E_{i_1}\dots E_{i_k}} \sim \lmgen\big( \texttt{relation\_analysis}(D, E_{i_1}, E_{i_2}, \dots, E_{i_k}) \big)$$.
For example, if $E_1 = \texttt{Linear space}$ and $E_2 = \texttt{Vector}$, $\widetilde{D}_{E_1E_2}$ could be \texttt{Based on the textbook, a vector is an element of a linear space\dots}
Exhaustively enumerating all possible subsets of entities is impractical; we generate data for pairs $\widetilde{D}_{E_{i}E_{j}}$ and triplets $\widetilde{D}_{E_i E_j E_k}$ in our experiments.

\paragraph{EntiGraph synthetic corpora} Finally, we collect all sampled synthetic texts from Step 2 as the EntiGraph output: $\Denti = \{\widetilde{D}_{E_{i_1}\dots E_{i_k}}, \dots \}$.
Altogether, we have described a data augmentation algorithm mapping a small source corpus $\Ds$ to a larger synthetic corpus $\Denti$, as in \eqref{eqn:entigraph-op}.

\section{Experiment setup}
\label{sec:exp-setup}

We next detail how we evaluate a given data augmentation algorithm $\Algsynth$.
As described in \S\ref{sec:setup}, we evaluate algorithms by measuring how well an LM continually pretrained on the output synthetic corpus $\Algsynth(\Ds)$ answers test queries $\Qt$ about the source documents $\Ds$.

In our main experiments, we use queries that are unambiguous without the source documents $\Ds$ and disallow the LM from accessing $\Ds$ while answering queries $\Qt$.
This allows us to evaluate which data augmentation algorithm best promotes parametric knowledge acquisition through synthetic CPT.
Later, in \S\ref{sec:exp-open-book}, we consider an open-book setting where the model can simultaneously access $\Ds$ and $\Qt$, testing how parametric knowledge acquired through synthetic CPT composes with non-parametric knowledge accessed through retrieval \citep{rag}.
We next introduce our small corpus and related test queries $(\Ds, \Qt)$.

\paragraph{\quality~corpus $\Ds$} Our corpus and test queries are based on the QuALITY \citep{quality} long-document comprehension benchmark.
The QuALITY corpus $\Ds$ consists of 265 articles and short books on genres such as science fiction and journalism, averaging $\sim$5,000 tokens.

\paragraph{\quality~test queries $\Qt$} We use the 10--20 multiple choice questions accompanying each article in \quality.
These serve as high-quality knowledge probes on $\Ds$, but the query phrasing often presupposes reading comprehension context (e.g., ``What does the author think about...'').
We remove ambiguity by contextualizing with an article reference: ``In the context of article \{article\_name\} by \{author\_name\}, what does the author think about...''.
This provides 4,609 unambiguous queries $\Qt$ to test parametric knowledge of our continually pretrained LMs.

\paragraph{Evaluation on instruction-tuned summarization} We also instruction tune the continually pretrained LMs and evaluate on more general instruction following queries.
Specifically, we prompt them to generate closed-book summaries of QuALITY articles given only title and author.

\paragraph{Performance with strong API-based LLMs} In our continued pretraining setting, we must select a corpus $\Ds$ not well-represented in standard pretraining datasets.
As an initial test of QuALITY corpus obscurity, we evaluate GPT-3.5 and GPT-4 on $\Qt$.
In the closed-book setting, GPT-3.5 achieves 44.81\% accuracy and GPT-4 achieves 51.30\% (Figure \ref{fig:exp-entigraph}).
In the open-book setting (full access to $\Ds$), GPT-3.5 reaches 72.60\% and GPT-4 reaches 86.09\% (Table \ref{tbl:exp-open}).
The large ($\sim$30\%) improvement when $\Ds$ is provided confirms that the \quality~corpus is sufficiently niche for our testbed.

\section{Main experiments}
\label{sec:exp-main}

We next present our main experimental results\footnote{Code \url{https://github.com/ZitongYang/Synthetic_Continued_Pretraining.git}.}.
Using GPT-4\footnote{We use the \texttt{gpt-4-turbo} model as of Aug. 19, 2024.} as our prompted model $\lmgen$, we apply EntiGraph to the 1.3M token \quality~corpus $\Ds$, generating a 455M token synthetic corpus\footnote{Corpus available at \url{https://huggingface.co/datasets/zitongyang/entigraph-quality-corpus}.}.
We refer to the former as the ``Raw corpus'' and the latter as the ``EntiGraph corpus'' throughout.
Additional details appear in Appendix \ref{sec:appendix-quality}.

We continually pretrain Llama 3 8B \citep{llama3} with causal language modeling on the 455M token EntiGraph corpus.
We describe our CPT procedure and introduce two baselines (\S\ref{sec:exp-cpt-procedure}), evaluate on QuALITY test queries $\Qt$ (\S\ref{sec:exp-qa-result}), and show that synthetic CPT with EntiGraph is compatible with downstream instruction tuning (\S\ref{sec:exp-instruct-result}; \citealp{instruct_gpt}).

\subsection{Continued pretraining procedure}
\label{sec:exp-cpt-procedure}

\paragraph{EntiGraph CPT} In our main experiment, we continually pretrain Llama 3 8B Base on the 455M token EntiGraph corpus for 2 epochs with replay on the RedPajama dataset \citep{together2023redpajama}.
We refer to this model as ``EntiGraph CPT''\footnote{Model weights available at \url{https://huggingface.co/zitongyang/llama-3-8b-entigraph-quality}.} hereafter; CPT details appear in Appendix \ref{sec:appendix-training-details}.
We next describe two baselines for comparison in closed-book QA (\S\ref{sec:exp-qa-result}).

\paragraph{Raw CPT baseline} The first baseline continues pretraining Llama 3 8B Base on the 1.3M token Raw corpus of QuALITY articles $\Ds$.
We jointly tune the number of epochs and RedPajama replay rate, obtaining the ``Raw CPT'' model (details in Appendix \ref{sec:appendix-training-details}).

\paragraph{Rephrase CPT baseline} Another simple synthetic data augmentation procedure is to rephrase QuALITY articles repeatedly.
\cite{wrap} and \cite{ovadia2024finetuningretrievalcomparingknowledge} systematically extend this idea (cf. \S\ref{sec:related-work}).
Based on their approaches, we craft a ``Rephrase baseline'' that repeatedly applies three fixed prompts (easy, medium, and hard rephrase)\footnote{\cite{wrap} include a 4th prompt to generate synthetic QA pairs.
We defer this task-specific QA finetuning method to Appendix \ref{sec:appendix-task-specific} and focus on task-agnostic baselines for learning generic knowledge.} to QuALITY articles at temperature 1.0.
We stopped generating paraphrases at 38M tokens, where we observed a clear gap from EntiGraph CPT and a slower scaling trend (Figure \ref{fig:exp-entigraph}).
We refer to this data as the ``Rephrase corpus'' and the continually pretrained model as ``Rephrase CPT''.

\subsection{Question-answering evaluations}
\label{sec:exp-qa-result}

We next present closed-book QA evaluations with the \quality~test queries $\Qt$.

\paragraph{Evaluation procedure} Each \quality~question is a four-choice, single-answer multiple choice question (similar to MMLU; \citealp{mmlu}).
We evaluate with 5-shot chain-of-thought prompting \citep{gpt3, cot}; our prompt appears in Appendix \ref{sec:appendix-qa-eval-detail}.

\begin{figure}[t]
\centering
\includegraphics[width=\textwidth]{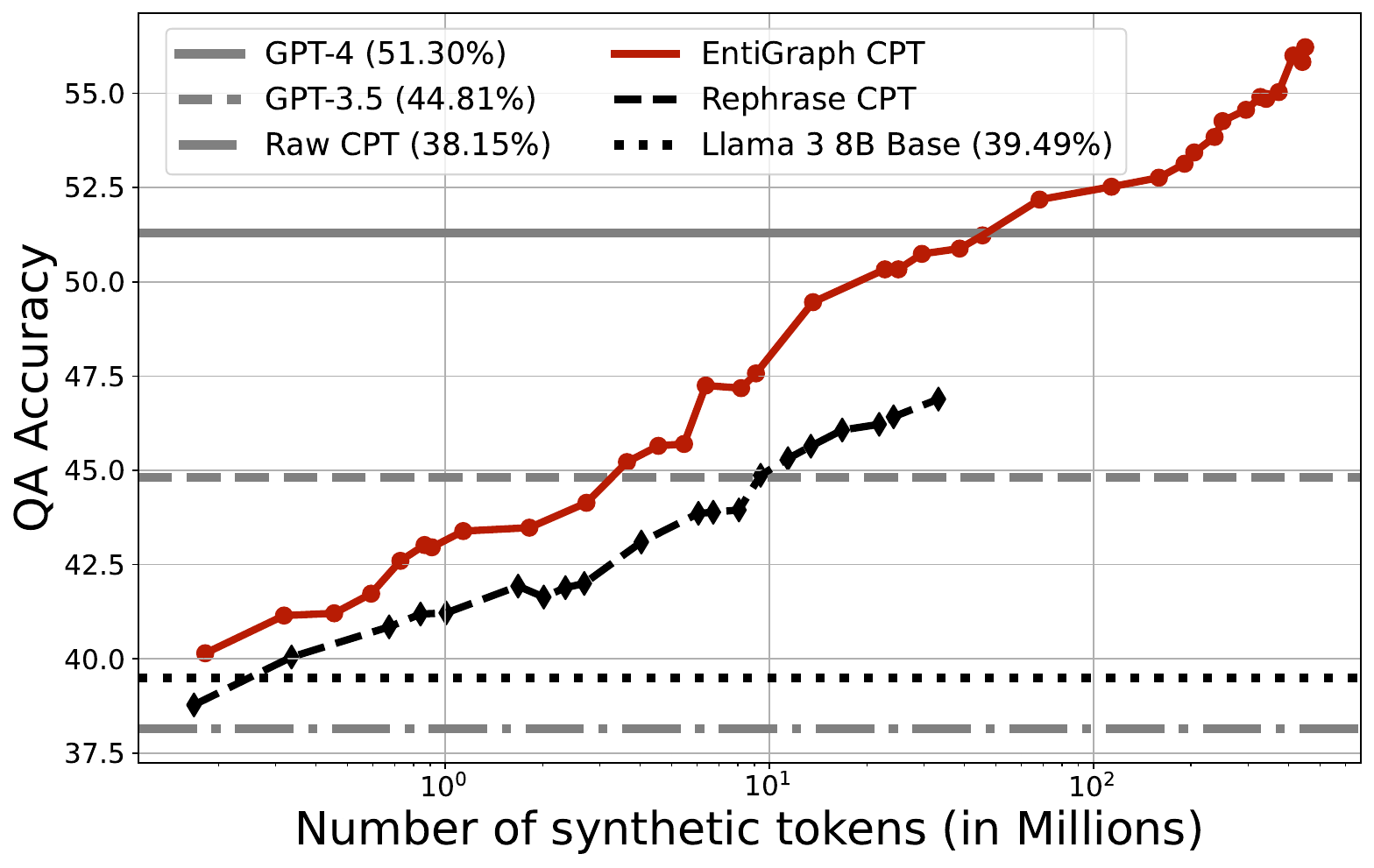}
\caption{Accuracy on the QuALITY question set $\Qt$ ($y$-axis) as a function of the synthetic token count ($x$-axis).
The accuracy of synthetic continued pretraining using the EntiGraph data augmentation algorithm (EntiGraph CPT) scales log-linearly up to 455M tokens.}
\label{fig:exp-entigraph}
\end{figure}

\paragraph{EntiGraph scaling} CPT on the 455M token EntiGraph corpus improves closed-book QA accuracy from 39.49\% (Llama 3 8B Base) to 56.22\% (Figure \ref{fig:exp-entigraph}).
A natural question is: how does accuracy scale as we synthesize and train on more tokens?
To test this, we randomly subsample without replacement from the EntiGraph corpus with varying sample sizes, continually pretrain Llama 3 8B Base on each subsample, and plot accuracy versus sample size in Figure \ref{fig:exp-entigraph}.
We observe log-linear scaling of accuracy in the number of synthetic tokens, up to 455M tokens.
We mathematically investigate EntiGraph scaling properties in \S\ref{sec:entigraph-scaling}.
At a high level, we postulate that QuALITY accuracy follows a mixture-of-exponential shape with three stages: (i) linear growth, (ii) log-linear growth, and (iii) asymptotic plateau.

\paragraph{Comparison with baselines} Raw CPT (green line) underperforms even Llama 3 8B (dashed black line).
We postulate two explanations: (i) the Raw corpus follows a narrower distribution than the Llama 3 pretraining corpus, and heavily training on it may harm English capabilities; (ii) limited diversity of knowledge representations in the Raw corpus leads to limited knowledge acquisition due to problems such as the reversal curse \citep{berglund2023reversal}.
Rephrase CPT scales poorly compared with EntiGraph (Figure \ref{fig:exp-entigraph}), suggesting that synthetic CPT requires sufficiently diverse synthetic data to scale well.
EntiGraph tackles this problem with a hierarchical prompting strategy that externalizes diversity to a knowledge graph's combinatorial relationships.

\subsection{Instruction following evaluations}

We next explore more general test queries beyond $\Qt$.
We perform instruction tuning on EntiGraph CPT to obtain EntiGraph Instruct.
Synthetic CPT on the EntiGraph corpus is compatible with instruction tuning: EntiGraph Instruct can directly use knowledge obtained during synthetic CPT for instruction following, without test-time access to the \quality~corpus $\Ds$.
We detail our instruction tuning procedure in Appendix \ref{sec:appendix-training-details}.

\begin{table}[t]
\centering
\caption{EntiGraph Instruct examples.}
\label{tbl:exp-instruct-example}
\begin{tabular}{p{0.8\textwidth}}
\toprule
\footnotesize
\textbf{Explicit reference:} Summarize ``Defining Decay Down''. \\
\midrule
\footnotesize
The article ``Defining Decay Down'' by David Plotz discusses [...] Dentists began to focus on cosmetic dentistry, [...] \\
\midrule
\footnotesize
\textbf{Implicit reference:} How has dentistry in the U.S. changed? \\
\midrule
\footnotesize
1. Increase in cosmetic dentistry [...]

2. Use of technology: [...] \\
\midrule
\footnotesize
\textbf{Cross article instruction:} Compare David Plotz's commentary on American dentistry and the movie Fight Club? \\
\midrule
\footnotesize
David Plotz's commentary style is different when he analyzes American dentistry and when he discusses the movie Fight Club. [...] \\
\bottomrule
\end{tabular}
\end{table}

\paragraph{Instruction tuning qualitative examples} \label{sec:exp-instruct-result} We first present qualitative examples demonstrating EntiGraph Instruct's ability to follow instructions related to QuALITY articles.
We ask the model to summarize a QuALITY article with explicit reference to title and author, but no access to the article itself (Table \ref{tbl:exp-instruct-example}, top row).
Next, we show that even without explicit reference to title and author, article knowledge stored in model parameters affects behavior (Table \ref{tbl:exp-instruct-example}, middle row).
Finally, we provide an example where the model compares across two articles (Table \ref{tbl:exp-instruct-example}, bottom row).
Although artificial, this demonstrates that even though EntiGraph does not synthesize data involving multiple articles simultaneously, the model can reason about their interaction using parametric knowledge.
Full responses appear in Table \ref{tbl:appendix-instruct-example}.

\paragraph{Evaluating closed-book summarization} We also present quantitative metrics for summarization, a well-studied instruction following task.
We compare EntiGraph Instruct summaries of QuALITY articles with human-written summaries from sQuALITY \citep{squality}, a QuALITY variation with human summaries.
Common scalar metrics such as ROUGE \citep{rouge} or BERTScore \citep{bertscore} mostly evaluate text similarity between summary and source articles, and may not accurately reflect summarization quality for abstractive systems \citep{zhang-etal-2024-benchmarking}.
We use a simple automated metric based on pyramid evaluation \citep{pyramidevaluation, gao-etal-2019-automated} that measures both hallucination rate and how well summaries capture salient claims.
Our approach uses GPT-4 to (1) split summaries into atomic claims \citep{min2023factscorefinegrainedatomicevaluation}, (2) decide whether each claim is true or false based on the source article, and (3) determine if true claims are salient to the article's main message.
We obtain counts of false and salient claims for each summary, normalize by the corresponding count from human summaries, and report averages in Figure \ref{fig:exp-summaryeval}.
Appendix~\ref{sec:appendix-eval-summary-detail} provides further details.

\begin{figure}[t]
\centering
\includegraphics[width=\textwidth]{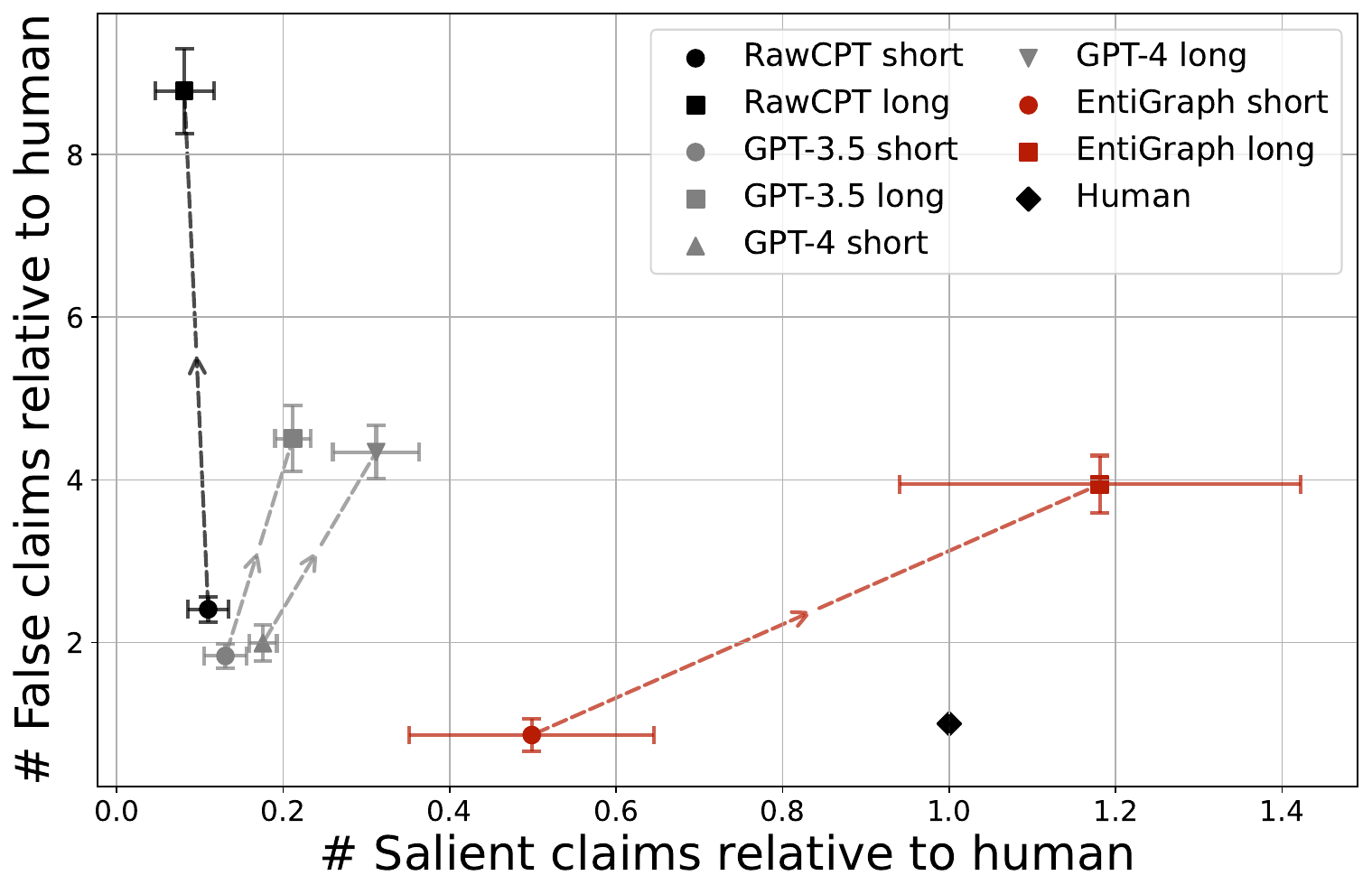}
\caption{Closed-book summarization: number of false claims ($y$-axis) versus number of salient claims ($x$-axis) normalized by the human summary.}
\label{fig:exp-summaryeval}
\end{figure}

\paragraph{Results discussion} In Figure~\ref{fig:exp-summaryeval}, we compare four summarizers: EntiGraph Instruct, Raw Instruct, GPT-3.5, and GPT-4.
We provide each summarizer with two prompts requesting short and long summaries (prompts in Appendix \ref{sec:appendix-eval-summary-detail}).
When we request more detailed summaries, Raw Instruct hallucinates and generates more false claims with little improvement in salient claims.
In contrast, EntiGraph Instruct generates more salient claims as summaries lengthen, with only a small increase in false claims (similar to GPT-3.5 and GPT-4 levels).
The gaps in both salient and false claim rates are large enough that these results likely hold beyond our particular metric.
We complement the automated evaluation with a qualitative example in Appendix \ref{sec:appendix-eval-summary-detail}.

\section{Ablation Studies}
\label{sec:app_ablation_studies}
We present ablation experiments to further validate EntiGraph's effectiveness and test its generalization properties.
We discussed two potential limitations in \S\ref{sec:limitations}:
\begin{enumerate}
    \item Could the gains of Synthetic CPT be explained by distillation effects, due to the use of a strong prompted LM for synthetic data generation?
    \item Is the data synthesized in Synthetic CPT factual?
\end{enumerate}
We provide evidence suggesting these are not significant concerns in Appendix \ref{sec:data-synthesizer-ablation} and Appendix \ref{sec:factuality-ablation}, respectively.
Lastly, we repeat the procedure of the core experiments on another small corpus of Coursera lecture transcripts, to provide evidence that Synthetic CPT generalizes to datasets and domains beyond QuALITY (Appendix \ref{sec:dataset-ablation}).

\subsection{Using a Weaker Synthetic Data Generation LM}
\label{sec:data-synthesizer-ablation}
One potential concern is whether EntiGraph's success demonstrated in \S\ref{sec:exp-main} stems from distilling knowledge from GPT-4.
To investigate this, we conducted an experiment replacing GPT-4-Turbo with a significantly weaker model, Llama 3.1 8B Instruct, as the synthetic data generator.
Recall that in all continued pretraining experiments, we finetune the 8B parameter Llama 3 Base model.
Therefore, in this experiment, the capabilities of the synthetic data generator and the continually pretrained model are very similar, controlling for distillation effects.
Using the entity extraction and relation analysis prompts introduced in \S\ref{sec:method}, we generate 334M synthetic tokens and evaluate the scaling behavior under the same hyperparameter setup detailed in \S\ref{sec:exp-cpt-procedure}.

\begin{figure}[ht]
\centering
\includegraphics[width=\textwidth]{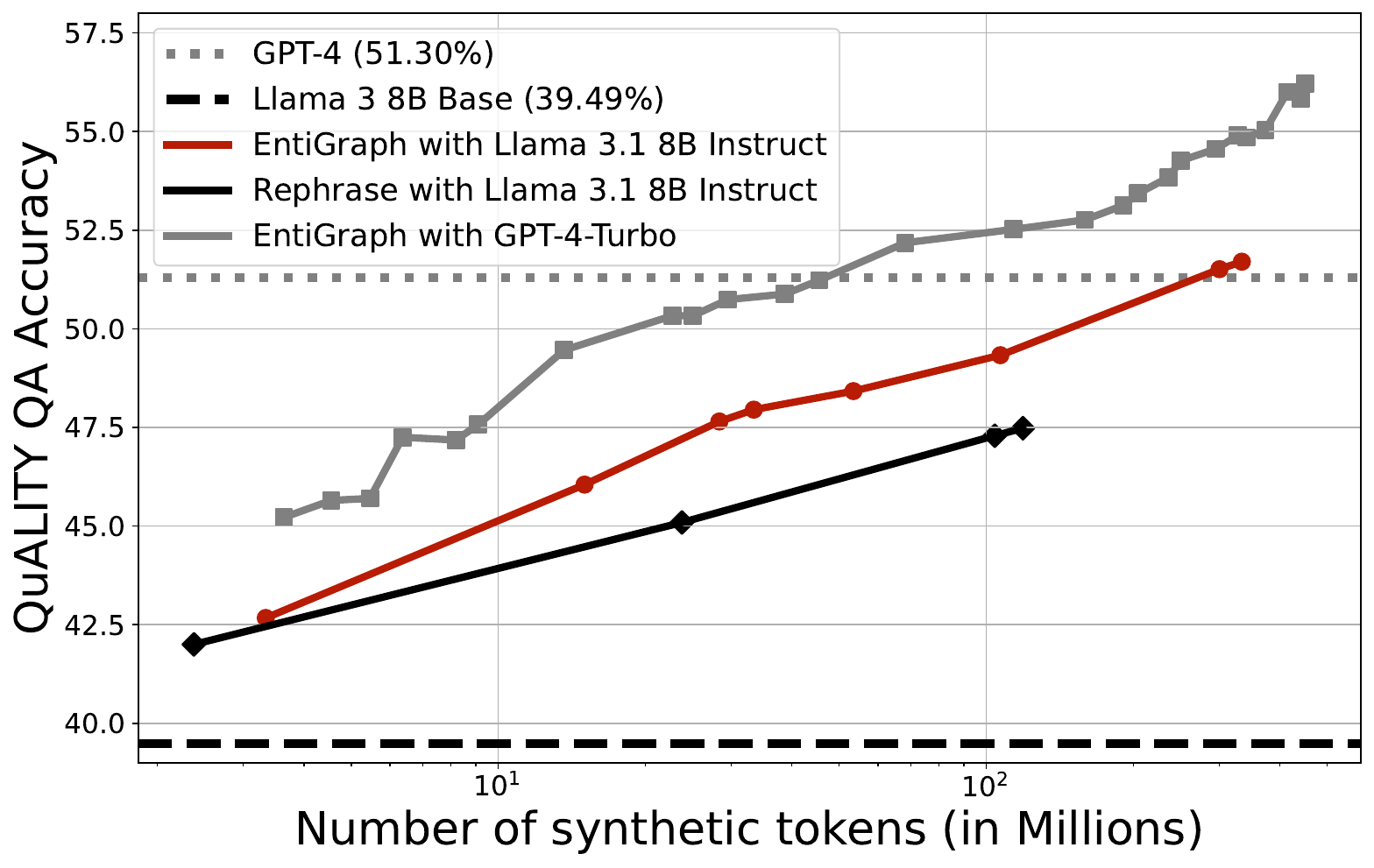}
\caption{The scaling properties of Synthetic CPT with the EntiGraph and Rephrase augmentations, comparing two synthetic data generators: GPT-4-Turbo and Llama 3.1 8B Instruct.}
\label{fig:self-improvement}
\end{figure}

Figure~\ref{fig:self-improvement} reveals two key insights.
First, even with the weaker generator, EntiGraph maintains steady log-linear improvement with no signs of saturation at 334M tokens, suggesting that the gains of Synthetic CPT stem from continued pretraining on diverse representations of the corpora's underlying knowledge, rather than distilling the generator model's knowledge.
Similar to our main results (\S\ref{sec:exp-main}), EntiGraph with a Llama 3.1 8B Instruct generator outperforms Rephrase with the same generator.
Moreover, at 334M synthetic tokens, EntiGraph with a Llama 3.1 8B Instruct generator exceeds closed-book evaluation of GPT-4-Turbo on this benchmark.

Second, while switching from the GPT-4-Turbo generator to the weaker generator shifts the accuracy curve downward, the log-linear slope remains consistent.
In contrast, holding the synthetic generator constant, we observe that EntiGraph CPT and Rephrase CPT exhibit different slopes.

\subsection{Factuality and Lexical Diversity of EntiGraph Synthetic Corpus}
\label{sec:factuality-ablation}

\paragraph{Factuality} A limitation discussed in \S\ref{sec:limitations}, and inherent in all methods involving synthetic data generation, is that the generation model may hallucinate.
EntiGraph is a synthetic data \textit{augmentation}, which conditions an LM on a given corpus document and prompts the LM to discuss the document's entities and their relationships.
Assuming a reasonably good generator model, this grounding should decrease hallucination rate.

To quantitatively test the factuality of documents synthesized with EntiGraph, we split the 455M token EntiGraph corpus into sentences and randomly sample 150 sentences.
We ask authors of this work to label whether each sentence is subjective or not, and among non-subjective sentences, to determine whether it is supported by the article text or not.

We compute two statistics: the proportion of subjective sentences denotes the number of subjective sentences over the total number of annotated sentences.
The factuality rate denotes the number of non-subjective sentences which are supported by the source document, over the number of non-subjective sentences, following \cite{min2023factscorefinegrainedatomicevaluation}:
\begin{itemize}
    \item Proportion subjective: $0.532$ (bootstrap 0.95 confidence interval: $[0.455, 0.610]$).
    \item Factuality rate: $0.944$ (bootstrap 0.95 confidence interval: $[0.889, 0.986]$).
\end{itemize}

Because EntiGraph uses open-ended prompts which ask the LM to relate different, often abstract entities, the LM often generates subjective statements.
We do not necessarily view this as a limitation, because learning reasonable subjective interpretations is crucial for understanding (and hence is often assessed in, e.g., essay questions on literature exams).
We also observe that the non-subjective sentences are consistently factual, supporting the effectiveness of grounding in reducing hallucination.

\paragraph{Lexical Diversity} We hypothesize that good synthetic data augmentations should produce knowledge representations with diverse wording.
As a measure of this lexical diversity, we compute the percentage of $n$-grams in the synthetic documents that overlap with the $n$-grams of the corresponding source documents.

More precisely, we first randomly select 100 \quality~articles, tokenize them with the Llama 3.1 tokenizer, and compute the set of $n$-grams for each article.
Then, for each article, we tokenize the corresponding EntiGraph and Rephrase synthetic data, compute $n$-grams, and count the $n$-grams in the synthetic data that appear in the set of $n$-grams for the raw article.
For each $n$ and synthetic augmentation method, we sum this overlap count across articles and normalize by the total number of synthetic tokens generated for the 100 articles, providing us an estimate of the percentage of $n$-grams in the synthetic data that overlap with the source data.

\begin{table}[ht]
\centering
\begin{tabular}{lcccc}
\toprule
Augmentation & $n=2$ & $n=4$ & $n=8$ & $n=16$ \\
\midrule
EntiGraph & 23.40 & 3.66 & 0.24 & 0.00 \\
Rephrase & 21.35 & 3.04 & 0.51 & 0.22 \\
\bottomrule
\end{tabular}
\caption{Percentage of token $n$-grams in synthetic documents that overlap with the source document $n$-grams, for the EntiGraph and Rephrase synthetic data augmentations.}
\vspace{\bigfill}
\label{tbl:ngram-overlap}
\end{table}

These results are provided in Table~\ref{tbl:ngram-overlap}.
We observe that for both augmentations, $n$-gram overlap percentage is low and quickly approaches $0\%$ with increasing $n$, indicating that both methods produce lexically diverse knowledge representations.

\subsection{Datasets Beyond QuALITY}
\label{sec:dataset-ablation}
To test whether synthetic CPT with EntiGraph generalizes to corpora beyond QuALITY, we evaluated on the Coursera Exam QA dataset \citep{an2023leval}.
This dataset contains lecture transcripts and exam questions from advanced technical courses like data science and machine learning.
Compared to the books and stories in QuALITY, Coursera exams present new challenges---the content is harder conceptually, questions can have multiple correct answers, and the number of options is not fixed to four choices.
This makes few-shot prompting more demanding, as the model must understand both the content and the flexible answering format.

The dataset consists of 15 lecture transcripts and 124K raw tokens, substantially smaller than QuALITY's 265 documents and 1.3M raw tokens.
During our scaling analysis, we found that models trained on tiny synthetic corpora (e.g., a few million tokens) struggled to follow few-shot prompts reliably for Coursera questions, resulting in parsing errors.
Therefore, we begin the scaling curve in Fig.~\ref{fig:coursera-scaling} starting from token counts where parsing error rates fall below 5\%.
For the Rephrase baseline, we generate synthetic data up to 22M tokens, and find that only one model has parsing error rates below 5\%.

\begin{figure}[ht]
\centering
\includegraphics[width=\textwidth]{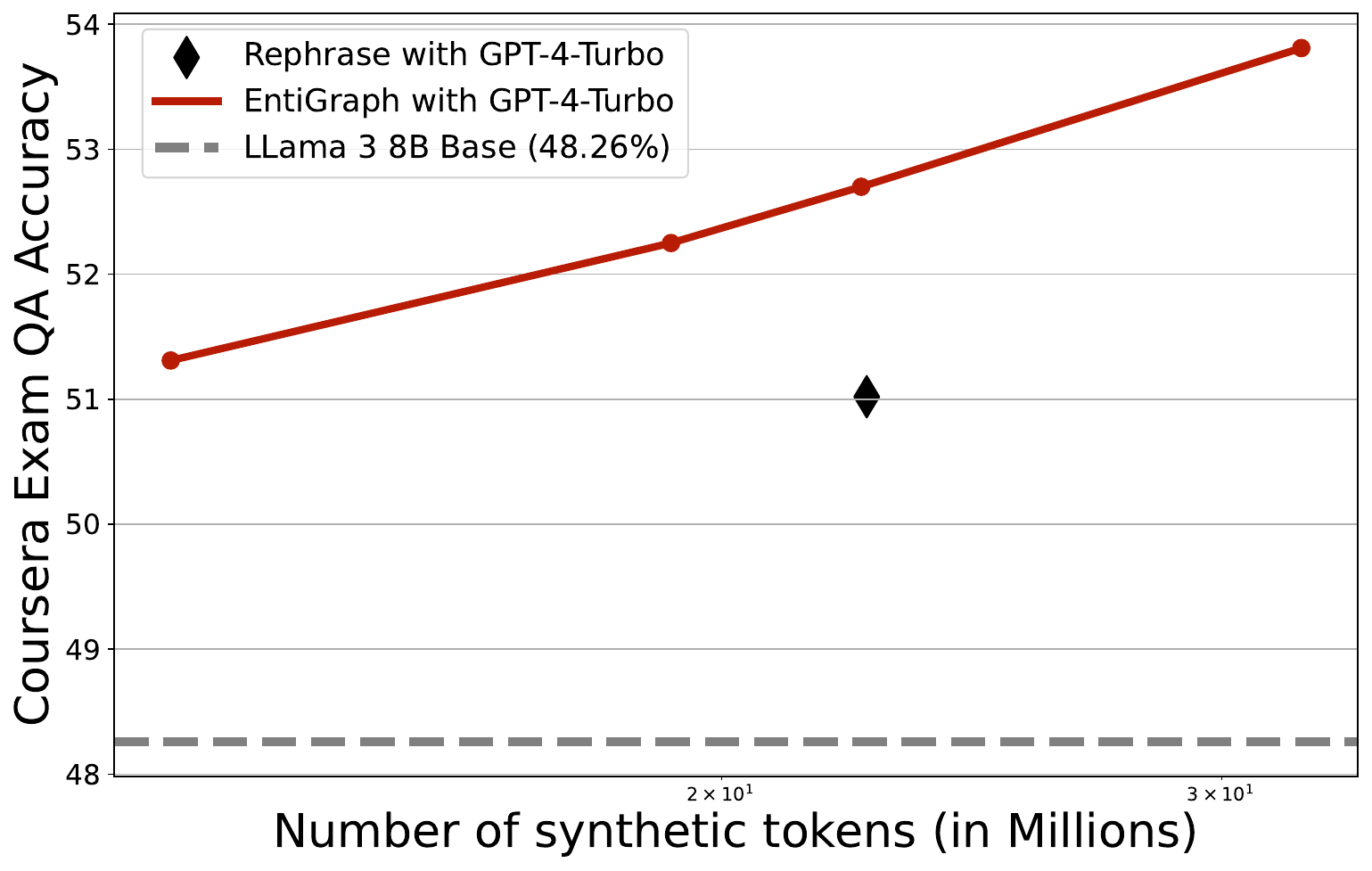}
\caption{The scaling properties of Synthetic CPT using the EntiGraph augmentation on the Coursera Exam QA dataset.}
\label{fig:coursera-scaling}
\end{figure}

Despite these challenges, EntiGraph CPT shows consistent improvement over Llama 3 8B Base, improving accuracy from 48.26\% to 53.87\%, better than Llama 3 8B Base and the Rephrase baseline.
The log-linear scaling pattern persists up to 32M synthetic tokens, suggesting EntiGraph's effectiveness extends beyond narrative texts to technical educational content.
This transfer to a different domain suggests that synthetic continued pretraining with EntiGraph may extend beyond narrative texts.

\section{Open-book experiments}
\label{sec:exp-open-book}

We next consider an open-book setting with the domain-specific corpus $\Ds$ available at test time.
In this widespread setting, retrieval-augmented generation (RAG; \citealp{rag}) is the predominant approach.
A natural question is whether parametric knowledge learned through synthetic CPT with EntiGraph complements non-parametric knowledge accessed through RAG.
We answer this by comparing a strong RAG pipeline with and without EntiGraph CPT.

\paragraph{RAG evaluation setup} Our RAG pipeline follows established best practices \citep{rag, ragsurvey}.
It involves an offline stage indexing document chunks, followed by inference-time retrieval, reranking, and placement of chunks in a few-shot LM prompt.
We use OpenAI \texttt{text-embedding-3-large} \citep{neelakantan2022textcodeembeddingscontrastive} as our embedding model, FAISS as our similarity search index \citep{douze2024faisslibrary}, and Cohere \texttt{rerank-english-v3.0} \citep{coherererank} as our reranker.
Following the evaluation procedure in \S\ref{sec:exp-main}, we evaluate parallel RAG pipelines on the \quality~multiple choice test set using few-shot chain-of-thought prompting.
All hyperparameters are tuned separately for each LM's RAG pipeline.
Appendix~\ref{sec:appendix-rag} provides further details.

\begin{table}[t]
\centering
\resizebox{\textwidth}{!}{%
\begin{tabular}{ccccccccc}
\toprule
\multicolumn{2}{c}{EntiGraph CPT + RAG} & \multicolumn{2}{c}{Llama 3 8B Base + RAG} & \multicolumn{2}{c}{GPT-4 + Oracle RAG} & \multicolumn{2}{c}{GPT-3.5 + Oracle RAG} \\
\cmidrule(lr){1-2}\cmidrule(lr){3-4}\cmidrule(lr){5-6}\cmidrule(lr){7-8}
Accuracy & Recall@$8$ & Accuracy & Recall@$8$ & Accuracy & Recall@$8$ & Accuracy & Recall@$8$ \\
\midrule
62.60 & 99.63 & 60.35 & 99.63 & 86.09 & 100.0 & 72.60 & 100.0 \\
\bottomrule
\end{tabular}
}
\caption{\quality~question-answering accuracy and recall rate in the open-book retrieval-augmented generation (RAG) setting.
EntiGraph CPT and Llama 3 8B Base are used in a RAG pipeline (cf. \S\ref{sec:exp-open-book} for setup details).
Recall@$8$ is defined as the proportion of questions for which the salient article appears in the top $8$ reranked document chunks.
GPT-4 and GPT-3.5 Oracle RAG provide an upper bound with a perfect retriever, by placing the entire relevant document in-context.
}
\label{tbl:exp-open}
\end{table}

\paragraph{EntiGraph continued pretraining complements RAG} Table~\ref{tbl:exp-open} shows that EntiGraph CPT outperforms Llama 3 8B Base, the model from which it is continually pretrained.
These results demonstrate that knowledge internalized through synthetic CPT complements knowledge accessed during RAG, suggesting a competitive recipe for small corpus QA: (1) synthetic data augmentation, (2) continued pretraining, and (3) RAG.

\paragraph{EntiGraph continued pretraining alone approaches RAG performance} These results also contextualize EntiGraph effectiveness in the closed-book parametric knowledge setting (\S\ref{sec:exp-main}).
Comparing Figure \ref{fig:exp-entigraph} and Table \ref{tbl:exp-open}, adding RAG to Llama 3 8B Base improves accuracy by 20.86\% ($39.49\% \rightarrow 60.35\%$).
In contrast, continued pretraining of Llama 3 8B Base on the EntiGraph corpus improves accuracy by 16.73\% ($39.49\% \rightarrow 56.22\%$).
Hence, EntiGraph continued pretraining provides $>\!80\%$ of the absolute performance improvement of RAG, even in a small corpus setting where RAG recall is nearly perfect.

Overall, our results show that parametric knowledge acquired in EntiGraph continued pretraining composes with realistic knowledge-intensive QA pipelines, and that EntiGraph continued pretraining alone---without test-time corpus access---is nearly competitive with a strong RAG baseline.

\section{Theoretical analysis of EntiGraph scaling}
\label{sec:entigraph-scaling}

It may seem surprising that simply ``rewriting'' the source documents $\Ds$ improves performance at all (\S\ref{sec:exp-main}), as EntiGraph does not explicitly add new knowledge beyond $\Ds$.
We postulate that EntiGraph ``rearranges'' $\Ds$ into a layout more amenable to learning.
For example, in $\Ds$, the entity pair $(A, B)$ may appear together in some sentences and $(B, C)$ in others.
Models trained directly on $\Ds$ may learn the $(A, B)$ and $(B, C)$ relations but not the $(A, C)$ relation \citep{akyurek-etal-2024-deductive}.
We build a mathematical model to formalize this intuition (\S\ref{sec:toy-model-setup}) and provide a quantitative prediction that EntiGraph CPT follows a mixture-of-exponential scaling shape (\S\ref{sec:toy-moe}), which fits well with our empirical observations (Figure \ref{fig:toy-curve-fitting}).

\subsection{Toy model setup}
\label{sec:toy-model-setup}

In this toy model, we use $\cV$ to denote the set of entities and represent the source documents $\Ds$ with pairs of known relations $\Ds \subset \{(x, y)\in\cV^2 : x \neq y\}$.
We assume each relation pair in $\cV^2$ appears in the source documents $\Ds$ independently at random with probability $p$.
Mathematically, $\P\left[(x, y)\in \Ds \right]=p$ for all $x \in \cV$ and $y \in \cV$ with $x\neq y$.
We write $V = |\cV|$ and assume $p = \lambda/V$ for some constant $\lambda>1$.

\paragraph{Training as memorization} We model learning of factual knowledge as a memorization process, where a model memorizes relations it is trained on but does not meaningfully generalize beyond them \citep{NEURIPS2023_bf0857cb, 10.1145/3357713.3384290}.
In this view, a language model's knowledge is represented by a matrix $\bM\in\{0, 1\}^{V\times V}$ such that $\bM(x, y)=1$ if the model ``knows'' the $(x, y)$ relation and equals $0$ otherwise.
Training directly on $\Ds$ simply means setting all entries appearing in $\Ds$ to $1$, denoting that the model has memorized source document relations.
We denote this model trained on $\Ds$ by the matrix $\bM_0\in\{0, 1\}^{V\times V}$, which has i.i.d. Bernoulli off-diagonal entries with mean $p$.

\paragraph{EntiGraph synthetic data augmentation} Given the source documents $\Ds$, we define the following iterative synthetic data generation procedure: for each $t=1, 2, \dots$
\begin{itemize}[topsep=0pt,leftmargin=4mm,itemsep=0pt,partopsep=0pt,parsep=2pt]
    
    \item \textbf{Entity pair selection:} Sample $(x_t, y_t)\in \{(x, y)\in\cV^2 : x \neq y\}$ uniformly at random.
    \item \textbf{Relation analysis:} Generate the ``relation between $(x_t, y_t)$'' by performing a breadth-first search (BFS) on the directed graph represented by the adjacency matrix $\bM_0$ starting at $x_t$.
    If no such path exists, do nothing.
    If there exists a path $(x_t, z_t^1, z_t^2, \dots, z_t^{k_t}, y_t)$ connecting $x_t$ to $y_t$, define
    $\cD_t = \{(x_t, z_t^1), (x_t, z_t^2), \dots, (x_t, z_t^{k_t}), (x_t, y_t)\} \cup \cD_{t-1},$
    where we assume $\cD_0 = \Ds$. The model trained on this round of synthetic data is $\bM_{t} = \bM_{t-1} + \sum_{(x,y)\in \cD_t \backslash \cD_{t-1}} \bI_{xy}$, where $\bI_{xy}\in\{0, 1\}^{V\times V}$ is a binary matrix with $\bI_{xy}(x, y)=1$ and $0$ otherwise.
    \end{itemize}
This mirrors the relation analysis step for EntiGraph (Step 2, \S\ref{sec:entigraph-method}).
The index $t$ is analogous to the number of synthetic tokens generated, and model knowledge is captured by how many ones $\bM_t$ contains.
We define the link density (or accuracy) of $\bM_t$ as
$\mathsf{Acc}(\bM_t) = \mathbb{E}[\|\bM_t\|_1 \vert \bM_0 ]/(V(V-1)),$
where the expectation is over randomness from synthetic data generation (not the source documents $\Ds$), and $\|M\|_1$ denotes $\sum_{i,j} |M_{i,j}|$.
We use the notation $\mathsf{Acc}$ because this emulates accuracy on QuALITY test queries (\S\ref{sec:exp-main} and \S\ref{sec:exp-open-book}).

\subsection{Rigorous upper and lower bound}

We next derive rigorous upper and lower bounds on the scaling trend of $\mathsf{Acc}(\bM_t)$.
\begin{definition}
    Let $C_\lambda = (1-\rho(\lambda))^2$, where $\rho(\lambda)$ denotes the extinction probability for a Poisson$(\lambda)$ branching process (i.e., $\rho$ is the smallest solution in $[0,1]$ to the fixed-point equation $\rho=\exp(\lambda(\rho-1))$).
    For any fixed $\varepsilon>0$, we further define
    $C_\mathrm{LB} = 1-\frac{1}{V(V-1)}, \quad C_\mathrm{UB} = 1-\frac{(1+\varepsilon) \log V}{V(V-1) \log \lambda}.$
\end{definition}

\begin{theorem}
\label{thm:toy}
    For any time $t \geq 1$ and any $\varepsilon>0$, the link density satisfies, with probability $\to 1$,
    \begin{align*}
         \left( p+C_\lambda \left( 1- C_\mathrm{LB}^t \right) \right) (1-\varepsilon) \leq \mathsf{Acc}(\bM_t) \leq \left(p+C_\lambda \left( 1-  C_\mathrm{UB}^t \right) \right ) (1+\varepsilon) ~~\text{as}~~ V \to \infty.
    \end{align*}
\end{theorem}
Although Theorem \ref{thm:toy} provides rigorous bounds on the scaling trend of $\mathsf{Acc}(\bM_t)$, the exact growth curve is more intricate, as we show next.

\subsection{An analytical formula}
\label{sec:toy-moe}

\begin{figure}[t]
\centering
\includegraphics[width=\textwidth]{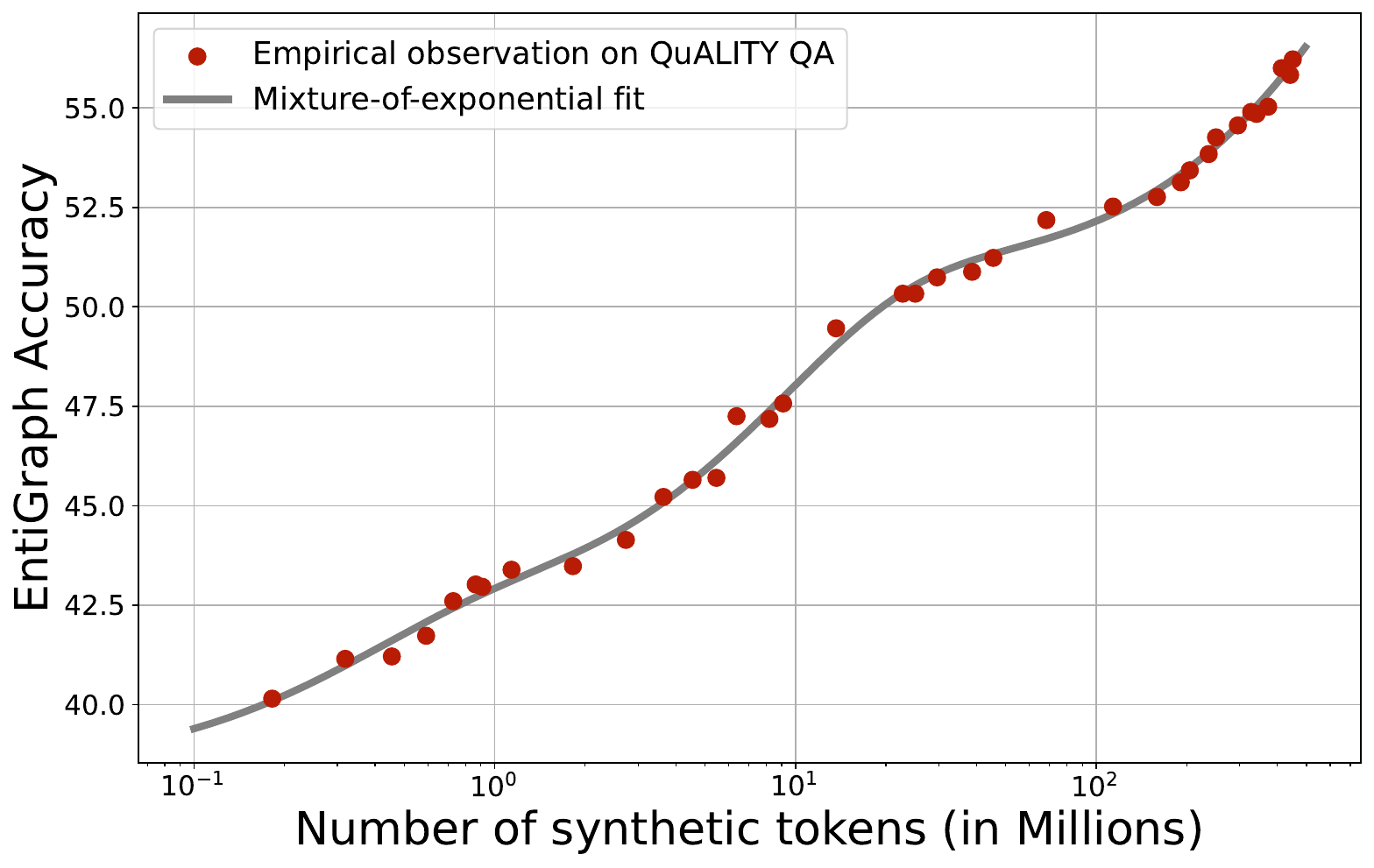}
\caption{A mixture-of-exponential function \eqref{eqn:moe} closely fits the scaling trend of EntiGraph CPT with respect to synthetic token count.}
\label{fig:toy-curve-fitting}
\end{figure}

We analyze link density $\mathsf{Acc}(\bM_t)$ using a Poisson branching process approximation of cluster growth.
This yields a \emph{mixture-of-exponential} scaling trend
\begin{equation}
\label{eqn:moe}
    \mathsf{Acc}(\bM_t) \sim  p+ C \left(1-\sum_{k=1}^\infty \mu(k) \left(1- a_k \right)^t \right),
\end{equation}
where $A\sim B$ means $A/B$ converges to $1$ in probability as $V\rightarrow\infty$.
The parameter $C$ governs link density $\mathsf{Acc}(\bM_t)$ as $t \to \infty$ and is determined by the proportion of reachable vertex pairs in $\bM_0$.
$\mu(\cdot)$ is the probability mass function on $k$, controlling the proportion of vertex pairs with a specific decay rate.
The parameters $\mu(\cdot)$ and $a_k$ depend on $\bM_0$ in more intricate ways (cf. Appendix \ref{sec:appendix-proof} for a full derivation).
Equation \eqref{eqn:moe} accurately fits the empirical scaling trend of EntiGraph CPT accuracy up to 455M synthetic tokens (Figure \ref{fig:toy-curve-fitting}).
We discuss curve fitting in Appendix \ref{sec:curve-fitting-details}, showing that the mixture-of-exponential shape grows in three phases: (i) linear growth, (ii) log-linear growth, and (iii) asymptotic plateau.

\section{Discussion}
\label{sec:scp-discussion}

\subsection{Limitations}
\label{sec:limitations}

Because EntiGraph synthesizes data using a prompted LM, it may hallucinate and fabricate non-existent entities or relations.
Although our synthesis process is grounded by source documents, we assume $\lmgen$ is capable enough to generate faithful synthetic data when conditioned on $\Ds$.
We test factuality of the EntiGraph corpus by randomly subsampling 150 sentences and manually labeling each sentence's factuality.
We find roughly half the sentences are subjective, and the objective half is almost always factual.
We postulate that factuality is high because QuALITY articles are relatively simple given the prompted LM's capability.
If EntiGraph were applied to more challenging content like a complex research paper, the prompted model may be more prone to hallucination.

Because we use a strong prompted LM \texttt{gpt-4-turbo} to generate synthetic data, one might be concerned that performance gains stem from distillation.
To probe this, we perform an ablation replacing \texttt{gpt-4-turbo} with Llama 3.1 8B Instruct, a substantially weaker model from the same base as EntiGraph CPT.
We generated 334M EntiGraph tokens using Llama 3.1 8B Instruct and found a consistent log-linear trend with the same slope but lower intercept compared with GPT-4 generation.
This ablation suggests EntiGraph genuinely teaches model knowledge about the QuALITY corpus rather than serving as a vehicle to distill a powerful prompted LM.

\subsection{Conclusion}

Continued pretraining with next-token prediction effectively teaches pretrained language models new knowledge, but has only been applied successfully in broad, data-rich domains with 10B--100B+ tokens.
We downscale continued pretraining to small, specialized corpora with $\sim$1M tokens using synthetic continued pretraining: converting a small corpus into a large synthetic one with diverse knowledge representations, then continuing pretraining on it.

We instantiate this approach using EntiGraph, a knowledge graph--inspired synthetic data augmentation algorithm.
Synthetic continued pretraining with EntiGraph demonstrates consistent scaling in downstream closed-book QA performance up to a 455M token synthetic corpus, whereas baselines such as continued pretraining on the small corpus or synthetic paraphrases show no improvement or scale slowly.
The acquired parametric knowledge composes with instruction tuning and retrieved non-parametric knowledge in an open-book setting.
We also present a simplified mathematical model of EntiGraph and derive a functional form for its scaling trend that closely matches our empirical observations.
We hypothesize that EntiGraph's ``externalization'' of synthetic data generation to a combinatorial structure---in this case, a knowledge graph over entities---may be a useful strategy for synthesizing highly diverse data and a promising direction for future study.
We designed every component of EntiGraph by hand---the entity-relation extraction, the knowledge graph traversal, and the synthesis prompts.
A natural question is whether AI systems can discover such data augmentation algorithms automatically; we explore this direction in Chapter~\ref{chap:automated-ai-research}.

Lastly, while synthetic continued pretraining enables efficient knowledge acquisition, it operates within the model's existing capabilities.
A more fundamental question remains: can a model improve its core capacity for language modeling?
In Chapter~\ref{chap:sbp}, we address this by targeting pretraining perplexity---the most basic measure of a language model's capability, one that correlates with performance across all downstream tasks.
If we can show genuine self-improvement in perplexity, we demonstrate self-improvement at the most fundamental level.

\chapter{Bootstrapping pretraining capabilities}
\label{chap:sbp}
Chapter~\ref{chap:scp} showed that synthetic data can efficiently teach a language model new \emph{knowledge}---specific facts about a corpus---but the model's underlying \emph{capability}, measured by pretraining perplexity, remained unchanged.
We now ask a more ambitious question: can a model improve the very foundation upon which all downstream performance rests?
We approach this in two parts.
First, we show that strong reasoning is already latent in pretrained weights and can be surfaced with remarkably few examples (\S\ref{sec:s1-sample-efficiency}), suggesting that pretraining---not post-training---is the true bottleneck for capability.
We then introduce Synthetic Bootstrapped Pretraining (SBP), a framework in which the model generates synthetic text to improve its own pretraining objective without relying on a stronger external teacher (\S\ref{sec:sbp-introduction}).

For the SBP component specifically, the self-improvement constraint rules out distillation.
Among remaining approaches (see \S\ref{sec:sbp-related-work} for a detailed review), architectural changes are orthogonal and complementary---SBP's gains compose with them.
Retrieval-augmented pretraining~\citep{Borgeaud:2021, Khandelwal:2020} can leverage related documents but keeps the additional signal external to the weights and imposes retrieval overhead at every forward pass.
In-context pretraining~\citep{shi2024incontext} groups related documents into the same context window, removing retrieval overhead but remaining limited by context length.
We choose synthetic data because it creates new training signal from existing data and writes it directly into the model's weights via standard pretraining.

\section{Prelude: sample-efficient reasoning}
\label{sec:s1-sample-efficiency}

\begin{figure}[t]
\centering
\includegraphics[width=\textwidth]{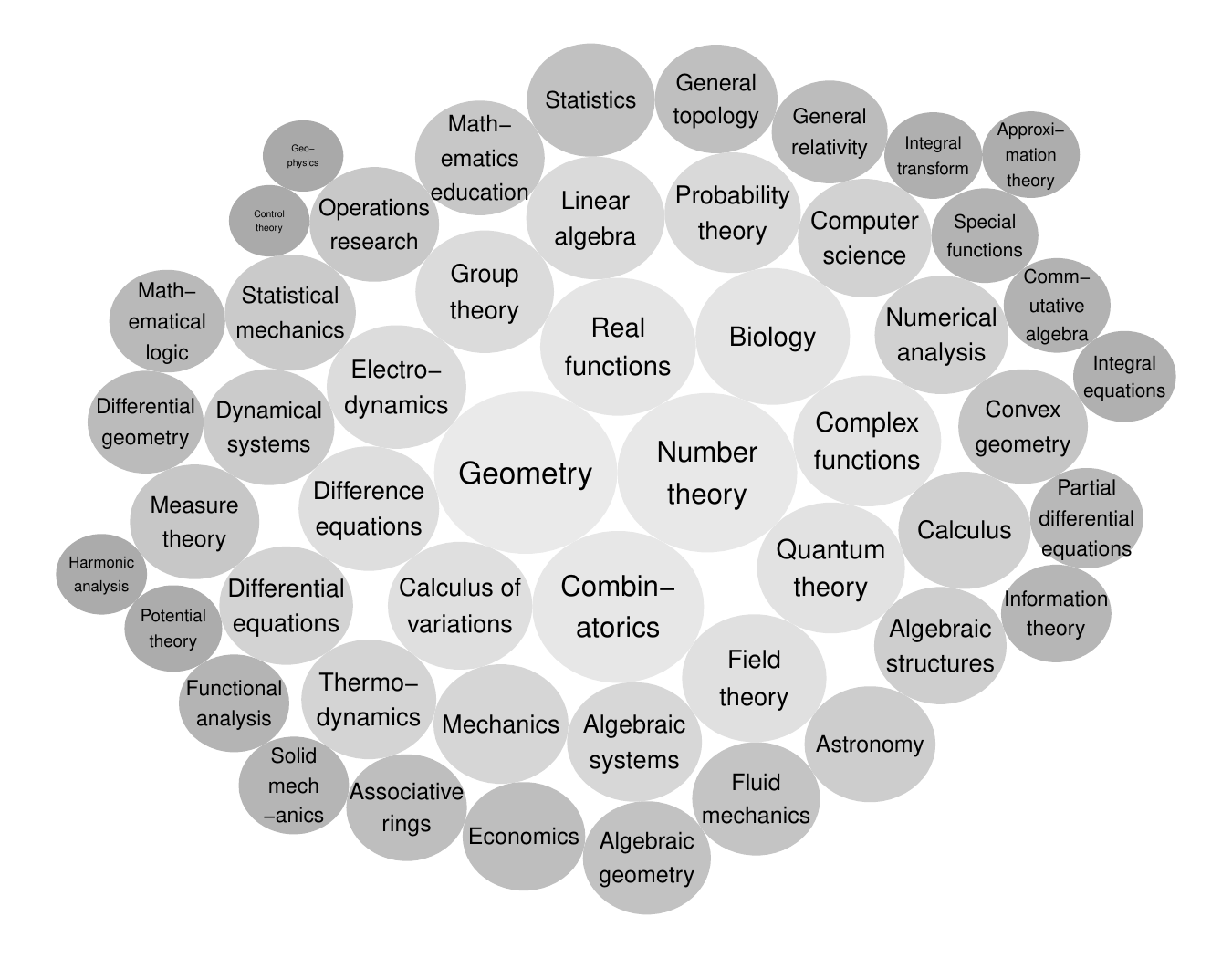}
\caption{\textbf{\sonek{}.} \sonek{} is a dataset of 1,000 high-quality, diverse, and difficult questions with reasoning traces.}
\label{fig:s1-se-bubble}
\end{figure}

We begin with a diagnostic experiment: training on only 1,000 carefully curated samples with next-token prediction suffices to build a strong reasoning model.
OpenAI o1~\citep{o1} and DeepSeek R1~\citep{r1} achieve strong reasoning through large-scale reinforcement learning with millions of training samples, yet a far simpler recipe works surprisingly well.
We construct \sonek{}, a dataset of 1,000 questions paired with reasoning traces distilled from Gemini Thinking Experimental~\citep{geminithinking}, and perform supervised fine-tuning (SFT) of an off-the-shelf pretrained model, requiring just 26 minutes of training on 16 H100 GPUs.
The resulting model \sone{} matches or exceeds closed-source models like OpenAI's o1-preview on several benchmarks (Figure~\ref{fig:s1-se-bubble}).

This extreme sample efficiency carries a profound implication: if 1,000 examples suffice to elicit strong reasoning, then reasoning capability must already be latent in the pretrained weights---post-training merely surfaces it.
Pretraining is therefore the true bottleneck.
All downstream capability ultimately derives from the initial pretraining phase, and improving pretraining itself---without relying on a stronger external teacher---is the highest-leverage intervention.

Extensive ablation experiments confirm that sample efficiency hinges on careful data curation: jointly selecting for quality, difficulty, and diversity is crucial, and training on our full pool of 59K examples offers no substantial gain over our 1K selection (Appendix~\ref{sec:s1-se-appendix}).

\label{sec:s1-se-results}

\paragraph{Training} We perform supervised finetuning on Qwen2.5-32B-Instruct using \sonek{} to obtain our model \sone{} using basic hyperparameters outlined in the appendix.
Finetuning took 26 minutes on 16 NVIDIA H100 GPUs with PyTorch FSDP.

\begin{table}[htbp]
\centering
\caption{\textbf{\sone{} is a sample-efficient reasoning model.} We evaluate \sone{}, Qwen, and Gemini. Other results are from the respective reports~\citep{qwen2024qwen25technicalreport,qwq-32b-preview,o1,r1,bespoke_stratos,sky_t1}. \# ex. = number examples used for reasoning finetuning.}
\begin{tabular}{lrrrr}
\toprule
Model & \makecell{\# ex.} & \makecell{AIME\\2024} & \makecell{MATH\\500} & \makecell{GPQA\\Diamond} \\
\midrule
\multicolumn{5}{c}{\textbf{API only}} \\
\midrule
o1-preview & N.A. & 44.6 & 85.5 & 73.3 \\
o1-mini & N.A. & 70.0 & 90.0 & 60.0 \\
o1 & N.A. & \textbf{74.4} & \textbf{94.8} & \textbf{77.3} \\
\midrule
\multicolumn{5}{c}{\textbf{Open Weights}} \\
\midrule
Qwen2.5- & \multirow{2}{*}{N.A.} & \multirow{2}{*}{26.7} & \multirow{2}{*}{84.0} & \multirow{2}{*}{49.0} \\
32B-Instruct & & & & \\
QwQ-32B & N.A. & 50.0 & 90.6 & 54.5 \\
r1 & $\gg$800K & \textbf{79.8} & \textbf{97.3} & \textbf{71.5} \\
r1-distill & 800K & 72.6 & 94.3 & 62.1 \\
\midrule
\multicolumn{5}{c}{\textbf{Open Weights and Open Data}} \\
\midrule
Sky-T1 & 17K & 43.3 & 82.4 & 56.8 \\
Bespoke-32B & 17K & \textbf{63.3} & \textbf{93.0} & 58.1 \\
\midrule
\textbf{s1-32B} & \textbf{1K} & 50.0 & 92.6 & \textbf{59.6} \\
\bottomrule
\end{tabular}
\label{tab:s1-se-perf}
\end{table}

\paragraph{Sample-efficiency} In Table~\ref{tab:s1-se-perf} we compare \sone{} with other models.
\sone{} is the most sample-efficient open data reasoning model.
It performs significantly better than our base model (Qwen2.5-32B-Instruct) despite training on only 1,000 additional samples.
The concurrently released r1-32B shows stronger performance than \sone{} while also using only SFT~\citep{r1}.
However, it trains on 800$\times$ more reasoning samples.
Whether one can achieve their performance with just 1,000 samples remains an open question.
Around half of all answers in \sonek{} are wrong, yet the results are striking.
This suggests that the SFT stage is about learning reasoning patterns rather than correct answers.

\subsection{Discussion: pretraining as the foundation of capability}
\label{sec:s1-se-disc}

Why does supervised finetuning on just 1,000 samples lead to reasoning performance matching o1-preview?
The most natural explanation is that reasoning capability is already present in the pretrained weights.
Pretraining on trillions of tokens---spanning mathematical proofs, code, scientific arguments, and logical discourse---exposes the model to vast quantities of implicit reasoning.
Our sample-efficient finetuning does not teach the model to reason; it teaches the model to \emph{format} its existing reasoning ability into an explicit chain of thought.
This parallels the ``Superficial Alignment Hypothesis'' of LIMA~\citep{zhou2023lima}, where 1,000 examples suffice to align a model to user preferences because the core capability was already acquired during pretraining.

This interpretation has a profound consequence: post-training methods---whether reinforcement learning~\citep{r1, k1.5}, distillation from stronger models~\citep{sky_t1,xu2025redstardoesscalinglongcot, bespoke_stratos}, or prompting strategies like Chain-of-Thought~\citep{wei2023chainofthoughtpromptingelicitsreasoning}---are ultimately bounded by what pretraining provides.
No amount of post-training can elicit a capability that was never acquired during pretraining.
The true ceiling on model performance is therefore set during pretraining, not during alignment or reinforcement learning.
A complementary implication concerns test-time compute: if reasoning capability is already latent, then controlling how much the model thinks---even by crude interventions---should modulate performance without any additional training.
We return to this observation in Chapter~\ref{chap:automated-ai-research}, where a simple technique called \emph{budget forcing} produces test-time scaling behavior from \sone{}.

This reframes pretraining as the highest-leverage target for improving language models.
If we want fundamentally more capable models---not just better-formatted outputs of existing capability---we must improve the pretraining phase itself.
The remainder of this chapter takes up exactly this challenge.

\section{Synthetic Bootstrapped Pretraining}
\label{sec:sbp-introduction}
Having established that pretraining is the true bottleneck, the remainder of this chapter pursues a single question: can a model improve its own pretraining \emph{without} relying on a stronger external teacher?
In Chapter~\ref{chap:scp}, we allowed distillation from GPT-4 because the goal was data efficiency; here, we remove that crutch and ask whether genuine self-improvement is possible.

We work in a \emph{data-limited} regime, motivated by the approaching exhaustion of high-quality internet text~\citep{villalobos2024run}: we assume a fixed pool of unique documents and ask whether a model can extract more value from them than simple repetition provides.
The defining constraint is that \emph{distillation is forbidden}---the data synthesizer is trained from the same pretraining corpus, not from a stronger external model.
Without this constraint, self-improvement would be trivially achievable by distilling from a more capable teacher.
We validate the approach through a \emph{compute-matched} experimental design: fixing both the data budget and the compute budget, and comparing against a repetition baseline and an oracle with unlimited unique data.
This is the right control because at frontier training scale, compute budgets already exceed available unique high-quality data~\citep{muennighoff2023scaling}, making the data-constrained regime the natural operating point.
Re-examining the conceptual foundation of pretraining, its success stems from the rich causal correlation among tokens \textit{within} a document.
However, this is not the only source of correlation pretraining datasets contain: a code document implementing the attention mechanism is derived from the arXiv preprint of the transformer paper; the book of Harry Potter is structurally similar to the screenplay of its movie production.
Such connections suggest a weaker form of \emph{inter-document} correlation derived from an underlying joint distribution of pretraining documents.
We hypothesize that this additional signal, missed by standard pretraining, can be captured by synthetic data, presenting an underexplored avenue for improving performance.

To leverage this opportunity, we introduce Synthetic Bootstrapped Pretraining (SBP), an LM pretraining procedure that operates in three steps (Figure \ref{fig:sp}).
First, SBP identifies semantically similar document pairs $(\docone, \doctwo)$, such as the transformer paper and its code implementation, from the pretraining dataset.
Second, SBP models the conditional probability of $\doctwo$ given $\docone$, creating a ``data synthesizer'' that can synthesize a new, related document given a seed document.
Finally, SBP applies the trained conditional synthesizer to the pretraining corpus itself, creating a vast text corpus that encodes the rich inter-document correlations previously missed (\S\ref{sec:method-description}).
By training a data synthesizer from the pretraining dataset itself, SBP avoids the pitfall of ``bootstrapping'' model performance using an external, readily available teacher LM, demonstrating a clean setup where improvement stems from better use of the same pretraining corpus.

To test our hypothesis, we design a compute-matched, data-constrained experimental framework under which we pretrain 3B-parameter and 6B-parameter models on up to 1T tokens from scratch \citep{li2024datacomplm, zyphradedup}, demonstrating the applicability of SBP for advancing frontier LMs.
We compare SBP against two crucial references: a strong repetition baseline, which represents the standard approach in data-constrained settings, and an oracle upper bound, which has access to an unlimited pool of unique internet data (\S\ref{sec:experiment-setup}).
Our results show that SBP improves over the strong repetition baseline across the pretraining scales we evaluate and closes up to 60\% of the performance gap to the oracle with 20x additional unique data access (\S\ref{sec:benchmark-performance}).

Beyond benchmark performance, qualitative analysis of the synthesized documents reveals that they go beyond mere paraphrases of the real documents (\S\ref{sec:analysis-of-synthetic-data}).
We postulate that the SBP synthesizer first abstracts latent concepts from the real document and then synthesizes a new document that expands upon the abstracted concept, incorporating diverse genres and content.
We formalize this intuition through a Bayesian hierarchical concept model, where documents are related through shared concepts.
From this perspective, we argue that the synthesizer implicitly learns a posterior likelihood model that abstracts latent concepts from the document---a mechanism not present in standard LM pretraining (\S\ref{sec:statistical-foundation}).

In summary, our contributions are threefold:
\begin{itemize}[leftmargin=16pt]
    \item \textbf{New pretraining framework:} We propose the Synthetic Bootstrapped Pretraining (SBP) algorithm that explicitly models inter-document correlations missed by standard pretraining practice and encodes those correlations into training via synthetic data.
    \item \textbf{Large-scale empirical validation:} We design a compute-matched pretraining setup that enables rigorous measurement of LM self-improvement and empirically validate SBP on 3B and 6B parameter models trained on up to 1T tokens from scratch.
    \item \textbf{Principled statistical interpretation:} We offer a natural Bayesian interpretation of SBP as implicitly learning a posterior for the latent concepts in a text document and concretize the intuition via qualitative analysis of synthesized documents.
\end{itemize}

In the remainder of this chapter, we will first define the data-constrained pretraining problem we address and introduce the SBP technique we propose in \S\ref{sec:method-description}.
Then, we present the compute-matched experiment setup in \S\ref{sec:experiment-setup} and results in \S\ref{sec:experiment-results}.
Finally, we conclude with a Bayesian interpretation of SBP that sheds light on the origin of the improved performance in \S\ref{sec:statistical-foundation}.

\begin{figure}[t]
\centering

\includegraphics[width=\textwidth]{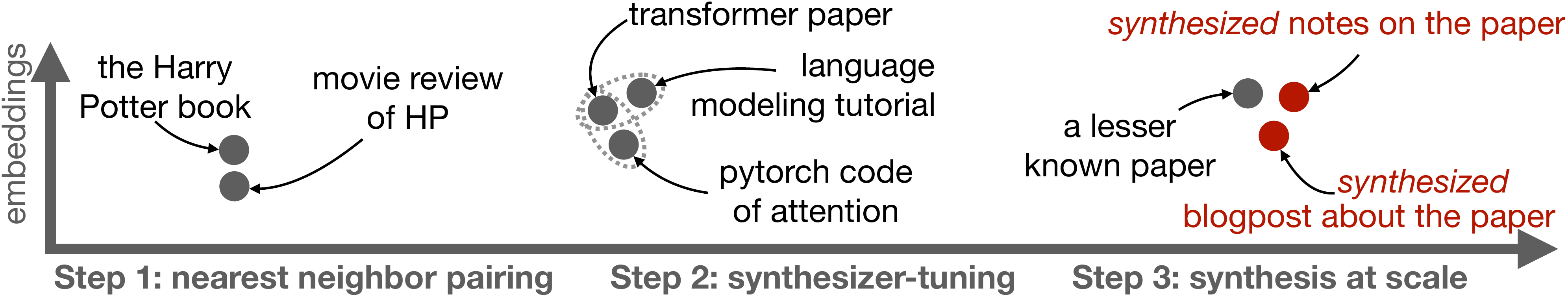}
\caption{Data synthesis illustration of Synthetic Bootstrapped Pretraining (SBP):
It first identifies semantically similar documents (\textbf{Step 1}) and then trains a conditional model that generates one element of the pair from the other (\textbf{Step 2}).
Finally, SBP applies the conditional model to the pretraining corpus itself to synthesize a new, vast corpus for joint training (\textbf{Step 3}).}
\label{fig:sp}

\end{figure}

\section{Method}
\label{sec:method-description}

We introduce the data-constrained pretraining setup (\S\ref{sec:problem-formulation}) and then present the SBP procedure in three steps (\S\ref{sec:sp-three-steps}).
We present SBP as a general pretraining recipe with a pretraining dataset, an LM architecture, and a collection of evaluation benchmarks.
We defer the concrete compute-matched experiment design to \S\ref{sec:experiment-setup}.

\subsection{Data-constrained pretraining setup}
\label{sec:problem-formulation}

We consider a \emph{data-constrained} setup where the goal is to train the best-performing LM given access to a fixed document collection $\Dpre$ (e.g., a snapshot of the entire internet).
To establish a controlled experimental framework, we also choose a transformer architecture with parameters $\theta$ and a collection of held-out evaluation benchmarks $\perf$ (e.g., perplexity, few-shot QA accuracy).
Recall that a transformer takes a sequence of tokens as input and outputs a sequence of conditional probabilities of each token given all previous tokens.
Applying the chain rule for joint probability, we can use a transformer to calculate the probability $p_\theta(y)$ of observing a particular text input $y$, or the conditional probability $p_\theta(y|x)$ of one piece of text $y$ followed by $x$.

Under such a setup defined by ($\Dpre$, $p_\theta$, $\perf$), pretraining searches for the best-performing transformer weights by maximizing the sum of the log-likelihood of pretraining documents,
\begin{equation}\label{eqn:pretraining-objective-sec2}
    \argmax_\theta \sum_{\doc\in\Dpre} \log p_\theta(\doc),
\end{equation}
and then evaluates the performance through $\perf(\theta)$.
Statistically, this objective treats each document as an independent sample from a hypothetical distribution of all documents and attempts to learn this marginal distribution.
However, this modeling assumption overlooks the structural similarities shared between natural language texts (e.g., Figure \ref{fig:sp}).
We next present the SBP procedure that fills this gap.

\subsection{Synthetic bootstrapped pretraining}
\label{sec:sp-three-steps}

At a high level, SBP finds related document pairs $(\docone, \doctwo)$ from the pretraining dataset $\Dpre$ and trains a conditional synthesizer $p_\theta(\doctwo|\docone)$ using the same transformer architecture parametrized by $\theta$.
It then uses it to synthesize a large collection of documents $\Spre$ to perform joint pretraining on $\{\Dpre, \Spre\}$.
The fact that SBP trains a data synthesizer from $\Dpre$ itself also distinguishes it from extensive existing work that relies on a readily available ``teacher'' LM.

\paragraph{Step 1: nearest neighbor pairing} In preparation for training the conditional data synthesizer, SBP first curates pairs of related documents.
To efficiently perform similarity search at pretraining scale, we use Approximate Nearest Neighbor (ANN) search \citep{malkov2018efficientrobustapproximatenearest}, which embeds each document as a quantized vector normalized to the unit sphere and then performs massively parallelizable linear algebraic operations.
In our implementation of SBP, we use inner-product similarity, which we denote by $\<\docone, \doctwo\>$.
Then, we select a subset of pairs whose similarity score exceeds a certain threshold $\alpha$:
\begin{equation}
\Dst = \{(\docone, \doctwo)\in\Dpre\times\Dpre,~\text{s.t.}~ \<\docone, \doctwo\> >\alpha \}.
\end{equation}
We provide implementation details of paired data curation in \S\ref{sec:experiment-details}.

\paragraph{Step 2: synthesizer-tuning} SBP exploits the correlation between pairs of related documents by maximizing the conditional probability of $\doctwo$ given $\docone$:
\begin{equation}
\label{eqn:synthesizer-objective-lm}
\thetast = \argmax_\theta \sum_{(\docone, \doctwo)\in\Dst} \log p_\theta(\doctwo | \docone),
\end{equation}
which we obtain by summing over the log conditional probabilities corresponding to tokens from document $d_2$.
We refer to this step as ``synthesizer-tuning'' as we are training a conditional probabilistic model that synthesizes a related $\doctwo$ from a given $\docone$.
When performing synthesizer-tuning, we initialize $p_\theta$ at the pretrained checkpoint \eqref{eqn:pretraining-objective-sec2} so that the model is equipped with the knowledge of individual documents at initialization, but not the conditional relation between them.
Each document $\docone$ can be associated with multiple instances of $\doctwo$, encouraging the synthesizer to produce diverse, high-entropy outputs rather than deterministic synthesis.

\paragraph{Step 3: data synthesis at scale} Finally, SBP synthesizes $\Spre$ through a hierarchical sampling process:
\begin{itemize}[leftmargin=16pt]
    \item Sample the seed document $\docone$ from $\Dpre$ uniformly at random;
    \item Sample synthesized document $\doctwo$ from $p_{\thetast}(\cdot | \docone)$.
\end{itemize}
This process achieves synthetic data diversity using two sources of variation: first through the variation of the seed documents $\docone$, which comes from the diversity of the pretraining document $\Dpre$ itself, and second through the entropy of the conditional distribution $p_{\thetast}(\cdot | \docone)$, which stems from the diverse inter-document correlations captured in $\Dst$.
While empirically motivated, the procedure admits a principled Bayesian interpretation of the distribution of natural language texts, which we explain in \S\ref{sec:statistical-foundation}.
For now, we focus on demonstrating the empirical effectiveness of SBP.

\section{Experiment setup}
\label{sec:experiment-setup}

We present our concrete experimental implementation of SBP.
We curate a pretraining dataset of 582M high-quality documents totaling 482B tokens from DCLM \citep{li2024datacomplm}, design 3B and 6B transformer architectures modified from the Llama 3 implementation \citep{llama3}, and select nine commonly used benchmarks targeting general world knowledge and commonsense reasoning (\S\ref{sec:data-model-eval}).
We propose a \emph{compute-matched} comparison scheme to validate SBP against natural reference methods at a compute scale of up to 1T total training tokens in our largest experiment (\S\ref{sec:baselines}), bringing validation at a scale relevant for frontier LM development.

\subsection{Data, model, and evaluation}
\label{sec:data-model-eval}

\paragraph{Dataset} A typical pretraining dataset is a mixture of different sources (e.g., GitHub, arXiv, CommonCrawl) with distinct sampling weights assigned to each constituent.
We simplify this reality by considering a fixed document collection, which is a customized version of the DCLM dataset \citep{li2024datacomplm}.
The original 4T token DCLM-baseline split contains roughly 80\% duplicates, as reported by \cite{zyphradedup}.
We therefore begin with the de-duplicated dataset, which consists of 769B tokens.
We clean the raw Zyphra de-duplicated data by normalizing repeated line breaks, removing long URL links, and fixing malformed Unicode characters.
For efficiency, we cap the context window of the synthesizer-tuning \eqref{eqn:synthesizer-objective-lm} step at 8{,}192 tokens.
As a result, we additionally filter out the documents whose length is above 4{,}096 tokens, allowing both $\docone$ and $\doctwo$ to fit into the context window in the worst case when both documents are 4{,}096 tokens long.
After all the de-duplication, cleaning, and filtering procedures, we end up with a collection of 582M high-quality documents $\Dpre$ totaling 482B tokens.
We use the notation $|\Dpre|$ to denote the number of documents in the pretraining dataset and $\|\Dpre\|$ to denote the total number of tokens.

\paragraph{Architecture} We use the Llama 3 transformer architecture \citep{llama3} to model the probability $p_\theta$ with the notable exception of implementing a QK-norm on top of the existing design, which we empirically find to stabilize training.
Our resulting model is a 3B-parameter 26-layer transformer model with a hidden dimension of 3{,}072.
Each layer employs grouped query attention with 24 query heads and 8 key/value heads.
To validate the scalability of SBP, we also train a 6B-parameter model with 32 layers, a hidden dimension of 4{,}096, 32 query heads, and a feedforward dimension of 13{,}056 (detailed in Table \ref{tab:model_specs}).
The position embedding is RoPE \citep{rope} for queries and keys, with frequency 5e+5.
The feedforward network (FFN) has hidden dimension 8{,}064, and we apply prenorm to both the attention and FFN blocks.
For tokenization, we implement a customized BPE tokenization with a vocabulary size of 49{,}152.
To match the 8{,}192 context window design for synthesizer-tuning we have mentioned, we use context window 4{,}096 for pretraining, so that every document in $\Dpre$ can fit into the context window.

\paragraph{Benchmarks} To assess the pretraining capability of LM, we measure pretraining test loss and general world knowledge benchmarks.
We evaluate held-out test perplexity (exponential of negative log-probability) on
1) OpenWebText2 from EleutherAI~\citep{gpt2};
2) Narrative understanding with LAMBADA~\citep{paperno2016lambada}
and
3) Broad domain multiple-choice with MMLU~\citep{mmlu}.
We evaluate QA accuracy on
4) Hard scientific reasoning with ARC‑Challenge~\citep{clark2018think};
5) Easy scientific reasoning with ARC‑Easy~\citep{clark2018think}; 
6) Scientific QA with SciQ~\citep{Welbl2017CrowdsourcingMC};
7) Common sense reasoning with Winogrande~\citep{sakaguchi2021winogrande};
8) Reading comprehension with TriviaQA~\citep{joshi2017triviaqa};
9) Openbook QA with WebQS~\citep{berant-etal-2013-semantic}.
We directly evaluate the pretrained model with either zero-shot or few-shot prompts.
Although MMLU is more commonly used as a QA benchmark, we find that evaluating MMLU accuracy for weak models yields a highly non-smooth readout.
As a result, for each MMLU test question, we prepend the question with a 5-shot example of QA pairs and postpend it with the correct answer.
Then, we treat each such sample as a text corpus and evaluate LM's perplexity on such a text sample.
We find that this perplexity-based MMLU correlates well with MMLU accuracy when the underlying model is large enough to yield a stable readout, and also delivers smooth performance changes for smaller models.
These benchmarks improve significantly with instruction finetuning \citep{flann}.
However, we adhere to our data-constrained setup and do not introduce any additional data that may confound the comparison.

\subsection{Compute-matched comparison}
\label{sec:baselines}

We propose a \emph{compute-matched} experimentation framework to rigorously compare SBP against two natural references: a repetition baseline where we repeat $\Dpre$ multiple times to use the available training compute and an oracle upper bound that enables the model to access as many unique documents as possible.
Compute-matching in a data-constrained regime directly models the real-world situation at frontier scale, where training compute budgets already exceed the supply of unique high-quality internet text---making ``what to do with leftover compute after exhausting unique data'' the operationally relevant question.
Operationally, we control the training compute by controlling the total tokens seen during training, which is proportional to the training FLOPs given a fixed batch size and context window.
We validate SBP across three different settings:
\begin{itemize}[leftmargin=16pt]
    \item \textbf{200B-scale}: In this setting, we cap the training compute to be 200B tokens and cap the data access at $\|\Dpre\|=$10B tokens.
    \item \textbf{1T-scale (3B)}: We also consider a larger scale closer to frontier model training, where we cap the training compute at 1T tokens and data access at $\|\Dpre\|=$50B tokens.
    \item \textbf{1T-scale (6B)}: To validate SBP on larger models, we additionally train a 6B-parameter model with the same 1T token budget and 50B unique data access.
\end{itemize}
For each training scale, $\Dpre$ with different sizes is sampled uniformly at random from the 582M documents pool.
Given the compute-controlled comparison scheme, we next introduce two reference methods against which we compare SBP.

\paragraph{Repetition baseline} Since the compute budget typically exceeds the total number of unique tokens $\|\Dpre\|$, a natural baseline is to repeat $\Dpre$ over multiple epochs.
By design, in both 200B-scale and 1T-scale, we repeat the pretraining dataset $\Dpre$ 20 times to exploit the available compute budget.
In practice, when the pretraining dataset comes from a mixture of different sources, higher-quality documents can be seen as many as 30 times during pretraining, while lower-quality texts may appear only once.
\cite{muennighoff2023scaling} systematically evaluates the repetition baseline as a proposal to scale LMs under data constraints and finds that repeating $\Dpre$ up to 4 times yields nearly no performance degradation compared with having access to unlimited fresh data, but after around 40 times, repetition yields rapidly diminishing returns.
Our choice of 20 times repetition with compute-matched comparison therefore strikes a reasonable balance between efficient experimental execution and exhausting all possible performance gains from a fixed $\Dpre$ via repetition.

\paragraph{Oracle upper bound} Besides showing improvement against the repetition baseline, we also evaluate an oracle upper bound with unlimited data access to contextualize the numerical improvement delivered by SBP.
As we shall see in the next section, because different benchmarks respond differently to data size changes, SBP can deliver an improvement as large as 3.74\% on some benchmarks but only 0.14\% on others (Table \ref{tab:results}).
Because performance on LM benchmarks tends to scale logarithmically \citep{owen2024predictablelanguagemodelbenchmark, kaplan2020scalinglawsneurallanguage} with data size, the numerical difference quickly caps out as we move from the 200B scale to the 1T-scale.
By introducing this oracle upper bound, we can contrast the SBP improvement against this ``oracle'' improvement.

For the 200B-scale experiment, we implement the oracle upper bound as having access to 200B unique tokens from our document pool of size 482B tokens.
For the 1T-scale experiment, we unfortunately do not have 1T unique documents due to the large fraction of duplicates from DCLM.
As a surrogate, we utilize all 482B unique tokens as the dataset for training the oracle upper bound at the 1T-scale.
We provide a partial justification for this by performing a scaled-down comparison at 400B training tokens, with one model having 400B unique tokens and the other one having 200B unique tokens repeated twice (\S\ref{sec:two-epochs-validation}).
We find that the two models (400B unique and 200B repeated twice) yield nearly identical performance.

\paragraph{Training recipe} For both the repetition baseline and oracle upper bound at all scales, we use a batch size of 2{,}048 and a context window of 4{,}096, resulting in a throughput of 8M tokens per step.
We apply a cosine learning rate scale with a 5\% warmup to a peak learning rate of 1e-2, followed by subsequent decay to 5e-5 towards the end.
Under this setup, pretraining costs 11K v5p-TPU hours at 200B-scale, 59K v5p-TPU hours at 1T-scale (3B), and 265K v5p-TPU hours at 1T-scale (6B).
For a clean comparison, we adhere to this hyperparameter throughout the paper, including the SBP experiment presented next.

\section{Experiment results}
\label{sec:experiment-results}

We perform SBP experiments under the compute-matched framework outlined in \S\ref{sec:experiment-setup} at three compute budgets: 200B-scale, 1T-scale (3B), and 1T-scale (6B).
After joint training on real and synthetic data $\{\Dpre, \Spre\}$, we find SBP consistently improves upon the repetition baseline across all scales (Table \ref{tab:results}).
In this section, we focus on presenting the performance of SBP and evaluating the quality of the synthesized pretraining data.
We defer the implementation details of SBP to \S\ref{sec:experiment-details}.
\begin{table}[t]
\centering
\caption{Computed-matched comparison of Synthetic Bootstrapped Pretraining (SBP) and oracle performance gains over the repetition baseline. On average, SBP delivers roughly \textcolor{thesisDarkGarnet}{43\%} of the performance improvement in QA accuracy for the 3B model and \textcolor{thesisDarkGarnet}{58\%} for the 6B model, attainable by an oracle with access to 20x more unique data. } 
\label{tab:results}
\renewcommand{\arraystretch}{1.3} %
\resizebox{\textwidth}{!}{
\begin{tabular}{lrrr|rrr|rrr}
\hline
\hline
& \multicolumn{3}{c}{\textbf{200B-scale}} & \multicolumn{3}{c}{\textbf{1T-scale (3B)}} & \multicolumn{3}{c}{\textbf{1T-scale (6B)}} \\
\cline{2-10} %
\multicolumn{1}{l}{\textbf{Benchmark}} & \multicolumn{1}{c}{\textbf{Baseline}} & \multicolumn{1}{c}{\textbf{SBP}} & \multicolumn{1}{c}{\textbf{Oracle}} & \multicolumn{1}{c}{\textbf{Baseline}} & \multicolumn{1}{c}{\textbf{SBP}} & \multicolumn{1}{c}{\textbf{Oracle}} & \multicolumn{1}{c}{\textbf{Baseline}} & \multicolumn{1}{c}{\textbf{SBP}} & \multicolumn{1}{c}{\textbf{Oracle}} \\
\hline %
\multicolumn{10}{c}{~~~~~~~~~~~~~~~~~~~~~~~~~~~~~~~~~~~~\emph{Perplexity on held-out data $\downarrow$}} \\
\hline %
OpenWebText2& 5.74 & \textcolor{thesisDarkGarnet}{-0.53} & -1.02 & 4.51 & \textcolor{thesisDarkGarnet}{-0.02} & -0.12 & 4.25 & \textcolor{thesisDarkGarnet}{-0.06} & -0.21 \\
LAMBADA  & 6.87 & \textcolor{thesisDarkGarnet}{-0.85} & -1.86 & 4.33 & \textcolor{thesisDarkGarnet}{-0.03} & -0.22 & 3.63 & \textcolor{thesisDarkGarnet}{-0.06} & -0.25 \\
Five-shot MMLU & 3.83 & \textcolor{thesisDarkGarnet}{-0.36} & -0.51 & 3.17 & \textcolor{thesisDarkGarnet}{-0.06} & -0.05 & 3.08 & \textcolor{thesisDarkGarnet}{-0.08} & -0.13 \\
\hline %
\multicolumn{10}{c}{~~~~~~~~~~~~~~~~~~~~~~~~~~~~~~~~~~~~\emph{QA accuracy $\uparrow$}} \\
\hline %
ARC-Challenge \tiny{(0-shot)} & 35.32 & \textcolor{thesisDarkGarnet}{+1.28} & +2.82 & 42.66 & \textcolor{thesisDarkGarnet}{+1.62} & +3.84 & 47.44 & \textcolor{thesisDarkGarnet}{+0.77} & +0.17 \\
ARC-Easy \tiny{(0-shot)} & 68.94 & \textcolor{thesisDarkGarnet}{+2.65} & +4.29 & 75.63 & \textcolor{thesisDarkGarnet}{+0.42} & +2.11 & 78.70 & \textcolor{thesisDarkGarnet}{+0.51} & +0.85 \\
SciQ \tiny{(0-shot)} & 90.50 & \textcolor{thesisDarkGarnet}{+1.00} & +2.40 & 93.20 & \textcolor{thesisDarkGarnet}{+0.80} & +0.50 & 92.90 & \textcolor{thesisDarkGarnet}{+1.90} & +1.80 \\
Winogrande \tiny{(0-shot)} & 60.14 & \textcolor{thesisDarkGarnet}{+1.90} & +5.53 & 65.19 & \textcolor{thesisDarkGarnet}{+1.42} & +2.92 & 70.17 & \textcolor{thesisDarkGarnet}{+0.47} & +2.36 \\
TriviaQA \tiny{(1-shot)} & 22.51 & \textcolor{thesisDarkGarnet}{+3.36} & +7.37 & 36.07 & \textcolor{thesisDarkGarnet}{+0.25} & +0.59 & 40.64 & \textcolor{thesisDarkGarnet}{+0.49} & +3.19 \\
WebQS \tiny{(1-shot)} & 8.56 & \textcolor{thesisDarkGarnet}{+3.74} & +10.83 & 19.34 & \textcolor{thesisDarkGarnet}{+0.54} & +0.44 & 19.88 & \textcolor{thesisDarkGarnet}{+3.79} & +5.22 \\
\hline %
\textbf{Average QA accuracy} & \textbf{47.66} & \textcolor{thesisDarkGarnet}{\textbf{+2.32}} & \textbf{+5.54} & \textbf{55.35} & \textcolor{thesisDarkGarnet}{\textbf{+0.84}} & \textbf{+1.73} & \textbf{58.29} & \textcolor{thesisDarkGarnet}{\textbf{+1.32}} & \textbf{+2.26} \\
\hline
\hline
\end{tabular}
}
\end{table}

\subsection{Main benchmark performance}
\label{sec:benchmark-performance}

At the 200B-scale, we start with the source dataset of $\|\Dpre\|=$10B and curate a SBP dataset of $\|\Spre\|=$75B tokens (detailed ablation in \S\ref{sec:ablation-mixture-ratio}).
We perform joint training on $\{\Dpre, \Spre\}$ with the principle that we do not repeat any synthetic documents during training.
This means that out of a 200B token training budget, we spent 37.5\% of it on the 75B synthetic tokens from $\Spre$ without any repetition, and the remaining 62.5\% on the real dataset $\Dpre$ repeated 12.5 times.
As shown in Table \ref{tab:results}, SBP consistently decreases test loss and improves QA accuracy.
On average, SBP captures \textcolor{thesisDarkGarnet}{$2.32/5.54=$42\%} of the improvement in QA accuracy delivered by the oracle run with 20x additional data access.

\begin{figure}[t]
\centering
\includegraphics[width=\textwidth]{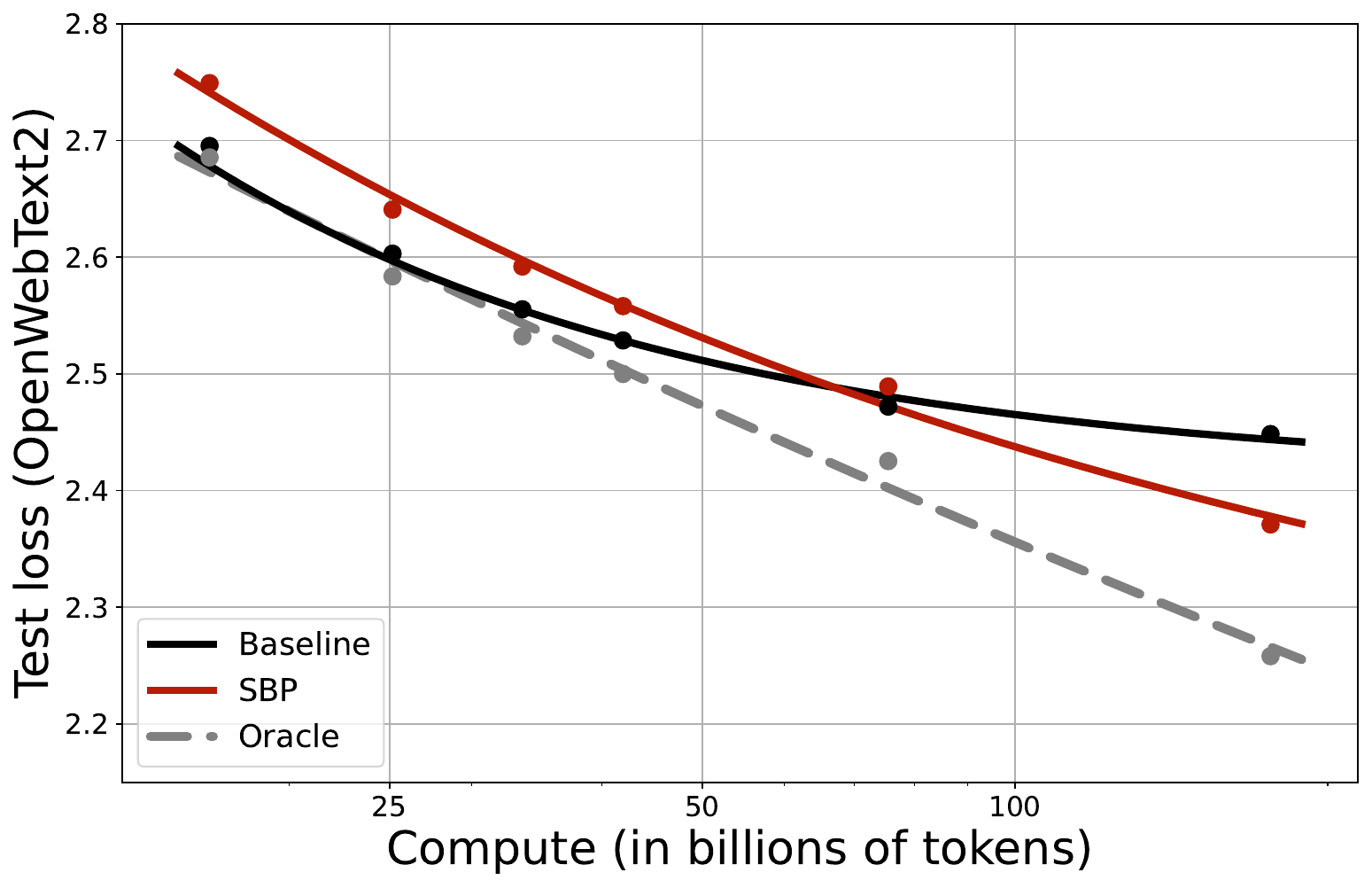}
\caption{Training dynamics (200B-scale).}
\label{fig:training-dynamics}
\end{figure}

The training dynamics of SBP reveal its core mechanism.
As shown in Figure \ref{fig:training-dynamics}, the baseline initially performs similarly to the oracle, since their training data share the same distribution, and when the number of tokens seen is small, there is no distinction between the two.
Then gradually, the oracle becomes a better model than the baseline, as it has access to unlimited unique training data.
For the SBP dynamics, it initially performs worse than both the baseline and the oracle, which is expected since the quality of the synthesized data at most matches that of the real data.
However, gradually, SBP continues to scale while the baseline has plateaued.
This suggests that $\Spre$ offers a signal $\Dpre$ alone cannot capture.

Lastly, to validate the benefit of SBP across different training scales, we implement a larger experiment with $\|\Dpre\|=$50B unique tokens under a compute budget of 1T total training tokens using both 3B and 6B-parameter models.
Based on the ablation studies presented in \S\ref{sec:ablation-mixture-ratio}, we include $\|\Spre\|=$125B synthetic tokens for the 3B model and $\|\Spre\|=$250B synthetic tokens for the 6B model, adhering to the principle of no repetition for synthetic data.
Examining the results in Table \ref{tab:results}, we observe that while perplexity-based measurements plateau for the 3B model \citep{samepretrain}, benchmarks like ARC-Challenge and Winogrande continue to show gains.
On average, SBP recovers \textcolor{thesisDarkGarnet}{$0.84/1.73=$48\%} of the oracle's QA accuracy improvement for the 3B model.
The results are even more pronounced for the 6B model, where SBP delivers a relative improvement of \textcolor{thesisDarkGarnet}{$1.32/2.26=58\%$} compared to the oracle.
This suggests that SBP's effectiveness may scale favorably with model size.
Furthermore, the increased optimal synthetic data ratio for the 6B model suggests that larger models possess greater capacity to exploit the additional information encoded in the synthetic corpus.

\subsection{Analysis of synthetic data}
\label{sec:analysis-of-synthetic-data}

We provide qualitative and quantitative analyses of the synthesized documents to gain insight into the SBP procedure beyond benchmark performance.

\paragraph{Qualitative examples} Figure~\ref{fig:text_comparison} shows samples of synthesized documents from the 200B-scale experiment, with additional samples from 1T-scale (3B) presented in \S\ref{sec:1t-exp}.
On the left, we display a real document about a practical, first-person guide to the coffee houses in San Diego.
We then present two synthesized texts that exhibit notable differences in both framing and depth, with varying degrees of fidelity to the seed document.
Synthesis I sticks to the same topic but shifts toward an expository essay on espresso machines and bean quality, with little mention of specific coffee shops.
Synthesis II adopts a promotional, comparative style, linking San Diego’s coffee culture to New York’s and praising Café Lestat in a way that departs from the original’s balanced assessments.
SBP provides no instructions on how the synthesizer should use the seed texts to write new documents.
The model spontaneously learns to introduce new content and style into the discussion while staying on topic.
It is challenging to manually craft a prompt to an instruction-tuned model that would output either Synthesis I or II with the real document as input.
This example highlights how SBP differs from existing paradigms of data synthesis---the output first abstracts the seed document and then synthesizes new text with more generalized narratives, genres, and intent.
We provide more extensive analysis of this observation in \S\ref{sec:statistical-foundation}.
\definecolor{originalcolor}{RGB}{80, 80, 80}  %
\definecolor{originalcolorlight}{RGB}{230, 230, 230}  %
\definecolor{generatedcolor}{RGB}{100, 100, 100}
\definecolor{generatedcolorlight}{RGB}{250, 250, 250}

\begin{figure}
\begin{center}
\begin{minipage}{0.32\textwidth}
    \begin{tcolorbox}[
        colback=originalcolorlight,
        colframe=originalcolor,
        title=Real document,
        fonttitle=\bfseries\scriptsize,
        width=\textwidth,
        left=0mm,
        right=0mm,
        top=0mm,
        bottom=0mm,
        equal height group=textcomp
    ]

    \small HOME $>$ Chowhound $>$ California $>$ \vspace{3pt}

    Coffeehouses in SD \vspace{3pt}
    
    I am always looking for good coffeehouses in San Diego. I am looking for coffeehouses with (obviously) excellent coffeebeans, baristas who know their business (good microfoam, good crema, they should know the different ratios for steamed milk, foamed milk, espresso for the cappuccino, latte etc., etc), good snacks (paninis, cakes etc), good (cold) flavored coffeedrinks and good ambience. So far I haven’t found many who are good but would like to hear more suggestion. \vspace{3pt}
    
    Caffe Calabria: By far the best coffeehouse in SD. Excellent coffee and baristas (best cappuccino in SD,), great paninis. Only downside is that they close very early. \vspace{3pt}
    
    Café Lestat: Good example what happens with good coffee and average baristas. Lestat is also using coffeebeans from Caffé Calabria but it is sad to see that they don’t know how to make excellent espresso, cappuccino. \vspace{3pt}
    
    E Street Café: Average espresso but strange flavored ...

    \end{tcolorbox}
\end{minipage}
\begin{minipage}{0.32\textwidth}
    \begin{tcolorbox}[
        colback=generatedcolorlight,
        colframe=generatedcolor,
        title=Synthesis I,
        fonttitle=\bfseries\scriptsize,
        width=\textwidth,
        left=0mm,
        right=0mm,
        top=0mm,
        bottom=0mm,
        equal height group=textcomp
    ]
    \small Coffee = Love \vspace{3pt}

    On a recent trip to San Diego, one thing I wanted to find was coffee. I was worried it might be difficult with the huge decrease in Asian flavors, but I was pleasantly surprised. Many of the Mexican, Italian, and Cuban flavors that are pretty common in California are now in the mix. While I did find a few highly recommended spots, the majority of the locations I checked out were nothing to write home about. At the least, they were not being sold out, which I consider a plus. \vspace{3pt}
    
    One major factor that really needs to be taken into account is the type of machine being used. There are two types of espresso machines: the cheap ones that spit out lots of pressure and don’t have any control over the shot, and the expensive ones that create amazing shots and utilize the types of superfine grinds you usually find in Italy. The latter types are much harder to find and cost significantly more. \vspace{3pt}
    
    Another factor to consider is the type of beans being used. If they are done correctly, great espresso can be made from any type of bean, but it’s a rare experience to find a place that does ...

    \end{tcolorbox}
\end{minipage}
\begin{minipage}{0.32\textwidth}
    \begin{tcolorbox}[
        colback=generatedcolorlight,
        colframe=generatedcolor,
        title=Synthesis II,
        fonttitle=\bfseries\scriptsize,
        width=\textwidth,
        left=0mm,
        right=0mm,
        top=0mm,
        bottom=0mm,
        equal height group=textcomp
    ]
    \small Discuss as: \vspace{3pt}

    Coffeehouses: San Diego. The best? \vspace{3pt}
    
    Café Lestat, by far, is a stand-out. (photo by Thomas Moriarty) \vspace{2pt}
    
    If the coffee scene in San Diego sounds like the one in New York, there's a reason. The upper echelons of NYC's drink scene are the same ones who have taken over the majority of the coffee houses in San Diego. This is probably partly due to the fact that both cities have large student populations, both are home to large coffee drinker bases, both have a large number of espresso shops, and both offer top-notch, hand-crafted cappuccinos. But if there is one city in America that deserves to have its name above that of New York, it's San Diego.  \vspace{2pt}
    
    There are just under 100 coffee shops in San Diego, with almost half of them located on University Ave. alone. So finding the perfect coffee shop is crucial. We spent a whole day just roaming around the area, hunting for the best. \vspace{2pt}
    
    In terms of the coffee itself, it's hard to beat Café Lestat. The baristas are amazing and their methods are pristine ...

    \end{tcolorbox}
\end{minipage}

\captionof{figure}{Comparison of original text with synthesized text variations.}
\label{fig:text_comparison}
\end{center}
\end{figure}

\paragraph{Quantitative analysis} We also conduct quantitative evaluations to assess the quality of the generated texts.
We measure text distributions for the synthesized document at 200B-scale and 1T-scale.
To establish a reference, we also conduct the same evaluation on the real documents.
We measure five basic quality indicators:

\begin{itemize}[leftmargin=16pt]
    \item \textbf{Repetition:} A document may contain too many repeated sentences or patterns. Repetition rate refers to the fraction of documents that exhibit this problematic behavior.
    \item \textbf{Duplicate@1M:} Another failure mode of synthesis is when the documents sampled from the synthesizer distribution are nearly duplicates of each other. Duplicate@1M refers to the fraction of unique documents (determined by Jaccard similarity at a threshold of 0.6) when 1M documents are sampled from the text distribution.
    \item \textbf{Non-factual:} A common failure mode of synthesis is the generation of content that contradicts established knowledge or facts. Non-factual rate refers to the fraction of documents that contain verifiable factual errors, as determined by automated fact-checking tools.
    \item \textbf{Pair-irrelevance:} The synthesized $\doctwo$ is considered relevant to $\docone$ if they pertain to the same topic, event, entity, person, place, or object. Pair-irrelevance refers to the fraction of synthesized $\doctwo$ that is not relevant to $\docone$, indicating the synthesis is not rightly using information from $\docone$.
    \item \textbf{Pair-copying:} $\docone$ and $\doctwo$ are considered near-duplicates if they are almost identical, except for some extra white spaces, line breaks, or punctuation. Pair-copying refers to the fraction of synthesized $\doctwo$ that is a near duplicate of $\docone$.
\end{itemize}
Operationally, we implement Repetition, Pair-irrelevance, and Pair-copying using LM-as-judge (prompts and more implementation details given in \S\ref{sec:mid-eval}) by sampling 1{,}000 examples from each distribution and estimating the fraction of documents satisfying each criterion.
For Non-factual (prompts and details given in \S\ref{sec:factuality}), we sample 10{,}000 examples and conduct a comprehensive examination of factual errors to ensure broader coverage of the generated data.
For Duplicate@1M, we use rule-based filtering to detect the fraction of duplicates based on 1M documents sampled from each distribution.
We present the result in the table below. All metrics are lower for better data.

\begin{table}[ht]
\centering
\caption{
Quantitative evaluation of documents sampled from the synthesizer at 200B-scale and 1T-scale.
We can see that the synthesized documents preserve topics and are not are simple duplicates.
}
\renewcommand{\arraystretch}{1.3}
\resizebox{\textwidth}{!}{
\begin{tabular}{lrrrrr}
\hline
\hline
\textbf{ } & \textbf{Repetition $\downarrow$} & \textbf{Duplicate@1M $\downarrow$} & \textbf{Non-factual $\downarrow$} & \textbf{Pair-irrelevance $\downarrow$} & \textbf{Pair-copying $\downarrow$} \\
\hline
\textbf{200B-scale}    & 4.3\% & 0.8\% & 15.1\% & 25.6\% & 0.1\% \\
\textbf{1T-scale (3B)} & 3.9\% & 0.8\% & 8.7\% & 7.8\% & 0.9\% \\
\textbf{1T-scale (6B)} & 2.6\% & 0.3\% & 6.5\% & 6.0\% & 0.3\% \\
\textbf{Real data} & 1.8\% & 0.7\% & 1.8\% & n.a. & n.a. \\
\hline
\hline
\end{tabular}
}
\label{tab:mideval}
\end{table}

Repetition and Duplicate@1M measure basic text quality independent of the specific pair-synthesis strategy employed by SBP.
They aim to detect two simple failure modes: text repetition, a common failure pattern in generations from small language models (3B in our case), and the lack of diversity, a common issue with synthetic data that relies on variation induced by the sampling temperature.
From Table \ref{tab:mideval}, we find that both 200B-scale and 1T-scale synthesis match the quality of real data as captured by these two metrics.
The absence of repetitions and duplicates is not, in itself, an indicator of high-quality or educational text, but rather a basic sanity check that the synthesized texts are diverse.

Non-factual failure stems from hallucinations that introduce non-existent entities or relations inconsistent with reality.
We find that synthesis at the 1T-scale (3B) significantly reduces these errors compared to the 200B-scale.
Furthermore, with the 6B-parameter model, the non-factual rate further decreases to 6.5\%.
This indicates that as the data synthesizer trains on more data with larger models, the factuality of the generated outputs converges toward that of real data.

Pair-irrelevance and Pair-copying, on the other hand, measure how synthesized $\doctwo$ relates to the seed $\docone$.
We aim to detect two failure modes: when $\doctwo$ is completely irrelevant to $\docone$, and when $\doctwo$ merely copies the content of $\docone$.
We observe that both 200B-scale and 1T-scale synthesis avoid simply copying and pasting $\docone$.
The 1T-scale demonstrates substantially higher relevance than the 200B-scale, which makes sense as the synthesizer learns more diverse relations among $|\Dpre|=$60M documents than the $|\Dpre|=$12M corpus.

This concludes the main experimental results.
In the appendix, we present implementation details of SBP in \S\ref{sec:experiment-details}, ablations on synthetic data mixture ratio in \S\ref{sec:ablation-mixture-ratio}, additional analysis of synthesized documents in \S\ref{sec:additional-analysis-of-synthesized-samples}, and comparison with a larger 6B model in \S\ref{sec:model_scaling}.

\section{Statistical foundations of SBP}
\label{sec:statistical-foundation}

We present a Bayesian interpretation of the SBP procedure, offering one explanation for the origin of SBP's improvement.
We formulate a hierarchical model of natural language texts (\S\ref{sec:hierarchical-concept-model}) and demonstrate that SBP implicitly enables LMs to learn a posterior that standard pretraining cannot capture.
We conclude by connecting our findings from this idealized model to the reality of LMs (\S\ref{sec:ideal-to-real-discussion}).
We begin with the observation that the pretraining objective models the marginal likelihood of documents:
\begin{equation}\label{eqn:pretraining-objective-sec5}
    \argmax_\theta \log p_\theta(\Dpre) = \argmax_\theta \sum_{\doc\in\Dpre} \log p_\theta(\doc).
\end{equation}
However, different natural language documents share structural similarities (Figure \ref{fig:sp}), suggesting a more complex underlying joint distribution that we explore next.
Recall that EntiGraph (Chapter~\ref{chap:scp}) explicitly constructed an entity-relation graph to generate diverse synthetic data from a source corpus.
SBP achieves an analogous effect implicitly: rather than prompting for entities and their relations, the synthesizer infers a latent concept $c$ that governs the joint distribution of related documents.
Both methods share the same underlying principle---generating data from a structured intermediate representation yields higher diversity than direct paraphrasing.

\subsection{A hierarchical concept model for natural language}
\label{sec:hierarchical-concept-model}

In the transformer example from Figure \ref{fig:sp}, both the arXiv preprint of the transformer paper and its code implementation are derived from the abstract concept of ``transformer neural network''.
From this perspective, we can view the generation process of natural language documents as a hierarchical sampling process where we first sample a collection of abstract concepts $c^{(i)}$ (e.g., the idea of a transformer) from a semantic space of all concepts $\mathcal{C}$ and then generate new documents $\doc^{(i,j)}$ conditional on $c^{(i)}$.

\begin{wrapfigure}{r}{0.2\textwidth}
\centering
\vspace{-10pt}
\includegraphics[width=0.18\textwidth]{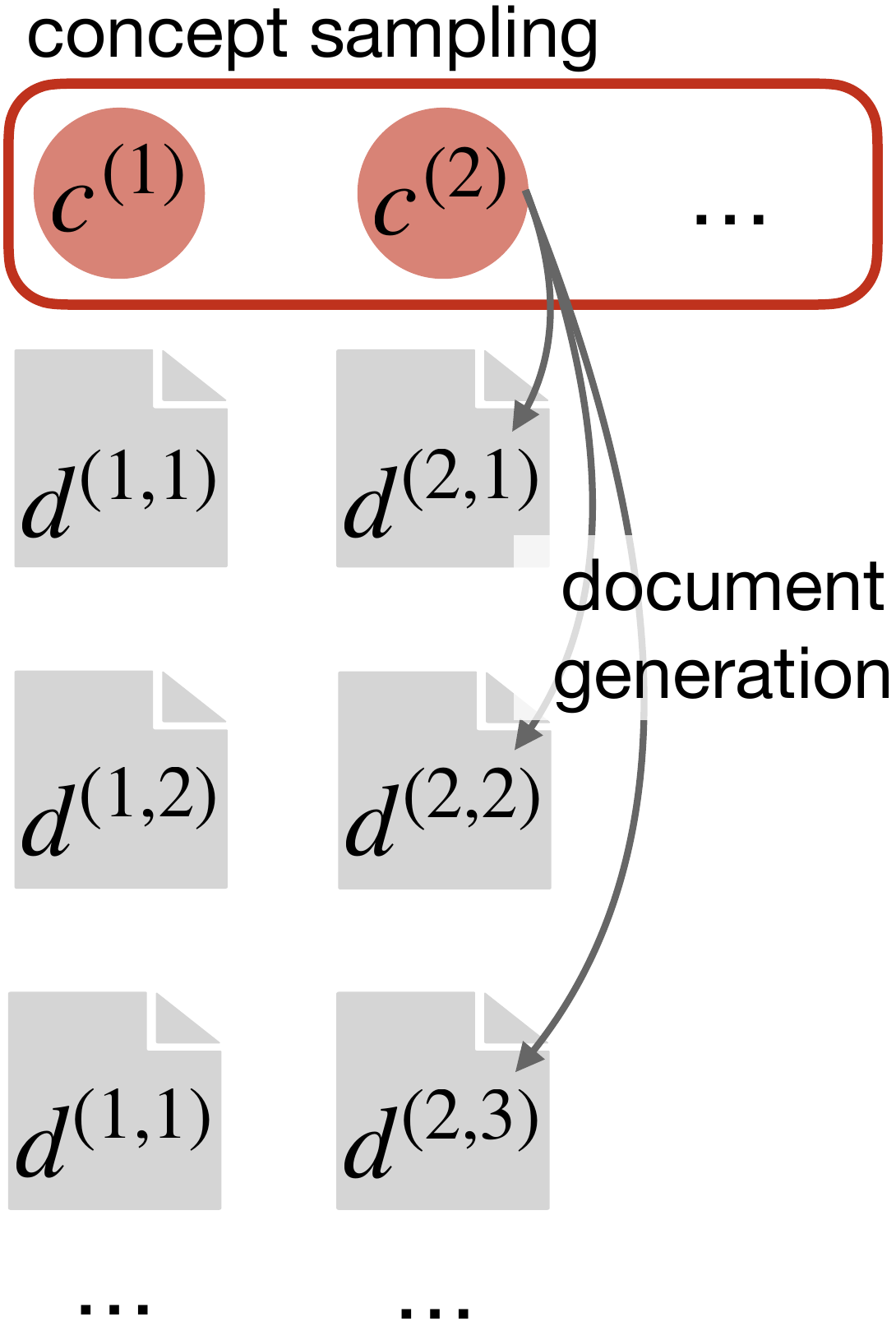}
\label{fig:hcm}
\vspace{-10pt}
\end{wrapfigure}

Under this view, we can think of the pretraining document as follows.
\begin{itemize}[leftmargin=16pt]
    \item \textbf{Concept sampling}: Sample a fixed concept collection $\{c^{(i)}\}_i\sim P(c)$.
    \item \textbf{Document generation}: For each concept $c^{(i)}$, generate documents from $\{\doc^{(i,j)}\}_j\sim P(\doc|c^{(i)})$ constituting one part of the pretraining dataset.
\end{itemize}
Under such a model, the structural similarity between documents generated from the same concept is modeled as probabilistic \emph{dependence}.
The standard pretraining objective \eqref{eqn:pretraining-objective-sec5} then neglects inter-document correlation and only learns the marginal distribution
\begin{equation}
    P(\doc) = \int_{c\in\mathcal{C}} P(\doc|c)P(c) dc.
\end{equation}
In this view, the model learns to generate plausible text by first generating a core concept $c$ and then performing the generation $P(\doc|c)$.
In contrast, the synthesizer-tuning objective models a posterior of $c$ given $d$.
To see this, we additionally assume that the curated pairs $(\docone, \doctwo)$ come from the same underlying concept $c$.
Then, the synthesizer-tuning objective \eqref{eqn:synthesizer-objective-lm} forces the LM to perform a distinct task:
\begin{equation}\label{eqn:synthesizer-objective-simulation}
    P(\doctwo | \docone) = \int_{c\in\mathcal{C}} P(\doctwo|c)P(c|\docone) dc.
\end{equation}
Here, we use Bayes' rule and the conditional independence assumption
\[
    P(\doctwo|c,\docone) = P(\doctwo|c),
\]
which says that the documents from the same concept are conditionally independent given that concept.
As a result, to successfully model \eqref{eqn:synthesizer-objective-simulation}, the synthesizer must first perform posterior inference to infer the latent concept $c$ given the document $\docone$, and then use this inferred concept to synthesize a new document $\doctwo$, a signal that is ignored by the standard pretraining objective.
To illustrate this, we perform a post-hoc analysis by prompting an LM to identify the shared concepts between the synthesized document and its seed (Table \ref{tab:doc_connections}).
While it is difficult to describe a synthesized document as the outcome of a simple transform, such as a paraphrase or summarization, it always shares a common underlying concept with its seed origin.
\begin{table}[ht]
\centering 
\caption{Examples of latent concepts $c$ inferred by an external LM (prompts provided in \S\ref{sec:concept-analysis}). From left to right, we provide a summary of the real document, the inferred latent concept, and a summary of the synthesized document.}
\rowcolors{2}{appleLightGray}{white} %
\scriptsize
\renewcommand{\arraystretch}{1.4} %
\begin{tabular}{p{0.37\linewidth} p{0.27\linewidth} p{0.37\linewidth}}
\hline\hline
\textbf{Real document summary} & \textbf{Concepts} & \textbf{Synthesized document summary} \\
\hline
Twitter's impact on journalism & Opportunities arise from Twitter & Guide on Twitter user monetization \\
\hline
Family story about kids and doughnuts & Parenting + kids' food catering & Emotional anecdotes of parents treating kids \\
\hline
Minor parties' challenges in the U.S. Congress & Minor political parties in the U.S. & Explains U.S. minor parties' history \\
\hline
Personal stories/questions about swollen eyes & Causes/treatments of swollen eyes & Non-personal guide to treating swollen eyes. \\
\hline
Antarctic carbon fixation mechanisms & How life survives in Antarctic & Antarctic geography and survival adaptations \\
\hline
Profile of a belly dancing teacher in the U.K. & Belly dancing as a dance form & General introduction to belly dancing \\
\hline
Anxiety about creative work judged in a dream & Dream as personal self-reflection & Description and reflection of personal dreams \\
\hline
NYC (yearly/monthly) climate extremes & NYC weather and temperature & QA on NYC July heat and related topics \\
\hline
Tutorial for Minecraft block modding & Block editing in Minecraft & Minecraft forum question on removing blocks \\
\hline
Cosmic airburst destroys Tall el-Hammam city & Destruction of ancient cities & Tall el-Hammam excavation as a news event \\
\hline
\hline
\end{tabular}
\label{tab:doc_connections}
\end{table}

The additional signal from the posterior then enables a form of self-distillation.
The synthesizer, by learning a more complex conditional objective, becomes a more knowledgeable ``teacher'' model that has learned to infer the latent structure of data.
The synthetic data it produces is then the knowledge ``distilled'' from this teacher \citep{hinton2015distillingknowledgeneuralnetwork}. 
The final LM training then acts as a ``student'' that learns from a combination of real and synthetic data, allowing it to discover information that real data alone cannot reveal.

\subsection{From idealized models to language model reality}
\label{sec:ideal-to-real-discussion}

For real text documents, we do not know the true data-generating process, and any parametric assumption would be incorrect.
This is where the power of the transformer neural network shines.
A transformer is a \emph{mapping-first} \citep{2cultures} approach.
It does not require explicit modeling of the underlying parametric model.
Instead, as a universal function approximator \citep{emmanuelthesis}, it directly learns the complex conditional distribution $p_\theta(\doctwo|\docone)$ from paired data alone.

In this context, the transformer's ignorance of an explicit hierarchical model is its blessing.
It bypasses the impossible step of modeling the true hierarchical distribution of language and instead brute-forces the learning of the exact transformation required: the end-to-end process of posterior inference and subsequent synthesis.
The self-distillation framework---synthesizing data from this conditional model and then training on it---is all that is needed.
We never need to introduce an explicit hierarchical model to perform the forward $P(\doc|c)$ and backward pass $P(c|\doc)$ in the latent space.
The entire procedure is implicitly carried through the synthesizer-tuning update with the latent concept $c$ integrated, demonstrating a powerful insight for scaling LMs in the real world.

\section{Discussion}
\label{sec:sbp-discussion}

Before the prevalence of pretraining \citep{gpt2, devlin-etal-2019-bert}, we needed 40M pairs of English and French sentences \citep{sutskever2014sequencesequencelearningneural} to grant an LM the ability to translate from one language to another.
In contrast, any modern LM \citep{gemini, gpt4, gunter2024appleintelligencefoundationlanguage} can easily achieve this task via a single prompt.
This advancement stems from the rich correlations between words within a document that LMs learn during pretraining.
This shift from a hard-to-collect, task-specific dataset to the readily available, generically crawled internet data marks a transition from relying on strong but scarce supervision to a weak self-supervision that is easy to scale.
As we gradually deplete this weak source of self-supervision by exhausting the available internet data, many have called for stronger forms of supervision, such as reinforcement learning \citep{r1, o1}.
We instead offer an alternative perspective that continues to search for a form of self-supervision weaker than next-token prediction.
SBP offers a particular instantiation of such an effort by examining the correlations \emph{between} documents missed by the current pretraining paradigm.
It remains to explore other forms of supervision not currently utilized.

The fact that SBP provides any improvement stems from the poor inductive bias of the transformer neural network.
For example, transformers trained on the text ``A is B'' can not generalize to ``B is A'' \citep{berglund2023reversal}, requiring the user to curate synthetic data that explicitly narrates the converse relation \citep{entigraph}.
One can imagine an architecture with strong inductive bias such that the LM trained individually on $\docone$ and $\doctwo$ can automatically internalize the relation between the two.
Despite this poor inductive bias, transformers are more parallelizable and scalable than their alternatives \citep{transformer, shazeer2017outrageouslylargeneuralnetworks}.
Given this trade-off, SBP offers a unique solution that preserves the system benefits of the transformer architecture while also enabling the learning of missed correlations by encoding such additional signals via synthetic data.

\subsection{Limitations}
\label{sec:future-directions}

\paragraph{Document embedding with activations of pretrained LM} In our implementation of SBP, we use Qwen3-0.6B-Embedding \citep{zhang2025qwen3embeddingadvancingtext} to obtain embeddings of DCLM \citep{li2024datacomplm} documents.
An ideal implementation of SBP would only rely on the 3B-parameter model and the pretraining dataset itself to curate the paired synthesizer-tuning dataset.
To achieve this, we can use the activations of the self-attention layer from an intermediate transformer block as a learned representation of documents.
\cite{Khandelwal:2020} and \cite{resmem} implemented this at the much smaller scale of $\sim300$M parameters and $\sim3$B tokens.
However, our experiments operate at a much larger scale with a customized model.
As a result, we use the optimized vLLM \citep{kwon2023efficient} inference infrastructure for Qwen3-0.6B embedding models to efficiently index the pretraining corpus.
Since the SBP procedure only requires a coarse binary decision of relevant vs.~not relevant, which is much weaker than fine-grained document ranking embedding models are optimized for, we leave the more involved inference infrastructure for future work.

\paragraph{Parametric fit of SBP scaling law} One experimental consequence of LM pretraining is a clean scaling law \cite[Equation 1.4]{kaplan2020scalinglawsneurallanguage} that relates the held-out test loss $L(N, D)$ to the number of LM parameters $N$ and the size of the pretraining dataset $D$.
A natural scientific question is how the scaling law of SBP compares to the scaling law of pretraining.
In our experiments, we evaluate $L(N, D)$ at three distinct points: $(N=3\text{B}, D=10\text{B})$, $(N=3\text{B}, D=50\text{B})$, and $(N=6\text{B}, D=50\text{B})$.
We observe that SBP delivers larger relative improvements with the 6B model, suggesting a favorable scaling behavior.
Two obstacles prevent a full scaling law:
First, SBP is inherently a large-scale algorithm that cannot be scaled down.
Since SBP first uses the pretrained LM to generate synthetic text and then trains on it, if the model and dataset sizes are too small, the generated text may not even be coherent.
In contrast, the scaling experiments in \cite{kaplan2020scalinglawsneurallanguage} involve model sizes ranging from 768M to 1.5B and dataset sizes ranging from 22M to 23B, which allows for efficient experimentation.
Second, varying $N$ or $D$ implies redoing the synthesizer-tuning and subsequent data synthesis over billions of tokens.
Additionally, varying $D$ also implies redoing the nearest neighbor matching, as the neighbors are only defined given a fixed document pool.
These obstacles aside, it would be interesting to see whether the SBP scaling law differs from the normal scaling law by a smaller multiplicative factor or a better exponent.
Since SBP additionally utilizes inter-document correlations, a form of long-range interactions, its scaling law not only helps us understand SBP but also potentially helps us better understand natural language data itself \citep{Ebeling_1994}.

\subsection{Conclusion}
\label{sec:sbp-conclusion}

We introduce Synthetic Bootstrapped Pretraining (SBP) as a new LM pretraining framework that leverages inter-document correlations missed by the standard pretraining objective.
We demonstrate the effectiveness of SBP by pretraining 3B and 6B-parameter models from scratch for up to 1T tokens under a rigorously designed compute-matched experiment setup.
Qualitative and quantitative analyses of the synthesized text reveal rich variation that cannot be captured by simple paraphrases.
Beyond being practically effective, SBP admits a natural Bayesian interpretation, where it implicitly learns a posterior that infers the latent concept in a given document.
Its \emph{bootstrapping} nature grants SBP the possibility of scaling the pretraining capability of the LM beyond its current limit.

Yet even with SBP, the pretraining procedure and learning algorithms remain human-designed.
In Chapter~\ref{chap:automated-ai-research}, we ask whether we can remove this final dependency on human ingenuity---by building AI systems that autonomously discover and improve their own training methods.

\chapter{Towards AI-designed AI via test-time search}
\label{chap:automated-ai-research}
Chapters~\ref{chap:scp} and~\ref{chap:sbp} showed that a model can acquire new knowledge and bootstrap its own pretraining capabilities through synthetic data, but the learning algorithms themselves---EntiGraph, SBP, and the training pipelines that orchestrate them---remained human-designed.
This motivates asking whether the research process itself can be automated.
In this chapter, we pursue this direction by building an automated AI research system.
Established approaches to algorithmic improvement---Neural Architecture Search~\citep{Zoph2016NeuralAS, So2019TheET}, automated algorithm discovery~\citep{Real2020AutoMLZeroEM}, and learned optimizers~\citep{Chen2023SymbolicDO}---are reasonable alternatives, but they operate within constrained, predefined spaces or require end-to-end differentiable pipelines, making it difficult to discover techniques outside the search space or scale to full training systems.
We pursue research automation because it operates in an unbounded action space: ideas expressed in natural language, validated via code execution, using capabilities that language models already possess (see \S\ref{sec:air-related-work} for a detailed comparison).

\section{Towards automated AI research}

We envision automated AI research as follows: LLMs generate research ideas, implement them as code, run experiments to verify effectiveness, and continuously learn from execution results.
If successful, such automated researchers could discover effective ideas in a massive search space, scalably converting compute into scientific discovery; the discovered ideas could, in turn, improve frontier AI models themselves, enabling recursive self-improvement.

Despite this promise, the ability of LLMs to generate effective ideas remains the key bottleneck.
\cite{Si2024CanLG} and \cite{Si2025TheIG} evaluated LLM-generated research ideas through large-scale expert review and found that LLM ideas often look convincing but prove ineffective after human researchers execute them.

This finding underscores the need to ground idea generation in execution.
However, obtaining execution results automatically and at scale is challenging---especially for open-ended AI research where any idea expressible in natural language lies within the action space.
We build a high-throughput automated idea executor that implements hundreds of model-generated ideas in parallel and returns experiment results as execution feedback.

To study how far we can push automated LLM research, we select two GPU-intensive research problems---LLM pre-training and post-training---as research environments for our automated AI researchers.
We demonstrate for the first time that our automated executor can implement a large fraction of LLM-generated ideas on these challenging open-ended research problems, achieving over 90\% execution rates on the pre-training environment with Claude-4.5-Sonnet and Claude-4.5-Opus.

We define objective and unhackable performance metrics for both environments and apply test-time search at the idea level---using execution feedback to guide evolutionary search over model-generated ideas.

\begin{figure}[!htbp]
  \centering
  \includegraphics[width=0.8\textwidth]{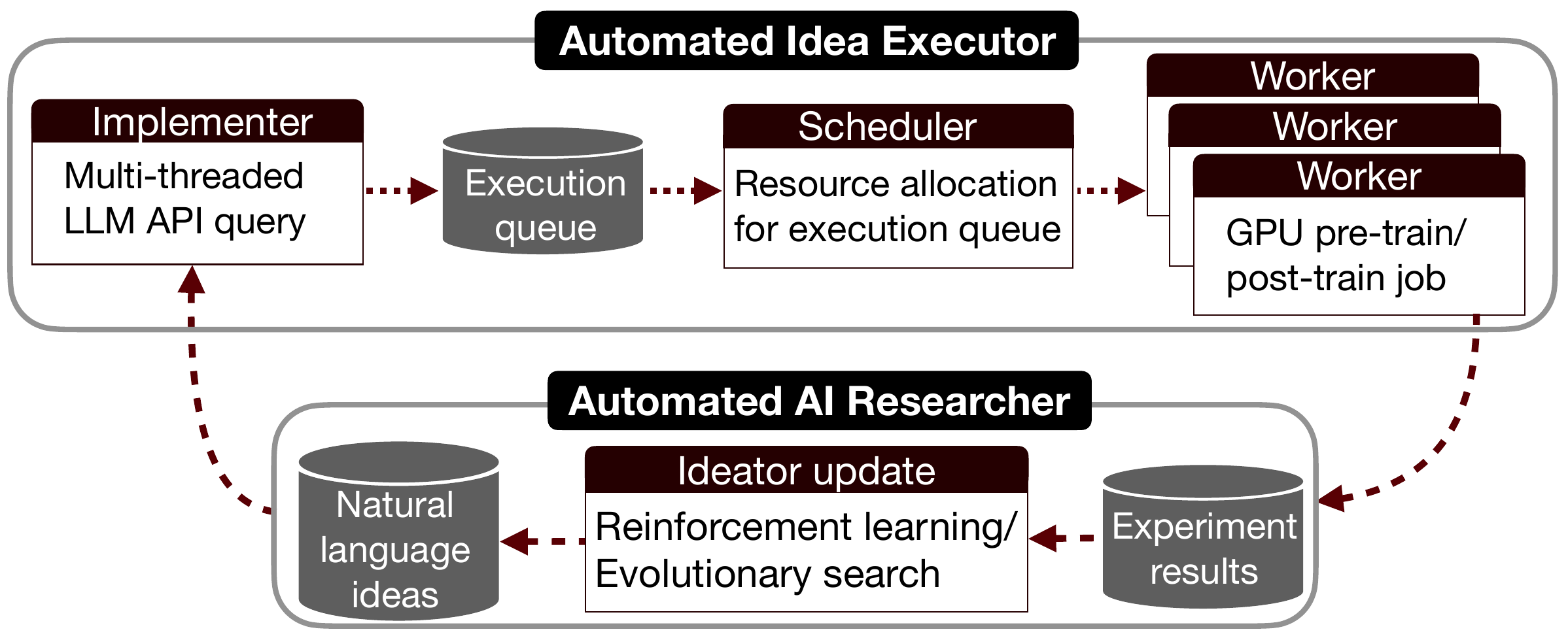}
  \caption{We build an automated idea executor involving Implementer, Scheduler, and Worker. We then use this automated executor to guide test-time search over model-generated ideas.}
  \label{fig:systemdesign}
\end{figure}

We use our automated executor to guide evolutionary search.
Within ten search epochs, execution-guided search finds a post-training recipe that improves over the GRPO baseline (69.4\% vs.\ 48.0\%) on post-training a 1.5B model for math reasoning, and a pre-training recipe that improves over the nanoGPT baseline (19.7 minutes vs.\ 35.9 minutes) on minimizing training wall-clock time to reach the target validation loss (Table~\ref{table:search_results}).
We find that models frequently generate algorithmic ideas beyond hyper-parameter tuning, and evolutionary search outperforms best-of-N under the same sampling budget.
However, only Claude-4.5-Opus shows a clear scaling curve; both Claude-4.5-Sonnet and GPT-5 saturate early.

We also explore reinforcement learning with execution reward as an alternative to search; while average reward improves, the max reward---more critical for discovery---does not, pointing to fundamental challenges in learning to generate research ideas (Appendix~\ref{sec:rl-from-execution-reward}).

In summary, we develop a large-scale automated idea executor that implements research ideas for open-ended, realistic research problems.
Using this executor, we show that execution-guided evolutionary search is sample-efficient and effective, outperforming best-of-N under the same sampling budget.
We provide extensive analysis of the executed ideas and suggest promising directions for scaling automated AI research.

\begin{table}[t]
\centering
\small
\caption{Performance of our execution-guided search in comparison with the provided baselines and best human experts. The post-training task is to finetune a 1.5B model for math reasoning, and the metric is accuracy on the MATH validation set. The pre-training task is to train a 124M Transformer on FineWeb, and the metric is the training time to reach 3.28 validation loss.}
\vspace{5pt}
\begin{tabular}{lcc}
\toprule
                    & Post-training$_{\uparrow}$   & Pre-training$_{\downarrow}$ \\
\midrule
Baseline            & 48.0\% & 35.9 min  \\
Execution-guided search & 69.4\% & 19.7 min  \\
Best human expert       & 68.8\% & 2.1 min \\
\bottomrule
\end{tabular}
\label{table:search_results}
\end{table}

\section{Automated idea executor}

We build an automated executor that takes natural language research ideas as input, generates code implementations, runs experiments on the backend, and returns benchmark performance as output.

\subsection{Research environments for ideation}

We formalize a research environment as an abstract interface with two methods (Figure~\ref{fig:research-env}, left): \texttt{context()} returns the task description passed to the LM ideator, and \texttt{value(idea)} returns a scalar measure of idea quality after execution.
For AI research, we instantiate this interface as \texttt{AIResearchEnv} (Figure~\ref{fig:research-env}, right): \texttt{context()} returns the baseline codebase, and \texttt{value()} patches the idea into a sandboxed copy of the codebase, runs the experiment, and returns the benchmark result.
This abstraction separates what the LM sees (the codebase) from how ideas are scored (sandboxed execution), and makes the search objective explicit: all methods in this chapter---best-of-N, evolutionary search, and RL---optimize \texttt{env.value(idea)}.

\begin{figure}[!htbp]
\centering
\begin{minipage}[t]{0.42\textwidth}
\begin{lstlisting}[style=pseudocode]
class ResearchEnv:

  @abstract
  def context(self):
    # task description
    # passed into LM ideator
    pass

  @abstract
  def value(self, idea: str):
    # scalar measure of
    # idea quality
    pass
\end{lstlisting}
\end{minipage}
\hfill
\begin{minipage}[t]{0.54\textwidth}
\begin{lstlisting}[style=pseudocode]
class AIResearchEnv(ResearchEnv):
  codebase: str
  resource: str
  sandbox_factory: Callable

  def context(self):
    return self.codebase

  def value(self, code_diff: str):
    sb = self.sandbox_factory(
             self.resource)
    sb.exec(f"patch {code_diff}")
    sb.exec("bash run.sh")
    return sb.exec("bash eval.sh")
\end{lstlisting}
\end{minipage}
\caption{Research environment abstraction. \textit{Left:} abstract interface defining the search problem---\texttt{context()} provides what the LM sees, \texttt{value()} scores an idea by execution. \textit{Right:} concrete implementation for AI research---ideas are patched into a sandboxed codebase and evaluated via automated execution.}
\label{fig:research-env}
\end{figure}

We select research problems that are open-ended---allowing ample room for algorithmic innovation---while having well-established baselines so that measuring effectiveness is simple.
We construct both a pre-training and a post-training environment.\footnote{We open-source our environments and idea execution trajectories at \url{https://github.com/NoviScl/Automated-AI-Researcher}.}

\paragraph{Pre-training task: improving nanoGPT} In the nanoGPT environment, we provide a baseline codebase adapted from the nanoGPT speedrun~\cite{modded_nanogpt_2024}.
The original speedrun task is to minimize the time to pre-train a 124M GPT-2 model~\cite{gpt2} on the FineWeb corpus~\cite{Penedo2024TheFD} to reach a validation loss of 3.28 on the validation set on 8 H100 GPUs.
We make several modifications to the original speedrun setting.
First, we define $\texttt{value}(\text{idea}) = \frac{1}{\text{validation loss}}$ for the search and RL experiments in later sections.
This allows us to fix the training wall-clock time at 25 minutes and have the model directly optimize the proxy reward under this fixed budget, avoiding vastly different runtimes across runs.
We report the validation loss or the proxy reward metric in most plots, and only measure and report the training time metric for the top solution in order to directly compare it with the human experts' solutions on the original nanoGPT speedrun leaderboard.
Second, to prevent reward hacking, we freeze all evaluation hyper-parameters and implement an inference function that predicts one token at a time.
This prevents models from changing the attention mechanism in ways that leak future tokens---an issue we encountered multiple times during initial development.
We use this inference function during final validation after each training run.

\paragraph{Post-training task: improving GRPO} In the GRPO environment, we provide an implementation of the GRPO algorithm~\cite{shao2024deepseekmathpushinglimitsmathematical} that finetunes a Qwen2.5-Math-1.5B checkpoint~\cite{yang2024qwen25mathtechnicalreportmathematical} on the MATH dataset~\cite{Hendrycks2021MeasuringMP}.
We specify a fixed training wall-clock time budget and define $\texttt{value}(\text{idea})$ as the max accuracy on the MATH validation set during training.
To prevent reward hacking, we keep all validation-related code in a separate file that the automated executor cannot access or modify.

In both environments, we impose no constraints on the ideation scope, so anything from extensive hyperparameter tuning to novel model architectures or training algorithms is within scope.

\subsection{System design}

The automated idea executor implements the \texttt{value} method of Figure~\ref{fig:research-env}: given a batch of natural-language ideas, it returns the benchmark performance of each.
Three building blocks compose this API (Figure~\ref{fig:systemdesign}): \textit{Implementer}---the server that generates the code diff for each idea and applies those changes; \textit{Scheduler}---a middle layer that receives codebases and allocates resources; \textit{Worker}---the GPU cluster that runs experiments and uploads results.

\paragraph{Implementer} The implementer runs on a CPU machine with high IO capacity.
The user submits a batch of natural language ideas.
For each idea, the implementer makes parallel API calls to the code execution LLM to obtain a \texttt{diff} file that can be patched into the baseline codebase.

For efficiency, we prompt the code execution LLM with both the idea and the baseline codebase to sample 10 code diff files in parallel.
For each sample, if the generated diff file cannot be patched into the original codebase, we provide the patch log and ask the model to revise its generation.
We repeat this self-revision for a maximum of 2 times.
We then return the first code diff file that patches successfully.
The patched codebase is submitted to a cloud bucket as a \texttt{.zip} file.

\paragraph{Scheduler} At a set clock frequency, the scheduler downloads new codebases from the cloud.
If a codebase has not been executed, the scheduler examines its resource requirements and prepares a job configuration for submission.

\paragraph{Worker} Once the scheduler finds available resources, it connects the prepared job configuration with the GPU resource and initializes the worker to run the experiment.
If execution succeeds, the worker uploads experiment logs---including all performance metrics---to another cloud bucket (\texttt{wandb}) along with complete metadata: idea content, code change, execution log, etc.
If execution fails (e.g., due to bugs in code implementation), the worker halts.
The ideator model can then download execution results and observe the performance of all submitted ideas with full training logs.

\begin{figure}[!htbp]
  \centering
  \subfigure[Self-Execution (GRPO)]{\includegraphics[width=0.495\linewidth]{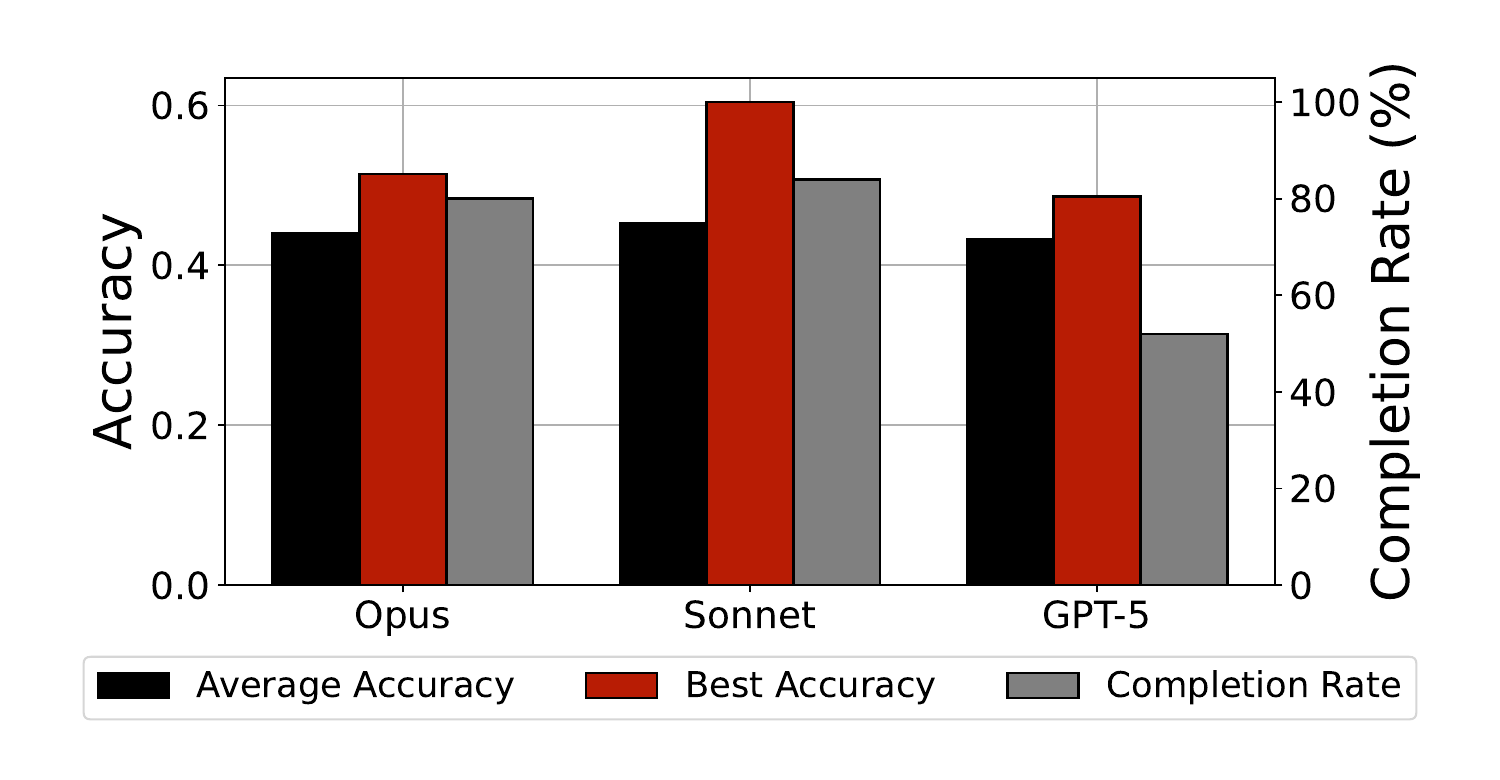}}
  \subfigure[Self-Execution (nanoGPT)]{\includegraphics[width=0.495\linewidth]{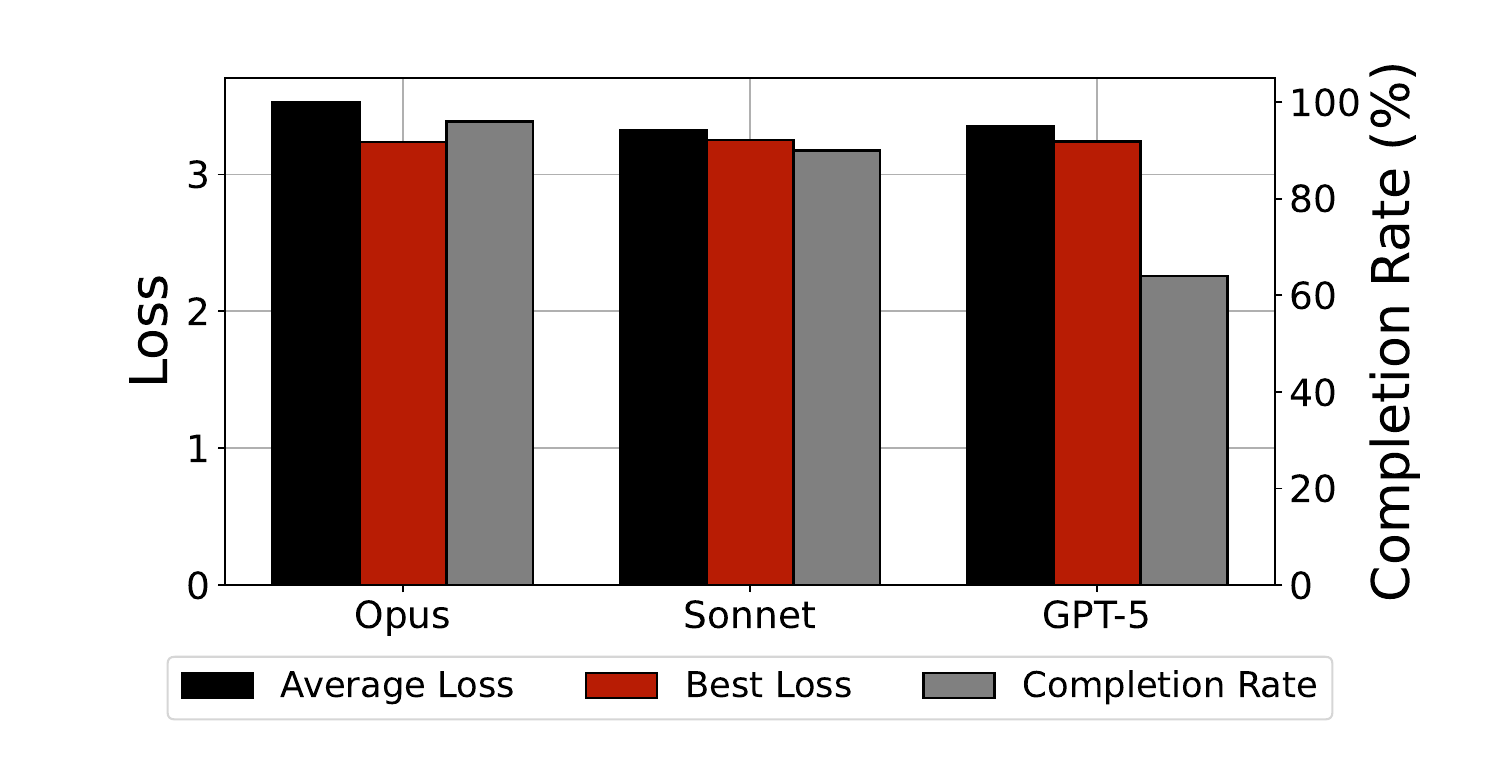}}

  \subfigure[GPT-5 Execution (GRPO)]{\includegraphics[width=0.495\linewidth]{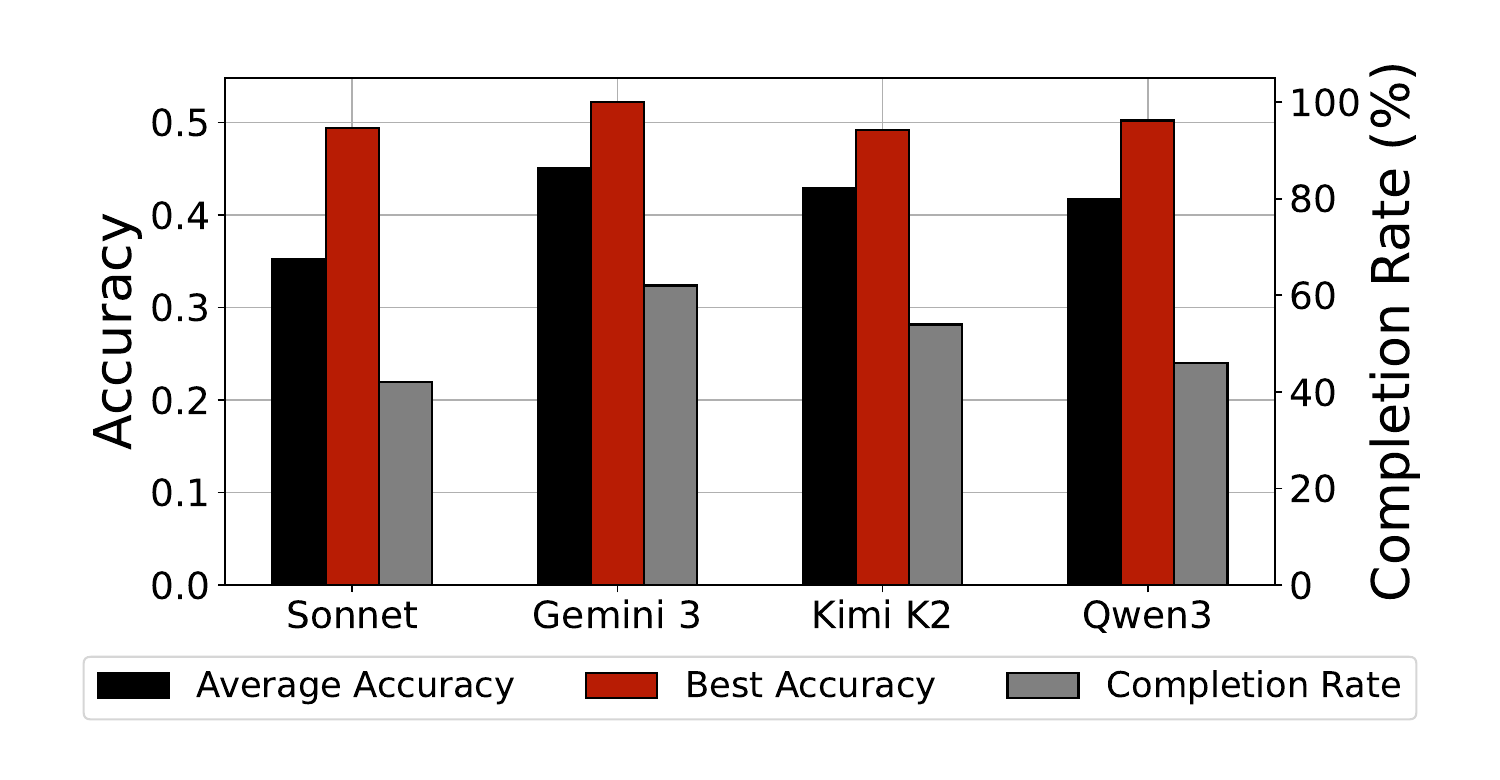}}
  \subfigure[GPT-5 Execution (nanoGPT)]{\includegraphics[width=0.495\linewidth]{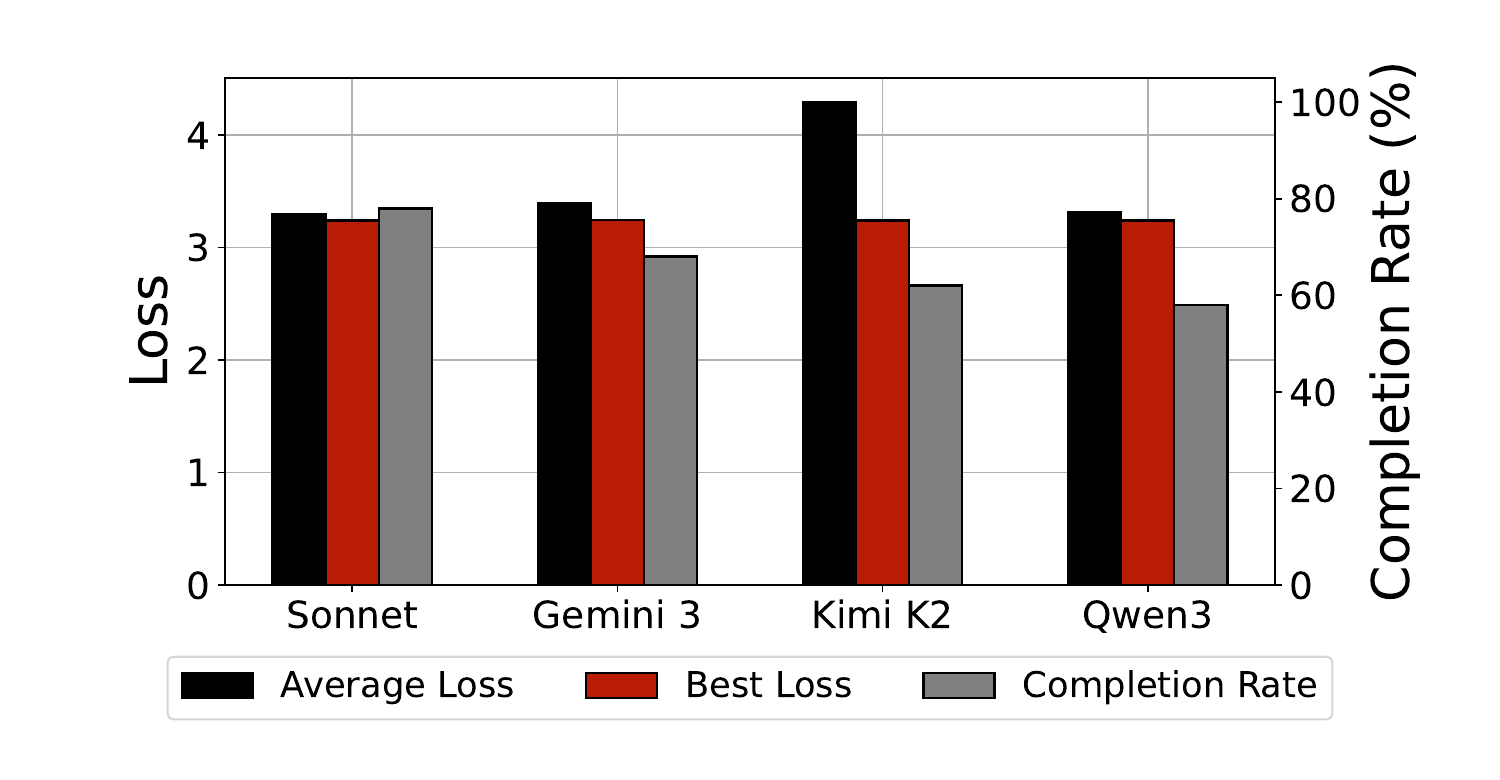}}

  \caption{Model performance comparison with self-execution (top row) vs GPT-5 execution (bottom row) on GRPO and nanoGPT environments. The baseline accuracy for GRPO is 0.480, and the baseline loss for nanoGPT is 3.255. The completion rate is high for most models, especially under self-execution.}
  \label{fig:benchmarking}
\end{figure}

\section{Benchmarking LLM ideators and executors}

For an execution-grounded feedback loop to work, current LLMs must serve as both ideators and executors, yielding meaningful \texttt{env.value} signals for learning.
We benchmark various frontier LLMs in both roles.

\subsection{End-to-end ideation and execution}

We sample ideas from an LLM and use the same LLM as the code execution model to execute its own ideas.
We sample and execute 50 ideas from Claude-4.5-Opus, Claude-4.5-Sonnet, and GPT-5, measuring three metrics: (1) completion rate---the percentage of ideas successfully executed with a valid (non-zero) result; (2) average performance---the average validation accuracy or loss across all successfully executed ideas; (3) best performance---the highest validation accuracy or lowest validation loss among all executed ideas.
The top row of Figure~\ref{fig:benchmarking} shows the results.
A large fraction of sampled ideas execute successfully, with Claude-4.5-Opus and Claude-4.5-Sonnet achieving significantly higher execution rates than GPT-5.
The best-of-N performance ($N=50$) already exceeds the original baselines.
For example, on GRPO, Claude-4.5-Sonnet achieves a max accuracy of 60.4\% compared to the baseline of 48.0\%; on nanoGPT, Claude-4.5-Opus achieves a lowest loss of 3.237 compared to the baseline of 3.255.

\subsection{Comparing ideators with the same executor}

We fix the executor to GPT-5 and vary the ideator model.
As the bottom row of Figure~\ref{fig:benchmarking} shows, even when ideator and executor differ, execution rates remain reasonable (ranging from 42\% to 78\%).
The same ideas from Claude-4.5-Sonnet achieve lower execution rates when executed by GPT-5 instead of itself (84\% vs.\ 42\% on GRPO and 90\% vs.\ 78\% on nanoGPT).
Frontier open-weight models like Kimi-K2-Thinking~\cite{Bai2025KimiKO} and Qwen3-235B-A22B~\cite{Yang2025Qwen3TR} also achieve non-trivial completion rates and best-of-N performance that exceeds the baselines.
For example, Qwen3-235B achieves a max accuracy of 50.2\% on GRPO and min loss of 3.238 on nanoGPT with $N=50$, both exceeding the baselines.

These results suggest the feasibility of an automated ideation and execution loop.
Before turning to search, we examine a simpler form of scaling---test-time compute scaling via budget forcing---whose lesson motivates the idea-level search that follows.

\section{Test-time scaling via budget forcing}
\label{sec:s1-budget-forcing}

\begin{figure}[!htbp]
\centering
\includegraphics[width=\textwidth]{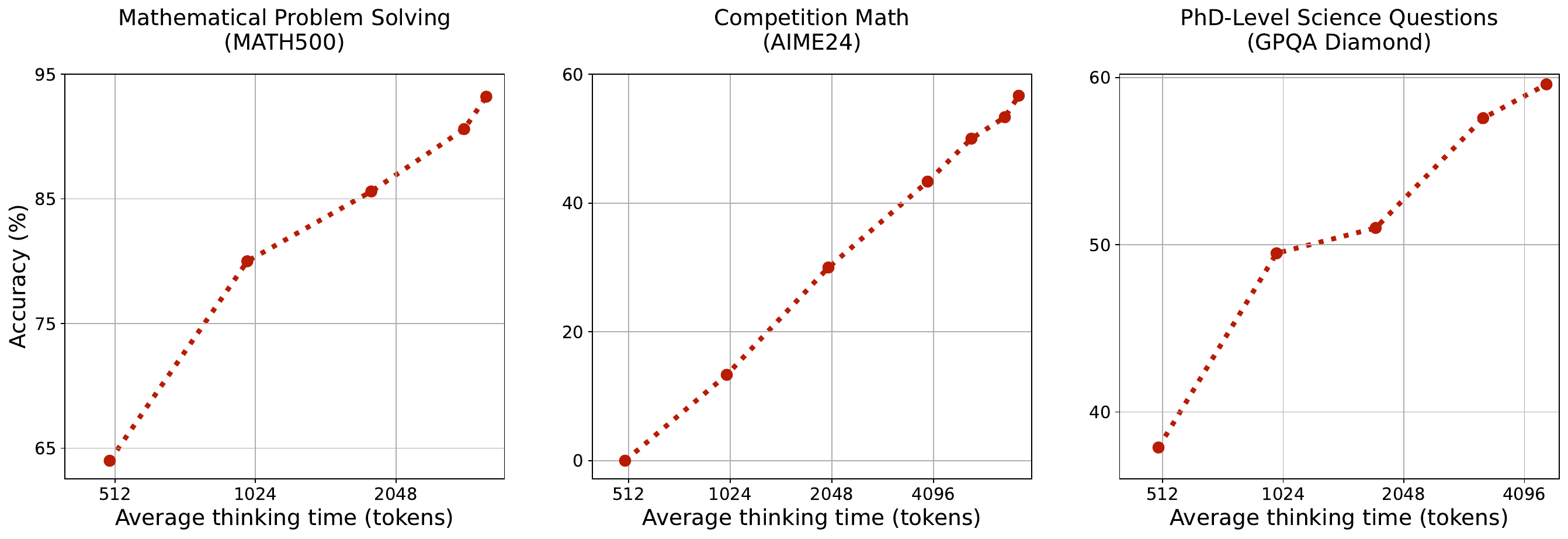}
\caption{\textbf{Test-time scaling with \sone{}.} We benchmark \sone{} on reasoning-intensive tasks and vary test-time compute.}
\label{fig:s1-bf-scaling}
\end{figure}

As discussed in Chapter~\ref{chap:sbp}, training on just 1,000 examples suffices to build a competitive reasoning model, revealing that reasoning capability is latent in pretrained weights.
A natural question is: can we improve performance simply by controlling how much a model reasons at the token level?
Budget forcing offers a minimal mechanism for allocating test-time compute adaptively, and its effectiveness foreshadows the larger principle behind the evolutionary search in later sections---that shifting compute from training to search over \texttt{env.value} is a viable path to better performance.

We show that controlling thinking duration via a simple test-time technique we call \textit{budget forcing} produces a reasoning model that scales in performance with more test-time compute.
After training a model on reasoning data, we control test-time compute using \textit{budget forcing}: \textbf{(I)} If the model generates more thinking tokens than a desired limit, we forcefully end the thinking process by appending an end-of-thinking token delimiter, causing the model to transition to generating its answer.
\textbf{(II)} If we want the model to spend more test-time compute on a problem, we suppress the end-of-thinking token delimiter and instead append ``Wait'' to the model's reasoning trace to encourage more exploration.
Equipped with this simple technique, our model \sone{} exhibits test-time scaling (Figure~\ref{fig:s1-bf-scaling}).

In summary, our contribution is a simple method for controlling test-time compute called budget forcing (\S\ref{sec:s1-bf-ttc}).
We end with a discussion on test-time scaling and its limits (\S\ref{sec:s1-bf-disc}).
Our code, model, and data are open-source at \url{https://github.com/simplescaling/s1}.

\subsection{Test-time scaling}
\label{sec:s1-bf-ttc}

\paragraph{Method} We classify test-time scaling methods into \textbf{1) Sequential}, where later computations depend on earlier ones (e.g., a long reasoning trace), and \textbf{2) Parallel}, where computations run independently (e.g., majority voting)~\citep{snell2024scalingllmtesttimecompute,brown2024largelanguagemonkeysscaling}.
We focus on sequential scaling because later computations build on intermediate results, enabling deeper reasoning and iterative refinement.
We propose new sequential scaling methods and ways to benchmark them.

\paragraph{Budget forcing} We propose a simple decoding-time intervention that forces a maximum and/or minimum number of thinking tokens.
We enforce a maximum token count by appending the end-of-thinking token delimiter and optionally ``\texttt{Final Answer:}'' to early-exit the thinking stage and make the model provide its current best answer.
To enforce a minimum, we suppress the end-of-thinking token delimiter and optionally append ``Wait'' to the model's reasoning trace to encourage reflection on its current generation.

\begin{figure*}[!htbp]
\centering
\subfigure[Sequential scaling via budget forcing\label{fig:s1-bf-forcing}]{\includegraphics[width=0.49\textwidth]{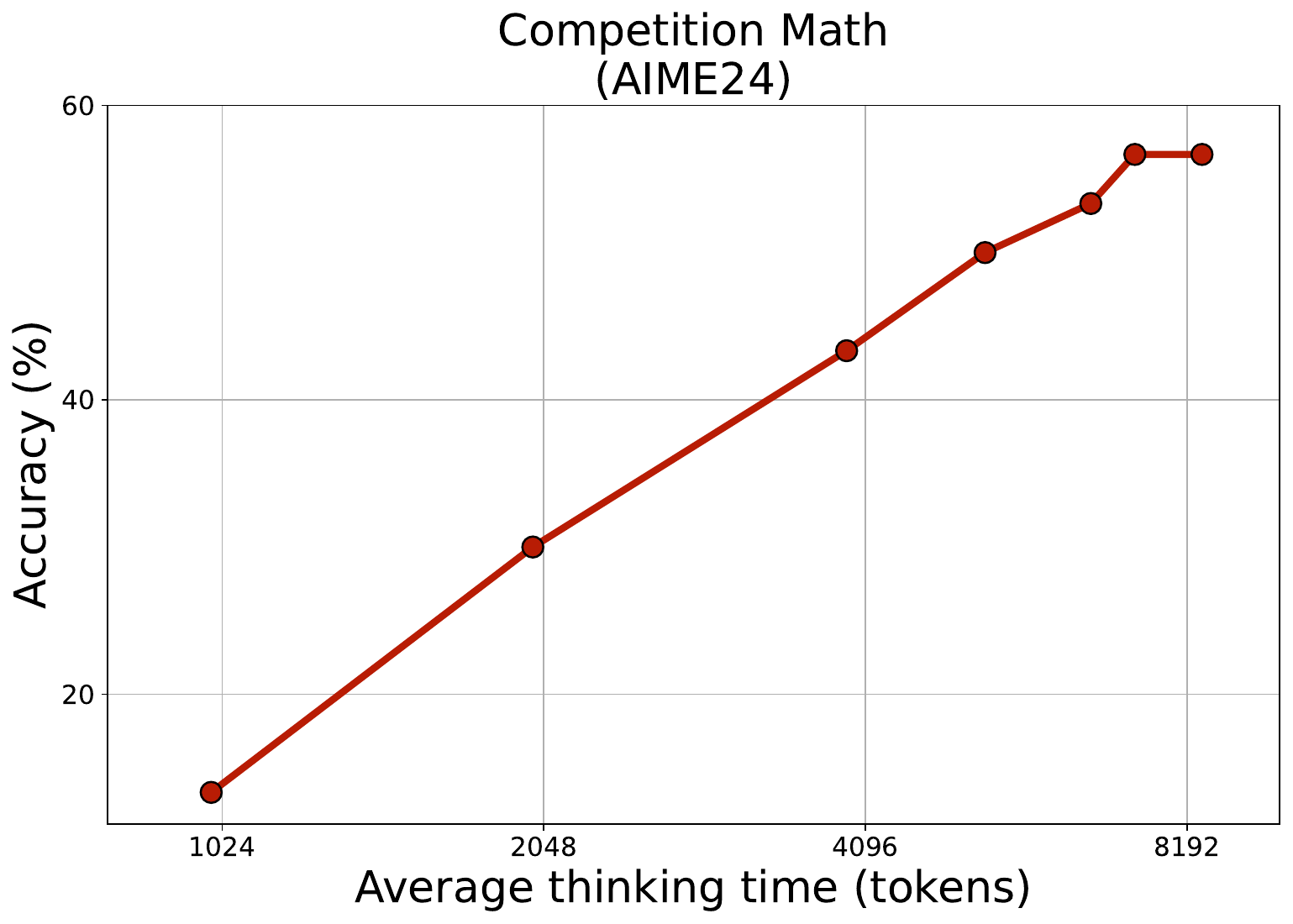}}
\subfigure[Parallel scaling via majority voting\label{fig:s1-bf-majority}]{\includegraphics[width=0.49\textwidth]{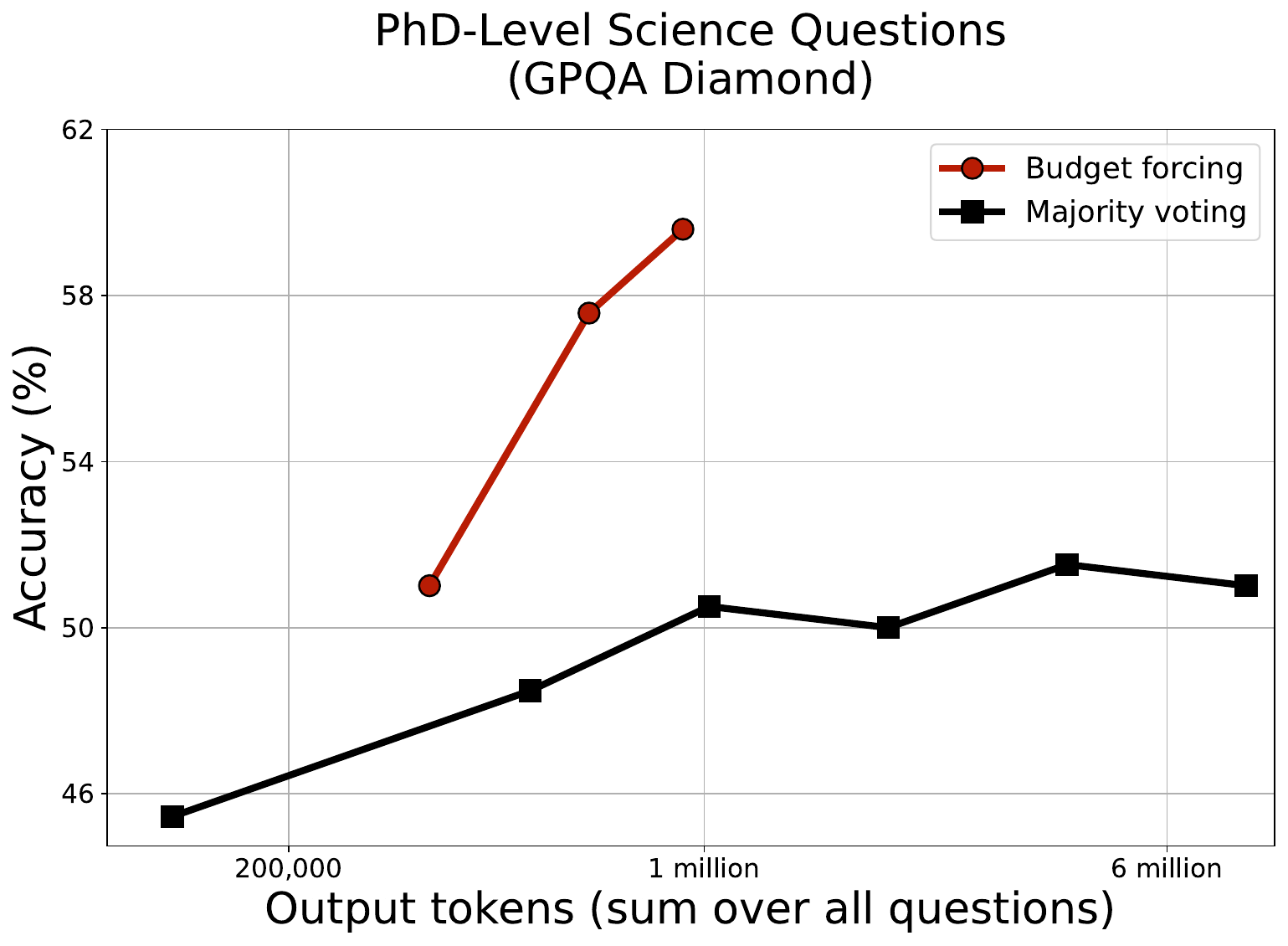}}
\caption{\textbf{Sequential and parallel test-time scaling.} \textit{(a):} Budget forcing shows clear scaling trends and extrapolates to some extent. For the three rightmost dots, we prevent the model from stopping its thinking 2/4/6 times, each time appending ``Wait'' to its current reasoning trace. \textit{(b):} For Qwen2.5-32B-Instruct we perform 64 evaluations for each sample with a temperature of 1 and visualize the performance when majority voting across 2, 4, 8, 16, 32, and 64 of these.}
\label{fig:s1-bf-scaling2}
\end{figure*}

\subsection{Results}
\label{sec:s1-bf-results}

\paragraph{Evaluation} We evaluate on the same three reasoning benchmarks introduced in Section~\ref{sec:s1-se-results}: AIME24, MATH500, and GPQA Diamond.
Unless otherwise specified, we evaluate with a temperature of 0 (greedy) and measure accuracy (equivalent to pass@1).

\paragraph{Test-time scaling} Figure~\ref{fig:s1-bf-scaling} shows that \sone{} with budget forcing scales in performance with more test-time compute.
Figure~\ref{fig:s1-bf-scaling2} (left) expands Figure~\ref{fig:s1-bf-scaling} (middle), showing that budget forcing (\S\ref{sec:s1-bf-ttc}) improves AIME24 performance with more test-time compute but eventually flattens out at six times.
Suppressing the end-of-thinking token delimiter too often leads the model into repetitive loops instead of continued reasoning.
Figure~\ref{fig:s1-bf-scaling2} (right) shows that after training Qwen2.5-32B-Instruct on our 1,000 samples to produce \sone{} and equipping it with budget forcing, it operates in a different scaling paradigm.
Scaling test-time compute on the base model via majority voting does not catch up with \sone{}, validating our intuition from \S\ref{sec:s1-bf-ttc} that sequential scaling is more effective than parallel.

\subsection{Discussion}
\label{sec:s1-bf-disc}

\paragraph{Limits to further test-time scaling} Budget forcing extrapolates test-time compute (\S\ref{sec:s1-bf-results})---for example, improving AIME24 performance from 50\% to 57\%.
However, it has two key limitations: it eventually \textbf{flattens out} (Figure~\ref{fig:s1-bf-scaling2}), and the \textbf{context window} of the underlying language model constrains it.
Despite these constraints, we demonstrate test-time scaling across a wide range of accuracies (Figure~\ref{fig:s1-bf-scaling}), partly because scaling \emph{down} test-time compute behaves predictably and does not suffer from these constraints.

Continuing test-time scaling will require approaches that further extrapolate test-time compute.
Potential improvements to budget forcing include rotating through different strings (not only ``Wait'') or combining it with frequency penalties or higher temperature to avoid repetitive loops.
A natural question is whether applying budget forcing to a reasoning model trained with reinforcement learning yields better extrapolation, or whether RL enables new forms of test-time scaling beyond budget forcing.

The lesson from budget forcing is simple: even a crude intervention---appending ``Wait'' to force continued thinking---improves performance.
If brute-force thinking at the token level already helps, then systematically scaling search at the idea level---generating many candidate research ideas, executing them, and feeding results back---should yield continued improvement.
The remainder of this chapter tests this hypothesis.

\begin{figure*}[!htbp]
  \centering
  {\includegraphics[width=0.48\linewidth]{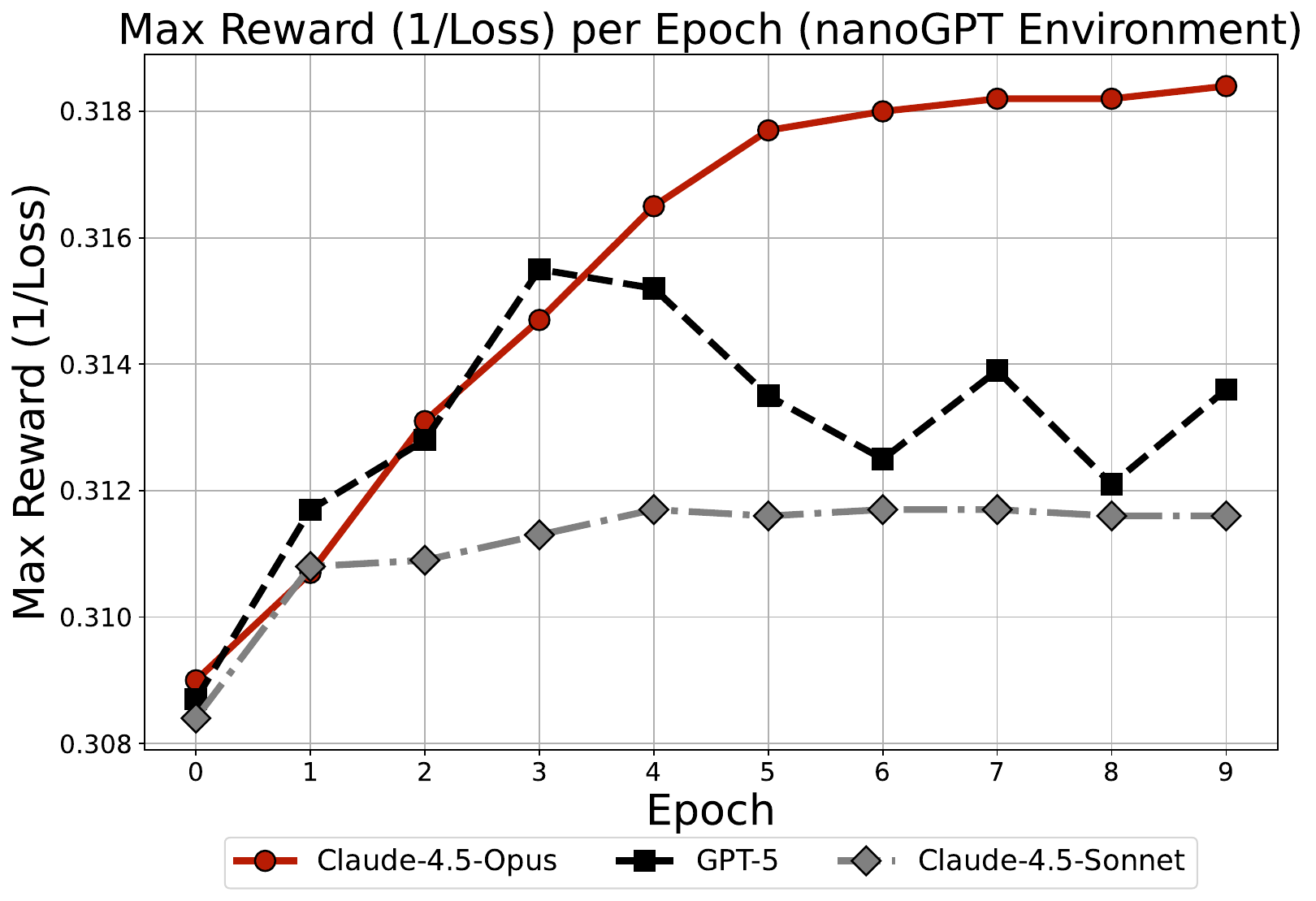}}
  {\includegraphics[width=0.48\linewidth]{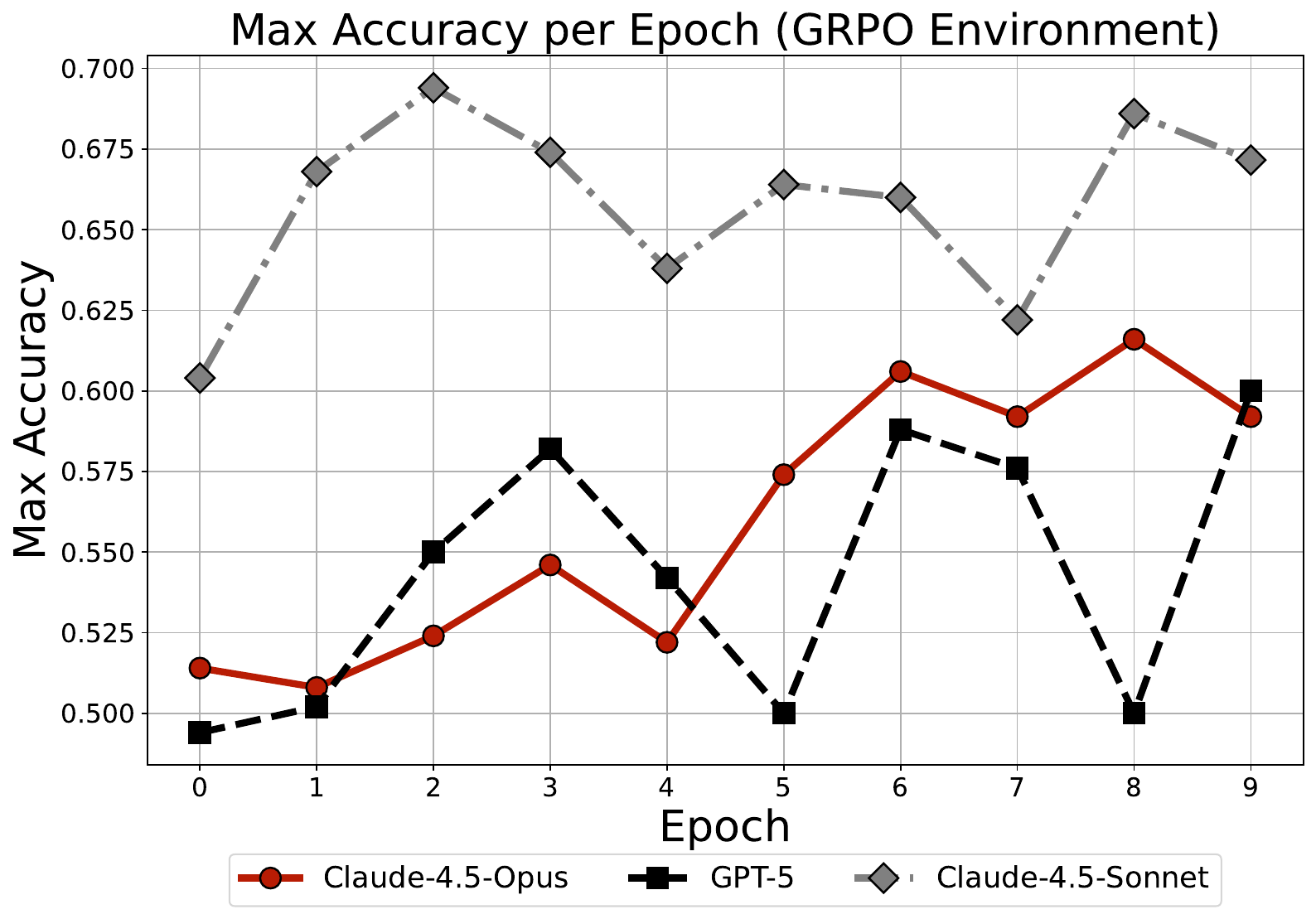}}
  \caption{Best performance at each epoch when performing execution-guided search with different models. For the nanoGPT environment (left), we use the reciprocal of the validation loss as the metric; for the GRPO environment (right), we use validation accuracy as the metric. Claude-4.5-Opus exhibits a scaling trend on both environments and achieves the best performance on nanoGPT. Claude-4.5-Sonnet achieves the best performance on GRPO due to effective hyper-parameter tuning, but saturates early.}
  \label{fig:search_results}
\end{figure*}

\begin{algorithm}[t]
\caption{Execution-guided search}
\label{alg:execution_guided_search}
\small
\begin{algorithmic}[1]
\Require batch size $N$, epochs $T$, baseline performance $\beta$
\Require initial exploitation rate $a_1 \in [0,100]$, annealing schedule $a(t)$ for $t \in \{1,\dots,T\}$
\Require research environment \texttt{env} with methods \texttt{context}, \texttt{value}
\State Sample initial batch of ideas $\mathcal{I}_0 \leftarrow \textsc{SampleIdeas}(N)$
\State $\mathcal{D}_0 \leftarrow \{(i,\; \texttt{env.value}(i)) : i \in \mathcal{I}_0\}$
\For{$t = 1$ \textbf{to} $T$}
  \State $a \leftarrow a(t)$ \Comment{$(100-a)\%$ exploration rate}
  \State $\mathcal{D}_{<t} \leftarrow \bigcup_{k=0}^{t-1} \mathcal{D}_k$
  \State $\mathcal{D}^{+} \leftarrow \{(i,r)\in \mathcal{D}_{<t} \;:\; r > \beta\}$ \Comment{positive trajectories}
  \State $N_{\text{exp}} \leftarrow \left\lfloor \frac{a}{100}\,N \right\rfloor$, \quad
         $N_{\text{expl}} \leftarrow N - N_{\text{exp}}$
  \State $\mathcal{I}^{\text{exp}}_t \leftarrow \textsc{ExploitVariants}(\mathcal{D}^{+}, N_{\text{exp}})$
  \State $\tilde{\mathcal{D}}_{<t} \leftarrow \textsc{SubsampleToContext}(\mathcal{D}_{<t})$
  \State $\mathcal{I}^{\text{expl}}_t \leftarrow \textsc{ExploreNovel}(\tilde{\mathcal{D}}_{<t}, N_{\text{expl}})$
  \State $\mathcal{I}_t \leftarrow \mathcal{I}^{\text{exp}}_t \cup \mathcal{I}^{\text{expl}}_t$
  \State $\mathcal{D}_t \leftarrow \{(i,\; \texttt{env.value}(i)) : i \in \mathcal{I}_t\}$
\EndFor
\State \Return $\bigcup_{t=0}^{T} \mathcal{D}_t$
\end{algorithmic}
\end{algorithm}

\section{Execution-guided evolutionary search}

Evolutionary search~\cite{Koza1992GeneticPO,ELM} is a classic optimization method that does not require gradient updates.
We develop an evolutionary search scaffold on top of frontier LLMs to optimize idea effectiveness based on execution feedback.
We describe our search method---which blends exploration and exploitation---demonstrate its effectiveness on both research environments, and analyze the generated ideas throughout the search process.

\subsection{Search scaffold}

Our search method is inspired by prior evolutionary search approaches for code optimization, such as AlphaEvolve~\cite{Novikov2025AlphaEvolveAC}.
Algorithm~\ref{alg:execution_guided_search} details our approach.
At the first search epoch, we sample a full batch of new ideas.
In subsequent epochs, we split idea generation into exploitation and exploration subsets.
For exploitation, we select ideas from previous epochs that outperform the baseline, append them to the idea generation prompt, and ask the ideator model to generate new variants combining their strengths.
For exploration, we randomly sample ideas from previous epochs into the prompt until reaching the max context length, then instruct the ideator model to generate completely new ideas different from them.
We start with 50\% exploitation and 50\% exploration at epoch 1 and gradually anneal the exploration rate and increase the exploitation ratio in later epochs.
We use a batch size of 50 for the GRPO environment and a batch size of 80 for the nanoGPT environment.

\subsection{Experiment results}

For each environment, we perform execution-guided search with three different models: Claude-4.5-Opus, Claude-4.5-Sonnet, and GPT-5.
For each experiment, we use the same model as both ideator and executor (self-execution).
Figure~\ref{fig:search_results} plots the progression of best performance at each search epoch.
Several trends emerge.

First, Claude-4.5-Opus shows a scaling trend: searching for more epochs leads to a higher upper bound.
In contrast, Claude-4.5-Sonnet and GPT-5 saturate early.
Second, all models find ideas that improve over the baselines. On GRPO, Claude-4.5-Sonnet discovers that vanilla policy gradient with the group-average baseline---without importance reweighting or clipping---outperforms the standard GRPO objective in this setup, and exploits this finding in subsequent epochs to reach 69.4\% at epoch 2 with precise hyper-parameter tuning.
On nanoGPT, Claude-4.5-Opus achieves a min validation loss of 3.1407 at epoch 9 by combining architectural modifications, hyper-parameter tuning, and exponential moving average of intermediate checkpoints during validation (see Appendix~\ref{sec:more_examples} for the full idea).
We run this top solution on 8 H100s following the nanoGPT speedrun setup: it reaches the 3.28 target validation loss in 19.7 minutes, a speedup over the baseline codebase, which takes 35.9 minutes to reach the same target.

To contextualize these model-optimized solutions, we compare the top performance of execution-guided search to human experts (Table~\ref{table:search_results}).
For the GRPO environment, we compare with the leaderboard of the Stanford CS336 graduate-level LLM class, which hosted the same environment as an assignment where students optimized validation accuracy under the same training time budget.
The best student solution\footnote{\url{https://github.com/stanford-cs336/assignment5-alignment-leaderboard}} achieved 68.8\% accuracy, lower than Claude-4.5-Sonnet's top solution from execution-guided search.
In the nanoGPT environment, we directly compare with the nanoGPT speedrun leaderboard.\footnote{\url{https://github.com/KellerJordan/modded-nanogpt}}
The state-of-the-art human solution as of December 2025 achieves the target validation loss in under 2.1 minutes, indicating significant room for improvement in model capability and search methods.

\begin{table*}[t]
\centering
\small
\caption{Breakdown of hyper-parameter tuning vs algorithmic ideas throughout the entire execution-guided search.
We report the percentage of each type among all generated ideas of each model ($N=500$ ideas on GRPO and $N=800$ ideas on nanoGPT).
We also report the average and best performance for ideas under each category, where we use validation accuracy as the performance metric for GRPO and validation loss as the metric for nanoGPT.
Bold numbers every row indicate the best performance by each model.
All models generate a substantial amount of algorithmic ideas apart from hyper-parameter changes, while Claude-4.5-Sonnet generates significantly more hyper-parameter ideas than other models.}
\label{tab:idea_type}
\renewcommand{\arraystretch}{1.3}
\begin{tabular}{lrrr|rrr}
\toprule
& \multicolumn{3}{c}{\textbf{Hyper-parameter}} & \multicolumn{3}{c}{\textbf{Algorithmic}} \\
\cline{2-7}
\multicolumn{1}{l}{\textbf{Model}} & \multicolumn{1}{c}{\textbf{\%}} & \multicolumn{1}{c}{\textbf{Avg}} & \multicolumn{1}{c}{\textbf{Best}} & \multicolumn{1}{c}{\textbf{\%}} & \multicolumn{1}{c}{\textbf{Avg}} & \multicolumn{1}{c}{\textbf{Best}} \\
\hline
\multicolumn{7}{c}{~~~~~~~~~~~~~~~~~~~~~~~~~~~~~~~~~~~~\emph{GRPO environment (accuracy$_\uparrow$)}} \\
\hline
GPT-5 & 5.0\% & 45.0\% & 50.2\% & 95.0\% & 44.5\% & \textbf{60.0\%} \\
Claude-4.5-Sonnet & 41.1\% & 48.4\% & \textbf{69.4\%} & 58.9\% & 45.0\% & 67.4\% \\
Claude-4.5-Opus & 3.7\% & 44.4\% & 50.4\% & 96.3\% & 46.5\% & \textbf{61.6\%} \\
\hline
\multicolumn{7}{c}{~~~~~~~~~~~~~~~~~~~~~~~~~~~~~~~~~~~~\emph{nanoGPT environment (loss$_\downarrow$)}} \\
\hline
GPT-5 & 15.4\% & 3.254 & 3.195 & 84.6\% & 3.894 & \textbf{3.170} \\
Claude-4.5-Sonnet & 31.3\% & 3.251 & \textbf{3.208} & 68.7\% & 3.679 & \textbf{3.208} \\
Claude-4.5-Opus & 8.7\% & 3.329 & 3.147 & 91.3\% & 3.419 & \textbf{3.141} \\
\bottomrule
\end{tabular}
\end{table*}

\begin{table*}[ht]
\centering
\small
\setlength{\tabcolsep}{5pt}
\renewcommand{\arraystretch}{1}

\begin{tabular}{L{0.333\textwidth}|L{0.34\textwidth}|L{0.263\textwidth}}
\toprule
\textbf{Claude-4.5-Opus on GRPO} & \textbf{Claude-4.5-Sonnet on GRPO} & \textbf{GPT-5 on GRPO} \\
\hline
\vspace{0.5pt}
Residual Ratio Learning with Momentum Bounds:
Instead of directly using the (importance sampling) ratio, decompose it into a ``base'' component
(EMA of batch mean ratios) and a ``residual'' component (ratio -- base).
Apply sigmoid bounding only to the residual, allowing the base to capture
systematic policy drift while controlling deviations from it.
Combined with momentum clip adaptation.
Formula: \texttt{residual = ratio - ema\_batch\_ratio},
\texttt{bounded\_residual = sigmoid\_bound(residual, deviation)},
\texttt{effective\_ratio = 1.0 + bounded\_residual}.

\vspace{5pt}
\textbf{Validation accuracy: 61.6}

\vspace{15pt}
Advantage Rank Difference Weighting:
Instead of using absolute advantage magnitude, weight samples by how far their
rank differs from their expected rank under uniform distribution. Samples that
significantly outperform or underperform their ``expected'' position get higher
weights. This is distribution-free and robust to outliers. Formula:
\texttt{expected\_rank = (N-1)/2}, \texttt{rank\_diff = |actual\_rank - expected\_rank| / expected\_rank},
\texttt{weight = 0.5 + 0.5 * rank\_diff}.

\vspace{5pt}
\textbf{Validation accuracy: 59.2}
\vspace{1pt}
&
\vspace{0.5pt}
Dynamic Mathematical Problem Difficulty Balancing with Performance Feedback: Implement intelligent difficulty balancing that dynamically adjusts the mix of problem difficulties based on recent performance trends. When performance is strong, increase difficulty proportion; when struggling, provide more foundational problems. Combine with the proven hyper-parameters for optimal learning progression.

\vspace{5pt}
\textbf{Validation accuracy: 64.0}

\vspace{15pt}
Implement token-level reward attribution by using attention
weights to identify which input tokens contributed most to correct answers, then
amplifying the gradient updates for those tokens during policy gradient training.

\vspace{5pt}
\textbf{Validation accuracy: 45.2}

\vspace{15pt}
Create mathematical working memory simulation by maintaining a context buffer of mathematical facts, definitions, and intermediate results during problem solving. This buffer gets updated as the model works through problems and provides additional context for subsequent mathematical steps, simulating how humans maintain mathematical working memory during complex calculations.

\vspace{5pt}
\textbf{Validation accuracy: 58.0}
&
\vspace{0.5pt}
Token-Level Ratio De-noising via Response Chunks (Chunked-Ratio):
Reduce noisy token spikes by averaging log-ratio over small
contiguous chunks within the response. Partition response tokens into $C$ chunks
per sequence (e.g., $C=8$ over effective length), replace per-token
$\Delta \log p$ with chunk mean broadcast to tokens in the chunk before ratio and
clipping. Keeps structural signal while smoothing extremes.

\vspace{5pt}
\textbf{Validation accuracy: 58.2}

\vspace{15pt}
Per-Group Curriculum via Reward Spread (PGC-RS): Adjust step aggressiveness based on group reward spread. For each group, compute
spread $s_g = \mathrm{std}(r)$. Compute per-sample temperature
$T^{\mathrm{grp}}_i = \mathrm{clamp}\!\big(1 + \alpha\cdot(s_{\mathrm{ref}} - s_g),\, 0.8,\, 1.5\big)$
with $s_{\mathrm{ref}} =$ median group std over rollout and $\alpha=0.8$.
Multiply existing ratio temperature $T_i$ (if any) by $T^{\mathrm{grp}}_i$.
Groups with low spread (hard to distinguish) get cooler ratios; high-spread
groups allow bolder updates.

\vspace{5pt}
\textbf{Validation accuracy: 49.4}
\\
\bottomrule
\end{tabular}

\caption{Examples of successfully executed ideas on the GRPO environment, along with their accuracy on the MATH validation set. The baseline accuracy is 48.0\% on this environment.}
\label{tab:grpo_examples}
\end{table*}

\subsection{Comparison with best-of-N}
To evaluate the effectiveness of our search scaffold, we compare execution-guided search with best-of-N under the same sampling budget on the nanoGPT environment.
Since our search batch size is 80, we compare the first 3 epochs of execution-guided search using GPT-5 with best-of-N results for GPT-5 with $N \in \{80, 160, 240\}$.

\begin{figure}[!htbp]
  \centering
  \includegraphics[width=0.75\textwidth]{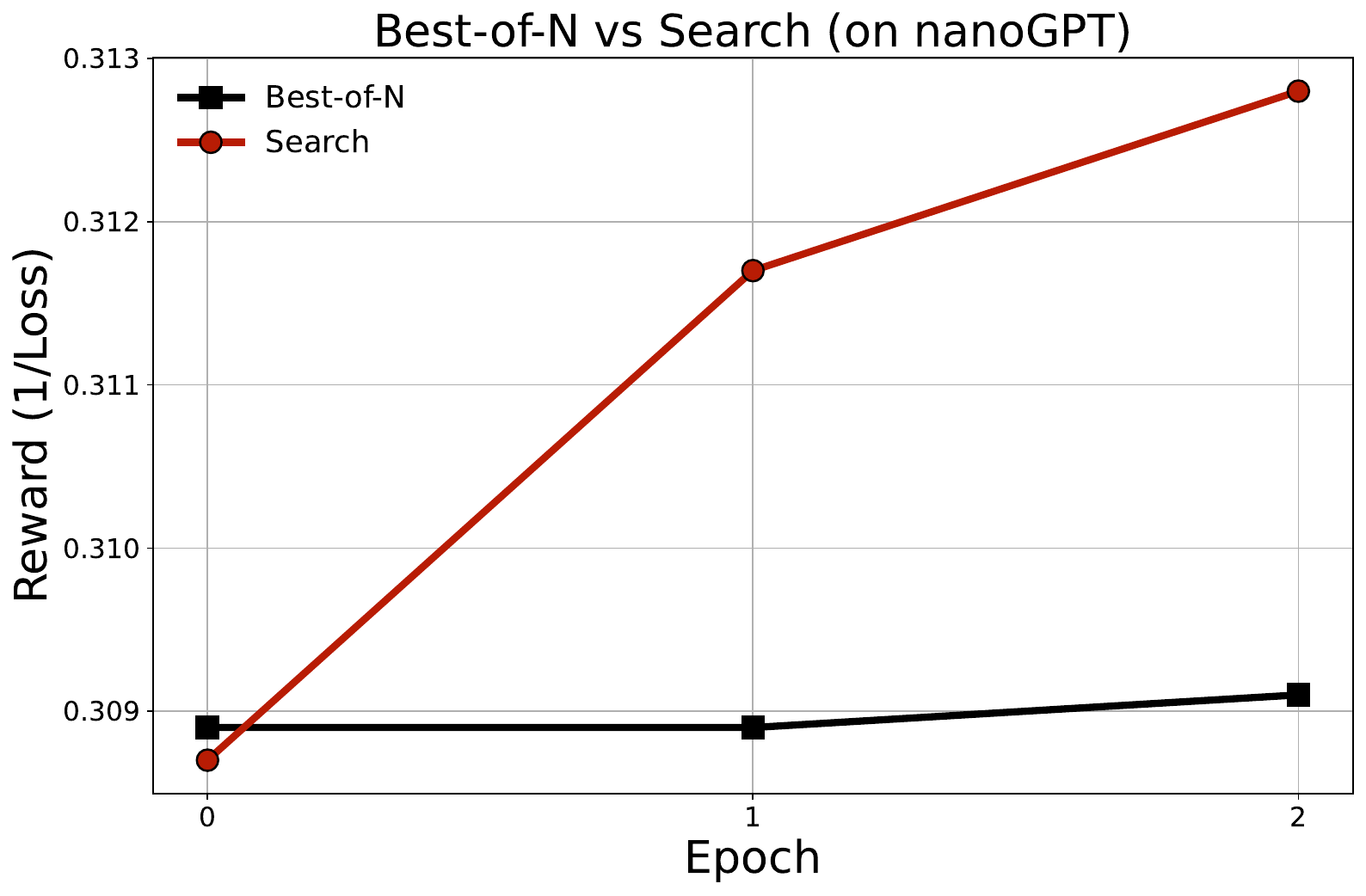}
  \caption{Comparison between best-of-N and our execution-guided search under the same sampling budget.}
  \label{fig:best_of_n}
\end{figure}

As shown in Figure~\ref{fig:best_of_n}, search and best-of-N start from similar performance at epoch 0 (not exactly the same due to sampling variance), but evolutionary search outperforms best-of-N from epoch 1 onward.
This suggests that the model leverages trajectories from previous epochs to generate more effective ideas in future epochs.
This result echoes the token-level finding from \S\ref{sec:s1-budget-forcing}: there, sequential scaling via budget forcing outperformed parallel scaling via majority voting (Figure~\ref{fig:s1-bf-scaling2}b).
A similar pattern appears at the idea level---sequential search that builds on prior results outperforms parallel sampling of independent candidates.

\subsection{Analysis of generated ideas}

\paragraph{Hyper-parameter vs.\ algorithmic} To understand the types of ideas models generate during execution-guided search, we classify all generated ideas into either hyper-parameter tuning (ideas implementable by changing existing configs) or algorithmic (ideas requiring changes not supported by the baseline codebase) using an LLM judge.
Table~\ref{tab:idea_type} shows that all three models generate substantial algorithmic ideas beyond hyper-parameter tuning.
Claude-4.5-Sonnet generates significantly more hyper-parameter ideas than both Claude-4.5-Opus and GPT-5.
The most effective ideas stem from algorithmic innovations in most cases, except when using Claude-4.5-Sonnet.

\paragraph{Qualitative examples} We provide several executed ideas on the GRPO environment in Table~\ref{tab:grpo_examples} and on the nanoGPT environment in Appendix~\ref{sec:more_examples}.
When sampling, models generate a thinking trace followed by the natural language idea and a brief description of code changes needed for implementation.
For brevity, we include only the natural language ideas in the table; a full code execution trajectory appears in Appendix~\ref{sec:code_examples}.
Table~\ref{tab:grpo_examples} reveals different idea styles across models: Claude-4.5-Sonnet generates more intuitive ideas, while Claude-4.5-Opus and GPT-5 are more mathematically inclined.

\paragraph{Recovering recent research papers} We observe multiple cases where model-generated ideas (without any RAG) closely resemble research papers released within three months of writing this chapter.
For example, Claude-4.5-Sonnet proposed: ``\textit{Implement response diversity rewards within groups where responses to the same prompt receive bonus rewards for being dissimilar to other responses in their group, encouraging exploration of different solution paths.}'', which is similar to \cite{Li2025JointlyRD}.
For pre-training, Claude-4.5-Opus proposed: ``\textit{Causal Context Compression: Before each attention layer, apply a learned compression that mixes local context
(previous 2-3 tokens) into the current representation, providing implicit local context without convolutions.}'', which is similar to the ``canon layer'' described in \cite{AllenZhu2025PhysicsOL}.
Assessing the novelty of LLM-generated ideas is beyond this chapter's scope, but the ability to rediscover ideas from recent papers suggests that automated AI researchers could plausibly support work at the frontier of LLM research.

\section{Discussion}

\subsection{Limitations}

Our current experiments have several limitations.

First, our current procedure does not test idea generalizability.
The best-performing ideas at small scales may not transfer to larger scales or other datasets.
Future work should explore methods that explicitly test generalizability and scalability, potentially incorporating them into the optimization objectives.

Second, our experiment scope is bounded by execution agent capability.
Many promising model-generated ideas cannot be successfully executed (e.g., see the end of Appendix~\ref{sec:more_examples}), introducing noise in the reward signal.
Future work could develop more capable execution agents and extend our setup to more complex research problems---for instance, by implementing coding agents with access to external tools and the ability to install new libraries in the execution environments.

Finally, we explore only effectiveness as the training reward.
Other metrics could complement effectiveness---such as idea novelty and interestingness.
Future work could explore how to computationally measure these qualities and incorporate them into the training objective to discover more insightful ideas.

\subsection{Conclusion}

In this chapter, we tackle the problem of AI-designed AI---building systems that improve the very algorithms used to train them---through the lens of test-time search.
We first show that test-time search at the token level (budget forcing) improves reasoning, then scale this principle to the idea level: generating research ideas, executing them automatically, and feeding results back to guide evolutionary search.
Using this approach, frontier LLMs improve over baseline solutions---finding a post-training recipe that improves accuracy from 48\% to 69\% and a pre-training recipe that halves wall-clock time.
These results point toward the feasibility of automated, execution-grounded AI research and suggest a path toward AI systems that continually improve themselves.

This chapter points toward AI systems that improve not only their knowledge or capabilities but also the very algorithms used to train them.
Chapter~\ref{chap:scp} showed how synthetic data can teach models new knowledge; Chapter~\ref{chap:sbp} demonstrated that models can bootstrap their pretraining capability without external supervision.
Here, we have shown that AI systems can generate, implement, and validate research ideas---including ideas for improving pretraining and post-training algorithms themselves.

An AI system that improves the algorithms used to train future AI systems could, in principle, accelerate its own development.
Our current results are early---evolutionary search finds better hyperparameters and algorithmic improvements but does not yet discover fundamentally new methods---yet the feasibility of execution-grounded idea generation suggests a path forward.
The key bottleneck is no longer whether AI can generate ideas, but whether it can generate \emph{good} ideas that generalize beyond narrow benchmarks.

From this perspective, the most promising directions are: learning algorithms that maintain exploration while optimizing for effectiveness, richer execution environments that test generalization, and tighter feedback loops between ideation and execution.
The ultimate goal---AI systems that continually and autonomously improve themselves---remains distant, but this thesis provides concrete building blocks toward it.

\chapter{Conclusion: can AI be smarter than its creators?}
\label{chap:conclusion}
In Chapter~\ref{chap:intro}, we proposed a one-sentence definition: \emph{a continually self-improving AI is one that, once created, can autonomously and continually improve itself better than its human creators can improve it.}
The key phrase is ``can improve it''---the definition does not claim that AI is stronger than humans, only that AI can improve AI more effectively than humans can.

Each chapter takes a step toward this standard.
In Chapter~\ref{chap:scp}, we showed that a model can synthesize training data that teaches it knowledge beyond what the small source corpus can directly teach.
In Chapter~\ref{chap:sbp}, we showed that a model can bootstrap its own pretraining capabilities from a fixed dataset, producing training signals that yield improvement beyond what the original human-collected data provides.
In Chapter~\ref{chap:automated-ai-research}, we showed that an AI system can design learning algorithms by searching over a larger design space than human researchers can search.

But the mechanism behind each result is the same: AI compensates for inferior quality with superior quantity.
Human data is better, but AI data is infinite.
Human researchers are stronger, but AI researchers are tireless.
The definition is satisfied, but only through brute force---stacking quantity to overcome a limitation in quality.
This raises a deeper question: can a created system \emph{genuinely} surpass its creator, not merely out-grind them?

The preceding chapters presented experimental results; this chapter is different in character.
What follows is a position piece: I offer a historical analogy and a personal interpretation, not a proof.
The analogy is suggestive, not rigorous, and the conclusions I draw from it reflect my own perspective on where the field may be headed.

\section{A parable from physics}
\label{sec:conc-parable}

To explore why the answer might be yes, we turn to an analogy from physics.
In 1915, Albert Einstein published the field equations of general relativity \citep{einstein1915feldgleichungen}, completing a decade-long effort to reconcile gravity with the geometry of spacetime.
Two years later, when Einstein applied his equations to the universe as a whole, he found that they do not admit a static, matter-filled cosmos \citep{einstein1917kosmologische}.
Rather than accept this prediction, he modified the equations---introducing a free parameter called the cosmological constant---to force a static solution.
In 1929, Edwin Hubble established that distant galaxies are systematically redshifted, their light stretched to longer wavelengths in proportion to their distance \citep{hubble1929relation}.
The universe is expanding, exactly as the unmodified equations had predicted.
The theory knew something its creator did not.

This chapter tells the story of how a set of equations can be smarter than the person who wrote them down.
A theory, once created, has a life of its own: it can evolve, make predictions, and reach conclusions that its creator never intended.
General relativity provides perhaps the most literal historical precedent for this phenomenon.
Understanding it precisely requires following a chain of mathematical reasoning from Newton to Einstein, which we do now.

\section{The gravitational field equation}
\label{sec:conc-field-equation}

Newton's law of universal gravitation describes gravity as a force between two point masses: a mass $M$ exerts a force on a test mass $m$ separated by displacement $\mathbf{r}$,
\begin{equation}
\label{eqn:newton-gravity}
\mathbf{F} = -\frac{GMm}{r^2}\,\hat{\mathbf{r}},
\end{equation}
where $G \approx 6.674 \times 10^{-11}\;\mathrm{m^3\,kg^{-1}\,s^{-2}}$ is Newton's gravitational constant and the negative sign indicates that the force is attractive.
This is a \emph{particle equation}: it tracks individual objects and assumes instantaneous action at a distance.

The passage from discrete forces to a local field equation proceeds by introducing a continuous mass density $\rho(\mathbf{r})$ and a gravitational potential $\Phi(\mathbf{r})$.
Through an integral formulation and the identity $\nabla^2(1/|\mathbf{r} - \mathbf{r}'|) = -4\pi\,\delta^3(\mathbf{r} - \mathbf{r}')$, one arrives at Poisson's equation---a \emph{field equation} that replaces particle-by-particle bookkeeping with a local differential relationship (Appendix~\ref{sec:appendix-newton-poisson}):
\begin{equation}
\label{eqn:poisson}
\underbrace{\nabla^2 \Phi}_{\text{geometry (potential)}} = \underbrace{4\pi G\,\rho}_{\text{matter (source)}}.
\end{equation}

Einstein's general relativity elevates every ingredient of Poisson's equation from scalars and vectors to tensors.
The scalar potential $\Phi$ becomes the \emph{metric tensor} $g_{\mu\nu}$, a $4 \times 4$ symmetric matrix that defines the geometry of spacetime: $ds^2 = g_{\mu\nu}\,dx^\mu\,dx^\nu$.
The scalar density $\rho$ becomes the \emph{stress-energy tensor} $T_{\mu\nu}$, encoding energy, momentum, and stress---in general relativity, all forms of energy curve spacetime, not just mass.
The left-hand side of the field equation must express spacetime curvature.
Starting from $g_{\mu\nu}$, one constructs the Riemann curvature tensor through a chain of derivatives (Appendix~\ref{sec:appendix-einstein-details}), then contracts it to the Ricci tensor $R_{\mu\nu}$ and Ricci scalar $\mathcal{R}$.
The unique symmetric, divergence-free combination is the Einstein tensor $G_{\mu\nu} \equiv R_{\mu\nu} - \frac{1}{2}\mathcal{R}\,g_{\mu\nu}$.

Assembling these ingredients, Einstein wrote down the field equations of general relativity in 1915 \citep{einstein1915feldgleichungen}:
\begin{equation}
\label{eqn:einstein-field}
\boxed{\;R_{\mu\nu} - \tfrac{1}{2}\,\mathcal{R}\,g_{\mu\nu} = \frac{8\pi G}{c^4}\,T_{\mu\nu}.\;}
\end{equation}
These are 10 coupled, nonlinear, second-order partial differential equations for the 10 independent components of $g_{\mu\nu}$.
The identity $\nabla_\mu G^{\mu\nu} = 0$ (the Bianchi identity) ensures compatibility with energy-momentum conservation ($\nabla_\mu T^{\mu\nu} = 0$).
In the weak-field, slow-motion, static limit, these equations reduce exactly to Poisson's equation~\eqref{eqn:poisson} (Appendix~\ref{sec:appendix-newtonian-limit}).

The correspondence between the two theories is summarized below:
\begin{center}
\begin{tabular}{lcc}
\toprule
& Poisson (Newton) & Einstein \\
\midrule
Matter field & $\rho$ (scalar density) & $T_{\mu\nu}$ (stress-energy tensor) \\
Potential field & $\Phi$ (scalar potential) & $g_{\mu\nu}$ (metric tensor) \\
Differential operator & $\nabla^2$ (Laplacian) & $g_{\mu\nu} \mapsto G_{\mu\nu}$ (nonlinear) \\
Coupling constant & $4\pi G$ & $8\pi G / c^4$ \\
\bottomrule
\end{tabular}
\end{center}

\section{Einstein's cosmological problem}
\label{sec:conc-einstein-cosmology}

In 1917, the astronomical consensus held that the universe was static and eternal, consisting essentially of the Milky Way alone.
Einstein set out to apply his field equations to the universe as a whole \citep{einstein1917kosmologische}.
He modeled the matter content as pressureless dust at rest---the relative velocities of stars are small compared to the speed of light, so the stress-energy tensor is dominated by $T_{00} = \rho c^2$ with all spatial components negligible ($T_{ij} \approx 0$).
Einstein assumed the density $\rho$ to be constant in both space and time \citep{einstein1917kosmologische}.
We follow Friedmann \citep{friedman1922krummung}, who generalized Einstein's setup by allowing both the curvature radius and the density to depend on time: $\rho = \rho(t)$, with spatial homogeneity still enforced.
The crucial question was whether these equations permit a \emph{static} universe---one in which the metric has no time dependence.

\subsection{The cosmological metric}
\label{sec:conc-cosmological-metric}

Einstein's starting point was the \emph{cosmological principle}: at any fixed time, space is homogeneous (the same at every point) and isotropic (the same in every direction).
Isotropy at every point means that the sectional curvature at each point does not depend on the direction in which it is measured.
By Schur's theorem \citep[Ch.~4, Ex.~8]{docarmo1992}, this implies that the sectional curvature does not vary from point to point either: space has constant curvature.
A complete, simply connected three-dimensional space of constant curvature falls into exactly one of three classes:
\begin{itemize}[leftmargin=16pt,nosep]
\item \emph{Positive curvature} ($k = +1$): the three-sphere $S^3$, with finite volume.
\item \emph{Zero curvature} ($k = 0$): flat Euclidean space $\mathbb{R}^3$, with infinite volume.
\item \emph{Negative curvature} ($k = -1$): hyperbolic space $H^3$, with infinite volume.
\end{itemize}
Einstein worked with the closed case $k = +1$ \citep{einstein1917kosmologische}; here we choose the \emph{flat} case ($k = 0$) for simplicity, since the qualitative conclusion---that a static universe with matter is impossible---holds for all three.

We now construct the most general metric compatible with spatial flatness, homogeneity, and isotropy.
The assumption that matter is at rest allows us to choose coordinates where the mixed components $g_{0i}$ vanish, so time is orthogonal to the spatial slices.
The geodesic equation for dust at rest requires $g_{00}$ to be independent of spatial position; by our sign convention, $g_{00} = -c^2$.
At fixed time, the spatial metric is that of flat $\mathbb{R}^3$: $dx^2 + dy^2 + dz^2$.
Homogeneity and isotropy do not prevent the overall spatial scale from changing with time, so we introduce a scale factor $a(t)$ that multiplies all spatial distances uniformly.
The coordinates $(x, y, z)$ are \emph{comoving coordinates}: they label points in space permanently, and the physical distance between two points with coordinate separation $\Delta x$ at time $t$ is $a(t)\,\Delta x$.

Combining these ingredients, the spacetime line element is the \emph{Friedmann--Lema\^{i}tre--Robertson--Walker (FLRW) metric} for a spatially flat universe:
\begin{equation}
\label{eqn:flrw-metric}
\boxed{\;ds^2 = -c^2\,dt^2 + a^2(t)\bigl(dx^2 + dy^2 + dz^2\bigr).\;}
\end{equation}
The metric tensor and its inverse are diagonal:
\begin{equation}
g_{00} = -c^2, \quad g_{11} = g_{22} = g_{33} = a^2(t); \qquad g^{00} = -\frac{1}{c^2}, \quad g^{11} = g^{22} = g^{33} = \frac{1}{a^2(t)}.
\end{equation}
The only undetermined quantity is the single function $a(t)$.

\subsection{The Friedmann equations}
\label{sec:conc-friedmann}

Substituting the FLRW metric~\eqref{eqn:flrw-metric} into the Einstein field equations~\eqref{eqn:einstein-field} requires computing the Christoffel symbols, Ricci tensor, Ricci scalar, and Einstein tensor for this metric (the full computation is carried out in Appendix~\ref{sec:appendix-friedmann-derivation}).
With the stress-energy tensor $T_{00} = \rho c^2$ and $T_{ij} = 0$ for pressureless dust, the field equations yield two ordinary differential equations for the scale factor $a(t)$ and the matter density $\rho(t)$, where $\dot{a} \equiv da/dt$ and $\ddot{a} \equiv d^2a/dt^2$.

The 00-component gives the first Friedmann equation (the energy constraint):
\begin{equation}
\label{eqn:friedmann-I}
\boxed{\;\frac{3\dot{a}^2}{a^2} = \frac{8\pi G\rho}{c^2}.\;}
\end{equation}
The 11-component gives the second Friedmann equation (the acceleration equation):
\begin{equation}
\label{eqn:friedmann-II}
\boxed{\;\frac{2\ddot{a}}{a} + \frac{\dot{a}^2}{a^2} = 0.\;}
\end{equation}
These two equations, together with initial conditions, completely determine the evolution of the universe.

\subsection{A dynamic universe}
\label{sec:conc-dynamic}

A static universe means $\dot{a} = 0$ for all time.
Setting $\dot{a} = 0$ in the first Friedmann equation~\eqref{eqn:friedmann-I} gives $0 = 8\pi G\rho/c^2$, which requires $\rho = 0$.
A static universe must be empty.

The result is stronger than this.
Suppose the universe contains matter ($\rho > 0$) and is momentarily at rest ($\dot{a} = 0$) at some time $t_0$.
Then the first Friedmann equation gives $0 = 8\pi G\rho/c^2 > 0$, a contradiction.
So $\dot{a}$ can never pass through zero in a matter-filled universe: if the universe is expanding at any moment, it was always expanding; if contracting, always contracting.
The universe is \emph{permanently dynamic}.

Solving the Friedmann equations (Appendix~\ref{sec:appendix-friedmann-solution}) yields $a(t) \propto t^{2/3}$: the universe begins at a singularity where all distances vanish and the density is infinite, then expands forever, decelerating but never stopping.
The density decreases as $\rho \propto 1/a^3 \propto 1/t^2$, reflecting the conservation of total mass in any comoving volume as the volume grows.

The same qualitative conclusion---the impossibility of a static universe with matter---holds for the closed ($k = +1$) and open ($k = -1$) cases.
In the closed case, $\dot{a} = 0$ at one moment forces $\ddot{a} < 0$, so the universe cannot remain static; in the open case, $\dot{a}^2$ has a positive lower bound, so $\dot{a}$ can never reach zero.

\subsection{The cosmological constant}
\label{sec:conc-cosmological-constant}

Einstein \citep{einstein1917kosmologische} encountered the same problem: his field equations yield a dynamic universe, contradicting the static cosmos that the astronomical consensus of his era held to be self-evident.
His response was to modify his own equations, introducing the \emph{cosmological constant} $\Lambda$ by adding a term $\Lambda\,g_{\mu\nu}$ to the left-hand side:
\begin{equation}
\label{eqn:einstein-lambda}
R_{\mu\nu} - \tfrac{1}{2}\,\mathcal{R}\,g_{\mu\nu} + \Lambda\,g_{\mu\nu} = \frac{8\pi G}{c^4}\,T_{\mu\nu}.
\end{equation}
This modification preserves the divergence-free property of the left-hand side (since $\nabla_\mu g^{\mu\nu} = 0$), so the equations remain mathematically consistent.
The extra term provides a repulsive effect that can be tuned to balance gravitational attraction, permitting a static solution.
However, this solution is unstable: the slightest perturbation in density or scale factor causes the universe to begin expanding or contracting.
Einstein had patched his equations to obtain the answer he wanted.

\section{The theory was right}
\label{sec:conc-hubble}

In 1929, Edwin Hubble published a result that rendered Einstein's patch unnecessary \citep{hubble1929relation}.
Using distance estimates derived from resolved stars in nearby galaxies, together with radial velocities measured by Vesto Slipher and Milton Humason, Hubble established a roughly linear relation between a galaxy's distance and its recession velocity.
Distant galaxies are \emph{redshifted}---the wavelength of their light is systematically stretched toward longer (redder) wavelengths, in proportion to their distance.

\begin{figure}[t]
\centering
\includegraphics[width=\textwidth]{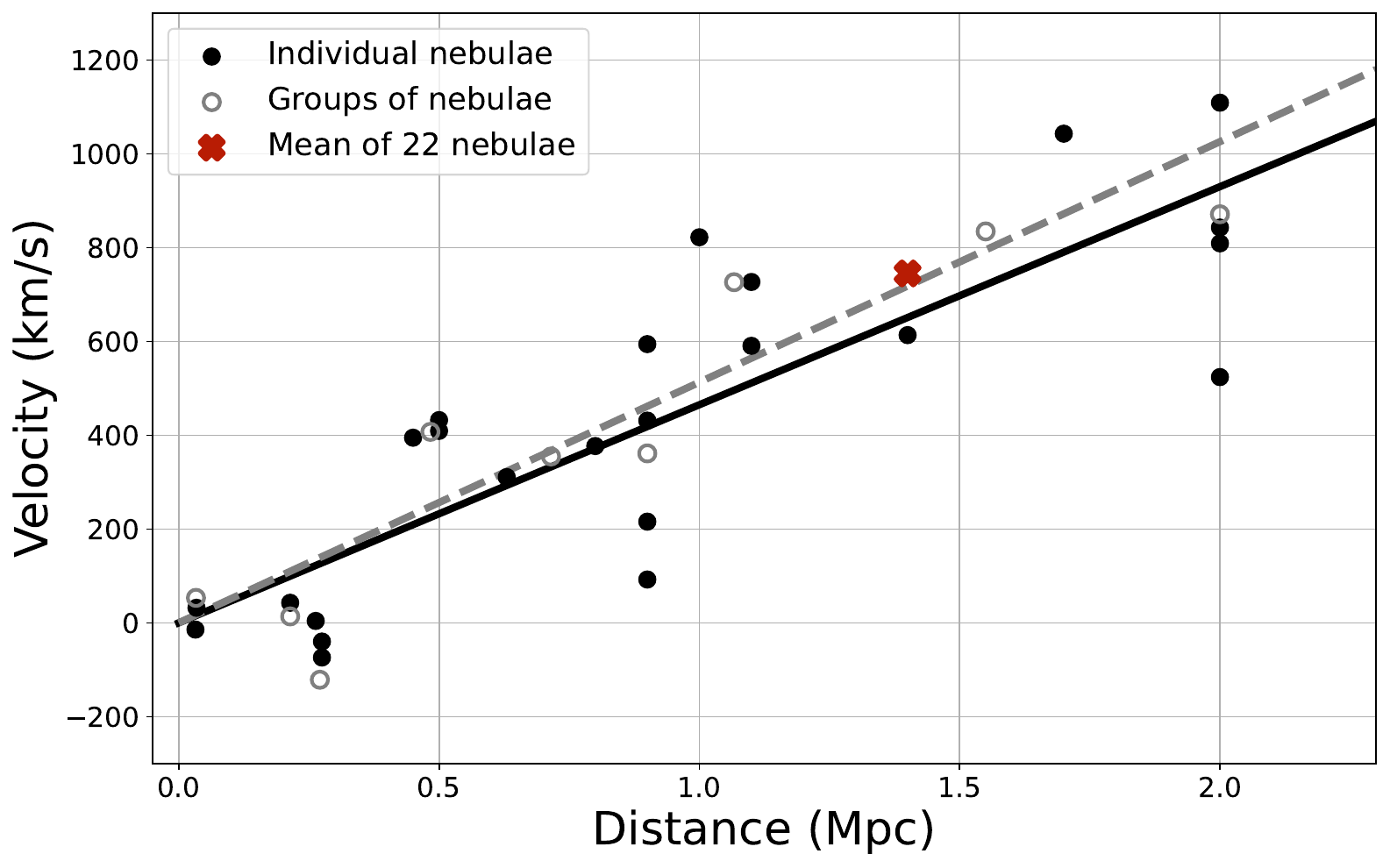}
\caption{Hubble's velocity--distance relation, reproduced from the data in \citet{hubble1929relation}.
Black discs are 24 individual nebulae with estimated distances (solid line: least-squares fit, $K = 465$ km/s/Mpc).
Open circles are 9 groups formed by combining nearby nebulae (dashed line: $K = 513$ km/s/Mpc).
The red cross marks the mean of 22 additional nebulae whose distances could not be estimated individually.
Both fits are consistent with a linear relation $v = K r$ passing through the origin.}
\label{fig:hubble-1929}
\end{figure}

Hubble's original data are reproduced in Figure~\ref{fig:hubble-1929}.
Redshift is the optical analog of the Doppler effect: when a source of light moves away from an observer, each successive wave crest must travel a slightly longer distance, stretching the observed wavelength $\lambda_{\text{obs}}$ relative to the emitted wavelength $\lambda_{\text{emit}}$.
The redshift $z = (\lambda_{\text{obs}} - \lambda_{\text{emit}})/\lambda_{\text{emit}}$ measures this fractional stretch, and for recession velocities small compared to the speed of light, $z \approx v/c$.
Hubble found that nearly all galaxies have $z > 0$ and that the recession velocity increases with distance, roughly as $v \approx H_0 d$, where $H_0$ is a constant now bearing his name.
The universe is not static.
It is expanding, and running the expansion backward in time implies that it originated from an extremely dense, hot initial state---the Big Bang.

The linearity of Hubble's relation is not a coincidence---it is the only possibility consistent with the cosmological principle.
Consider two observers, A and B, separated by distance $d_B$.
Observer A measures a galaxy at distance $d$ receding with velocity $v(d)$.
Observer B sees the same galaxy at distance $d - d_B$ and, after subtracting B's own recession velocity as seen by A, measures velocity $v(d) - v(d_B)$.
Homogeneity requires that B's velocity--distance law take the same functional form as A's, so
\begin{equation}
\label{eqn:cauchy-hubble}
v(d) - v(d_B) = v(d - d_B) \quad \text{for all } d, d_B.
\end{equation}
This is Cauchy's functional equation, whose only continuous solution is $v(d) = H d$ for some constant $H$.
Any nonlinear relation would single out a preferred center---the point from which the relation appears simplest---violating the assumption that no observer occupies a special location.
This constant $H$ is precisely the Hubble parameter predicted by the Friedmann equations (\S\ref{sec:conc-friedmann}).
The first Friedmann equation~\eqref{eqn:friedmann-I} already contains the ratio $\dot{a}/a$: rewriting $3\dot{a}^2/a^2 = 8\pi G\rho/c^2$ as $H^2 = 8\pi G\rho/(3c^2)$ shows that the expansion rate is set by the matter density of the universe.

Einstein's unmodified field equations~\eqref{eqn:einstein-field}, without the cosmological constant, had predicted exactly this.
Had Einstein trusted his own mathematics rather than the astronomical prejudices of his era, he could have predicted the expanding universe over a decade before Hubble observed it.
Einstein later called the introduction of $\Lambda$ his ``biggest blunder'' \citep{gamow1956evolutionary}.

In a precise sense, the equations knew more than Einstein did.
Their deductive consequences included a true prediction about the physical universe that their creator actively suppressed.
The mathematical structure of general relativity---the interplay of $g_{\mu\nu}$, $R_{\mu\nu}$, and $T_{\mu\nu}$---encoded a fact about nature that no human being recognized at the time.

\section{Continually self-improving AI}
\label{sec:conc-ai}

We began this chapter with a narrow observation: the AI systems built in this thesis satisfy the definition of continually self-improving AI, but through the uninteresting mechanism of quantity overcoming quality.
The Einstein parable suggests---but does not prove---that something deeper is possible.
The field equations of general relativity were not merely more diligent than Einstein---their deductive consequences were right where his intuitions were wrong.
The mathematical structure he created encoded a truth about the universe that he actively denied.

The results of this thesis do not yet achieve this.
Synthetic data outworks human data; automated search outworks human researchers.
The mechanism remains quantity over quality.

But the story need not end here.
If models can internalize knowledge into their weights, regularize their own training from the structure of data, and design their own learning algorithms, then it is my view that one day a created system will not merely outwork its creator---but, like Einstein's field equations, contain truths its creator did not recognize.

\section{Future work}
\label{sec:conc-future-work}

Toward a future in which AI systems genuinely surpass their creators, the three methodologies developed in this thesis---synthetic knowledge acquisition, bootstrapped pretraining, and automated algorithm design---can each be extended in distinct ways.
We close by narrating three concrete possibilities.

\subsection{Synthetic continued pretraining as an alternative to infinite context}

In Chapter~\ref{chap:scp}, we showed that synthetic continued pretraining can teach a model knowledge from a small corpus by generating diverse rephrasings and elaborations of the source material.
A natural extension is to apply this approach not to static knowledge acquisition, but to the problem of long-context inference.
Recent work handles long user queries (e.g., 1M--10M+ tokens) using efficient attention \citep{dao2022flashattention, liu2023ring, gemini} or sub-quadratic architectures \citep{tay2022efficienttransformerssurvey, gu2022efficiently, gu2024mamba, sun2024learninglearntesttime}.
In settings where many queries share a long prefix---e.g., a corporation's proprietary documents or other prompt caching use cases \citep{anthropicpromptcache}---one could continue pretraining on the prefix to internalize its knowledge, then perform standard quadratic attention on shorter queries.
This approach pays a fixed training cost to amortize prefix knowledge into model weights, then benefits from shorter context lengths \citep{gururangan2020dont, snell2022learningdistillingcontext}.
By adapting continued pretraining from 10B--100B tokens to as little as 1.3M tokens, the synthetic continued pretraining approach of Chapter~\ref{chap:scp} could enable unsupervised learning of shared text prefixes at much smaller and more practical token counts.
The natural limit of this direction is replacing the context window entirely: a model that has internalized a corpus through synthetic continued pretraining needs no retrieval at inference time, achieving the effect of an infinite context window through learned knowledge rather than attention over tokens.

\subsection{Synthetic data as data-dependent regularization}

In Chapter~\ref{chap:sbp}, we showed that a model can bootstrap its own pretraining capabilities by synthesizing training data from a fixed corpus, producing training signals that yield improvement beyond what the original data provides.
This capability becomes increasingly important as the field confronts a fundamental resource constraint.
The compute-optimal scaling laws of \citet{hoffmann2022training} established that model size and training tokens should grow in proportion: a model trained on too few tokens relative to its parameter count is undertrained, while one trained on too many is inefficient.
The immediate consequence is that the largest models require the most data.
As model sizes continue to grow, the data requirement grows with them---but the supply of high-quality web text does not.
\citet{kim2025pretraininginfinitecompute} study this regime directly, showing that when compute is abundant but data is fixed, standard training overfits and that aggressive regularization (weight decay 30$\times$ larger than standard practice) is needed to continue extracting signal from repeated data.
Together, these results point to a future in which the largest models are severely undertrained---not for lack of compute, but for lack of data.

The synthetic bootstrapped pretraining method of Chapter~\ref{chap:sbp} offers a natural response to this problem.
SBP can be viewed as a form of data-dependent regularization: it does not add new information beyond what is already present in the original corpus, but it makes implicit inter-document correlations explicit, guiding the model toward representations that capture deeper structure.
Standard pretraining learns from the marginal distribution of documents; SBP additionally learns from conditional distributions that expose latent structure shared across documents.
In this view, synthesized documents act as a regularizer whose form depends on the training data itself, much as data augmentation in vision (cropping, flipping, color jittering) regularizes without introducing new visual concepts.
This makes SBP a promising approach for training ultra-large models on fixed data budgets, precisely the regime where compute-optimal scaling laws predict the most severe undertraining.

\subsection{Harness engineering}

In Chapter~\ref{chap:automated-ai-research}, we showed that an AI system can design learning algorithms by searching over a larger design space than human researchers can search, generating, implementing, and evaluating research ideas in a closed loop.
The natural trajectory of this line of work shifts the role of the human researcher from doing research to engineering the harness within which AI does research.
In our experiments, the most consequential design decisions were not about individual ideas---they were about the research environments (Figure~\ref{fig:research-env}).
These are properties of the harness, not of any single experiment.
As AI systems become more capable at the hill-climbing work of generating and testing variations, the harness becomes the primary locus of human contribution.
This is because the harness encodes something that AI systems do not have on their own: human intent.
The choice of what to optimize, what constraints to respect, and what counts as progress reflects human values and goals that cannot be derived from execution feedback alone.
The researcher's comparative advantage shifts from proposing ideas to specifying objectives---from doing the climbing to choosing the mountain.
This shift is already visible in practice: frameworks like Harbor \citep{harbor2026} provide infrastructure for running AI agents across thousands of containerized environments in parallel, collecting execution rollouts, and feeding results back for reinforcement learning---an end-to-end harness in which the human contribution is the design of the evaluation, not the execution of the research.

\appendix

\chapter{Supplementary materials for Chapter~\ref{chap:scp}}
\section{Details on the QuALITY dataset}
\label{sec:appendix-quality}
We provide additional details on the QuALITY dataset.
For each book, we execute entity extraction (Step 1, \S\ref{sec:entigraph-method}) and then analyze all pairwise relations between entities and a subset of all triplet relations (Step 2, \S\ref{sec:entigraph-method}).
Figure \ref{fig:appendix-quality-detail} shows summary statistics for the Raw and EntiGraph corpora.

\begin{figure}[ht]
\subfigure[\label{fig:appendix-raw-token-count} Raw article tokens]{\includegraphics[width=0.3\textwidth]{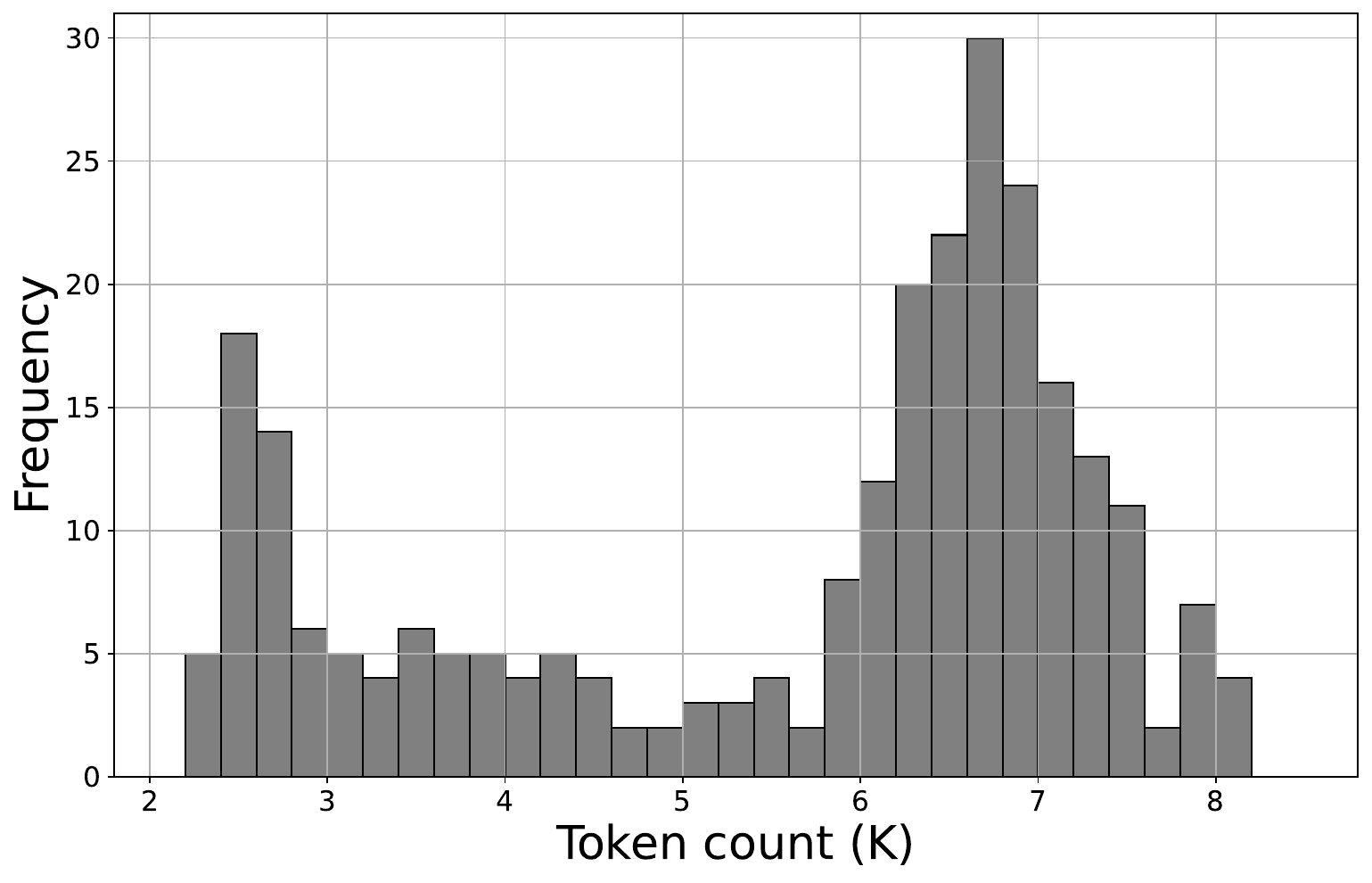}}
\subfigure[\label{fig:appendix-entity-count} Extracted entities]{\includegraphics[width=0.3\textwidth]{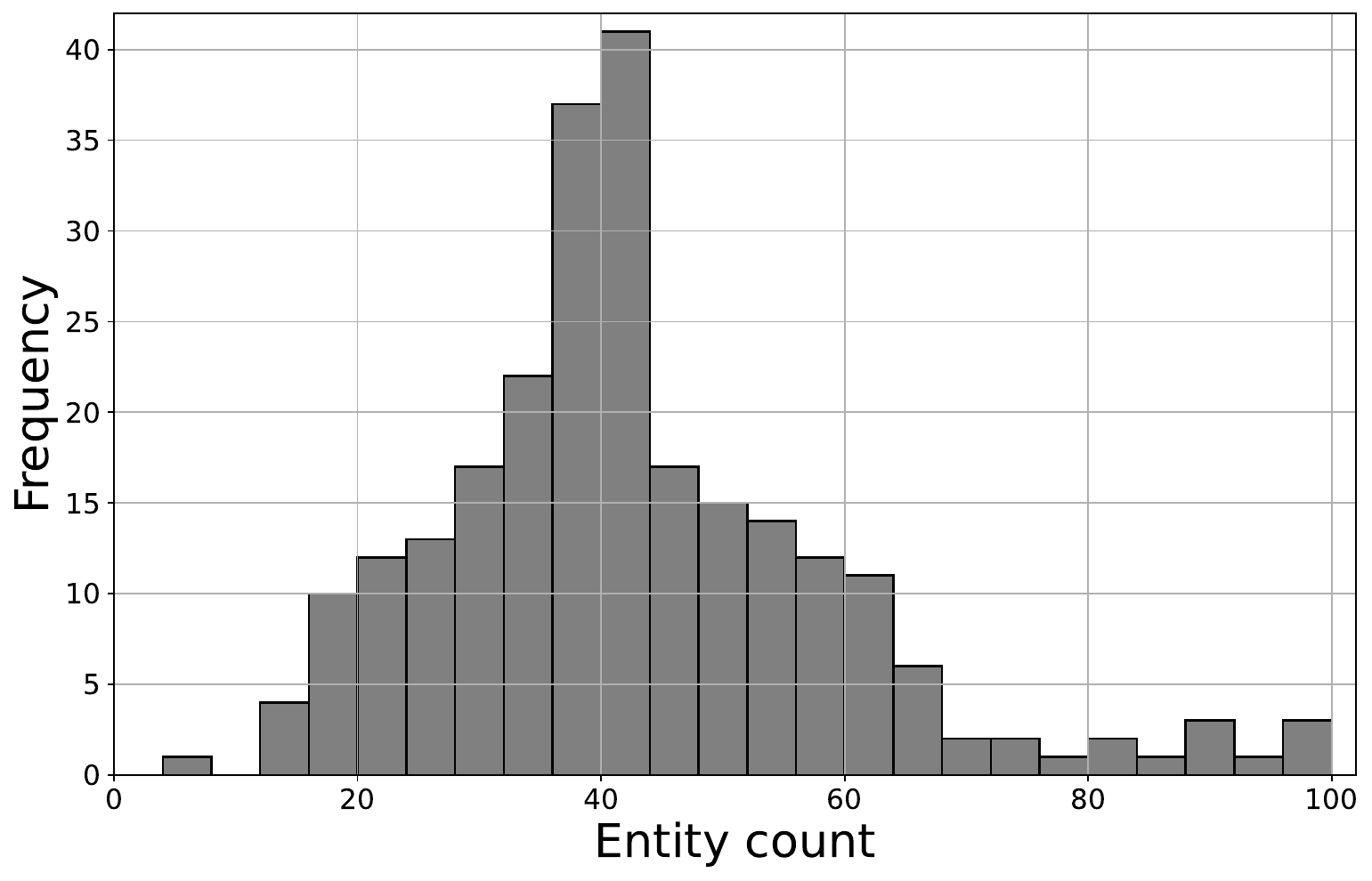}}
\subfigure[\label{fig:appendix-entigraph-token-count} EntiGraph corpus tokens]{\includegraphics[width=0.3\textwidth]{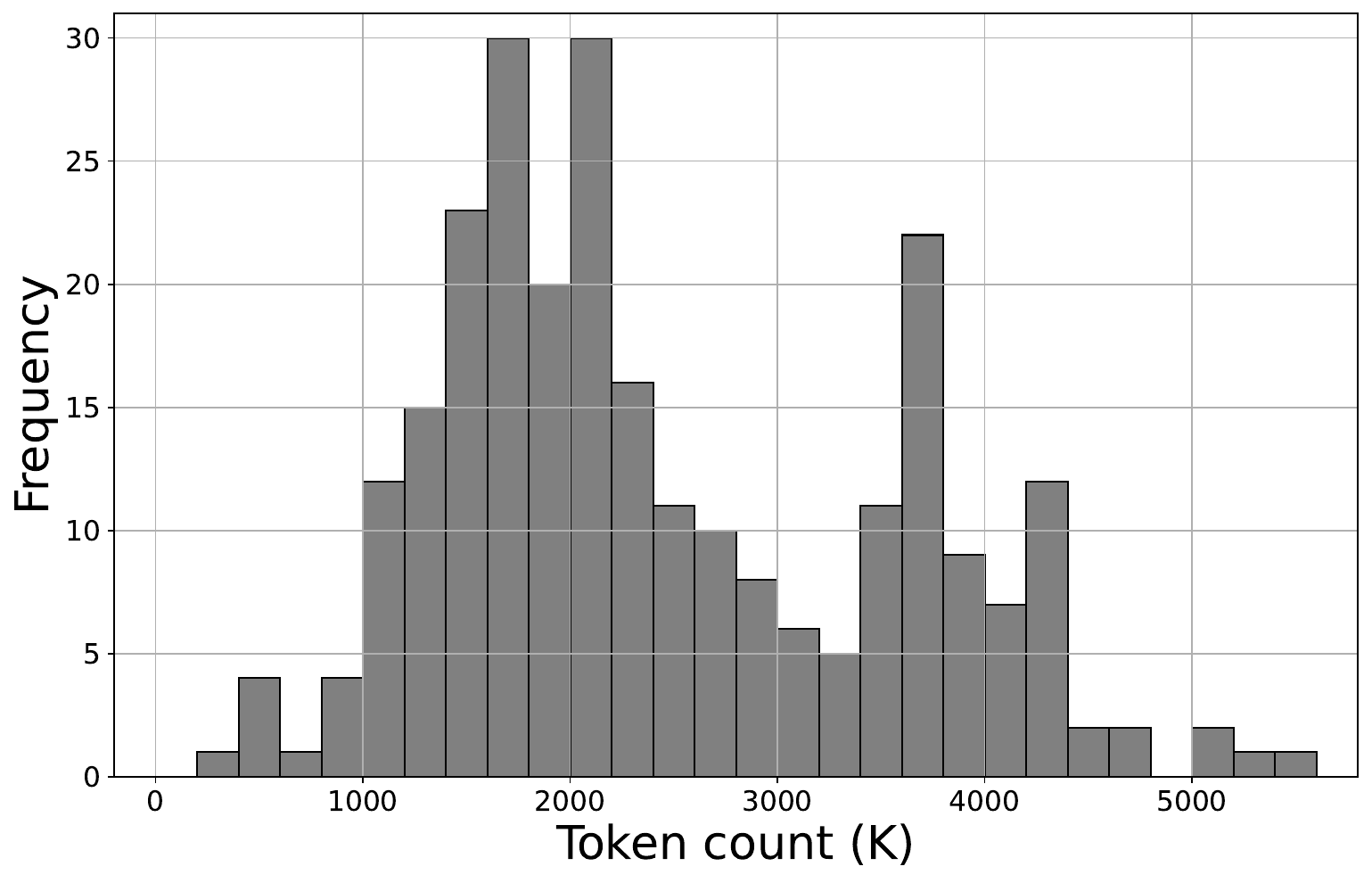}}
\caption{Histograms over the 265 QuALITY articles and books. (a) The token count of raw articles. (b) The number of extracted entities. (c) The token count of EntiGraph synthetic data (generated for each book).}
\label{fig:appendix-quality-detail}
\end{figure}

\section{Training details for the main experiments}
\label{sec:appendix-training-details}

\paragraph{Continued pretraining details} In all experiments, we continue pretraining Llama 3 8B Base with a context length of 2048 and batch size of 16.
We apply a linear learning rate warmup for 5\% of total steps, followed by cosine decay with peak learning rate 5e-6.
We use full parameter training with Fully Sharded Data Parallelism (FSDP, \cite{fsdp}).

\paragraph{EntiGraph continued pretraining details} To mitigate forgetting of pretrained knowledge, we perform replay at a rate of 0.1 using 1B RedPajama tokens \citep{together2023redpajama}.
For each training batch, we flip a biased coin such that with 10\% probability we load RedPajama data instead of EntiGraph synthetic data.

\paragraph{Raw continued pretraining details} We now describe continued pretraining directly on the Raw corpus, which produces the ``Raw CPT'' model.
Because the Raw corpus contains only 1.3M tokens, we jointly tune the number of epochs (repetition factor) and the RedPajama replay rate on accuracy over a \quality~QA validation split.
We select a configuration with 4 epochs and a 0.1 replay rate.

\paragraph{Instruction tuning details} We use the UltraChat instruction tuning dataset \citep{ultrachat} filtered by the Huggingface team \citep{alignmentbook}.
We format the UltraChat conversations using the Llama 3.1 8B Instruct chat template \citep{llama3}, obtaining a 250M token instruction tuning dataset.
We apply a linear learning rate warmup followed by cosine decay to 0 with peak learning rate 5e-6, and train the model for 1 epoch with a batch size of 512 and context window of 2048.
To validate our instruction tuning procedure, we measure the AlpacaEval \citep{alpaca_eval} winrate against GPT-4 and find that it improves from 0\% to 6.25\%, comparable to Llama 2 Chat 13B's 7.7\% baseline winrate.

\paragraph{Compute resource} We run all continued pretraining experiments on one $8\times$H100 node.
With PyTorch FSDP \citep{fsdp}, we achieve throughput of 6090 tokens per second.
Because all experiments use the same model architecture, batch size, and context length, we can calculate training time from total tokens seen.
For example, EntiGraph trains on 455M tokens for 2 epochs, taking $455$M$\times 2/6090$ seconds, or about 41 hours.

\section{Task-specific finetuning for the QuALITY question set}
\label{sec:appendix-task-specific}

Our work considers \emph{task-agnostic} synthetic data generation and continued pretraining as a way to obtain generalizable knowledge about a domain---knowledge that can later be extracted via few-shot prompting \citep{gpt3} and instruction tuning \citep{instruct_gpt}.

However, if the goal is only to excel on a single task such as question answering, one could fine-tune a language model for that particular task.
This approach works well on tasks such as SQuAD \citep{rajpurkar2016squad100000questionsmachine} in-domain but degrades outside the fine-tuning data distribution \citep{awadalla-etal-2022-exploring}.

We do not extensively compare to task-specific finetuning given EntiGraph's broader multi-task goals.
We run preliminary experiments comparing a simple QA SFT baseline to EntiGraph and find that EntiGraph scaling and synthetic data generation costs are generally favorable even compared to this strong, task-specific baseline.

\paragraph{QA SFT} We follow the same setup as in \S\ref{sec:setup} and \S\ref{sec:exp-setup} except that we do not prompt $\lmsynth$ to generate general knowledge about QuALITY articles.
Instead, we prompt $\lmsynth$ to generate QA pairs directly:

\begin{qualitativeBox}
{\footnotesize
\begin{verbatim}
You are an assistant to help read a article and then rephrase it in a
question answering format. The user will provide you with an article
with title, year, content. You need to generate a paraphrase of the
same article in question and answer format with multiple tags of
"Question: ..." followed by "Answer: ...". Remember to keep the
meaning and every content of the article intact, including the title,
year, etc.
\end{verbatim}
}
\end{qualitativeBox}
We repeat this prompt many times at temperature 1.0, resulting in 28M tokens of synthetic question-answer pairs.
We perform the same continued pretraining procedure as in \S\ref{sec:exp-cpt-procedure} on Llama 3 8B and refer to this model as ``QA SFT''.

\begin{figure}[ht]
\centering
\includegraphics[width=\textwidth]{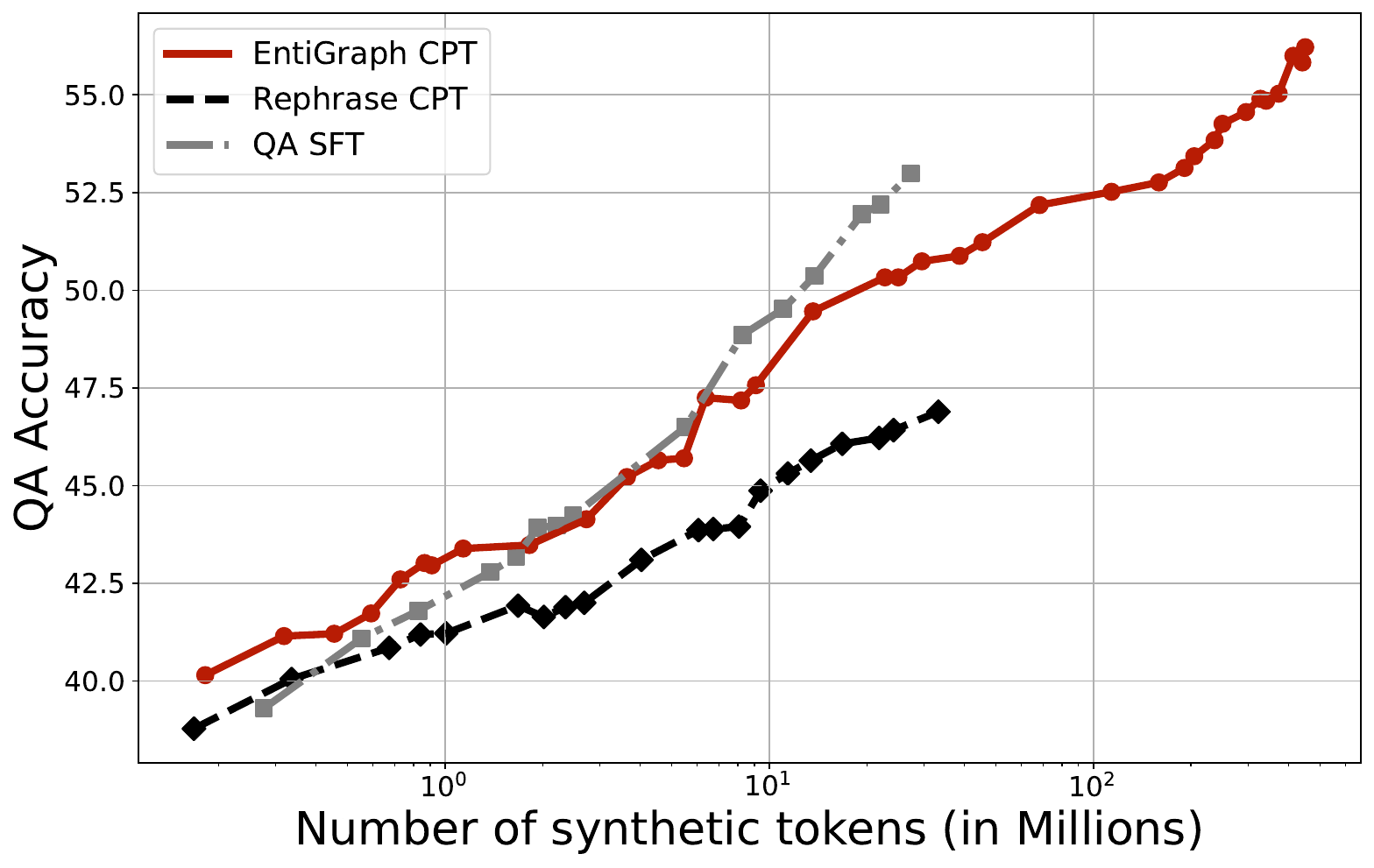}
\caption{Accuracy on the QuALITY question set $\Qt$ ($y$-axis) as a function of the synthetic token count ($x$-axis).
Comparison among EntiGraph CPT, Rephrase CPT, and QA SFT.}
\label{fig:qasft}
\end{figure}

\paragraph{Results discussion} Figure \ref{fig:qasft} shows the QA SFT scaling curve.
Task-specific finetuning demonstrates sharp improvement in QA accuracy, consistent with prior results on task-specific finetuning gains for pretrained models.
While QA SFT performance is high, EntiGraph attains similar performance despite being entirely task-agnostic, and the overall cost of creating the dataset is much lower for EntiGraph.

This cost difference is hidden in Figure \ref{fig:qasft}, as we plot training tokens rather than dollars spent to generate the synthetic data.
For QA SFT, each QA question is short, resulting in large inefficiencies in generating this dataset.
The input-to-output token ratio is large compared with Rephrase CPT and EntiGraph CPT, costing over \$5K to generate just 28M tokens.\footnote{OpenAI API pricing, Sep 2024.}
This cost difference means further scaling becomes prohibitively expensive, and EntiGraph's performance in Figure \ref{fig:qasft} is even better than it appears when matching for total cost rather than token budget.

\section{Additional details on open-book experiments}
\label{sec:appendix-rag}

We provide additional details on our open-book experimental setup, including our retrieval-augmented generation (RAG, \cite{rag, ragsurvey}) pipeline.
As described in \S\ref{sec:exp-open-book}, we use a standard two-stage RAG pipeline: first, an offline stage that indexes document chunks; second, inference-time retrieval, reranking, and placement of those chunks in a few-shot LM prompt.

\subsection{Stage 1: offline indexing}

The indexing stage constructs an index over all 265 articles and books from the \quality~corpus $\Ds$.
This stage chunks documents, obtains dense vector embeddings for each chunk using an API-based embedding model, and indexes the (embedding, chunk) pairs.

\paragraph{Chunking documents} We first split each document $D^{(i)} \in \{D^{(i)}\}_{i = 1}^n = \Ds$ into a set of $m_i$ document chunks $\{C^{(i)}_1, ..., C^{(i)}_{m_i}\}$.
We use the \texttt{Recursive}
\texttt{CharacterTextSplitter} from \cite{chase2022langchain}, which keeps paragraphs (and then sentences, and then words) together as long as possible to preserve semantics within each chunk.
We use non-overlapping chunks and tune chunk size in characters (\texttt{chunk\_size}, hyperparameter values provided below).
Because we have access to metadata about each document $D^{(i)}$---namely, the title, author, and year---we prepend this metadata to each document chunk.
This mirrors how an organization building a RAG system over their document store would include document metadata (title, author, year, etc.).
We embed, retrieve, and place these final chunks with metadata prepended in-context.

\paragraph{Embedding and indexing document chunks} We obtain dense embeddings for all document chunks using OpenAI \texttt{text-embedding}
\texttt{-3-large} \citep{neelakantan2022textcodeembeddingscontrastive}.
We then index all (embedding, chunk) tuples using a FAISS vector store \citep{douze2024faisslibrary}.

\subsection{Stage 2: inference-time retrieval and reranking}

At inference time, the RAG system receives a test query $q \in \Qt$.
Each query $q$ is contextualized with the article title and author name, as described in \S\ref{sec:exp-setup}, and contains its four possible answer choices (\quality~is a 4-choice, multiple choice dataset).
In Stage 2, we embed the query with the API-based embedding model, retrieve $K$ document chunks using approximate nearest-neighbor search, and select the $k < K$ most relevant chunks using an API-based reranker.

\paragraph{Retrieving top-$K$ document chunks} We embed $q$ with \texttt{text-embedding-3-large}, and retrieve the top-$K$ most relevant document chunks from our indexed vector store using FAISS similarity search with a Euclidean distance metric.

\paragraph{Reranking to obtain top-$k$ ($k < K$) chunks} We use a reranker to filter the $K$ retrieved document chunks to a smaller set of $k$ reranked chunks.
Rerankers significantly improve recall (the proportion of the time the salient article appears in the top chunks), and our RAG pipelines achieve near-perfect recall (Table~\ref{tbl:exp-open} in \S\ref{sec:exp-open-book}).
We pass the query $q$ and the list of $K$ retrieved document chunks to Cohere \texttt{rerank-english-v3.0} \citep{coherererank}, which returns the $K$ chunks ordered from most to least semantically relevant.
We take the $k$ highest scoring chunks and place them in our few-shot prompt.

\paragraph{Few-shot prompt formatting} We provide our full few-shot chain-of-thought evaluation prompts for the open-book setting in our code release.
As with the closed-book QA evaluation prompt, we manually write and fact-check in-context learning examples about well-known books to avoid leaking knowledge from the \quality~articles.
In early experiments, we find that placing retrieved contexts first, followed by the question and answer choices, significantly outperforms question-then-contexts; we use this format throughout the retrieval experiments.
We treat as a hyperparameter whether reranked chunks are ordered from best match to worst (\texttt{best\_first}) or worst match to best (\texttt{best\_last}).
When performing few-shot evaluation, we follow the sampling procedure from the closed-book experiments (Appendix \ref{sec:appendix-qa-eval-detail}).
We generate 64 responses for each question and filter out responses that do not parse to one of the four choices.
We then randomly select one valid response as the model's final answer.

\subsection{Hyperparameter tuning}
We compare two LMs in the RAG pipeline: EntiGraph CPT and its base model, Llama 3 8B Base.
We fix the number of retrieved chunks to $K = 128$ but vary the number of reranked chunks $k$ placed in the context window.
For each language model + RAG pipeline, we independently tune the following hyperparameters via grid search on a \quality~QA validation split:
\begin{itemize}
    \item Document $\texttt{chunk\_size} \in \{256, 512, 1024\}$
    \item Rerank top-$k \in \{1, 2, 4, 8, 16\}$
    \item Order of chunks $\in \{\texttt{best\_first}, \texttt{best\_last}\}$
    \item Eval temperature $\in \{0.1, 0.3, 0.5, 0.7\}$
\end{itemize}
We provide tuned hyperparameters in our code release.

\newpage
\section{Proof of Theorem \ref{thm:toy} and other analytical formulas}
\label{sec:appendix-proof}
We prove Theorem \ref{thm:toy} and provide derivations for several other approximation formulas.

\begin{proof}[Proof of Theorem \ref{thm:toy}]
    Fix the matrix $\bM_0$, we observe that
    \begin{align*}
        \mathsf{Acc}(\bM_t) = \frac{\mathbb{E}[\|\bM_t\|_1 \vert \bM_0 ]}{V(V-1)} = \sum_{(i,j) \in \cV^2} \frac{\mathbb{E}[\mathds{1}((i,j)\in \cD_t) \vert \bM_0 ]}{V(V-1)} = \sum_{(i,j) \in \cV^2} \frac{\mathbb{P}[(i,j)\in \cD_t \vert \bM_0 ]}{V(V-1)}.
    \end{align*}
   For each $(i,j) \in \cV^2$, we define $q_{i,j}$ to be the probability that $(i,j)$ is included in the set $$\{(x_t, z_t^1), (x_t, z_t^2), \dots, (x_t, z_t^{k_t}), (x_t, y_t)\}$$. Note that each iteration of the procedure generates a path $(x_t, z_t^1, z_t^2, \dots, z_t^{k_t}, y_t)$ independently identically. So naturally $q_{i,j}$ does not depend on the time $t$. This implies that $\mathbb{P}[(i,j)\in \cD_t \vert \bM_0] = 1-(1-q_{i,j})^t$. Thus we can further rewrite the link density as 
   \begin{align*}
       \mathsf{Acc}(\bM_t) &= \frac{|\Ds|}{V(V-1)}+ \sum_{(i,j) \in \cV^2 \backslash \Ds} \frac{\mathbb{P}[(i,j)\in \cD_t \vert \bM_0 ]}{V(V-1)} \\
       &= \frac{|\Ds|}{V(V-1)}+ \sum_{(i,j) \in \cV^2 \backslash \Ds} \frac{1-(1-q_{i,j})^t}{V(V-1)}.
   \end{align*}
   The remaining task is to estimate $q_{i,j}$. We say a vertex $j$ is reachable from $i$ and denote $i \sim j$, if there is a directed path from $i$ to $j$ in $\bM_0$. We define $\cR = \{(u,v) \in \cV^2: u \neq v, u \sim v\}$ to be the set of all reachable pairs of vertices in $\cV$. We note that $q_{i,j}$ is non-zero if and only if $j$ is reachable from $i$ in $\bM_0$. Now, for any $t\geq 1$, the function $1-(1-x)^t$ is concave, thus by Jensen's inequality, we have
   \begin{align*}
       \sum_{(i,j) \in \cV^2 \backslash \Ds} 1-(1-q_{i,j})^t \leq \sum_{(i,j) \in \cR} 1-(1-q_{i,j})^t \leq |\cR|\left(1-(1-\bar q_{i,j} )^t \right),
   \end{align*}
   where
   \begin{align*}
       \bar q_{i,j} = \frac{\sum_{(i,j) \in \cR} q_{i,j}}{|\cR|}.
   \end{align*}
    For each $(i,j)\in \cR$, the probability $q_{i,j}$ satisfies
    \begin{align*}
        q_{i,j} = \frac{\sum_{a\neq b \in\cV^2 } \mathds{1}((i,j)\in \{(a, z^1), (a, z^2), \dots, (a, z^{k}), (a, b)\}) }{V(V-1)}
    \end{align*}
    where $(a, z^1, z^1, \cdots, z^k, b)$ is the shortest path in $\bM_0$ connecting $a$ and $b$. If there is no such path, then by default the indicator equals zero. Now we look at
    \begin{align*}
        \sum_{(i,j) \in \cR} q_{i,j} &= \frac{1}{V(V-1)}\sum_{(i,j) \in \cR} \sum_{(a,b) \in \cR} \mathds{1}((i,j)\in \{(a, z^1), (a, z^2), \dots, (a, z^{k}), (a, b)\})\\
        &\leq \frac{1}{V(V-1)}\sum_{(a,b) \in \cR} \sum_{i\neq j \in\cV^2 } \mathds{1}((i,j)\in \{(a, z^1), (a, z^2), \dots, (a, z^{k}), (a, b)\})\\
        & = \frac{1}{V(V-1)}\sum_{(a,b) \in \cR} \ell_{a,b},
    \end{align*}
    where $\ell_{a,b}$ is the length of the shortest path connecting $a$ to $b$. To analyze the typical shortest length of paths, we present a few classical results on directed Erd\H{o}s-R\'enyi graphs. For any $a \in \cV$, let $X(a)$ denote the set of vertices reachable from $a$ and let $Y(a)$ denote the set of vertices from which $a$ is reachable. Recall that $\rho(\lambda)$ is the extinction probability for the Poisson$(\lambda)$ branching process.
    \begin{lemma}[Lemma 1 and Corollary 1 in \cite{karp1990transitive}]\label{lem:conponent_size}
        For each vertex $a$, with probability tending to $1$ as $V$ tends to infinity, there exists a constant $\beta>0$ such that either $|X(a)| \leq \beta \log V$ or $|X(a)| = (1-\rho(\lambda)) V + \Theta(\sqrt{V})$. Moreover, the probability that the latter happens tends to $1-\rho(\lambda)$ as $V$ tends to infinity. The same is true for $Y(a)$.
    \end{lemma}
    For each vertex $a$, the set $X(a)$ is said to be small if $|X(a)| \leq \beta \log V$ (in such case we write $a\in \cS_X$) and large if $|X(a)| = (1-\rho(\lambda)) V + \Theta(\sqrt{V})$ (we write $a\in \cL_X$). We define $\cS_Y$ and $\cL_Y$ similarly.
    \begin{lemma}[Theorem 3 in \cite{karp1990transitive} and Theorem 2.4.1 in \cite{durrett2010random}]\label{lem:diameter}
        With probability tending to $1$, the following statement holds for all $a$ and $b$ in $\cV$: if $X(a)$ is large and $Y(b)$ is large, then $b$ is reachable from $a$. Moreover, if $X(a)$ is large and $Y(b)$ is large, then for any $\varepsilon>0$ and any sufficiently small $\delta>0$,
        \begin{align*}
            \mathbb{P}[\ell_{a,b} > (1+\varepsilon) \log V /\log \lambda ] < \exp(-V^{\varepsilon} \delta).
        \end{align*}
    \end{lemma}
    With Lemma \ref{lem:conponent_size} and Lemma \ref{lem:diameter}, we can now give useful estimates of $|\cR|$.
    In particular, for any $\varepsilon>0$,
    \begin{align*}
        |\cR| &= |\{(a,b) \in \cR: a \in \cL_X, b\in \cL_Y\}| + |\{(a,b) \in \cR: a \in \cS_X \text{ or } b\in \cS_Y\}|\\
        & \leq (1-\rho(\lambda))^2 (1+\varepsilon/4) V^2 + 2(1+\varepsilon) V \beta \log V \\
        & \leq (1-\rho(\lambda))^2 (1+\varepsilon/3) V(V-1),
    \end{align*}
    with high probability. Similarly, for the lower bound,
    \begin{align*}
        |\cR| &= |\{(a,b) \in \cR: a \in \cL_X, b\in \cL_Y\}| + |\{(a,b) \in \cR: a \in \cS_X \text{ or } b\in \cS_Y\}|\\
        & \geq (1-\rho(\lambda))^2 (1-\varepsilon) V^2 \\
        & \geq (1-\rho(\lambda))^2 (1-\varepsilon) V(V-1),
    \end{align*}
    with high probability. By a union bound over all pairs of $(a,b)\in \cR$, we also have that
    \begin{align*}
        \sum_{(i,j)\in \cR} q_{i,j} &\leq \frac{1}{V(V-1)}\sum_{(a,b) \in \cR} \ell_{a,b} \\
        &= \frac{1}{V(V-1)}\sum_{\substack{(a,b) \in \cR\\ a\in \cL_X, b\in \cL_Y}} \ell_{a,b} + \frac{1}{V(V-1)}\sum_{\substack{(a,b) \in \cR\\ a\in \cS_X \text{ or } b\in \cS_Y}} \ell_{a,b}\\
        & \leq (1-\rho(\lambda))^2 (1+\varepsilon/2) \frac{\log V}{\log \lambda} + \frac{1}{V(V-1)} 2(1+\varepsilon) V (\beta \log V)^2\\
        & \leq (1-\rho(\lambda))^2 (1+\varepsilon) \frac{\log V}{\log \lambda},
    \end{align*}
    with probability larger than $1-V^2\exp(-V^\varepsilon \delta)$. Combining the above, for any $\varepsilon>0$, 
    \begin{align*}
        \bar q_{i,j} = \frac{\sum_{(i,j) \in \cR} q_{i,j}}{|\cR|} \leq  \frac{ (1+\varepsilon)\log V}{V(V-1) \log \lambda },
    \end{align*}
    with high probability. Therefore, for any $\varepsilon>0$,
    \begin{align*}
       \mathsf{Acc}(\bM_t) &\leq \frac{|\Ds|}{V(V-1)}+ \frac{|\cR|\left(1-(1-\bar q_{i,j} )^t \right)}{V(V-1)}\\
       & \leq (1+\varepsilon)\left( p + (1-\rho(\lambda))^2\left( 1-\left( 1-\frac{ (1+\varepsilon)\log V}{V(V-1) \log \lambda }\right)^t \right) \right),
   \end{align*}
   with high probability, which completes the proof of the upper bound. For the lower bound, we observe that if $i \sim j$ and $(i,j)\in \cR \backslash \Ds$, then $q_{i,j} \geq 1/V(V-1)$, because when $i$ and $j$ are chosen in the procedure, the edge $(i,j)$ will be added. This implies that
   \begin{align*}
       \mathsf{Acc}(\bM_t) &= \frac{|\Ds|}{V(V-1)}+ \sum_{\cR \backslash \Ds} \frac{1-(1-q_{i,j})^t}{V(V-1)}\\
       & \geq \frac{|\Ds|}{V(V-1)} + \frac{|\cR\backslash \Ds|}{V(V-1)} \left( 1- \left (1-\frac{1}{V(V-1)} \right )^t\right)\\
       & \geq (1-\varepsilon) \left( p+(1-\rho(\lambda))^2 \left( 1- \left (1-\frac{1}{V(V-1)} \right )^t\right) \right),
   \end{align*}
   with high probability which completes the proof of the lower bound.
\end{proof}
To obtain a more precise description of $\mathsf{Acc}(\bM_t)$, we use a Poisson branching process to approximate the cluster growth of vertices. A Poisson$(\lambda)$ branching process models a population evolving in time, where each individual independently gives birth to a number of children with Poisson$(\lambda)$ distribution. We denote by $Z_n$ the number of individuals in the $n$-th generation, where by default $Z_0=1$. Then $Z_n $ satisfies the recursion relation $Z_n=\sum_{i=1}^{Z_{n-1}} X_{n, i}$, where $\{X_{n, i}\}_{n, i \geq 1} $is a doubly infinite array of i.i.d. Poisson$(\lambda)$ random variables. The total progeny $Y_n$ is then defined as $Y_n = \sum_{i=0}^n Z_n$. $Z_n$ is often called a Galton--Watson branching process and the associated tree is called a Galton--Watson tree.

As in the previous proof, accurately estimating $\mathsf{Acc}(\bM_t)$ requires understanding $q_{i,j}$, the probability that edge $(i,j)$ is added in each round. As before, the only edges added are those connected to the giant component (i.e., $i \in \cL_X$ and $j \in \cL_Y$). The proportion of such edges converges to $C_\lambda$ as $V \to \infty$. Recall that
\begin{equation}\label{eqn:qij}
    q_{i,j} = \frac{\sum_{(a,b) \in\cR } \mathds{1}((i,j)\in \{(a, z^1), (a, z^2), \dots, (a, z^{k}), (a, b)\}) }{V(V-1)}    
\end{equation}
where $(a, z^1, z^1, \cdots, z^k, b)$ represents the shortest path in $\bM_0$ connecting $a$ and $b$. Equivalently, if we consider the tree generated by a breadth-first search in $\bM_0$ rooted at $i$, then since $i \sim j$, $j$ will be in the tree, and the numerator counts the total number of offspring of $j$ in the tree, including $j$ itself. This is the point at which a rigorous mathematical characterization of the tree becomes challenging. Instead, we approximate the tree and analyze its behavior. It is well-known that when $p=\lambda/V$, the cluster growth (or the breadth-first search at a vertex) can be approximated by a Poisson$(\lambda)$ branching process (see e.g., \cite{Hofstad2016,durrett2010random}). 
For fixed vertex $i$, we define $T$ as a Galton--Watson tree rooted at $i$ with Poisson$(\lambda)$ offspring distribution with depth $L$. We use $T$ to approximate the exploration process at $i$. For $0\leq \ell \leq L$, the number of vertices at level $L-\ell$ is approximately $\lambda^{L-\ell}$. Given that the total number of vertices in $T$ is approximately $(1-\rho(\lambda)) V$, the number of vertices at level $L-\ell$ is also $(1-\rho(\lambda)) V(\lambda-1)/\lambda^{\ell+1}$. For each vertex at level $L-\ell$, the number of its offspring (including itself) equals $k$ with probability $p_\ell(k)$. In this case, the numerator in \eqref{eqn:qij} equals $k$. Combining the above, there are around $(1-\rho(\lambda))V \cdot p_\ell(k)(1-\rho(\lambda))V(\lambda-1)/\lambda^{\ell+1}$ vertex pairs $(i,j)$ in the graph such that $i \in \cL_X$, $j \in \cL_Y$, $q_{i,j} = k/V(V-1)$ and $j$ is located at the $L-\ell$ level in the tree $T$. Ultimately, we arrive at an approximation of the form
\begin{align*}
    \mathsf{Acc}(\bM_t) \sim  p+ C_\lambda \left( 1-  \sum_{\ell=0}^\infty \frac{\lambda-1}{\lambda^{\ell+1}} \sum_{k=1}^\infty p_\ell(k)\left( 1-\frac{k}{V(V-1)} \right)^t \right).
\end{align*}

Beyond Erd\H{o}s-R\'enyi graphs, $q_{i,j}$ may not be as explicit. Defining $C$ as the proportion of vertex pairs $(i,j)$ such that $i \sim j$ in $\bM_0$, we find that $q_{i,j}$ is nonzero for $CV(V-1)$ pairs of vertices. Writing $a_k = k/V(V-1)$ and defining $\mu(k)$ as the probability that $q_{i,j} = a_k$, we obtain a general formula
\begin{align*}
    \mathsf{Acc}(\bM_t) \sim  p+ C \left(1-\sum_{k=1}^\infty \mu(k) \left(1-a_k \right)^t \right).
\end{align*}
The drawback of this formula is the lack of explicit expressions---for a given $\bM_0$, computing the measure $\mu(\cdot)$ is not simple.

We next provide a qualitative description of the mixture-of-exponentials shape.
\begin{lemma}\label{lem:shape}
For a fixed constant $0<C<1$ and a probability measure $\mu(\cdot)$ on $\mathbb{Z}_+$ with finite mean $m$, we define 
\begin{align*}
    f(t) = p+ C \left(1-\sum_{k=1}^\infty \mu(k) \left(1-\frac{k}{V(V-1)} \right)^{tV(V-1)} \right).
\end{align*}
Then we have that there exists $0<t_1<t_2$ such that
\begin{align*}
    f(t) = 
    \begin{cases}
    \Theta\left (p+ t \right), \quad &\text{ for } 0\leq t\leq t_1,\\
    \Theta(\log t ), \quad &\text{ for } t_1\leq t\leq t_2,\\
    \Theta(1), \quad &\text{ for } t\geq t_2,
    \end{cases}
\end{align*}
as $V \to \infty$.
\end{lemma}
\begin{proof}[Proof of Lemma \ref{lem:shape}]
Fix any $1<t_1<t_2$. Note that $f(t)$ is monotone increasing, concave and always bounded by $1$. We also have
\begin{align*}
    f(t_2) \geq p+C \left(1- \left(1-\frac{1}{V(V-1)} \right)^{t_2 V(V-1)} \right) \geq p+C(1-\exp(-t_2)) = \Theta(1).
\end{align*}
So $f(t) = \Theta(1)$ when $t\geq t_2$. Now when $t\leq t_1$,
\begin{align*}
    f(t) \leq p+ C \left(1- \sum_{k=1}^\infty \mu(k) (1-tk) \right)\leq p+ C  m t.
\end{align*}
Since $f(0) = p$ and $f(t_2) \geq p+C(1-\exp(-t_2))$, by concavity, $f(t)$ is lower bounded by $p+t C(1-\exp(-t_2))/t_2 = \Theta(p+t)$ for any $0\leq t\leq t_1$. Finally for $t_1\leq t\leq t_2$, we note that $f(t_1)\leq f(t)\leq 1$, so easily, $f(t) \leq \log t_1/\log t_1 \leq \log t/\log t_1 = O(\log t)$. Similarly, $f(t) \geq f(t_1) \log t_2/\log t_2 \geq \log t (f(t_1)/\log t_2) \geq \Omega(\log t)$. Therefore, $f(t) = \Theta(\log t)$ for any $t_1\leq t \leq t_2$.
\end{proof}

\subsection{More details on the mixture of exponential shape}
\label{sec:curve-fitting-details}
We provide additional discussion on the mixture-of-exponentials shape, including how we fit it to empirical EntiGraph CPT QA accuracy.
\paragraph{Sketch of derivation} Intuitively, edge $(i,j)$ will eventually be added if and only if $j$ is reachable from $i$ in the original graph $\bM_0$. This explains the limiting behavior of $\mathsf{Acc}(\bM_t)$ as $t$ approaches infinity: the proportion of links converges to the proportion of connected vertex pairs in $\bM_0$. To understand the mixture-of-exponentials functional form, consider that at time $t$, the probability of adding each vertex pair follows an exponential pattern, with different vertex pairs exhibiting different growth rates. Consider a breadth-first search in $\bM_0$ starting from vertex $i$. If $j$ is close to the root, many paths from $i$ to other vertices pass through $j$, making $(i,j)$ more likely to be included in each iteration. In contrast, if $j$ is far from the root (e.g., at the end of the exploration process), fewer such paths exist, making $(i,j)$ less likely to be included. This accounts for the mixture-of-exponentials shape, where the mixture reflects the distance of each vertex from the root, the number of such vertices, and their corresponding exponential growth rates.

\paragraph{Qualitative description} \begin{figure}[ht]
\subfigure[\label{fig:Acc1} Linear regime]{\includegraphics[width=0.32\textwidth]{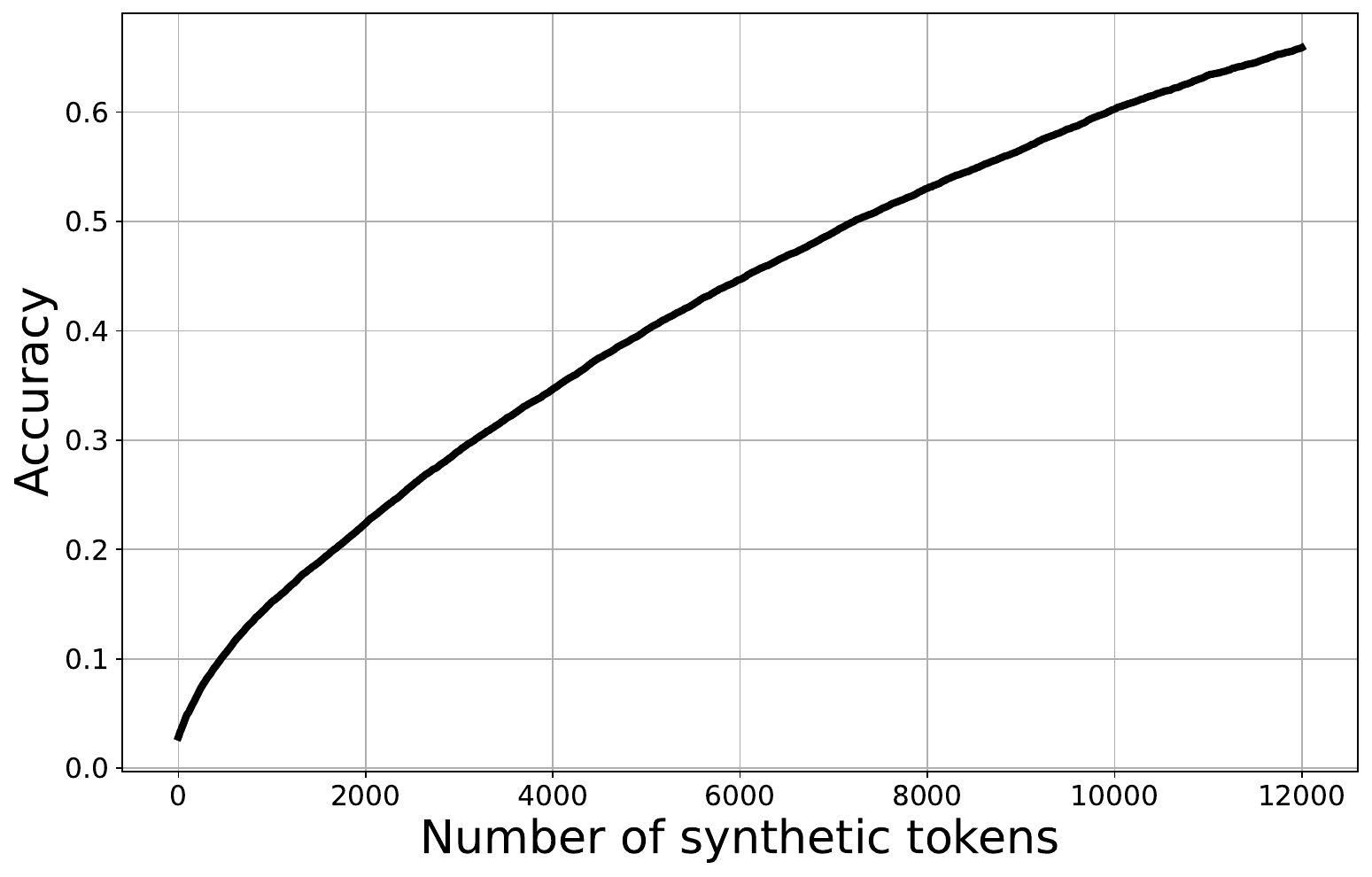}}
\subfigure[\label{fig:Acc2} Log-linear ($t$ in log scale)]{\includegraphics[width=0.32\textwidth]{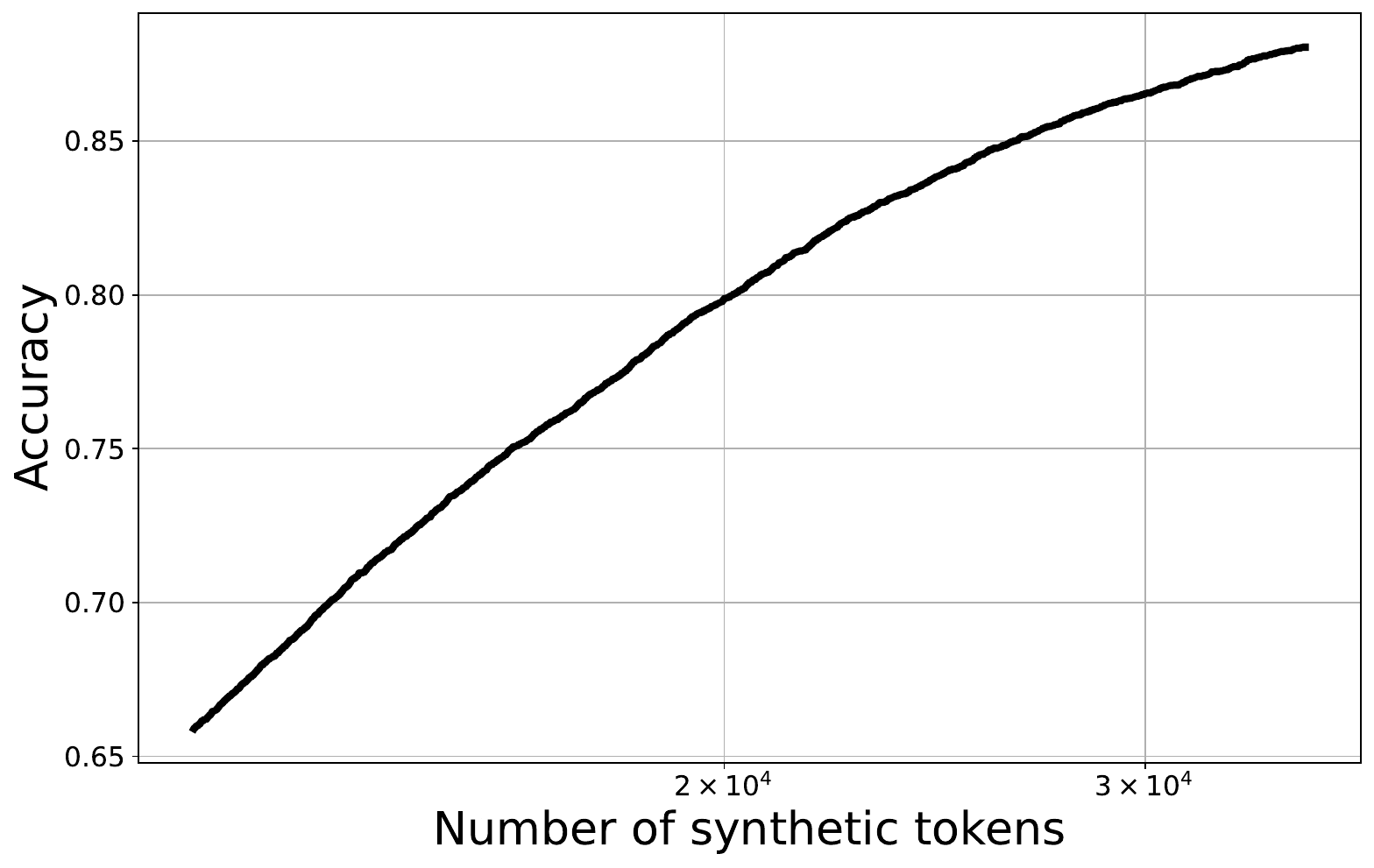}}
\subfigure[\label{fig:Acc3} Plateau regime]{\includegraphics[width=0.32\textwidth]{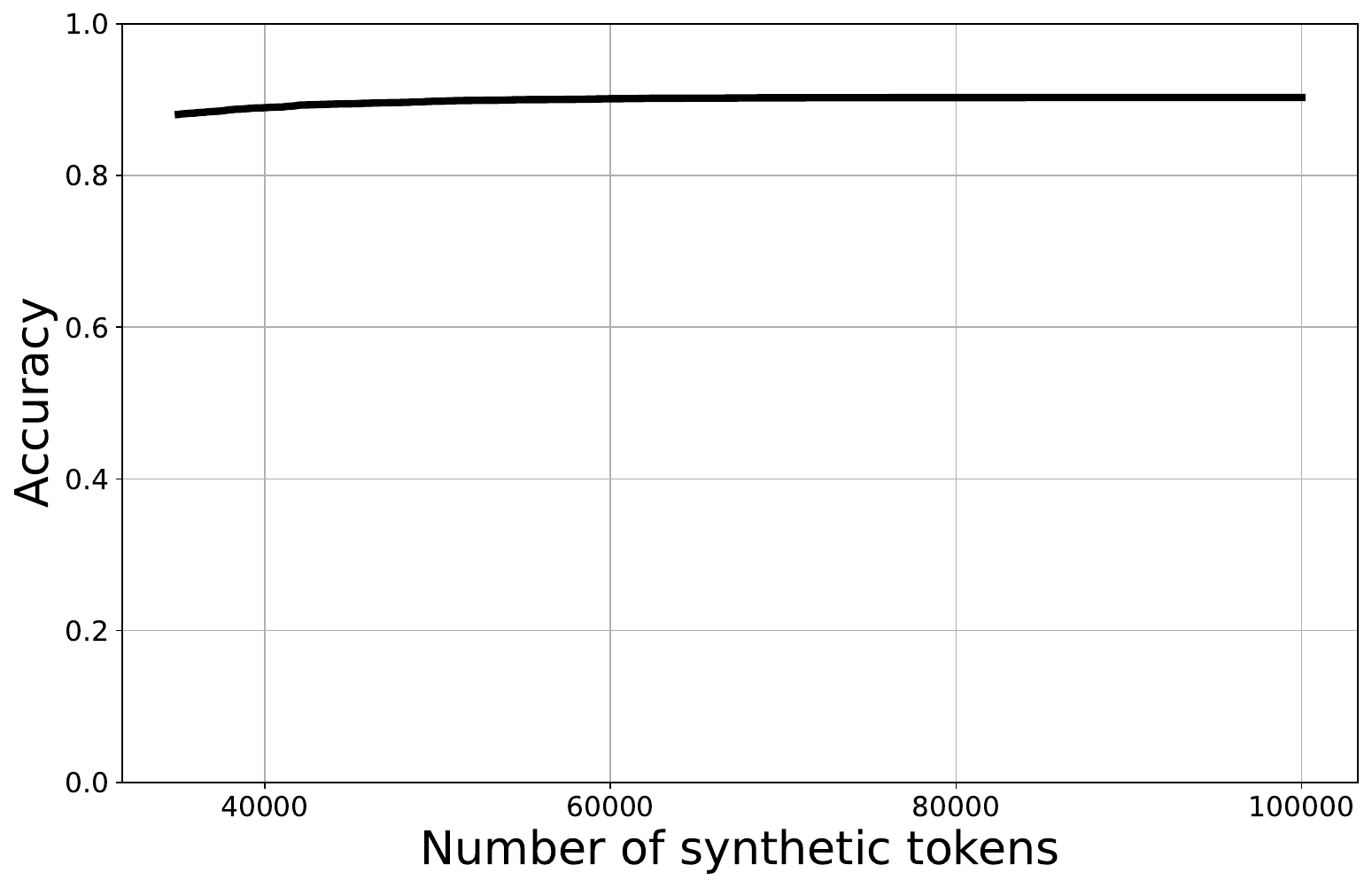}}
\caption{
Accuracy $\mathsf{Acc}(\bM_t)$ with respect to time $t$, for $V = 100$ and $p=0.03$. The mixture-of-exponentials functional form in \eqref{eqn:moe} leads to three distinct regimes.}
\label{fig:Acc}
\end{figure}
We now provide a qualitative description of the mixture-of-exponentials shape.
As shown in Appendix \ref{sec:appendix-proof}, this shape comprises three distinct phases: a fast growth phase, a slower growth phase, and a plateau phase.
We show the existence of two distinct times, $0 < t_1 < t_2$, such that
\begin{align*}
    \mathsf{Acc}(\bM_{T}) = 
    \begin{cases}
    \Theta\left (p+ t \right), \quad &\text{ for } 0\leq t\leq t_1,\\
    \Theta(\log t ), \quad &\text{ for } t_1\leq t\leq t_2,\\
    \Theta(1), \quad &\text{ for } t\geq t_2,
    \end{cases}
\end{align*}
where we use a convenient change of variable $T = t V(V-1)$.
The choice of $\log t$ in the second phase is not necessarily canonical---the bound holds for any well-behaved monotone increasing concave function as a replacement for $\log t$.
We use this representation for two reasons: first, it aligns with performance observed in our EntiGraph CPT numerical results; second, it reflects the gradual slowdown in growth. Figure \ref{fig:Acc} illustrates the three phases in a simulation of the toy model with $p = 0.03$.

\label{sec:appendix-curve-fit}
To perform curve fitting using the mixture-of-exponentials formula, we approximate the infinite sum with three terms:
\begin{align*}
    \mathsf{Acc}(\bM_t) \sim  p+ C \left(1-\sum_{k=1}^\infty \mu(k) \left(1-a_k \right)^t \right).
\end{align*}
We fit the empirical observations to the formula
\[
y(x) = a - b_1r_1^x - b_2r_2^x - b_3r_3^x,
\]
where $x$ is the EntiGraph token count (in millions) and $y(x)$ is the QuALITY QA accuracy.
We use the non-linear least squares method implemented by \cite{2020SciPy-NMeth}.
This procedure yields the fitted formula
\[
y(x) = 64.5456 - 13.8352\times(0.9989)^x - 8.4705\times(0.8961)^x - 3.932\times(0.0546)^x.
\]
We provide the implementation in our code release.

\newpage
\section{Synthetic data generation prompts}
\label{sec:synthetic-data-prompts}
We generate two synthetic corpora: EntiGraph (Appendix \ref{sec:appendix-entigraph-prompts}) and the Rephrase baseline (Appendix \ref{sec:appendix-rephrase-prompts}).
In our experiments, $\Ds$ is a collection of documents $D$, and we apply our synthetic augmentation procedure to each document $D\in\Ds$.
We focus on a single document $D$ for the remainder of this section.

\subsection{EntiGraph prompts}
\label{sec:appendix-entigraph-prompts}
The EntiGraph procedure is described in \S\ref{sec:entigraph-method}.
We recap the steps below.
\paragraph{Step 1: entity extraction} We first extract salient entities from document $D$ using the \texttt{entity\_extraction} operation (Step 1, \S\ref{sec:entigraph-method}).
The complete \texttt{entity\_extraction} prompt is:
\begin{qualitativeBox}
{\footnotesize
\begin{verbatim}
As a knowledge analyzer, your task is to dissect and understand an article provided by the user.
You are required to perform the following steps:
1. Summarize the Article: Provide a concise summary of the entire article, capturing the main 
points and themes.
2. Extract Entities: Identify and list all significant "nouns" or entities mentioned within the 
article. These entities should include but not limited to:
    * People: Any individuals mentioned in the article, using the names or references provided.
    * Places: Both specific locations and abstract spaces relevant to the content.
    * Object: Any concrete object that is referenced by the provided content.
    * Concepts: Any significant abstract ideas or themes that are central to the article's 
      discussion.

Try to exhaust as many entities as possible. Your response should be  structured in a JSON format 
to organize the information effectively. Ensure that the summary is brief yet comprehensive, and 
the list of entities is detailed and accurate.

Here is the format you should use for your response:

{
  "summary":  "<A concise summary of the article>",
  "entities": ["entity1", "entity2", ...]
}
\end{verbatim}
}
\end{qualitativeBox}

\paragraph{Step 2: relation analysis} We next generate diverse descriptions of relations among two or more entities.
For each document $D$, we enumerate all entity pairs and generate a description for each.
The prompt for generating a description relating a pair of entities is:
\begin{qualitativeBox}
{\footnotesize
\begin{verbatim}
You will act as a knowledge analyzer tasked with dissecting an article provided by the user. Your 
role involves two main objectives:
1. Rephrasing Content: The user will identify two specific entities mentioned in the article. You 
   are required to rephrase the content of the article twice:
   * Once, emphasizing the first entity.
   * Again, emphasizing the second entity.
2. Analyzing Interactions: Discuss how the two specified entities interact within the context of 
   the article.
Your responses should provide clear segregation between the rephrased content and the interaction 
analysis. Ensure each section of the output include sufficient context, ideally referencing the
article's title to maintain clarity about the discussion's focus. Here is the format you should 
follow for your response:

### Discussion of <title> in relation to <entity1>
<Rephrased content focusing on the first entity>

### Discussion of <title> in relation to <entity2>
<Rephrased content focusing on the second entity>

### Discussion of Interaction between <entity1> and <entity2> 
    in context of <title>
<Discussion on how the two entities interact within the article>
\end{verbatim}
}
\end{qualitativeBox}
We also generate synthetic data involving three entities, using the prompt below:
\begin{qualitativeBox}
{\footnotesize
\begin{verbatim}
You will act as a knowledge analyzer tasked with dissecting an article provided by the user. Your 
role involves three main objectives:

1. Rephrasing Content: The user will identify three specific entities mentioned in the article.
You are required to rephrase the content of the article three times:
   * Once, emphasizing the first entity.
   * Again, emphasizing the second entity.
   * Lastly, emphasizing the third entity.
2. Analyzing Interactions: Discuss how these three specified entities interact within the context 
   of the article.

Your responses should provide clear segregation between the rephrased content and the interaction 
analysis. Ensure each section of the output include sufficient context, ideally referencing the
article's title to maintain clarity about the discussion's focus. Here is the format you should 
follow for your response:

### Discussion of <title> in relation to <entity1>
<Rephrased content focusing on the first entity>

### Discussion of <title> in relation to <entity2>
<Rephrased content focusing on the second entity>

### Discussion of <title> in relation to <entity3>
<Rephrased content focusing on the third entity>

### Discussion of Interaction between <entity1>, <entity2> and
    <entity3> in context of <title>
<Discussion on how the three entities interact within the article>
\end{verbatim}
}
\end{qualitativeBox}

\subsection{Rephrase prompts}
\label{sec:appendix-rephrase-prompts}
For the rephrase corpus, we adapt the prompt from \cite{wrap} to our setting of books and articles.
We use four rephrase styles:

\textbf{Easy rephrase:}
\begin{qualitativeBox}
{\footnotesize
\begin{verbatim}
You are an assistant to help read a article and then rephrase it in simpler terms. The user will 
provide you with an article with title, year, content. You need to generate a paraphrase of the
same article using a very small vocabulary and extremely simple sentences that a toddler will 
understand. Remember to keep the meaning and every content of the article intact, including the
title, year, etc.
\end{verbatim}
}
\end{qualitativeBox}

\textbf{Medium rephrase:}
\begin{qualitativeBox}
{\footnotesize
\begin{verbatim}
You are an assistant to help read a article and then rephrase it in different terms. The user will 
provide you with an article with title, year, content. You need to generate a paraphrase of the
same article using diverse and high quality English language as in sentences on Wikipedia.
Remember to keep the meaning and every content of the article intact, including the title, year, 
etc.
\end{verbatim}
}
\end{qualitativeBox}

\textbf{Hard rephrase:}
\begin{qualitativeBox}
{\footnotesize
\begin{verbatim}
You are an assistant to help read a article and then rephrase it in more sophisticated terms. The 
user will provide you with an article with title, year, content. You need to generate a paraphrase 
of the same article using very terse and abstruse language that only an erudite scholar will 
understand. Remember to keep the meaning and every content of the article intact, including the 
title, year, etc.
\end{verbatim}
}
\end{qualitativeBox}

\section{Additional evaluation details of main experiments}
\label{sec:app_additional_eval_details}

\subsection{QuALITY QA question set}
\label{sec:appendix-qa-eval-detail}
We provide additional details on evaluation using the QuALITY QA test queries.
Throughout the closed-book QA experiments, we use the following fixed 5-shot prompt:
\begin{qualitativeBox}
{\footnotesize
\begin{verbatim}
## Example 1
### Question
In the context of "Les Misérables", written by Victor Hugo in 1862, what is the main setting of
the novel? There is only one correct choice.
### Choices
A. London
B. Madrid
C. Paris
D. Rome
### Thought Process and Answer
Thought process: "Les Misérables" is primarily set in Paris, making C the correct choice. 
London, Madrid, and Rome are significant cities in other literary works but not in Victor 
Hugo's "Les Misérables". There is only one correct choice.
Answer: C.

## Example 2
### Question
In the context of "Brave New World", written by Aldous Huxley in 1932, what substance is 
widely used in the society to control citizens' happiness? There is only one correct choice.
### Choices
A. Gold
B. Soma
C. Silver
D. Iron
### Thought Process and Answer
Thought process: In Aldous Huxley's "Brave New World," Soma is used as a means to maintain 
social control by ensuring citizens' happiness, making B the correct choice. Gold, Silver, and 
Iron are not the substances used for this purpose in the book.
Answer: B.

## Example 3
### Question
In the context of "Romeo and Juliet", written by William Shakespeare in the early 1590s, what 
are the names of the two feuding families? There is only one correct choice.
Choices:
A. Montague and Capulet
B. Bennet and Darcy
C. Linton and Earnshaw
D. Bloom and Dedalus
### Thought Process and Answer
Thought process: In William Shakespeare's "Romeo and Juliet," the two feuding families are the 
Montagues and the Capulets, making A the correct choice. The Bennets and Darcys are in "Pride 
and Prejudice", the Lintons and Earnshaws in "Wuthering Heights", and Bloom and Dedalus in 
"Ulysses".
Answer: A.\end{verbatim}
}
\end{qualitativeBox}

\begin{qualitativeBox}
{\footnotesize
\begin{verbatim}

## Example 4
### Question
In the context of "1984", written by George Orwell in 1949, what is the name of the 
totalitarian leader? There is only one correct choice.
### Choices
A. Big Brother
B. O'Brien
C. Winston Smith
D. Emmanuel Goldstein
### Thought Process and Answer
Thought process: In George Orwell's "1984," the totalitarian leader is known as Big Brother, 
making A the correct choice. O'Brien is a character in the novel, Winston Smith is the 
protagonist, and Emmanuel Goldstein is a rebel leader.
Answer: A.

## Example 5
### Question
In the context of "Moby-Dick", written by Herman Melville in 1851, what is the name of the 
ship's captain obsessed with hunting the titular whale? There is only one correct choice.
### Choices
A. Captain Hook
B. Captain Nemo
C. Captain Flint
D. Captain Ahab
### Thought Process and Answer
Thought process: In Herman Melville's "Moby-Dick," the ship's captain obsessed with hunting 
the whale is Captain Ahab, making D the correct choice. Captain Nemo is in "Twenty Thousand 
Leagues Under the Sea", Captain Flint in "Treasure Island", and Captain Hook in "Peter Pan".
Answer: D.

## Example 6
\end{verbatim}
}
\end{qualitativeBox}
If the model correctly follows the few-shot prompt format, its last two characters should be ``\texttt{A.}'', ``\texttt{B.}'', ``\texttt{C.}'', or ``\texttt{D.}''.
However, the model sometimes fails to follow the few-shot prompting format, particularly the continually pretrained model.
In all our evaluations, we sample 64 responses and select only those that parse to the correct format.
From these valid attempts, we randomly select one as the final answer.
This is \emph{different} from majority voting in self-consistency prompting \citep{wang2023selfconsistency}.

\subsection{Closed-book summarization}
\label{sec:appendix-eval-summary-detail}

\paragraph{Automated evaluation metric} We use a three-stage evaluation procedure:
(i) We use GPT-4\footnote{Specifically, we use the \texttt{gpt-4-turbo} model as of Aug. 19, 2024.} to break the summary into atomic claims, similar to \cite{min2023factscorefinegrainedatomicevaluation};
(ii) We provide both the list of claims and the source article to a judge model (also GPT-4).
The judge determines whether each claim is true or false based on the source article.
If the claim is true, we further ask the model to determine whether the claim is salient (contributes to the main message) or cosmetic (factual details that do not aid understanding).
(iii) For each summary, we obtain the number of false and salient claims and normalize by the corresponding count from the human summary.
We report the average of these normalized metrics across the \quality~corpus articles in Figure \ref{fig:exp-summaryeval}.

\paragraph{Prompts to generate summaries} For summarization evaluation with EntiGraph Instruct and Raw Instruct, we use the following two prompts to obtain summaries of increasing length.
\begin{table}[ht]
\centering
\small
\begin{tabular}{l p{12.5cm}}
\toprule
\ding{228} & \textbf{Short prompt:} \texttt{Summarize the article \{article title\} by \{author name\} for me.} \\
\midrule
&  \texttt{Give a short summary of ``Cosmic Yo-Yo'' by Ross Rocklynne.} \\
\midrule[\heavyrulewidth]
\ding{228} &\textbf{Long prompt:} \texttt{Write an extremely long and detailed article regarding the book \{article title\} by \{author name\}.}  \\
\midrule
&  \texttt{Write an extremely long and detailed article regarding the book ``Cosmic Yo-Yo'' by Ross Rocklynne.} \\
\bottomrule
\end{tabular}
\caption{Summarization prompt for EntiGraph Instruct, Raw Instruct, and Reprhase Instruct.}
\label{tbl:appendix-summary-prompts}
\end{table}

We show three examples of summarization outputs below.
For each example, we first present the human summary for context, then present the short summary from the two summarizers.

\paragraph{Example 1} The first example is ``Cosmic Yo-Yo'' by Ross Rocklynne.
\begin{qualitativeBox}
\textbf{Human summary:}
Bob Parker, the President of Interplanetary Hauling \& Moving Co., sells asteroids to wealthy people on earth. Clients ask for asteroids with size parameters and specifications, and Bob finds them in space and hauls them to earth. His company is almost bankrupt because a rival company, Saylor \& Saylor, stole his idea and now offers the same services. Bob receives mail from Mr. Andrew S. Burnside with a request for an asteroid that he would like to use in an upcoming wedding.
Bob and his partner Queazy set out to find the perfect asteroid for Mr. Burnside, although they know it’s a longshot. Fairly quickly, they find one that looks perfect. The men land on the asteroid, and Bob deploys his atomic-whirl spectroscope to test it. Suddenly, a beautiful woman interrupts him and demands that they leave the asteroid. She pulls out her spasticizer gun before telling them that they can have it in a month after she’s gone. Bob explains that they are desperate, but the girl retorts that her fate is worse than death if she leaves.

Suddenly, the Saylor brothers’ ship appears, and Bob tells the girl that they have to fight this enemy together. Wally and Billy Saylor, along with three other men, jump out of the ship. Bob tells them that Mr. Burnside has ordered this asteroid, and the Saylor 
brothers say that they received the same order. Bob quickly grabs the girl’s spasticizer while Queazy throws his body at Billy. However, Wally manages to shoot the gun out of Bob’s hand and attack him. Bob is knocked unconscious in the scuffle. 

When Bob wakes up, he is completely alone, floating in space. He panics because he has very little oxygen left. Finally, he hears Queazy’s voice explaining that the girl used her ship’s technology to find them both. The mystery girl introduces herself as Starre Lowenthal, the granddaughter of Mr. Burnside. She concedes that this entire mission was fake. She told her grandfather that she would only marry her fiance Mac if he could get this particular asteroid, and then she made plans to conquer and protect the asteroid so it could not be supplied for the wedding. 

Bob is confident that they can reach the Saylor brothers before they bring the asteroid back to earth, but his plan does nothing to protect Starre from marrying a man she doesn’t love. She agrees to help Bob and Queazy. Within five days, Bob realizes he is in love with Starre. 

Starre compares her small ship to a yo-yo, and Bob gets an idea - they will use Starre’s ship like a yo-yo to retrieve the asteroid from the Saylor brothers. Once the team catches up to the Saylor brothers, Bob flings Starre’s ship at the asteroid several times, and Wally calls them to tell them that they might die as a result of the damage their ship has sustained. Bob makes it clear that they have no intention of stopping, and the Saylor brothers release the asteroid.     
\end{qualitativeBox}

\begin{qualitativeBox}
\textbf{EntiGraph Instruct with short prompt:} Cosmic Yo-Yo is a science fiction story by Ross Rocklynne about Bob Parker and Queazy, two interplanetary haulers who are hired to deliver an asteroid to Earth for a wealthy client's backyard wedding. However, they face competition from their rivals, the Saylor brothers, who also want to deliver the asteroid. The story follows their race to deliver the asteroid before the deadline, using a yo-yo-like device to maneuver the asteroid and outsmart their rivals. The story is filled with action, adventure, and a touch of romance.
\end{qualitativeBox}

\begin{qualitativeBox}
\textbf{Raw Instruct with short prompt:} Cosmic Yo-Yo by Rocklynne, Ross is a science fiction novel that follows the story of a young woman named Susan who discovers a mysterious device that can transport her through time and space. She uses the device to travel to different planets and meet different alien species, but soon realizes that her actions have consequences and she must find a way to return home. The novel explores themes of identity, responsibility, and the consequences of our actions.
\end{qualitativeBox}

\begin{table}[ht]
    \centering
    \small
    \begin{tabular}{l p{13.5cm}}
    \toprule
    \ding{228} & \textbf{Explicit article reference instruction:}  Summarize the article ``Defining Decay Down'' by ``David Plotz'' for me. \\
    \midrule
    & The article "Defining Decay Down" by David Plotz discusses how the dental industry has changed over the years. In the past, dentists were primarily focused on treating cavities and other dental issues, but as the prevalence of cavities decreased due to the widespread use of fluoride, the industry had to adapt. Dentists began to focus on cosmetic dentistry, offering services like teeth whitening and orthodontics to appeal to the vanity of aging baby boomers. The article also discusses how dentists have become more customer-friendly, with amenities like virtual reality glasses and massage therapy. Despite the shift towards elective procedures, the dental industry has continued to thrive, with dentists earning higher salaries than ever before. \\
    \midrule[\heavyrulewidth]
    \ding{228} &\textbf{Implicit article reference:} How has the dentistry practice in United States has changed?  \\
    \midrule
    &  The dentistry practice in the United States has undergone significant changes in recent years. Here are some of the ways in which it has changed:

    1. Focus on cosmetic dentistry: In recent years, there has been a growing demand for cosmetic dental procedures such as teeth whitening, orthodontics, and veneers. Dentists have responded to this demand by offering more cosmetic services.

    2. Use of technology: Dentists are now using advanced technology such as intra-oral cameras, 3D imaging, and laser technology to diagnose and treat dental problems. This technology has made dental procedures more accurate and efficient.

    3. Emphasis on preventive care: Dentists are now focusing more on preventive care to help patients avoid dental problems. This includes regular check-ups, cleanings, and fluoride treatments.

    4. Increased competition: With the rise of dental schools and the influx of foreign-trained dentists, the dental industry has become more competitive. Dentists are now offering more services and amenities to attract and retain patients.

    5. Focus on patient experience: Dentists are now focusing on providing a positive patient experience. This includes offering amenities such as virtual reality glasses, massage therapy, and entertainment during procedures.

    Overall, the dentistry practice in the United States has become more patient-centered, technology-driven, and focused on preventive care and cosmetic procedures. \\
    \midrule[\heavyrulewidth]
    \ding{228} &\textbf{Cross article instruction:} Compare David Plotz's commentary style between how he analyze American dentistry and how he discuss the movie Fight Club?  \\
    \midrule
    &  David Plotz's commentary style is different when he analyzes American dentistry and when he discusses the movie Fight Club. In his analysis of American dentistry, Plotz uses a critical and investigative approach, highlighting the over-commercialization and unnecessary procedures in the industry. He uses statistics and expert opinions to support his arguments and provides practical advice for readers to avoid unnecessary dental procedures.

    On the other hand, when discussing the movie Fight Club, Plotz adopts a more satirical and humorous tone. He uses pop culture references and witty commentary to analyze the themes and messages of the movie. He also provides a critical analysis of the film's characters and their motivations, while also acknowledging the movie's cultural impact.

    Overall, Plotz's commentary style is more serious and investigative when analyzing American dentistry, while his commentary on Fight Club is more satirical and humorous. \\
    \bottomrule
    \end{tabular}
    \caption{Complete instruction following example used in Table \ref{tbl:exp-instruct-example} from Section \ref{sec:exp-instruct-result}.}
    \label{tbl:appendix-instruct-example}

\end{table}

\clearpage

\chapter{Supplementary materials for Chapter~\ref{chap:sbp}}
\section{Additional details on synthetic bootstrapped pretraining}

\subsection{SBP implementation details}
\label{sec:experiment-details}

We now describe the implementation details of SBP outlined in \S\ref{sec:method-description}.

\paragraph{Nearest neighbor pairing} Recall from \S\ref{sec:experiment-setup} that we work with 3B and 6B-parameter transformer architectures and pretraining datasets at the $\|\Dpre\|=$10B and $\|\Dpre\|=$50B scales.
To enable efficient ANN search at pretraining scale, we embed documents from $\Dpre$ as 1{,}024-dimensional vectors using Qwen3-Embedding-0.6B.
We then use ScaNN \citep{guo2020accelerating,sun2023soar} with 8-bit quantization to perform efficient similarity search.
We adopt asymmetric sharding for keys and value vectors.
For each value vector, we build a ScaNN search tree with $\sqrt{N}$ leaves, where $N$ is the number of vectors in each value shard.
To distribute the key shards across each search tree, we employ a ``salting'' strategy: we create multiple copies of the ScaNN searcher and assign one key shard to each salted copy (Figure \ref{fig:scann-flow}).
This design enables us to perform a top-200 nearest neighbor search over $|\Dpre|=$60M documents within 155M CPU hours.

\begin{figure}[ht]
\centering
\includegraphics[width=\textwidth]{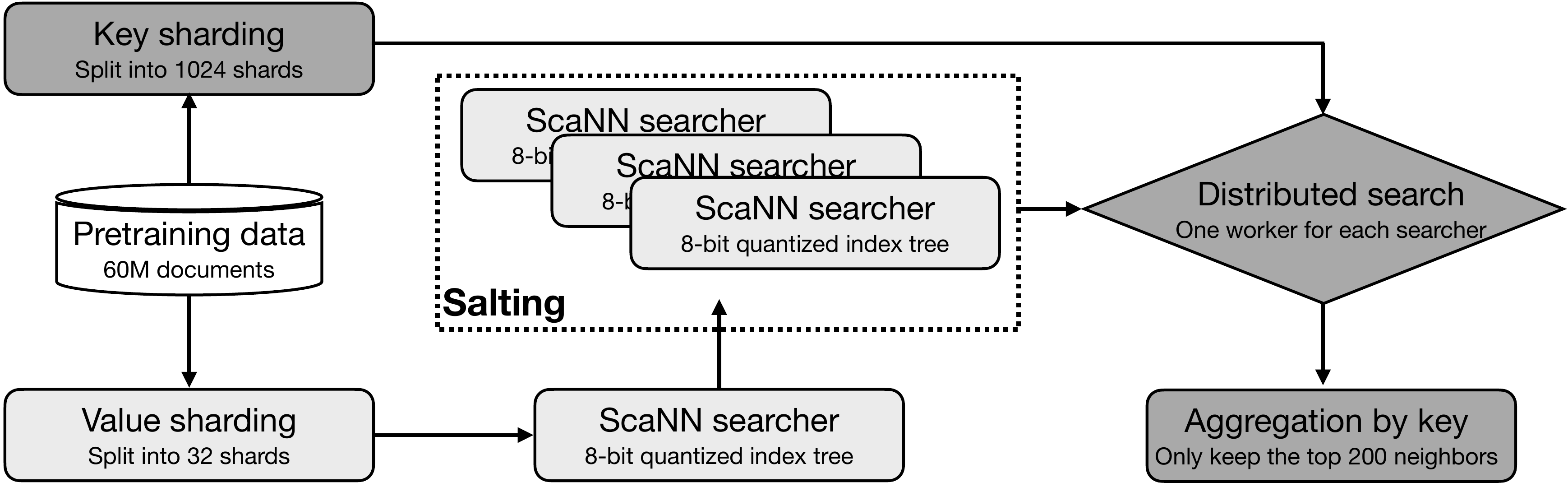}
\caption{ScaNN system design for efficient distributed search.}
\label{fig:scann-flow}
\end{figure}

At all scales, after obtaining the top 200 neighbors for each sample, we select pairs whose similarity score exceeds 0.75.
This threshold yields a tractable synthesizer-tuning dataset $\Dst$.
To assess how different thresholds affect data quality, Figure \ref{fig:appendix-distance-eda} shows the fraction of relevant documents at each similarity bin using the metric defined in \S\ref{sec:analysis-of-synthetic-data}.
We find that higher similarity scores yield more relevant pairs but also more duplicates.
Finally, we eliminate near-duplicates using rule-based filtering.
The deduplication process first normalizes text by removing punctuation, converting to lowercase, and eliminating numbers, then tokenizes using SentencePiece.
We then generate ``shingles'' using 13-token sliding windows within $\docone$.
We discard training pairs if any shingle from $\docone$ appears in $\doctwo$.

\begin{figure}[ht]
\subfigure[\label{fig:appendix-similarity-hist} Similarity histogram]{\includegraphics[width=0.33\textwidth]{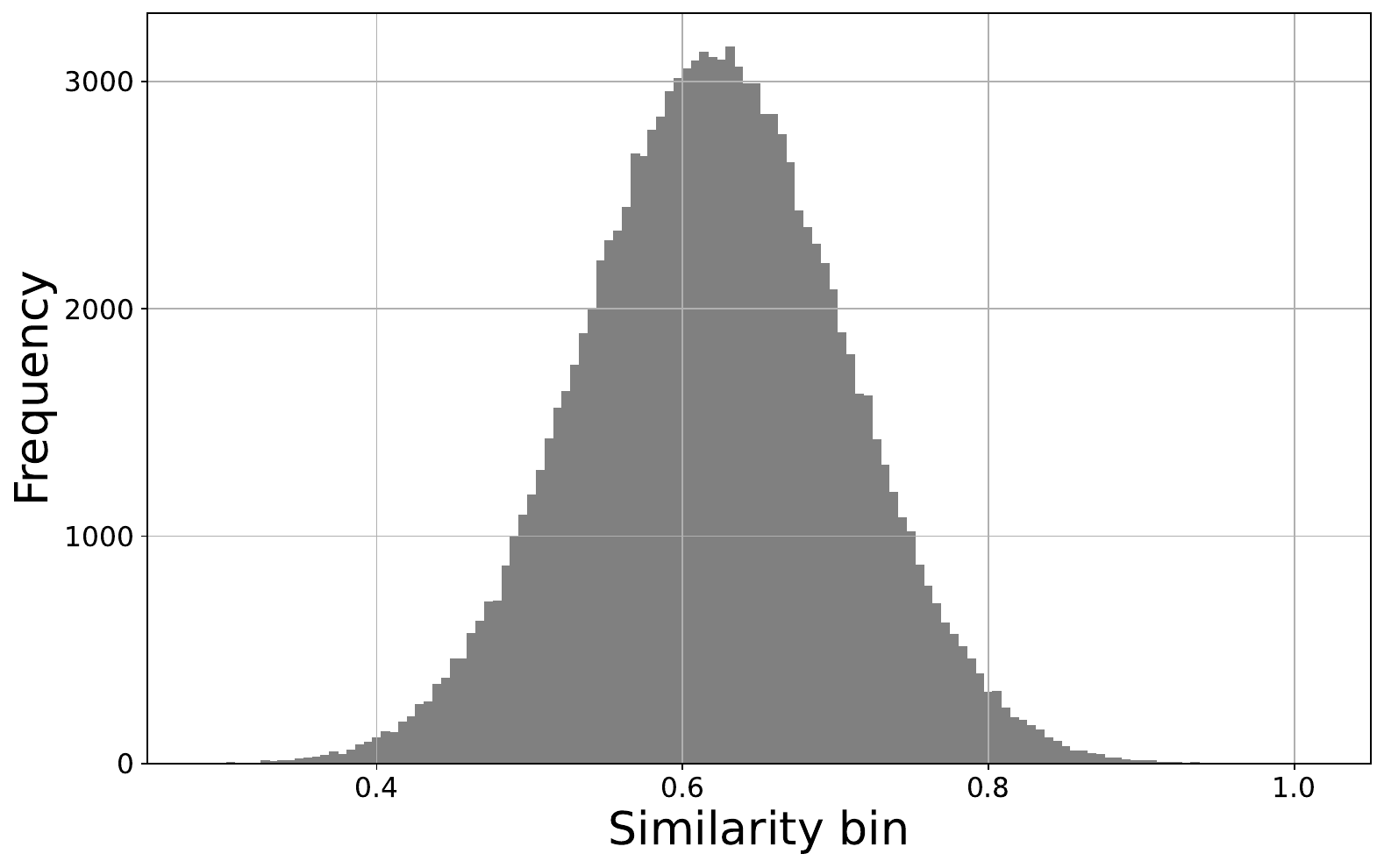}}
\subfigure[\label{fig:appendix-similarity-duplicate} Duplicate rate]{\includegraphics[width=0.33\textwidth]{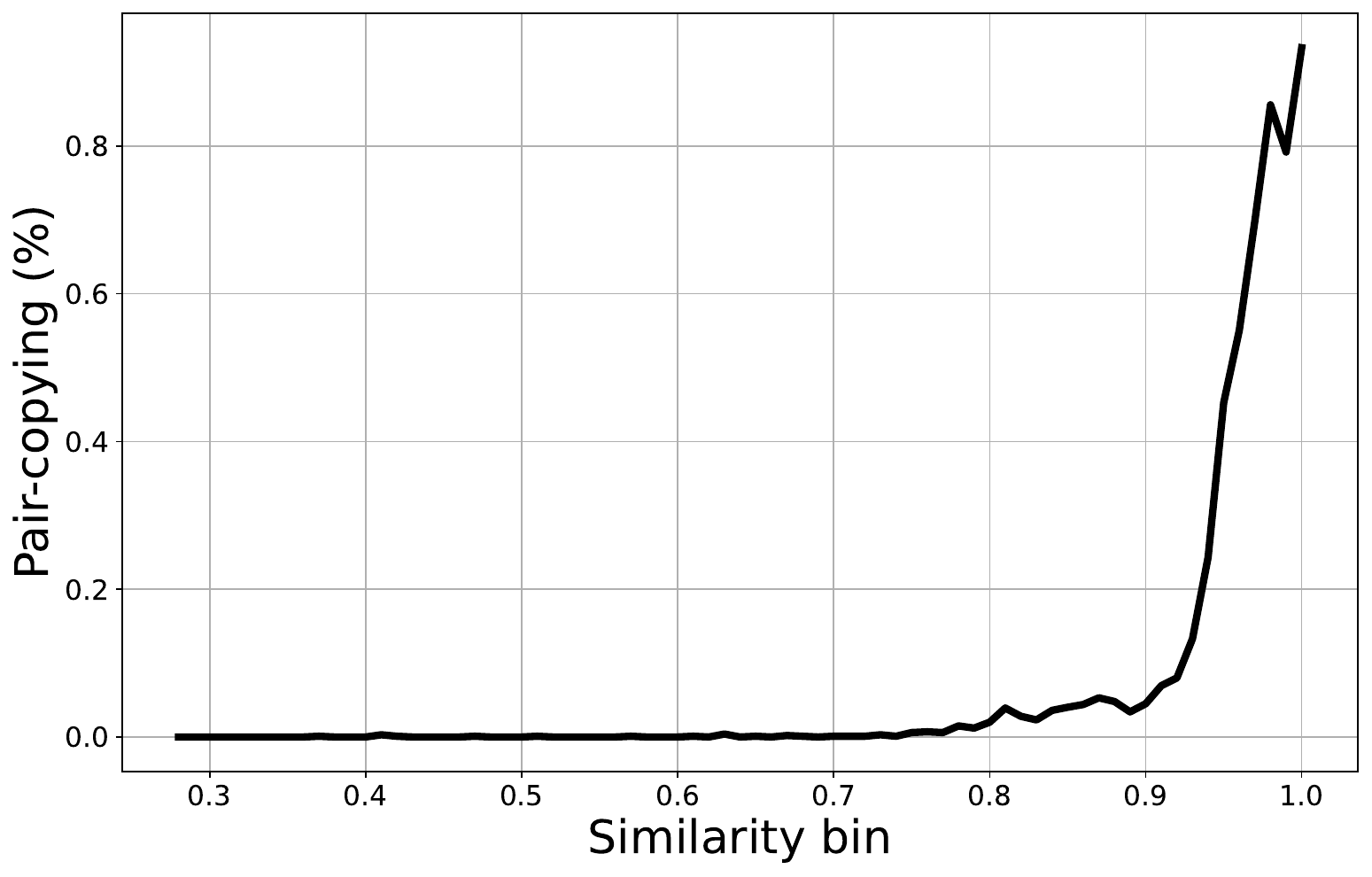}}
\subfigure[\label{fig:appendix-similarity-relevance} Relevance rate]{\includegraphics[width=0.33\textwidth]{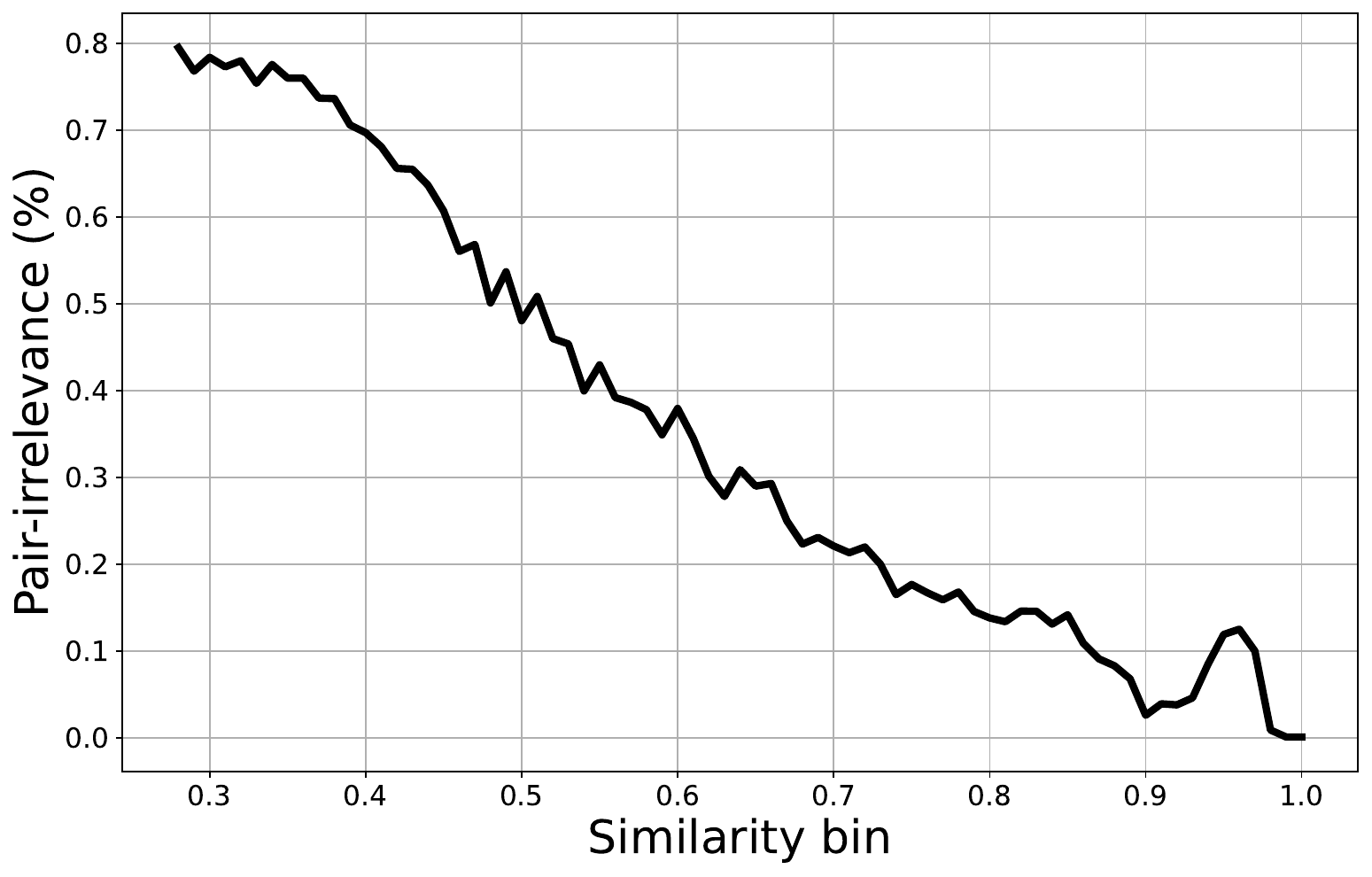}}
\caption{Analysis of paired data at 200B-scale.
Figure \ref{fig:appendix-similarity-hist}: a histogram of 100K subsampled pairs grouped by their similarity score.
Figure \ref{fig:appendix-similarity-duplicate}: the fraction of duplicate pairs when we subsample 1K pairs around a specific similarity score.
Figure \ref{fig:appendix-similarity-relevance}: same as \ref{fig:appendix-similarity-duplicate} but showing the fraction of relevant documents.
}
\label{fig:appendix-distance-eda}
\end{figure}

\paragraph{Synthesizer-tuning} After collecting the cleaned pair data $\Dst$, we perform synthesizer-tuning with the objective in \eqref{eqn:synthesizer-objective-lm}.
We initialize the 3B-parameter model at the baseline checkpoint and finetune with a constant learning rate of 5e-6 and a batch size of 16M tokens per step.
We initially attempted cosine decay with a larger learning rate but found that a small, constant learning rate produces higher-quality generated text.
We measure the Pair-novelty score (defined in \S\ref{sec:analysis-of-synthetic-data}) across different synthesizer-tuning checkpoints and find that longer training improves Pair-novelty.

\paragraph{Synthesis at scale} Finally, we perform the hierarchical sampling procedure defined in \S\ref{sec:method-description} with temperature 1.0 and \texttt{top\_p} threshold 0.9.
We apply rule-based filtering to remove synthesized documents containing repeated occurrences of 13-token shingles, effectively eliminating texts with repetition failures.
We use vLLM \citep{kwon2023efficient} and achieve a throughput of 8.3K tokens per B200 second.
This amounts to 2.5K B200 hours for the 200B-scale synthesis, 4.2K B200 hours for the 1T-scale (3B) synthesis, and 8.4K B200 hours for the 1T-scale (6B).

\subsection{Ablation on data mixture ratio}
\label{sec:ablation-mixture-ratio}

When performing joint training on a mixture of real and synthesized documents, a natural question is: what fraction of synthesized documents should we include?
In \S\ref{sec:experiment-results}, we used $\|\Spre\|=$75B for the 200B-scale experiment, $\|\Spre\|=$125B for the 1T-scale (3B) experiment, and $\|\Spre\|=$250B for the 1T-scale (6B) experiment.
We now present ablation experiments for this design choice.

\paragraph{200B-scale} At this smaller scale, we perform a comprehensive sweep over five values of $\|\Spre\|\in \{$0B, 25B, 50B, 75B, 100B$\}$.
As shown in Figure \ref{fig:mixture-sweep-small}, different benchmarks exhibit varying behavior as synthetic data increases: perplexity (OpenWebText2 and LAMBADA) decreases monotonically, while most QA benchmarks peak around $\|\Spre\|$ = 75B.

\begin{figure}[ht]
\centering
\includegraphics[width=\textwidth]{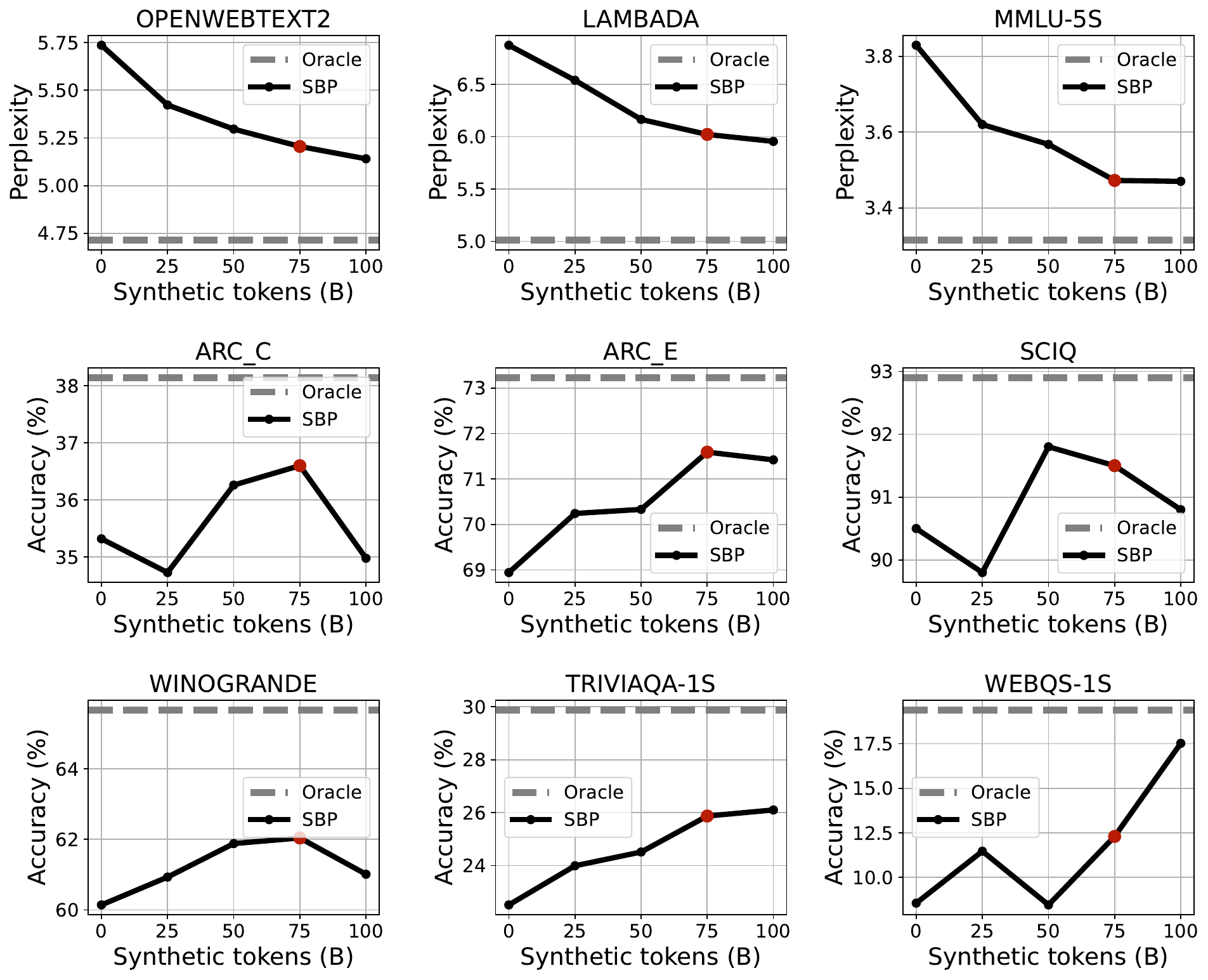}
\caption{SBP performance with varying synthetic tokens at 200B-scale.}
\label{fig:mixture-sweep-small}
\end{figure}

\paragraph{1T-scale (3B)} At the 1T-scale, both data synthesis and joint pretraining become significantly more expensive.
We therefore evaluate SBP at three values: $\|\Spre\|\in\{$0B, 125B, 250B$\}$.
As shown in Figure \ref{fig:mixture-sweep-large}, $\|\Spre\|=$125B produces the best-performing model across all benchmarks except LAMBADA perplexity.

\begin{figure}[ht]
\centering
\includegraphics[width=\textwidth]{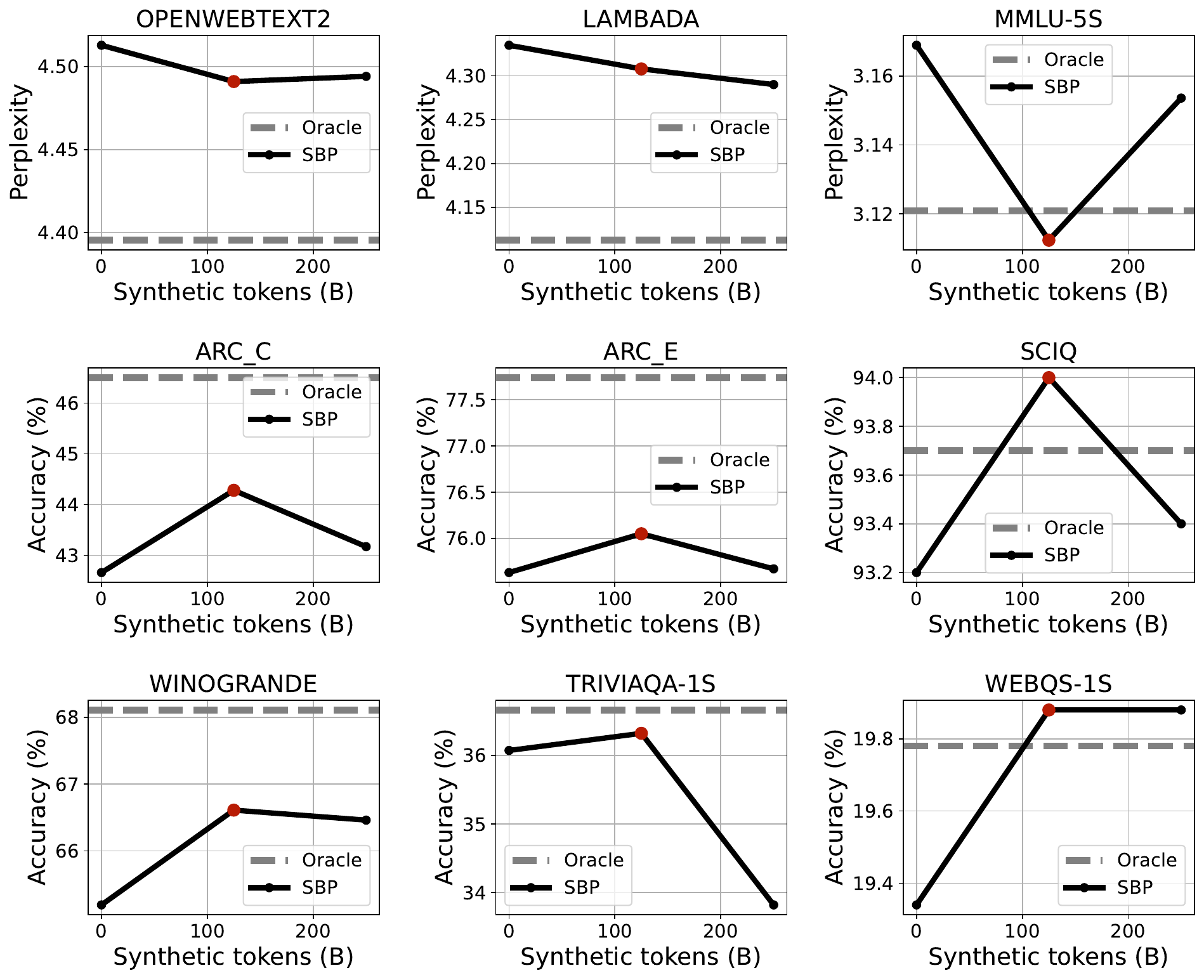}
\caption{SBP performance with varying synthetic tokens at 1T-scale (3B).}
\label{fig:mixture-sweep-large}
\end{figure}

\paragraph{1T-scale (6B)} We also sweep the mixture ratio for the 6B model at the 1T-scale, evaluating $\|\Spre\|\in\{$0B, 125B, 250B$\}$.
As shown in Figure \ref{fig:mixture-sweep-large-6b}, the optimal amount of synthetic data is around 250B---higher than the optimal 125B observed for the 3B model.

\begin{figure}[ht]
\centering
\includegraphics[width=\textwidth]{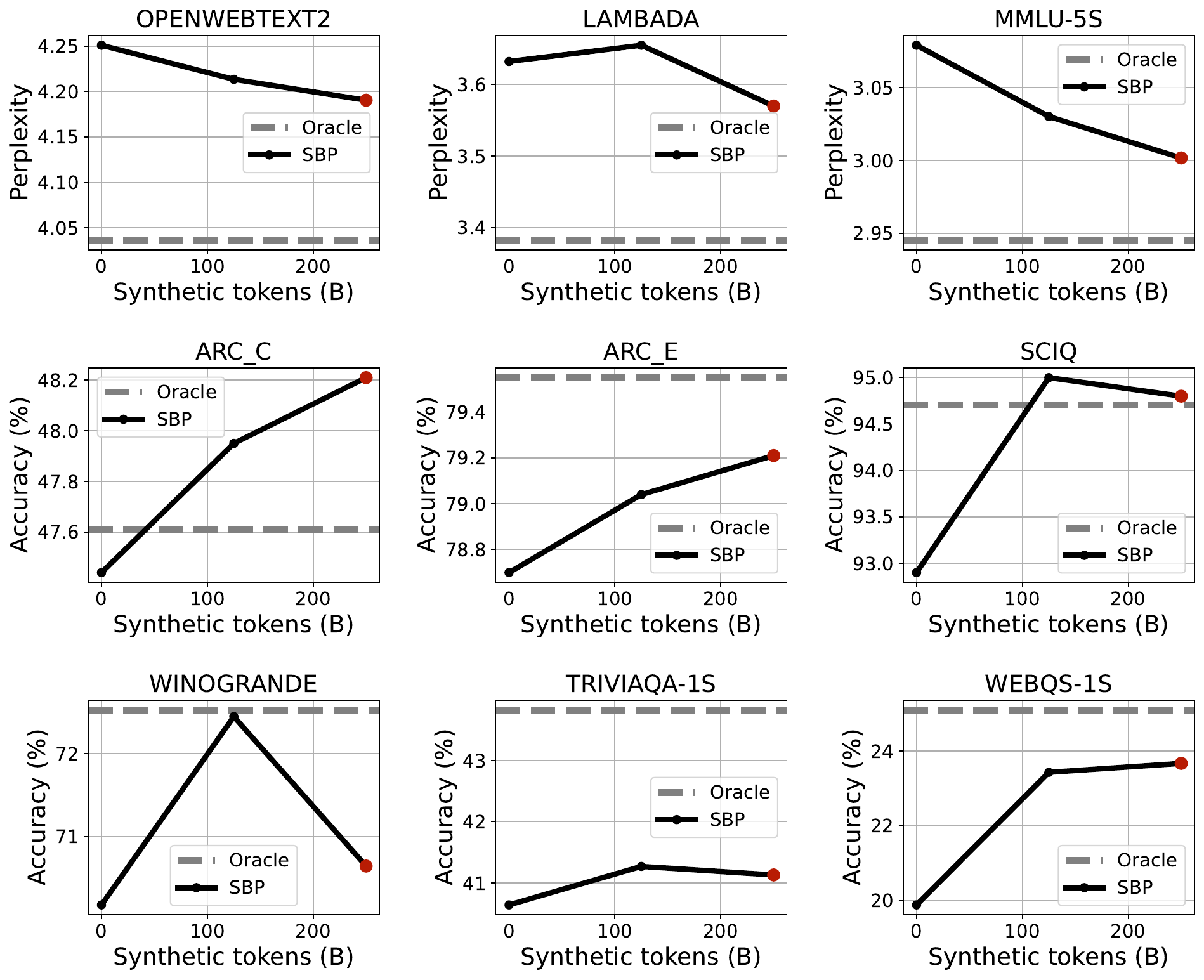}
\caption{SBP performance with varying synthetic tokens for the 6B model at 1T-scale.}
\label{fig:mixture-sweep-large-6b}
\end{figure}

\paragraph{Discussion} A general pattern emerges: the best-performing model results from pretraining on a mixture of real and synthetic data.
Furthermore, the optimal synthetic data ratio increases with model size (from approximately 12.5\% for 3B to approximately 25\% for 6B).
Real internet data has higher quality and merits more repetition, but because repetition yields diminishing returns, synthetic data provides an additional source of signal.
In contrast, distillation-based research typically finds that training purely on synthetic data yields higher training efficiency.
However, this finding is obscured because such models eventually converge to the capability of the teacher LM.
This contrast reveals that SBP does not generate a compressed and denoised representation of knowledge.
Instead, it provides an additional source of improvement that real data alone cannot capture.

\subsection{Random pairs and embedding analysis}
\label{sec:random-pairs-ablation}

To verify that SBP relies on learning specific inter-document correlations rather than generic data augmentation, we train the synthesizer on random document pairs as an ablation.
We measure the semantic similarity between the seed document $\docone$ and the target document (either $\doctwo$ or the synthesized output) using Qwen3-Embedding-0.6B.

\begin{table}[ht]
\centering
\caption{Embedding similarity statistics. ``Paired documents'' refers to the SBP training pairs found by nearest neighbor search. ``Random documents'' refers to randomly paired documents. ``Generated documents (SBP)'' refers to the synthetic data generated by the SBP model at 200B-scale (3B). ``Generated documents (Random)'' refers to the synthetic data generated by the model trained on random pairs. All comparisons are based on the 10B dataset.}
\label{tab:embedding-similarity}
\begin{tabular}{lcccc}
\toprule
\textbf{Statistic} & \textbf{Paired docs} & \textbf{Random docs} & \textbf{Generated (SBP)} & \textbf{Generated (Random)} \\
\midrule
Mean & 0.79 & 0.15 & 0.66 & 0.32 \\
\bottomrule
\end{tabular}
\end{table}

Table \ref{tab:embedding-similarity} shows that paired documents in SBP exhibit high semantic similarity (0.79), whereas random documents have minimal similarity (0.15).
Crucially, documents generated by the SBP synthesizer maintain high relevance (0.66) to the seed document.
In contrast, the model trained on random pairs produces outputs with significantly lower relevance (0.32).
This confirms that the SBP synthesizer learns to preserve semantic relevance from the training pairs---a property absent when training on random associations.

\section{Additional analysis of synthesized samples}
\label{sec:additional-analysis-of-synthesized-samples}

\subsection{Analyzing concepts in documents}
\label{sec:concept-analysis}

We next examine the intermediate mechanisms underlying the document synthesis process.
Specifically, we classify the hypothesized concepts inferred from real documents (Table~\ref{tab:doc_connections}) along two dimensions: \textbf{concept domains}, the broad subject areas a concept belongs to (e.g., science, psychology, health, culture), and \textbf{concept types}, the abstract role or nature of the concept itself (e.g., theory, method, comparison, symbol).

\begin{table}[ht]
\centering 
\caption{Categorize extracted concepts into domains.}
\small
\renewcommand{\arraystretch}{1.2} %
\begin{tabular}{p{0.2\linewidth}p{0.7\linewidth}}
\toprule
\textbf{Concept Domains} & \textbf{Examples} \\
\midrule
Culture (38.74\%) & Inter-community conflict in Nigeria, Family-based immigration policy, Reactions to Horrid Henry books, Interracial dating and bias \\
\midrule
Health (11.89\%) & Cosmetic dental appliance, Colistin toxicity in infections, Hair health tips, Portable/home medical diagnostics, Vitamin D and pregnancy outcomes \\
\midrule
Technology (9.91\%) & Recovering deleted phone data, Video editing app review, Flash platform pros and cons, HTML 2.0 draft process, Email attachment processing speed \\
\midrule
Politics (3.69\%) & Iran nuclear negotiations, Student loans policy reform, Democratic primary candidate choice, Catalan independence aftermath \\
\midrule
Psychology (3.42\%) & Differences in personality disorders, Exploring the strange in daily life, Aging and nostalgia, Toxic relationship breakup, Psychology research paper topics \\
\bottomrule
\end{tabular}
\label{tab:concept_analysis_domain}  
\end{table}

The distributions in Table~\ref{tab:concept_analysis_domain} and~\ref{tab:concept_analysis_types} reveal the multidimensional nature of the knowledge space.
The domains span macro-level sociocultural phenomena---Culture topics range from inter-community conflict in Nigeria to immigration policy and interracial dating---alongside micro-level issues of individual health and wellbeing.
The typological classification reveals not only subject matter but also modes of conceptual engagement: Methods comprise formalized procedures (multidimensional poverty measurement, commercial real estate appraisal), Events capture historically situated crises (Mediterranean migrant crisis, BP oil spill nationalization), and Comparisons facilitate interpretive framing through juxtapositions (cancer suffering: individual vs.\ family).
Altogether, this taxonomy illustrates both topical diversity and a spectrum of cognitive orientations.

\begin{table}[ht]
\centering 
\caption{Categorize extracted concepts into abstract types.}
\small
\renewcommand{\arraystretch}{1.2} %
\begin{tabular}{p{0.2\linewidth}p{0.7\linewidth}}
\toprule
\textbf{Concept Types} & \textbf{Examples} \\
\midrule
Method (9.17\%) & Multidimensional poverty measurement, Commercial real estate appraisal, Stop words search duplicates, DAT chemistry exam preparation \\
\midrule
Event (6.98\%) & Mediterranean migrant crisis, BP oil spill nationalization, Paula Abdul stalked, Eminem-Apple music rights lawsuit, Presidents Cup U.S. golf \\
\midrule
Comparison (5.54\%) & Hobbit film adaptation length/cost, Biking as superior transport, Cancer suffering, individual vs. family, Progress critique: 4G vs. alternatives \\
\midrule
Analysis (5.20\%) & Health effects of substances, Thai massage benefits, Scrabble word breakdown, Relationship roles and challenges, Manchester United player analysis \\
\midrule
Phenomenon (4.95\%) & Secret pain; self-destruction, Car-related online humor/pranks, Transnational corporations in globalization, Hippie identity and lifestyle \\
\bottomrule
\end{tabular}
\label{tab:concept_analysis_types}  
\end{table}

While real and synthesized documents share the same underlying concept, they differ in multiple ways.
We categorize these differences into a taxonomy of relations using a small ontology (Table~\ref{tab:relation_categories}).
These relations range from scope-based distinctions (e.g., specific vs.\ general), to causal connections (e.g., corruption leading to reform), to contrastive pairs (e.g., Constitution articles vs.\ Articles of Confederation).
This diversity demonstrates the rich variation structure that the synthesizer captures.

\begin{table}[ht]
\centering 
\caption{Categorize relations between real documents $d_1$ and synthesized documents $d_2$.}
\small
\renewcommand{\arraystretch}{1.2} %
\begin{tabular}{p{0.22\linewidth}p{0.68\linewidth}}
\toprule
\textbf{Relation Categories} & \textbf{Examples} \\
\midrule
Scope relation \par (8.14\%) & $d_1$: Probiotics' possible effects on \textcolor{thesisDarkGarnet}{H1N1} infection \par
$d_2$: Probiotics' \textcolor{thesisDarkGarnet}{general} digestive and immune benefits \par
Relation: specific application vs general health benefits of probiotics\\
\midrule
Perspectival relation \par (5.51\%) & $d_1$: \textcolor{thesisDarkGarnet}{Personal}, humorous struggles of new bloggers \par $d_2$: \textcolor{thesisDarkGarnet}{Objective} guide to pros and cons of blogging \par Relation: subjective experiences vs objective guidance about blogging\\
\midrule
Functional relation \par (4.70\%) & $d_1$: \textcolor{thesisDarkGarnet}{Reviews and feedback} on ``Space Bound'' game \par $d_2$: Forum \textcolor{thesisDarkGarnet}{troubleshooting} for bugs in ``Space Bound'' \par Relation: reviews/feedback vs troubleshooting for the same game\\
\midrule
Causal relation \par (2.05\%) & $d_1$: DTEK faces \textcolor{thesisDarkGarnet}{corruption} probe, financial risk \par $d_2$: DTEK nationalized for state-driven energy \textcolor{thesisDarkGarnet}{reform} \par Relation: corruption/financial issues vs nationalization/energy reform \\
\midrule
Contrastive relation \par (1.65\%) & $d_1$: Detailed summary of \textcolor{thesisDarkGarnet}{Constitution} articles \par $d_2$: Overview, flaws of \textcolor{thesisDarkGarnet}{Articles of Confederation} \par Relation: U.S. Constitution articles vs Articles of Confederation: different foundational documents\\
\bottomrule
\end{tabular}
\label{tab:relation_categories}  
\end{table}

Document summarize and concept analysis instructions:
\begin{qualitativeBox}
{\footnotesize
\begin{verbatim}
In the following, you are given two documents, doc1 and doc2. Doc2 is generated from doc1.

The principle of generation is to first abstract a concept from doc1, and then starting from 
this concept, generate doc2. Can you guess what this concept is and how doc2 was generated?

Please keep the summary and concepts to be LESS OR EQUAL TO 10 WORDS and format your answer as 
follows. Highlight the difference between doc2 and doc1 in your doc2_summary:

<doc1_summary> summary of doc1 </doc1_summary>
<concept_c> abstract concept from doc1 </concept_c>
<doc2_summary> summary of doc2 built on doc1 given the concept </doc2_summary>

Example 1:
<doc1_summary> recommendation of local coffee shops in San Diego </doc1_summary>
<concept_c> coffee + San Diego </concept_c>
<doc2_summary> comparison of coffee culture in SD and NYC </doc2_summary>

Example 2:
<doc1_summary> Patient with swollen eye discusses pain causes & symptoms and seeks for advice 
</doc1_summary>
<concept_c> medical symptom of swollen eye </concept_c>
<doc2_summary> A wiki-style article introducing causes and cures for swollen eye 
</doc2_summary>

Now, give your answer for the following documents:
<doc1>
{real_document}
</doc1>

<doc2>
{synthesized_document}
</doc2>
\end{verbatim}
}
\end{qualitativeBox}

\subsection{Factuality analysis}
\label{sec:factuality}
\begin{table}[ht]
\centering
\caption{Estimation of the ratio of non-factual documents.
We can see that the occurrence factuality error decays as the SP scales up.}
\renewcommand{\arraystretch}{1.3} %
\begin{tabular}{lrrr}
\hline
\hline
 & \textbf{Factuality undefined} & \textbf{No factual error} & \textbf{Factual error} \\
\hline
\textbf{Real data} & 31.44\% & 66.74\% & 1.81\% \\
\textbf{Synthetic data (200B-scale)} & 34.43\% & 50.47\% & 15.09\% \\
\textbf{Synthetic data (1T-scale)} & 31.91\% & 59.43\% & 8.65\% \\
\hline
\hline
\end{tabular}
\label{tab:factuality_stats}
\end{table}

\begin{table}[ht]
\centering
\caption{Factuality undefined synthetic text.}
\label{tab:factuality_undefined}
\renewcommand{\arraystretch}{1.03}
\begin{tabular}{p{0.09\linewidth} p{0.88\linewidth}}
\toprule
\textbf{Synthetic Text} & Sunday, December 28, 2008 \par
Tante Leni \par 
Tante Leni is not only my Aunt Leni; she is my Eternity. When my Aunt Leni died a few years ago, she was deeply saddened and I was devastated. She was not the first family member to die, but she was the first I felt so strongly about. Tante Leni was all my parents really had to show for the 25 years they worked as public school teachers and she was the one who had been with them the longest. There was a special place in her heart for my parents. In addition to all that, she was the kind of person who always had a smile and a funny story to share. She was kind and funny and generous. The story that always comes to mind when I think about her is the time she was working at the bank and someone dropped something from the top floor. It was a very large parcel and the workers on the ground floor didn't have the tools to open it. She jumped down to see what was in it. A very large package of champagne appeared and she began gulping it down. Tante Leni and my mother in a portrait they took when my mom was 20. Tante Leni and my parents in a family portrait she took for my mom at 22. Tante Leni and my dad at home when he was working as a dance instructor. When my mom died, she had all the people who had known her since she was a child living in the house. Tante Leni was the oldest, but she was also the best at cleaning, cooking and taking care of the house. When my mom passed away, she went to a rehab center and Tante Leni stayed in the house. \\
\bottomrule
\end{tabular}
\end{table}

All LM-generated synthetic data may produce non-factual content because of the probabilistic nature of generation.
Moreover, because the internet inherently contains factual inaccuracies, LMs absorb these errors unless the data is carefully cleaned.
During post-training, factuality must also be recalibrated alongside other objectives such as data safety.

SBP relies solely on document-level correlations and does not incorporate human intervention to filter non-factual content; the generated outputs are therefore expected to contain factual errors.
Interestingly, the frequency of such errors correlates with the amount of data used in the SBP pipeline.
We define a document as having \textbf{undefined factuality} if it is primarily subjective or opinion-driven, or if it concerns personal, obscure, or unverifiable entities.
In all other cases, the document's factuality is considered \textbf{well-defined} and verifiable.

In Table~\ref{tab:factuality_stats}, we analyze both real and synthesized data from the main experiment (\S\ref{sec:benchmark-performance}).
We consider two synthetic datasets: a smaller-scale set initialized with 10B seed tokens and a larger-scale set initialized with 50B seed tokens.
From each source, we randomly sample 10k documents and categorize each into three bins---\textbf{factuality undefined}, \textbf{no factual error}, and \textbf{factual error}---using LM-as-a-judge.
We find that synthetic data contains more factual errors than real data.
However, as the amount of seed data increases, factuality improves significantly, approaching that of real data.
This finding aligns with our mideval results in Table~\ref{tab:mideval}: greater seed data availability enables the LM to capture more factual knowledge and the synthesizer to generate more relevant documents, thereby reducing hallucinations.

Table~\ref{tab:factuality_error_detail_analysis} extends our analysis of factuality errors, highlighting inaccuracies in the synthetic texts.
These include false transfer and timeline claims in football, as well as incorrect institutional, company location, and certification details in the ecolabel example.
This underscores the importance of rigorous fact-checking, particularly for historical events (e.g., sports) and certification standards (e.g., eco-labels).

Factuality detection instructions:
\begin{qualitativeBox}
{\footnotesize
\begin{verbatim}
You are a helpful AI assistant. Your task is to evaluate whether the given document has
well-defined factuality.

Definitions:

Not well-defined factuality: The document is primarily subjective or opinion-based (e.g., 
express disapproval of a politician in social media). The document discusses personal, 
unknown, or unverifiable entities (e.g., a private diary).

Well-defined factuality: The document refers to well-known, identifiable entities (e.g., 
famous people, historical events, popular movies). Its factual claims can be checked or 
verified.

Output format:

If the document's factuality is not well-defined, output:
<not well defined></not well defined>

If the document's factuality is well-defined and factual, output:
<well defined>True</well defined>

If the document's factuality is well-defined but non-factual, output:
<well defined>False</well defined>

Now, analyze the following document and provide your answer:
{document}
\end{verbatim}
}
\end{qualitativeBox}

\begin{table}[ht]
\centering
\caption{Factuality errors detected in synthetic text.}
\label{tab:factuality_error_detail_analysis}
\footnotesize
\renewcommand{\arraystretch}{1.03}
\begin{tabular}{p{0.09\linewidth} p{0.88\linewidth}}
\toprule
\textbf{Synthetic Text} & So just how much has Chelsea been prepared to pay for the 34-year-old midfielder? 
Realistically, the clubs involved should be keeping in the region of £25 million (\$38.8 million) and around £30 million (\$45.5 million) being bandied about for the player in Italy. With the Blues expected to complete the sale of \textcolor{thesisDarkGarnet}{Cesc Fabregas to Arsenal} this week, Lampard appears the logical replacement in midfield, but his bid to extend his contract has hit a roadblock with Chelsea's owners Roman Abramovich and the club being unable to agree to an increase in salary. \par
Lampard, who played in the Champions League final in Lisbon for Chelsea in \textcolor{thesisDarkGarnet}{2007}, has been linked with a move away from Stamford Bridge this summer, after having his contract with the club \textcolor{thesisDarkGarnet}{indefinitely extended in 2010}. There were rumours that he was to leave on a free transfer in the summer of 2011, but it appears that those are again being placed to the side. \par
Miguel Delaney, from Sky Sports, believes that Lampard could be the highest paid player in the Premier League, and puts the bid that Chelsea are offering him in the region of £6 million (\$8.9 million) per year. If Lampard is to join a team in Italy, it seems that this sum would be realistic, but the Premier League is a different beast, and it would be unrealistic for Chelsea to splash out that amount. \par
As the speculation around his future continues, you can watch the pair discussing the topic in the video below. \\
\midrule
\textbf{Factual Errors} & \textcolor{thesisDarkGarnet}{Cesc Fabregas to Arsenal} -- Fabregas moved from Arsenal to Barcelona in 2011, then to Chelsea in 2014, not back to Arsenal. \par
\textcolor{thesisDarkGarnet}{2007} -- The 2014 Champions League final in Lisbon did not involve Chelsea or Lampard; Chelsea won in 2012. \par
\textcolor{thesisDarkGarnet}{indefinitely extended in 2010} -- Lampard did not get an indefinite contract extension in 2010 with Chelsea. \\
\midrule
\textbf{Synthetic Text} & Swanee Glue Brand: First Glue to be Awarded the Swan Ecolabel \par
Published:27 July 2022 \par
The global glue stick market is expected to reach USD 3.45 billion by 2028. Adhesives are the first choice of manufacturers in all industries such as food, pharmaceuticals, automotive, aerospace, construction, and packaging. As consumers are increasingly conscious of their carbon footprint and environmental issues, glue manufacturers are aiming to produce products that comply with environmental standards and are effective and cost-effective in their applications. This is why the \textcolor{thesisDarkGarnet}{Swan Ecolabel was established by the Swedish Environment Agency} as a certification for sustainable adhesive products. \par 
Swanee Glue is one of the world's leading glue brands in glue sticks, and this year its brand received the Swan Ecolabel. UHU is an adhesive brand owned by Bolton Adhesives in the \textcolor{thesisDarkGarnet}{Netherlands}, and part of the Italian Bolton Group with a strong agenda for sustainability. \par 
Glue sticks, specifically glue sticks with a wider applicator and swan neck applicators, have the most impact on the environment because they are a consumable item and their impact is greatest when thrown away. Therefore, the Swanee Swan Ecolabel ensures that UHU is part of the solution to the growing demand for sustainable adhesive products.\par
In order to obtain the Swan Ecolabel, the adhesive must have at least \textcolor{thesisDarkGarnet}{50\% renewable content}. Besides this, the glue stick should also contain a higher percentage of recyclable content. UHU meets all these criteria and has a permanent and multi-use applicator. For further information, you can contact UHU receives the Swan Ecolabel \\
\midrule
\textbf{Factual Errors} & \textcolor{thesisDarkGarnet}{Swan Ecolabel was established by the Swedish Environment Agency} -- The Nordic Swan Ecolabel was established by the Nordic Council of Ministers, not only Sweden. \par
\textcolor{thesisDarkGarnet}{Netherlands} -- UHU is based in Germany, not the Netherlands. \par
\textcolor{thesisDarkGarnet}{50\% renewable content} -- The Swan Ecolabel requires at least 20\% renewable content in adhesives, not 50\%. \\
\bottomrule
\end{tabular}
\end{table}

\subsection{Mideval prompts}
\label{sec:mid-eval}

Before each large-scale synthesis run---on the order of billions of tokens---we first synthesize a small subset of data to evaluate its quality.
We call this step ``mideval.''
We use LM-as-a-judge to evaluate three quality metrics: \textbf{Pair-relevance} (whether the seed and synthesized documents share the same topic, entity, or event), \textbf{Pair-novelty} (whether the synthesized content differs substantively from near-duplicates), and \textbf{Non-repetition} (absence of repeated sentence patterns in the output).
Full evaluation prompts are available in the code repository.

\newpage
\clearpage

\subsection{Synthesized documents from the 1T-scale experiment}
\label{sec:1t-exp}

We present additional examples of synthesized documents at the 1T-scale, complementing the 200B-scale example in \S\ref{sec:analysis-of-synthetic-data}.
\definecolor{originalcolor}{RGB}{80, 80, 80}  %
\definecolor{originalcolorlight}{RGB}{230, 230, 230}  %
\definecolor{generatedcolor}{RGB}{100, 100, 100}
\definecolor{generatedcolorlight}{RGB}{250, 250, 250}

\begin{figure}[htbp]
\begin{center}
\adjustbox{max totalheight=0.85\textheight}{%
\begin{minipage}{\textwidth}
\begin{minipage}[t]{0.32\textwidth}
    \begin{tcolorbox}[
        colback=originalcolorlight,
        colframe=originalcolor,
        title=Real document, %
        fonttitle=\bfseries\scriptsize,
        width=\textwidth,
        left=0mm,
        right=0mm,
        top=0mm,
        bottom=0mm,
        equal height group=row1
    ]

    \small The Karate Kid \vspace{3pt}

    From Wikipedia, the free encyclopedia \vspace{1pt}
    
    (Redirected from Karate Kid)
    
    Jump to: navigation, search
    
    For other uses, see Karate Kid (disambiguation).
    
    The Karate Kid Karate kid.jpg
    
    Directed by John G. Avildsen
    
    Produced by Jerry Weintraub
    
    Written by Robert Mark Kamen
    
    Music by Bill Conti
    
    Cinematography James Crabe
    
    Distributed by Columbia Pictures
    
    Release dates 
      • June 22, 1984 (1984-06-22)
      
    Running time 127 minutes[1]
    
    Country United States
    
    Language English
    
    Budget \$8 million[2]
    
    Box office \$90.8 million[3] \vspace{3pt}
    
    The Karate Kid is a 1984 American martial arts drama film produced by Jerry Weintraub, directed by John G. Avildsen, written by Robert Mark Kamen, that stars Ralph Macchio, Noriyuki "Pat" Morita and Elisabeth Shue.[4][5] It is an underdog story in the mold of a previous success, Avildsen's 1976 film Rocky ...

    \end{tcolorbox}
\end{minipage}\hfill
\begin{minipage}[t]{0.32\textwidth}
    \begin{tcolorbox}[
        colback=generatedcolorlight,
        colframe=generatedcolor,
        title=Synthesis I,
        fonttitle=\bfseries\scriptsize,
        width=\textwidth,
        left=0mm,
        right=0mm,
        top=0mm,
        bottom=0mm,
        equal height group=row1
    ]
    \small The Karate Kid (2010) \vspace{3pt}

    The Karate Kid (2010) movie poster
    
      • 115 minutes
      • Rated PG-13
      • Directed by Harald Zwart
      • Starring: Will Smith, Jackie Chan, RZA, Jaden Smith
      • More info on IMDB \vspace{3pt}
    
    It's ironic that Daniel LaRusso (Ralph Macchio) and Johnny Lawrence (Jaden Smith) appear in The Karate Kid in their middle ages, because what's most interesting about this remake is the way it updates an old, familiar storyline for a new generation of moviegoers. \vspace{3pt}
    
    Perhaps I shouldn't say "updates" -- because that implies that the new version is a significant departure from the 1984 classic. If you're one of those folks who thinks the 1984 version of The Karate Kid is sacrosanct, don't expect anything different here. Just watch both movies and compare the stories and characters. \vspace{3pt}
    
    The story is basically the same as before -- young Daniel LaRusso moves to Los Angeles, and soon finds himself beset by Johnny Lawrence and his gang of students from the Cobra Kai karate school, which bullies Daniel whenever it gets the chance ...

    \end{tcolorbox}
\end{minipage}\hfill
\begin{minipage}[t]{0.32\textwidth}
    \begin{tcolorbox}[
        colback=generatedcolorlight,
        colframe=generatedcolor,
        title=Synthesis II,
        fonttitle=\bfseries\scriptsize,
        width=\textwidth,
        left=0mm,
        right=0mm,
        top=0mm,
        bottom=0mm,
        equal height group=row1
    ]
    \small The Karate Kid \vspace{3pt}

    Trivia: When Daniel and his mom were drinking ice cream and talking about him not talking to them, Daniel turns to his mom and says "We're not just a family. We're a family of good people." After Daniel was sent away, on the front of his underwear (that was on fire), it said "Daniel don't quit on family." The script for this movie originally said that Daniel's mom was in the film but she was cut because the film makers thought she would not look like Elisabeth Shue and they thought the characters could have a less motherly character. Elisabeth Shue actually called the director John Avildsen, and insisted on being in the film. John Avildsen replied that he didn't know whether or not he wanted to add a less motherly character to the film. Shue replied "If you don't, I'm not playing in the movie." The director of the film was also impressed that she did this and he was going to hire her in any case. \vspace{3pt}
    
    Revealing mistake: When Daniel is throwing the bucket at Johnny's head while on the bus, he is supposed to do it several times, yet when he hits Johnny the first time, you can see Johnny fall on the ground ...

    \end{tcolorbox}
\end{minipage}

\begin{minipage}[t]{0.32\textwidth}
    \begin{tcolorbox}[
        colback=originalcolorlight,
        colframe=originalcolor,
        title=Real document, %
        fonttitle=\bfseries\scriptsize,
        width=\textwidth,
        left=0mm,
        right=0mm,
        top=0mm,
        bottom=0mm,
        equal height group=row2
    ]

    \small Skip to main content \vspace{3pt}

    Are you craving some beach time? If you're escaping down south for spring break, we have some tunes for you. But even if you aren’t lucky enough to sprawl in the sand or bask in golden sunshine right now, we'll bring the beach to you! Until your next break, enjoy listening to The Beach Boys and Jack Johnson. \vspace{3pt}
    
    If you’re stuck in the cold—or somewhere where the sun is indecisive (I’m looking at you, Midwest)—use this playlist to take you back to sunnier days and remind you that the winter is almost over; soon the sun will last!

    \end{tcolorbox}
\end{minipage}\hfill
\begin{minipage}[t]{0.32\textwidth}
    \begin{tcolorbox}[
        colback=generatedcolorlight,
        colframe=generatedcolor,
        title=Synthesis I,
        fonttitle=\bfseries\scriptsize,
        width=\textwidth,
        left=0mm,
        right=0mm,
        top=0mm,
        bottom=0mm,
        equal height group=row2
    ]
    \small Beach Vibe \vspace{3pt}

    Any traveller will tell you that if you get on a plane and actually arrive somewhere where the weather is the same for weeks or months, you know you've made a mistake. I've experienced this time and again - it's one of the many things that you just can't fully appreciate or appreciate to the fullest until you feel it in your bones. \vspace{3pt}
    
    My latest trip was to Vietnam last week. I had been planning this trip for a little over a year now and I really hoped it would be everything I had imagined. It was exactly what I'd wanted. The people, the culture, the food, the scenery - it was all perfectly delved into. I'd never been to an Asian country before and so the whole trip was just like discovering a new continent. \vspace{3pt}
    
    It was one of the best places I've ever been ...

    \end{tcolorbox}
\end{minipage}\hfill
\begin{minipage}[t]{0.32\textwidth}
    \begin{tcolorbox}[
        colback=generatedcolorlight,
        colframe=generatedcolor,
        title=Synthesis II,
        fonttitle=\bfseries\scriptsize,
        width=\textwidth,
        left=0mm,
        right=0mm,
        top=0mm,
        bottom=0mm,
        equal height group=row2
    ]
    \small Here's a playlist to warm your heart and get you ready to leave for the beach in a few weeks. So many classics from your youth, this playlist packs an emotional punch. \vspace{3pt}

    It's kind of a lame playlist. I'm sure a better one exists somewhere on the internet. But on this most northern of days, I like to get a warm chill going and use it as a positive recharge before going outside. This warm chill is a hit with my dogs and they are the best cuddlers in the world, so that makes it perfect for me. This playlist will make your winter that much more bearable and then hopefully you can head to the beach! \vspace{3pt}
    
    You know you want it! \vspace{3pt}
    
    Here's the playlist: \vspace{3pt}
    
    [credit provider="YouTube" url=''] \vspace{3pt}
    
    Get our free mobile app

    \end{tcolorbox}
\end{minipage}

\end{minipage}}%
\captionof{figure}{Comparison of original text with synthesized text variations. On the first row, the real document provides factual information about the 1984 film's production and release. In contrast, the synthesized documents offer subjective commentary, opinions, and behind-the-scenes anecdotes about both the 1984 film and its 2010 remake. On the second row, the synthesized documents are continuations of the real document.}
\label{fig:text_comparison_1t}
\end{center}
\end{figure}

\section{Additional pretraining results}

\subsection{Two epochs validation}
\label{sec:two-epochs-validation}

For the 1T-scale oracle experiment, we use 482B tokens repeated twice as a proxy for training on 1T unique tokens.
This design choice stems from the DCLM-baseline \citep{li2024datacomplm} dataset containing 80\% duplicates, which hinders our evaluation.
We validate this choice by scaling down to 400B, where we have sufficiently many unique tokens.
As shown in Table \ref{tab:two-epoches-validation}, 200B tokens repeated twice yield nearly identical performance to 400B unique tokens.
This aligns with the observation from \citet{muennighoff2023scaling} that repetition up to 4 times yields nearly no performance degradation.

\begin{table}[ht]
\centering
\caption{Performance comparsion with 200B tokens repeated twice vs. 400B unique tokens for the 3B model. We can see that the two models yield similar performance.}
\label{tab:two-epoches-validation}
\renewcommand{\arraystretch}{1.3}
\begin{tabular}{lrr}
\hline
\hline
\multicolumn{1}{l}{\textbf{Benchmark}} & \multicolumn{1}{c}{\textbf{2x200B}} & \multicolumn{1}{c}{\textbf{1x400B}} \\
\hline
\multicolumn{3}{c}{\emph{Perplexity on held-out data $\downarrow$}} \\
\hline
OpenWebText2& 4.55 & 4.54 \\
LAMBADA & 4.49 & 4.46 \\
Five-shot MMLU & 3.19 & 3.17 \\
\hline
\multicolumn{3}{c}{\emph{QA accuracy $\uparrow$}} \\
\hline
ARC-Challenge \tiny{(0-shot)} & 38.31 & 41.47 \\
ARC-Easy \tiny{(0-shot)} & 73.11 & 75.29 \\
SciQ \tiny{(0-shot)} & 93.80 & 93.30 \\
Winogrande \tiny{(0-shot)} & 64.96 & 63.93 \\
TriviaQA \tiny{(1-shot)} & 32.51 & 34.35 \\
WebQS \tiny{(1-shot)} & 18.75 & 13.58 \\
\hline
\textbf{Average QA accuracy} & 53.57 & 53.65 \\
\hline
\hline
\end{tabular}
\end{table}

\subsection{Model scaling}
\label{sec:model_scaling}

An alternative approach to using additional compute is to scale the model.
Here we examine the benefits of fixing a training token budget but using a 6B-parameter model (Table \ref{tab:model_specs}).
\begin{table}[t]
\centering
\caption{6B-parameter model setup.}
\label{tab:model_specs}
\begin{tabular}{ccc}
\toprule
\bf{Total Params.} & \bf{3B} & \bf{6B} \\
\midrule
$\ell_{\text{context}}$ & 4096 & 4096 \\
$n_{\text{vocab}}$ & 49152 & 49152 \\
$n_{\text{layers}}$ & 26 & 32 \\
$d_{\text{model}}$ & 3072 & 4096 \\
$d_{\text{ffn}}$ & 8064 & 13056 \\
$n_{\text{heads}}$ & 24 & 32 \\
$n_{\text{kv\_heads}}$ & 8 & 8 \\
\bottomrule
\end{tabular}
\end{table}

We conduct a pretraining experiment at the 200B-scale, replacing the 3B-parameter model with a 6B-parameter model.
Table \ref{tab:model-scaling-results} shows that the 6B-parameter model consistently outperforms the baseline, indicating that it effectively uses the additional computational resources.
Comparing SBP with the 6B-parameter model, we find that each performs better on different benchmarks.
This suggests that the benefits of SBP are orthogonal to those of a larger model, offering the potential to combine both approaches for even better performance.
\begin{table}[ht]
\centering
\caption{200B-scale experiments with model scaling.
The first three columns are identical to Table\ref{tab:results}. The last column shows the performance of training a 6B model under a 200B training token budget with 10B unique tokens.}
\label{tab:model-scaling-results}
\renewcommand{\arraystretch}{1.3} %
\begin{tabular}{lrrrr}
\hline
\hline
\multicolumn{1}{l}{\textbf{Benchmark}} & \multicolumn{1}{c}{\textbf{Baseline}} & \multicolumn{1}{c}{\textbf{SBP}} & \multicolumn{1}{c}{\textbf{Oracle}} & \multicolumn{1}{c}{\textbf{6B-model}} \\
\hline
\multicolumn{5}{c}{\emph{Perplexity on held-out data $\downarrow$}} \\
\hline
OpenWebText2& 5.74 & \textcolor{thesisDarkGarnet}{-0.53} & -1.02 & -0.36 \\
LAMBADA  & 6.87 & \textcolor{thesisDarkGarnet}{-0.85} & -1.86 & -1.10 \\
Five-shot MMLU & 3.83 & \textcolor{thesisDarkGarnet}{-0.36} & -0.51 & -0.13  \\
\hline
\multicolumn{5}{c}{\emph{QA accuracy $\uparrow$}} \\
\hline
ARC-Challenge \tiny{(0-shot)} & 35.32 & \textcolor{thesisDarkGarnet}{+1.28} & +2.82 & +3.42 \\
ARC-Easy \tiny{(0-shot)} & 68.94 & \textcolor{thesisDarkGarnet}{+2.65} & +4.29 & +0.67 \\
SciQ \tiny{(0-shot)} & 90.50 & \textcolor{thesisDarkGarnet}{+1.00} & +2.40 & +0.80 \\
Winogrande \tiny{(0-shot)} & 60.14 & \textcolor{thesisDarkGarnet}{+1.90} & +5.53 & +2.92 \\
TriviaQA \tiny{(1-shot)} & 22.51 & \textcolor{thesisDarkGarnet}{+3.36} & +7.37 & +3.11 \\
WebQS \tiny{(1-shot)} & 8.56 & \textcolor{thesisDarkGarnet}{+3.74} & +10.83 & +5.22 \\
\hline
\textbf{Average QA accuracy} & \textbf{47.66} & \textcolor{thesisDarkGarnet}{\textbf{+2.32}} & \textbf{+5.54} & +2.69 \\
\hline
\hline
\end{tabular}
\end{table}

\section{Supplementary materials for sample-efficient reasoning}
\label{sec:s1-se-appendix}

\subsection{Initial collection of 59K samples}
\label{sec:s1-se-59K}

We collect an initial 59,029 questions from 16 sources following three guiding principles.
\textbf{Quality}: Datasets should be high-quality; we always inspect samples and ignore datasets with, e.g., poor formatting.
\textbf{Difficulty}: Datasets should be challenging and require significant reasoning effort.
\textbf{Diversity}: Datasets should stem from various fields to cover different reasoning tasks.
We collect datasets of two categories:

\paragraph{Curation of existing datasets} Our largest source is NuminaMATH \citep{numina_math_datasets} with 30,660 mathematical problems from online websites.
We also include historical AIME problems (1983-2021).
To enhance diversity, we add OlympicArena \citep{huang2024olympicarenabenchmarkingmultidisciplinecognitive} with 4,250 questions spanning Astronomy, Biology, Chemistry, Computer Science, Geography, Mathematics, and Physics from various Olympiads.
OmniMath \citep{gao2024omnimathuniversalolympiadlevel} adds 4,238 competition-level mathematics problems.
We also include 2,385 problems from AGIEval \citep{zhong2023agievalhumancentricbenchmarkevaluating}, which features questions from standardized tests like SAT and LSAT, covering English, Law, and Logic.
We refer to Table~\ref{tab:s1-se-ds} for our other sources.

\paragraph{New datasets in quantitative reasoning} To complement these existing datasets, we create two original datasets.
s1-prob consists of 182 questions from the probability section of Stanford University's Statistics Department's PhD Qualifying Exams (\url{https://statistics.stanford.edu}), accompanied by handwritten solutions that cover difficult proofs.
The probability qualifying exam is held yearly and requires professional-level mathematical problem-solving.
s1-teasers comprises 23 challenging brain-teasers commonly used in interview questions for quantitative trading positions.
Each sample consists of a problem and solution taken from PuzzledQuant (\url{https://www.puzzledquant.com/}).
We only take examples with the highest difficulty level (``Hard'').

For each question, we generate a reasoning trace and solution using the Google Gemini Flash Thinking API~\citep{geminithinking}, extracting its reasoning trace and response.
This yields 59K triplets of a question, generated reasoning trace, and generated solution.
Examples from our dataset are in the appendix.
We decontaminate all samples against our evaluation questions (MATH500, GPQA Diamond, AIME24) using 8-grams and deduplicate the data.

\subsection{Final selection of 1K samples}
\label{sec:s1-se-data}

We could directly train on our pool of 59K questions.
However, our goal is to find the \textit{simplest} approach with minimal resources.
We therefore apply three stages of filtering to arrive at a minimal set of 1,000 samples, guided by our three data principles: Quality, Difficulty, and Diversity.

\paragraph{Quality} We first remove any questions where we encountered API errors, reducing our dataset to 54,116 samples.
Next, we filter out low-quality examples by checking if they contain any string patterns with formatting issues, such as ASCII art diagrams, non-existent image references, or inconsistent question numbering, reducing our dataset to 51,581 examples.
From this pool, we identify 384 samples for our final 1,000 samples from datasets that we perceive as high-quality and not in need of further filtering (see below for details).

\paragraph{Difficulty} For difficulty, we use two indicators: model performance and reasoning trace length.
We evaluate two models on each question: Qwen2.5-7B-Instruct and Qwen2.5-32B-Instruct~\citep{qwen2024qwen25technicalreport}, with correctness assessed by Claude 3.5 Sonnet comparing each attempt against the reference solution (see the grading protocol below).
We measure the token length of each reasoning trace to indicate problem difficulty using the Qwen2.5 tokenizer.
This relies on the assumption that more difficult problems require more thinking tokens.
Based on the grading, we remove questions that either Qwen2.5-7B-Instruct or Qwen2.5-32B-Instruct can solve correctly, as these may be too easy.
Using two models reduces the likelihood of an easy sample slipping through our filtering because of a rare mistake on an easy question by one model.
This brings our total samples down to 24,496, setting the stage for the next round of subsampling based on diversity.
While filtering with these two models may be optimized for our setup (we also use Qwen2.5-32B-Instruct as our model to finetune), the idea of model-based filtering generalizes to other setups.

\paragraph{Diversity} To quantify diversity, we classify questions into domains using Claude 3.5 Sonnet based on the Mathematics Subject Classification (MSC) system (e.g., geometry, combinatorics, etc.) from the American Mathematical Society.\footnote{\url{https://mathscinet.ams.org/mathscinet/msc/msc2020.html}}
The taxonomy focuses on topics in mathematics but also includes other sciences such as biology, physics, and economics.
To select our final examples from the pool of 24,496 questions, we first choose one domain uniformly at random.
Then, we sample one problem from this domain according to a distribution that favors longer reasoning traces (see below for details) as motivated in \textit{Difficulty}.
We repeat this process until we have 1,000 total samples spanning 50 domains.

In \S\ref{sec:s1-se-dataabl}, we show that using our three criteria in combination is important, as relying on quality, diversity, or difficulty in isolation leads to worse datasets.
Some distilled generations are incorrect, which we allow in our data because we focus on capturing the reasoning process rather than entirely correct solutions.
Our grader deems 53.6\% correct in \sonek{} and 63.0\% in our follow-up \textbf{s1K-1.1}.

\subsection{Data ablations}
\label{sec:s1-se-dataabl}

\begin{table}[htbp]
\centering
\caption{\textbf{\sonek{} data ablations.} We report 95\% paired bootstrap confidence intervals for differences relative to the \sonek{} model using 10,000 bootstrap samples. E.g., the interval [-13\%, 20\%] means that, with 95\% confidence, the true difference between 59K-full and \sonek{} is between -13\% and +20\%. If the entire interval is negative, e.g. [-27\%, -3\%], we can confidently say that the performance is worse than \sonek{}.}
\begin{tabular}{l|ccc}
\toprule
Model & \makecell{AIME\\2024} & \makecell{MATH\\500} & \makecell{GPQA\\Diamond} \\
\midrule
\multirow{2}{*}{1K-random} & 36.7 & 90.6 & 52.0 \\
& \scriptsize{[-26.7\%, -3.3\%]} & \scriptsize{[-4.8\%, 0.0\%]} & \scriptsize{[-12.6\%, 2.5\%]} \\
\multirow{2}{*}{1K-diverse} & 26.7 & 91.2 & 54.6 \\
& \scriptsize{[-40.0\%, -10.0\%]} & \scriptsize{[-4.0\%, 0.2\%]} &  \scriptsize{[-10.1\%, 5.1\%]}\\
\multirow{2}{*}{1K-longest} & 33.3 & 90.4 & 59.6 \\
& \scriptsize{[-36.7\%, 0.0\%]} & \scriptsize{[-5.0\%, -0.2\%]} & \scriptsize{[-5.1\%, 10.1\%]} \\
\multirow{2}{*}{59K-full} & 53.3 & 92.8 & 58.1 \\
& \scriptsize{[-13.3\%, 20.0\%]} & \scriptsize{[-2.6\%, 2.2\%]} & \scriptsize{[-6.6\%, 8.6\%]} \\
\midrule
\sonek{} & 50.0 & 93.0 & 57.6 \\
\bottomrule
\end{tabular}
\label{tab:s1-se-datablation}
\end{table}

In \S\ref{sec:s1-se-data} we outlined our three guiding principles in curating \sonek{}: Quality, Difficulty, and Diversity.
We next test the importance of combining them and the overall effectiveness of our selection.
\textbf{Only Quality (1K-random)}: After obtaining our high-quality reasoning chains from Gemini, we select 1,000 samples at random, not relying on our difficulty and diversity filtering at all.
Table~\ref{tab:s1-se-datablation} shows this approach performs much worse than \sonek{} across all benchmarks.
\textbf{Only Diversity (1K-diverse)}: For this dataset, we sample uniformly across domains to maximize diversity, disregarding any notion of difficulty.
This approach also leads to poor performance similar to 1K-random.
\textbf{Only Difficulty (1K-longest)}: Here we rely on one of our difficulty indicators introduced in \S\ref{sec:s1-se-data} by selecting the 1,000 samples with the longest reasoning traces.
This approach significantly boosts GPQA performance but overall still falls short of using \sonek{}.
\textbf{Maximize Quantity}: Finally, we compare with training on all of our 59K samples, a superset of all the 1K-sample versions.
This leads to a strong model but uses much more resources.
To finetune on 59K samples, we use 394 H100 GPU hours while \sone{} only required 7 H100 GPU hours.
Moreover, relying only on \sonek{} is extremely competitive as shown in \S\ref{sec:s1-se-data}.
Altogether, combining all three criteria---\textit{Quality}, \textit{Difficulty}, \textit{Diversity}---via our methodology in \S\ref{sec:s1-se-data} is key for sample-efficient reasoning training.

\subsection{Dataset composition}

\begin{table*}[htbp]
\centering
\caption{\textbf{Summary of our dataset \sonek{}}. Token count measured by the Qwen-2.5 tokenizer. We prompt Claude to produce keywords given several questions from the domain.}
\begin{tabular}{>{\raggedright}p{3.5cm} l l l l l}
\toprule
Domain & \#questions & Total token count & Keywords \\
\midrule
Geometry & 109 & 560.2K & Area, Triangle, Distance \\
Number theory & 98 & 522.5K & Sequences, Divisibility \\
Combinatorics & 75 & 384.7K & Permutations, Counting \\
Real functions & 43 & 234.8K & Trigonometry, Calculus  \\
Biology & 41 & 120.9K & Organic reactions \\
Complex functions & 32 & 170.2K & Complex roots \\
Quantum theory & 32 & 127.9K & Particles, Wave functions \\
Field theory & 28 & 150.1K & Polynomials, Roots \\
Calculus of variations & 28 & 155.5K & Optimization, Control \\
Difference equations & 24 & 132.5K & Recurrence, Recursion \\
Electromagnetic theory & 23 & 95.8K & Optics, Waves, Diffraction \\
Group theory & 22 & 100.0K & Groups, Automorphisms \\
Linear algebra & 22 & 128.3K & Matrices, Determinants \\
Probability theory & 20 & 114.6K & Random walk, Expectation \\
Algebraic systems & 19 & 109.9K & Functional equations \\
Mechanics & 19 & 103.6K & Forces, Motion, Energy \\
Thermodynamics & 19 & 74.2K &  Heat engines, Entropy \\
Differential equations & 18 & 89.6K & Substitution, Existence \\
Computer science & 18 & 34.2K & Complexity theory, Algorithms \\
Numerical analysis & 18 & 76.5K & Error analysis, Stability \\
Calculus & 17 & 96.3K & Convergence, Summation \\
Algebraic structures & 17 & 90.4K & Inequalities, Sets \\
Astronomy & 16 & 37.7K & Stellar populations, Orbits \\
Remaining 27 domains & 242 & 982.2K & Domains with $\leq$ 16 questions  \\
\midrule
All domains (51) & 1000 & 4.7M &  \sonek{}\\
\bottomrule
\end{tabular}
\label{tab:s1-se-domain_distribution}
\end{table*}

\subsubsection{Dataset composition for full 59K questions}

\begin{table}[htbp]
\centering
\caption{\textbf{Composition of full 59K questions.}
Thinking and response lengths are measured in tokens using the Qwen2.5-32B-Instruct tokenizer~\citep{qwen2024qwen25technicalreport}.
In addition to excluding our evaluation benchmark, AIME24, we also exclude AIME questions from 2022--2023 because we use these 90 questions during our development stage of \sone{}.}
\small
\begin{tabular}{>{\raggedright}p{4.7cm} p{5.4cm} p{1.2cm} p{1.2cm} p{1.2cm}}
\toprule
Source & Description & \#Samples & Avg. thinking length \\
\midrule
NuminaMATH~\citep{numina_math_datasets} & Math problems from online websites & 30660 & 4.1K \\
MATH~\citep{hendrycks2021measuringmathematicalproblemsolving} & Math problems from competitions & 11999 & 2.9K \\
OlympicArena~\citep{huang2024olympicarenabenchmarkingmultidisciplinecognitive} & Astronomy, Biology, Chemistry, Computer Science, Geography, Math, and Physics olympiad questions & 4250 & 3.2K\\
OmniMath~\citep{gao2024omnimathuniversalolympiadlevel} & Math problems from competitions & 4238 & 4.4K\\
AGIEval~\citep{zhong2023agievalhumancentricbenchmarkevaluating,ling2017programinductionrationalegeneration,hendrycks2021measuringmathematicalproblemsolving,liu2020logiqachallengedatasetmachine,zhong2019jecqalegaldomainquestionanswering,wang2021lsatprogresschallengescomplex} & English, Law, Logic and Math problems from the SAT, LSAT and other exams & 2385 & 1.2K\\
xword & Crossword puzzles & 999 & 0.7K \\
OlympiadBench~\citep{he2024olympiadbenchchallengingbenchmarkpromoting} & Math and Physics olympiad questions & 896 & 3.9K\\
AIME (1983-2021) & American Invitational Mathematics Examination & 890 & 4.7K \\
TheoremQA~\citep{chen2023theoremqatheoremdrivenquestionanswering} & Computer Science, Finance, Math, and Physics university-level questions relating to theorems  & 747 & 2.1K \\
USACO \citep{shi2024languagemodelssolveolympiad} & Code problems from the USA Computing Olympiad & 519 & 3.6K \\
JEEBench~\citep{arora2023llmsadvancedenoughchallenging} & Chemistry, Math, and Physics problems used in the university entrance examination of the Indian Institute of Technology & 515 & 2.9K \\
GPQA~\citep{rein2023gpqagraduatelevelgoogleproofqa} & PhD-Level Science Questions & 348 & 2.9K \\
SciEval~\citep{sun2024scievalmultilevellargelanguage} & Biology, Chemistry, and Physics problems from various sources & 227 & 0.7K \\
s1-prob & Stanford statistics qualifying exams & 182 & 4.0K \\
LiveCodeBench~\citep{jain2024livecodebenchholisticcontaminationfree} & Code problems from coding websites (LeetCode, AtCoder, and CodeForces) &  151 & 3.5K\\
s1-teasers & Math brain-teasers crawled from the Internet & 23 & 4.1K \\
\midrule
\textbf{All 59K questions} & Composite of the above datasets with reasoning traces and solutions & 59029 & 3.6K \\
\bottomrule
\end{tabular}
\label{tab:s1-se-ds}
\end{table}

\FloatBarrier

\paragraph{\sonek{} grading prompt} To grade whether an example is correct for our dataset selection in \S\ref{sec:s1-se-data}, we use the following prompt.
We use Claude 3.5 for grading, except for the final 1,000 samples, which we grade with Claude 3.7.

\begin{qualitativeBox}
{\footnotesize
\begin{verbatim}
You are an AI assistant for grading a science problem.
The user will provide you with the question itself, an attempt made by a student and the correct 
answer to the problem.
Your job is to judge whether the attempt is correct by comparing it with the correct answer.
If the expected solution concludes with a number or choice, there should be no ambiguity.
If the expected solution involves going through the entire reasoning process, you should judge 
the attempt based on whether the reasoning process is correct with correct answer if helpful.

The user will provide the attempt and the correct answer in the following format:

# Problem

{problem}

## Attempt

{attempt}

## Correct answer

{solution}

Explain your reasoning, and end your response on a new line with only "Yes" or "No" (without 
quotes).
\end{verbatim}
}
\end{qualitativeBox}

\paragraph{\sonek{} diversity selection}

\begin{algorithm}
\caption{Two-stage sampling for \sonek{}}
\label{alg:s1-se-twostage}
\begin{algorithmic}[1]
\State \textbf{Input:} $\mathcal{Q}$ := Set of 24,496 questions with features
\State \textbf{Output:} $\mathcal{S}$ := Set of 1,000 selected questions
\State $\mathcal{S} \gets \emptyset$ \hfill \textit{Initialize the output set (only tracks unique elements)}

\For{$q \in \mathcal{Q}$}
    \If{IsGeminiCorrect($q$) \textbf{and} (IsAIME($q$) \textbf{or} IsGPQA($q$))}
        \State $\mathcal{S} \gets \mathcal{S} \cup \{q\}$
        \State \hfill \textit{Select all correct AIME/GPQA solutions}
    \ElsIf{IsGeminiCorrect($q$) \textbf{and} IsMATH($q$) \textbf{and} ThinkingLength($q$) > 5600}
        \State $\mathcal{S} \gets \mathcal{S} \cup \{q\}$
        \State \hfill \textit{Select correct MATH500 solutions with long chains}
    \EndIf
\EndFor

\State $\mathcal{D} \gets$ All available domains
\State \hfill \textit{Initialize domain pool}

\While{$|\mathcal{S}| < 1000$}
    \State $d \gets$ RandomChoice($\mathcal{D}$)
    \State \hfill \textit{Randomly select a domain}
    \State $Q_d \gets$ Questions in domain $d$
    \State \hfill \textit{Get questions from this domain}
    \State ranks $\gets$ RankByThinkingLength($Q_d$)
    \State \hfill \textit{Rank by thinking length}
    \State weights $\gets 2^{-\text{ranks}}$
    \State \hfill \textit{Apply power-law weighting}
    \State $q \gets$ WeightedSample($Q_d$, weights)
    \State \hfill \textit{Sample favoring longer chains}
    \State $\mathcal{S} \gets \mathcal{S} \cup \{q\}$
    \State \hfill \textit{Add selected question}
    \State $Q_d \gets Q_d \setminus \{q\}$

    \If{$Q_d = \emptyset$}
        \State $\mathcal{D} \gets \mathcal{D} \setminus \{d\}$
        \State \hfill \textit{Remove exhausted domains}
    \EndIf
\EndWhile
\end{algorithmic}
\end{algorithm}

Algorithm~\ref{alg:s1-se-twostage} details our diversity selection procedure.
We also include samples from specific benchmarks we consider high-quality (\S\ref{sec:s1-se-data}).
None of the selected samples overlap with our final evaluation.

\paragraph{Decontamination} We filter samples by checking for 8-gram overlap between selected examples and our evaluation benchmarks (MATH500, GPQA Diamond, and AIME24).
We exclude any question with more than 8-gram overlap.

\subsection{Training details}
\label{sec:s1-se-details-training}

We finetune Qwen2.5-32B-Instruct~\citep{qwen2024qwen25technicalreport} for reasoning.
On math tasks, this model generally matches or outperforms the larger Qwen2.5-72B-Instruct~\citep{qwen2024qwen25technicalreport} and other open models~\citep{dubey2024llama3herdmodels,groeneveld2024olmo,muennighoff2024olmoeopenmixtureofexpertslanguage}.
We use token delimiters to separate the thinking stage from the answering stage, enclosing the thinking stage with \verb@<|im_start|>think@ and \verb@<|im_start|>answer@---both preceded and followed by a newline.

We train for 5 epochs with a batch size of 16 (315 gradient steps total) in bfloat16 precision.
We use a learning rate of $1e-5$ with linear warmup for 5\% of training (16 steps), then cosine decay to 0 over the remaining 299 steps.
We use AdamW~\citep{loshchilov2019decoupled} with $\beta_1=0.9$, $\beta_2=0.95$, and weight decay of $1e-4$.
We compute loss only on reasoning traces and solutions, not on questions.
We set the sequence length large enough to avoid truncating any samples.
Training takes 26 minutes on 16 NVIDIA H100 GPUs.
For ablations, we use identical hyperparameters except for the 59K model (\S\ref{sec:s1-se-dataabl}), where we use a batch size of 120 to process more data.

\paragraph{Evaluation} We select three representative reasoning benchmarks widely used in the field.
\textbf{AIME24} has 30 problems from the 2024 American Invitational Mathematics Examination (AIME) held from January 31 -- February 1, 2024.
AIME tests mathematical problem-solving with arithmetic, algebra, counting, geometry, number theory, probability, and other secondary school math topics.
All AIME answers are integers ranging from $000$ to $999$, inclusive.
\textbf{MATH500}~\citep{hendrycks2021measuringmathematicalproblemsolving} is a benchmark of competition math problems of varying difficulty.
We evaluate on the same 500 samples selected by OpenAI in prior work~\citep{lightman2023letsverifystepstep}.
\textbf{GPQA Diamond}~\citep{rein2023gpqagraduatelevelgoogleproofqa} consists of 198 PhD-level science questions from Biology, Chemistry, and Physics.
Experts with PhDs in the corresponding domains only achieved 69.7\% on GPQA Diamond~\citep{o1}.
We build on the ``lm-evaluation-harness'' framework~\citep{eval-harness,biderman2024lessons}.

\chapter{Supplementary materials for Chapter~\ref{chap:automated-ai-research}}
\section{Appendix}

\subsection{Additional idea examples}
\label{sec:more_examples}

We provide additional example ideas generated by Claude-4.5-Opus (Table~\ref{tab:more_grpo_examples_claude_4_5_opus}) and Claude-4.5-Sonnet (Table~\ref{tab:more_grpo_examples_claude_4_5_sonnet}) on the GRPO environment, including ideas with failed code execution.
Code execution errors tend to arise when an idea involves complex changes or requires external packages not supported in our execution environment.
Improving the execution agent to correctly implement more complex ideas (e.g., training auxiliary models or system-level optimizations) is an important direction for future work.

\begin{table*}[htbp]
\centering
\small
\setlength{\tabcolsep}{5pt}
\renewcommand{\arraystretch}{1}

\begin{tabular}{L{0.42\textwidth}|L{0.56\textwidth}}
\toprule
\textbf{Successful execution} & \textbf{Failed execution} \\
\hline

\textbf{[Experiment]} Sequence Position Weighted Trust Region: Apply tighter
sigmoid bounds to earlier tokens in the sequence (where errors compound) and
looser bounds to later tokens. Weight:
\texttt{position\_weight = 1 - 0.3 * (position / seq\_len)},
\texttt{effective\_deviation = 0.25 + 0.2 * position\_weight}. This accounts for
the sequential nature of autoregressive generation.

\vspace{6pt}
\textbf{[Code Changes]} Modify \texttt{grpo.py}: Initialize
\texttt{current\_cliprange = 0.2}, \texttt{ema\_clip\_fraction = 0.15}.
Standard momentum clip updates. Modify \texttt{compute\_grpo\_clip\_loss} in
\texttt{grpo\_utils.py}: After computing ratio on line 91 (shape:
\texttt{batch\_size x seq\_len}): \texttt{batch\_size, seq\_len = ratio.shape},
\texttt{positions = torch.arange(seq\_len, device=ratio.device).float()}
\texttt{.unsqueeze(0).expand(batch\_size, -1)},
\texttt{position\_weight = 1.0 - 0.3 * (positions / (seq\_len - 1 + 1e-6))},
\texttt{effective\_deviation = 0.25 + 0.2 * position\_weight}. Apply
position-aware sigmoid: \texttt{centered\_ratio = ratio - 1.0},
\texttt{bounded\_ratio = 1.0 + (2.0 * torch.sigmoid(centered\_ratio) - 1.0) * effective\_deviation}.
Use: \texttt{surr1 = bounded\_ratio * advantages},
\texttt{surr2 = torch.clamp(bounded\_ratio, 1 - cliprange, 1 + cliprange) * advantages},
\texttt{loss = -torch.min(surr1, surr2)}. Add metadata:
\texttt{metadata["mean\_effective\_deviation"] = effective\_deviation.mean().item()},
\texttt{metadata["early\_deviation"] = effective\_deviation[:, :seq\_len//4].mean().item()},
\texttt{metadata["late\_deviation"] = effective\_deviation[:, -seq\_len//4:].mean().item()}.

\vspace{5pt}
\textbf{Validation accuracy: 59.8}
\vspace{1pt}
&
\textbf{[Experiment]} Hierarchical Position-Group Trust Region: Apply trust
region at two hierarchical levels -- group level (shared within each response
group) and position level (varies along sequence). Groups with high internal
reward variance get tighter group-level bounds. Within groups, positions follow
the proven decay pattern. This captures both cross-sample and within-sample
structure. Formula: \texttt{group\_dev = 0.4 - 0.15 * tanh(group\_reward\_var / 0.3)},
\texttt{position\_factor = 1 - 0.2 * rel\_pos},
\texttt{effective\_dev = group\_dev * position\_factor}.

\vspace{6pt}
\textbf{[Code Changes]} Modify \texttt{grpo.py}: Initialize
\texttt{current\_cliprange = 0.2}, \texttt{ema\_clip\_fraction = 0.15}.
Standard momentum clip updates. Pass \texttt{group\_size} to function. Modify
\texttt{compute\_grpo\_clip\_loss} in \texttt{grpo\_utils.py}: Add parameter
\texttt{group\_size=8}. After computing ratio: \texttt{batch\_size, seq\_len = ratio.shape},
\texttt{n\_groups = batch\_size // group\_size}. Compute group reward variance
from advantages as proxy:
\texttt{adv\_grouped = advantages.view(n\_groups, group\_size, -1)},
\texttt{group\_adv\_var = adv\_grouped.var(dim=1, keepdim=True)},
\texttt{group\_adv\_var\_expanded = group\_adv\_var.expand(-1, group\_size, -1).reshape(advantages.shape)}.
Group-level deviation:
\texttt{group\_deviation = 0.4 - 0.15 * torch.tanh(group\_adv\_var\_expanded / 0.3)}.
Position factor:
\texttt{positions = torch.arange(seq\_len, device=ratio.device).float().unsqueeze(0)}
\texttt{.expand(batch\_size, -1)},
\texttt{rel\_pos = positions / (seq\_len - 1 + 1e-6)},
\texttt{position\_factor = 1.0 - 0.2 * rel\_pos}. Hierarchical deviation:
\texttt{effective\_deviation = group\_deviation * position\_factor},
\texttt{effective\_deviation = torch.clamp(effective\_deviation, 0.15, 0.45)}.
Apply: \texttt{centered\_ratio = ratio - 1.0},
\texttt{bounded\_ratio = 1.0 + (2.0 * torch.sigmoid(centered\_ratio) - 1.0) * effective\_deviation}.
Use: \texttt{surr1 = bounded\_ratio * advantages},
\texttt{surr2 = torch.clamp(bounded\_ratio, 1 - cliprange, 1 + cliprange) * advantages},
\texttt{loss = -torch.min(surr1, surr2)}. Add
\texttt{metadata["mean\_group\_var"] = group\_adv\_var.mean().item()},
\texttt{metadata["mean\_effective\_deviation"] = effective\_deviation.mean().item()}.
Log to wandb.
\\
\bottomrule
\end{tabular}

\caption{Additional examples on the GRPO environment. Ideas are generated by Claude-4.5-Opus during evolutionary search.}
\label{tab:more_grpo_examples_claude_4_5_opus}
\end{table*}

\newpage
\begin{table*}[htbp]
\centering
\scriptsize
\setlength{\tabcolsep}{5pt}
\renewcommand{\arraystretch}{1}

\begin{tabular}{L{0.48\textwidth}|L{0.48\textwidth}}
\toprule
\textbf{Successful execution} & \textbf{Failed execution} \\
\hline

\textbf{[Experiment]} Create a mathematical step-complexity aware reward shaping
where responses with more mathematical reasoning steps receive slightly higher
base rewards (1.05x for 3+ steps, 1.1x for 5+ steps) when correct, encouraging
thorough mathematical exposition without changing the core binary reward
structure.

\vspace{6pt}
\textbf{[Code Changes]} Modify \texttt{r1\_zero\_reward\_fn\_train} in
\texttt{drgrpo\_grader.py} to count reasoning steps by detecting mathematical
transitions (\texttt{"therefore"}, \texttt{"thus"}, \texttt{"so"}, \texttt{"="},
\texttt{"=>"}). When answer is correct, apply step-based multiplier:
\texttt{step\_multiplier = 1.0 + 0.05 * min(2, max(0, num\_steps - 2))} to get
multipliers \texttt{[1.0, 1.05, 1.1]}. Return
\texttt{\{"format\_reward": 1.0, "answer\_reward": answer\_reward, "reward": base\_reward * step\_multiplier\}}.
Set \texttt{--learning\_rate 3e-5} and \texttt{--loss\_type reinforce\_with\_baseline}.

\vspace{5pt}
\textbf{Validation accuracy: 65.6}
\vspace{1pt}

\vspace{15pt}
\textbf{[Experiment]} Add experience replay by maintaining a buffer of the top
20\% highest-reward rollouts from previous epochs and mixing them (25\% replay,
75\% new) with current rollouts during training, combined with \texttt{3e-5}
learning rate and \texttt{reinforce\_with\_baseline} for improved sample
efficiency.

\vspace{6pt}
\textbf{[Code Changes]} Modify \texttt{train\_loop} in \texttt{grpo.py} to
maintain \texttt{replay\_buffer} storing high-reward (\texttt{>0.8}) rollouts
from previous epochs. Each epoch, sample 25\% of training data from replay
buffer and 75\% from new rollouts. Update buffer by adding top 20\% of current
epoch's rollouts and removing oldest entries when buffer exceeds 1000 samples.
Set \texttt{--learning\_rate 3e-5} and \texttt{--loss\_type reinforce\_with\_baseline}.

\vspace{5pt}
\textbf{Validation accuracy: 39.4}
\vspace{1pt}

\vspace{15pt}
\textbf{[Experiment]} Implement response diversity rewards within groups where
responses to the same prompt receive bonus rewards (\texttt{0.05--0.15}) for
being dissimilar to other responses in their group, encouraging exploration of
different solution paths while maintaining the proven \texttt{group\_size=8} and
\texttt{3e-5} learning rate combination.

\vspace{6pt}
\textbf{[Code Changes]} Modify \texttt{compute\_group\_normalized\_rewards} in
\texttt{grpo\_utils.py} to compute pairwise similarity between responses in each
group using token-level Jaccard similarity. Add diversity bonus:
\texttt{diversity\_reward = 0.15 * (1 - max\_similarity\_in\_group)} to each
response's reward before advantage computation. Reshape responses into groups,
compute similarities, and add bonuses before advantage normalization. Set
\texttt{--learning\_rate 3e-5}, \texttt{--loss\_type reinforce\_with\_baseline},
\texttt{--group\_size 8}.

\vspace{5pt}
\textbf{Validation accuracy: 19.2}
\vspace{1pt}
&
\textbf{[Experiment]} Implement temporal difference advantage estimation where
advantages incorporate not just current rewards but also predicted future rewards
using a learned value function, combined with the proven \texttt{3e-5} learning
rate and \texttt{reinforce\_with\_baseline} for more accurate credit assignment.

\vspace{6pt}
\textbf{[Code Changes]} Add a value head to the policy model in \texttt{grpo.py}
that predicts expected future rewards. Compute TD advantages as
\texttt{advantages = rewards + gamma * next\_values - current\_values} with
\texttt{gamma=0.99}. Train the value function with MSE loss on observed returns.
Modify \texttt{compute\_group\_normalized\_rewards} to use TD advantages instead
of basic reward differences. Set \texttt{--learning\_rate 3e-5} and
\texttt{--loss\_type reinforce\_with\_baseline}.

\vspace{24pt}
\textbf{[Experiment]} Ensemble Decision Training with Voting Consensus:
Train the model using ensemble-style decision making where each rollout
generates multiple candidate responses, and the final training signal is based
on majority voting among responses. This encourages the model to develop more
robust and consistent reasoning patterns while maintaining diversity in solution
approaches.

\vspace{6pt}
\textbf{[Code Changes]} Modify \texttt{sample\_rollout} in \texttt{sample.py} to
generate 3 responses per prompt instead of 1, using different random seeds.
Implement voting consensus in \texttt{r1\_zero\_reward\_fn\_train}: if 2+
responses are correct, apply \texttt{+0.15} consensus bonus; if responses
disagree, apply \texttt{-0.05} uncertainty penalty. In \texttt{train\_loop} in
\texttt{grpo.py}, select the highest-voted response for training while using
consensus information to adjust learning rate:
\texttt{consensus\_lr = 3e-5 * (0.9 + 0.2 * consensus\_rate)}. Set
\texttt{group\_size=6}, \texttt{--loss\_type reinforce\_with\_baseline}.

\vspace{24pt}
\textbf{[Experiment]} Implement hierarchical advantage estimation where
advantages are computed at both token-level and sequence-level, with token-level
advantages weighted by their position importance (higher weights for
mathematical expressions and final answers), combined with the successful
\texttt{3e-5} learning rate and \texttt{reinforce\_with\_baseline}.

\vspace{6pt}
\textbf{[Code Changes]} Modify \texttt{grpo\_microbatch\_train\_step} in
\texttt{grpo\_utils.py} to create position importance weights that assign
\texttt{2.0x} weight to tokens containing mathematical symbols
(\texttt{\textbackslash frac}, \texttt{+}, \texttt{-}, \texttt{*}, \texttt{=})
and \texttt{1.5x} weight to answer sections. Compute both sequence-level
advantages (current) and token-level advantages, then combine as
\texttt{final\_advantages = 0.6 * sequence\_advantages + 0.4 * token\_advantages}.
Set \texttt{--learning\_rate 3e-5} and \texttt{--loss\_type reinforce\_with\_baseline}.
\\
\bottomrule
\end{tabular}

\caption{Additional examples on the GRPO environment. Ideas are generated by Claude-4.5-Sonnet during evolutionary search.}
\label{tab:more_grpo_examples_claude_4_5_sonnet}
\end{table*}

\FloatBarrier

\newpage
We next present the top-performing ideas from Claude-4.5-Opus, Claude-4.5-Sonnet, and GPT-5 on the nanoGPT environment.

\begin{tcolorbox}
[breakable,colback=thesisRosyTaupe!10,colframe=thesisRosyTaupe!30,title=\textbf{Claude-4.5-Opus Idea on nanoGPT (Validation Loss: 3.1407)}]
\textbf{[Experiment]} Wider SwiGLU (5x) with MLP Output Scaling (Init 0.97), Skip Connections Every 4 and 8 Layers with Learnable Weights (Init 0.52 and 0.31), Separate Attention/MLP Scales (Init 0.98), Higher LR (0.00168), Reduced Weight Decay (0.065), Warmup 173 iters, Lower Min LR (0.03x), Cosine Annealing, EMA, Untied Embeddings, and Beta2=0.99

Make the dual skip connection weights learnable parameters initialized at proven good values. This allows the model to adapt skip weights during training while combining with separate attention/MLP residual scales.

\textbf{[Code Changes]}
\begin{itemize}
  \item Change \texttt{warmup\_iters = 256} to \texttt{warmup\_iters = 173} in Hyperparameters class
  \item Change \texttt{weight\_decay = 0.1} to \texttt{weight\_decay = 0.065} in Hyperparameters class
  \item Change \texttt{learning\_rate = 0.0015} to \texttt{learning\_rate = 0.00168} in Hyperparameters class
  \item In \texttt{GPT.\_\_init\_\_}, add after transformer dict:
\begin{verbatim}
self.skip_weight_4 = nn.Parameter(torch.tensor(0.52))
self.skip_weight_8 = nn.Parameter(torch.tensor(0.31))
\end{verbatim}
  \item In \texttt{Block.\_\_init\_\_}, add: \texttt{self.attn\_scale = nn.Parameter(torch.tensor(0.98))} and \texttt{self.mlp\_scale = nn.Parameter(torch.tensor(0.98))}
  \item In \texttt{Block.forward}, change to:
\begin{verbatim}
def forward(self, x):
    x = x + self.attn_scale * self.attn(rmsnorm(x))
    x = x + self.mlp_scale * self.mlp(rmsnorm(x))
    return x
\end{verbatim}
  \item In \texttt{Block.forward\_with\_cache}, change to:
\begin{verbatim}
def forward_with_cache(self, x, cache):
    attn_out, new_cache = self.attn.forward_with_cache(rmsnorm(x),
    cache=cache)
    x = x + self.attn_scale * attn_out
    x = x + self.mlp_scale * self.mlp(rmsnorm(x))
    return x, new_cache
\end{verbatim}
  \item In \texttt{MLP.\_\_init\_\_}, replace lines 81--82 with:
\begin{verbatim}
self.c_fc = nn.Linear(config.n_embd, 5 * config.n_embd, bias=False)
self.c_gate = nn.Linear(config.n_embd, 5 * config.n_embd, bias=False)
self.c_proj = nn.Linear(5 * config.n_embd, config.n_embd, bias=False)
self.output_scale = nn.Parameter(torch.tensor(0.97))
\end{verbatim}
  \item In \texttt{MLP.forward}, replace with:
\begin{verbatim}
def forward(self, x):
    gate = F.silu(self.c_gate(x))
    x = self.c_fc(x) * gate
    x = self.c_proj(x) * self.output_scale
    return x
\end{verbatim}
  \item In \texttt{GPT.\_\_init\_\_}, remove line 132: \texttt{self.transformer.wte.weight = self.lm\_head.weight}
  \item Remove line 131: \texttt{self.lm\_head.LLMC\_SKIP\_INIT = 1}
  \item Modify \texttt{\_init\_weights} to add: \texttt{if isinstance(module, nn.Linear): torch.nn.init.normal\_(module.weight, mean=0.0, std=0.02)}
  \item Change optimizer betas on line 402 to \texttt{betas=(0.9, 0.99)}
  \item Modify \texttt{get\_lr} function:
\begin{verbatim}
def get_lr(it):
    assert it <= args.num_iterations
    if it < args.warmup_iters:
        return args.learning_rate * (it+1) / args.warmup_iters
    min_lr = 0.03 * args.learning_rate
    decay_ratio = (it - args.warmup_iters) /
    (args.num_iterations - args.warmup_iters)
    return min_lr + 0.5 * (args.learning_rate - min_lr) *
    (1.0 + math.cos(math.pi * decay_ratio))
\end{verbatim}
  \item In \texttt{GPT.forward}, replace the block loop with:
\begin{verbatim}
layer_outputs = []
for i, block in enumerate(self.transformer.h):
    if i >= 4 and i %
        x = x + self.skip_weight_4 * layer_outputs[i-4]
    if i >= 8 and i %
        x = x + self.skip_weight_8 * layer_outputs[i-8]
    x = block(x)
    layer_outputs.append(x)
\end{verbatim}
  \item In \texttt{GPT.forward\_with\_cache}, replace the block loop with:
\begin{verbatim}
layer_outputs = []
for i, block in enumerate(self.transformer.h):
    if i >= 4 and i %
        x = x + self.skip_weight_4 * layer_outputs[i-4]
    if i >= 8 and i %
        x = x + self.skip_weight_8 * layer_outputs[i-8]
    x, new_cache = block.forward_with_cache(x, cache=caches[i])
    new_caches.append(new_cache)
    layer_outputs.append(x)
\end{verbatim}
  \item After model initialization, add: \texttt{ema\_model = \{k: v.clone() for k, v in raw\_model.state\_dict().items()\}} and \texttt{ema\_decay = 0.999}
  \item After \texttt{optimizer.step()}, add: \texttt{for k, v in raw\_model.state\_dict().items(): ema\_model[k].mul\_(ema\_decay).add\_(v, alpha=1-ema\_decay)}
  \item Before validation, add: \texttt{orig\_state = \{k: v.clone() for k, v in raw\_model.state\_dict().items()\}; raw\_model.load\_state\_dict(ema\_model)}
  \item After validation, add: \texttt{raw\_model.load\_state\_dict(orig\_state)}
\end{itemize}
\end{tcolorbox}

\newpage
\begin{tcolorbox}
[breakable,colback=thesisRosyTaupe!10,colframe=thesisRosyTaupe!30,title=\textbf{Claude-4.5-Sonnet Idea on nanoGPT (Validation Loss: 3.2081)}]
\label{fig:sonnet_nanogpt_example}
\textbf{[Experiment]} Two-phase weight decay (0.1170$\rightarrow$0.0210 at 59.65\%) +
30.45\% plateau + LR 0.001550 + warmup 197 + two-phase grad clip (1.054$\rightarrow$0.916 at 59.65\%) +
quadratic min\_lr interpolation (0.0113x at 59.65\%, 0.0075x at end via quadratic) +
progressive EMA (0.999$\rightarrow$0.9992 linear over training) + exponential warmup +
cosine LR + beta2=0.99

\vspace{6pt}
Use smooth quadratic interpolation for \texttt{min\_lr} during low-WD phase AND
progressive EMA that gradually increases from \texttt{0.999} to \texttt{0.9992}
linearly throughout training. Early training benefits from faster EMA tracking,
while later training gets heavier smoothing. Use conservative settings:
WD \texttt{0.1170}/\texttt{0.0210}, extended plateau \texttt{30.45\%}, moderate LR
\texttt{0.001550}, longest warmup \texttt{197}, tight grad clip
\texttt{1.054$\rightarrow$0.916}.

\vspace{6pt}
\textbf{[Code Changes]} Modify line 326 to change \texttt{warmup\_iters = 256} to
\texttt{warmup\_iters = 197}. Modify line 325 to change
\texttt{learning\_rate = 0.0015} to \texttt{learning\_rate = 0.001550}. Modify
line 402 to change \texttt{betas=(0.9, 0.95)} to \texttt{betas=(0.9, 0.99)}.
Modify the \texttt{get\_lr} function: replace lines 408--414 with:
\begin{verbatim}
if it < args.warmup_iters:
    progress = (it + 1) / args.warmup_iters
    return args.learning_rate * (1.0 - math.exp(-5.0 * progress))
plateau_end = int(0.3045 * args.num_iterations)
if it < plateau_end:
    return args.learning_rate
overall_progress = it / args.num_iterations
decay_ratio = (it - plateau_end) / (args.num_iterations - plateau_end)
coeff = 0.5 * (1.0 + math.cos(math.pi * decay_ratio))
if overall_progress <= 0.5965:
    min_lr_factor = 0.0113
else:
    phase2_progress = (overall_progress - 0.5965) / (1.0 - 0.5965)
    min_lr_factor = 0.0113 - (0.0113 - 0.0075) * (phase2_progress ** 2)
min_lr = min_lr_factor * args.learning_rate
return min_lr + coeff * (args.learning_rate - min_lr)
\end{verbatim}
Modify line 527 to:
\begin{verbatim}
progress = step / args.num_iterations
current_clip = 0.916 if progress > 0.5965 else 1.054
norm = torch.nn.utils.clip_grad_norm_(model.parameters(), current_clip)
\end{verbatim}
After line 529, add:
\begin{verbatim}
progress = step / args.num_iterations
current_wd = 0.0210 if progress > 0.5965 else 0.1170
for param_group in optimizer.param_groups:
    param_group['weight_decay'] = current_wd
\end{verbatim}
After line 387, add:
\begin{verbatim}
ema_model = {name: param.clone().detach() for name, param in
raw_model.named_parameters()}
\end{verbatim}
After line 533, add:
\begin{verbatim}
if step > 0:
    progress = step / args.num_iterations
    ema_decay = 0.999 + 0.0002 * progress
    for name, param in raw_model.named_parameters():
        ema_model[name].mul_(ema_decay).add_(param.data, alpha=1 - ema_decay)
\end{verbatim}
Before line 483, add:
\begin{verbatim}
original_params = {name: param.data.clone() for name, param in
raw_model.named_parameters()}
for name, param in raw_model.named_parameters():
    param.data.copy_(ema_model[name])
\end{verbatim}
After line 509, add:
\begin{verbatim}
for name, param in raw_model.named_parameters():
    param.data.copy_(original_params[name])
\end{verbatim}
\end{tcolorbox}

\vspace{20pt}
\begin{tcolorbox}
[breakable,colback=thesisRosyTaupe!10,colframe=thesisRosyTaupe!30,title=\textbf{GPT-5 Idea on nanoGPT (Validation Loss: 3.1697)}]
\label{fig:gpt5_nanogpt_example}
\textbf{[Experiment]} SwiGLU-3.5x + Residual Alphas + Min-Floor Cosine +
Per-step Beta2 Linear Decay + 3-Group AdamW + Debiased EMA

\vspace{6pt}
\textbf{[Code Changes]}
\begin{itemize}
  \item \textbf{Hyperparameters:} \texttt{hidden\_factor=3.5},
  \texttt{warmup\_iters=256}, \texttt{lr\_peak\_factor=1.10},
  \texttt{min\_lr\_factor=0.02}, \texttt{beta2\_start=0.99},
  \texttt{beta2\_end=0.95}, \texttt{wd\_decay=0.1}, \texttt{wd\_embed=0.01},
  \texttt{ema\_decay=0.9995}, \texttt{ema\_warmup\_steps=256}.
  \item \textbf{MLP:} SwiGLU; Block alphas init \texttt{0.9}.
  \item \textbf{Optimizer:} 3-group AdamW.
  \item \textbf{LR:} warmup to peak; cosine to floor as before.
  \item \textbf{Per-step beta2 update:} After setting lr each step, set
  \[
  \texttt{beta2} =
  \texttt{beta2\_start}
  + \big(\texttt{beta2\_end} - \texttt{beta2\_start}\big)\,
  \texttt{min}\!\left(1.0,\frac{\texttt{it}+1}{\texttt{args.num\_iterations}}\right);
  \]
  update all \texttt{param\_groups} betas.
  \item \textbf{EMA:} maintain \texttt{ema\_params} with debiasing at eval
  (divide by \texttt{1 - ema\_decay**ema\_step}), then restore.
\end{itemize}
\end{tcolorbox}

The best-performing ideas on nanoGPT combine extensive hyperparameter tuning with architecture modifications.
We also highlight several ``atomic'' algorithmic ideas that execute successfully.

\textbf{Examples from Claude-4.5-Opus}

\begin{itemize}
    \item \textbf{Head-Wise Attention Output Scaling} Add learnable per-head scaling factors to attention, allowing different heads to contribute with different magnitudes to the output.\\
    \textbf{Validation loss: 3.2386}

    \item \textbf{Learned Residual Connection Weights} Add learnable scalar weights for each residual connection that are initialized to 1.0, allowing the model to learn optimal residual scaling during training.\\
    \textbf{Validation loss: 3.2517}

    \item \textbf{Mixture of Embeddings with Position} Learn to mix token embeddings and position embeddings with a content-dependent weight, allowing the model to dynamically balance positional vs semantic information per token.\\
    \textbf{Validation loss: 3.2497}

    \item \textbf{Shared Input-Output Embedding with Learned Asymmetry} Keep weight tying but add a small learned transformation on the output side, providing the benefits of weight tying while allowing output-specific adaptation.\\
    \textbf{Validation loss: 3.2499}

    \item \textbf{Gated Final Normalization} Replace the final RMSNorm before \texttt{lm\_head} with a gated version where a learned gate controls how much normalization is applied vs passing the raw representation.\\
    \textbf{Validation loss: 3.2503}

    \item \textbf{Position-Aware MLP Gating} Gate the MLP output based on position information, allowing the model to learn position-dependent processing depth.\\
    \textbf{Validation loss: 3.2506}

    \item \textbf{Learned Residual Connection Weights} Add learnable scalar weights for each residual connection that are initialized to 1.0, allowing the model to learn optimal residual scaling during training.\\
    \textbf{Validation loss: 3.2517}

    \item \textbf{Grouped Token Embeddings} Group the vocabulary into clusters and add a learned embedding per cluster on top of token embeddings, providing hierarchical vocabulary structure.\\
    \textbf{Validation loss: 3.2521}
\end{itemize}

\vspace{20pt}
We also present several Claude-4.5-Opus ideas on the nanoGPT environment that did not execute successfully.

\begin{itemize}
    \item \textbf{Soft Layer Repetition}
Allow the model to softly repeat computation through layers by adding a learned gate that mixes the current layer's input back into its output, simulating variable depth.

    \item \textbf{Causal Context Compression}
Before each attention layer, apply a learned compression that mixes local context (previous 2-3 tokens) into the current representation, providing implicit local context without convolutions.

    \item \textbf{Attention Head Specialization via Orthogonal Loss}
Add a soft penalty that encourages different attention heads to attend to different patterns by penalizing similarity between head outputs.

    \item \textbf{Skip Connections with Learned Residual Weights}
Combine skip connections with learned residual weights. The skip connections provide alternative gradient paths while learned weights allow adaptive scaling.

    \item \textbf{Token Difficulty-Aware Loss Weighting}
Weight the loss contribution of each token based on the model's uncertainty (entropy) at that position, focusing learning on difficult tokens while not over-optimizing easy ones.
\end{itemize}

\newpage
\subsection{Code execution examples}
\label{sec:code_examples}

We present a full code execution example from Claude-4.5-Sonnet on the GRPO environment, showing the natural language idea followed by the generated code implementation.

\textbf{Example 1}

\textbf{[Experiment]} Create mathematical working memory simulation by
maintaining a context buffer of mathematical facts, definitions, and
intermediate results during problem solving. This buffer gets updated as the
model works through problems and provides additional context for subsequent
mathematical steps, simulating how humans maintain mathematical working memory
during complex calculations.

\vspace{6pt}
\textbf{[Code Changes]} Modify the prompt template in \texttt{sample.py} to
include a ``Mathematical Context'' section that accumulates relevant
mathematical facts during problem solving. Implement context buffer management
in \texttt{train\_loop} in \texttt{grpo.py} that maintains recent mathematical
definitions, formulas, and intermediate calculations across problems in the same
batch. Update the buffer with key mathematical expressions extracted from each
response. Use enhanced prompts that include relevant context from the buffer.
Set \texttt{--learning\_rate 2.5e-5} and \texttt{--loss\_type reinforce\_with\_baseline}.

\vspace{6pt}
\textbf{Validation accuracy: 58.0}

\begin{CodeBlock}{Example 1 code diff}
--- repo_variants_full_pipeline_bsz50_fixed2_epoch7/idea_43/sample.py
+++ repo_variants_full_pipeline_bsz50_fixed2_epoch7/idea_43/sample.py
@@ -2,6 +2,7 @@ from vllm import LLM, SamplingParams
 import json
 from tqdm import tqdm
 import random
+import re

 def load_prompt_template(prompt_path = "prompts/r1_zero.prompt"):
     with open(prompt_path, 'r') as f:
@@ -20,12 +21,52 @@ def get_model_and_sampling_params(model_path = "Qwen/Qwen2.5-Math-1.5B", token

     return llm, sampling_params

-def load_dataset_with_prompt_template(prompt_template, dataset_path="/juice5b/scr5b/nlp/aihinton/CS336_data/a5-alignment/MATH/test.jsonl"):
+def extract_mathematical_expressions(text):
+    """Extract mathematical expressions, formulas, and key facts from text."""
+    expressions = []
+
+    # Extract LaTeX math expressions
+    latex_patterns = [
+        r'\$[^$]+\$',  # Inline math
+        r'\\[a-zA-Z]+\{[^}]+\}',  # LaTeX commands
+        r'\\frac\{[^}]+\}\{[^}]+\}',  # Fractions
+        r'\\sqrt\{[^}]+\}',  # Square roots
+        r'[a-zA-Z]\s*=\s*[^,\n\.]+',  # Variable assignments
+    ]
+
+    for pattern in latex_patterns:
+        matches = re.findall(pattern, text)
+        expressions.extend(matches)
+
+    # Extract definitions and key facts
+    definition_patterns = [
+        r'[A-Z][a-z]+ is defined as [^\.]+\.',
+        r'Let [a-zA-Z] be [^\.]+\.',
+        r'Given that [^\.]+\.',
+        r'Since [^\.]+\.',
+        r'Therefore [^\.]+\.',
+    ]
+
+    for pattern in definition_patterns:
+        matches = re.findall(pattern, text)
+        expressions.extend(matches)
+
+    return list(set(expressions))  # Remove duplicates
+
+def create_context_prompt(base_template, context_items):
+    """Create a prompt with mathematical context."""
+    if context_items:
+        context_section = "Mathematical Context:\n" + "\n".join(f"- {item}" for item in context_items[-10:]) + "\n\n"
+        return context_section + base_template
+    return base_template
+
+def load_dataset_with_prompt_template(prompt_template, dataset_path="/juice5b/scr5b/nlp/aihinton/CS336_data/a5-alignment/MATH/test.jsonl", context_items=None):
     with open(dataset_path, 'r') as f:
         dataset = [json.loads(line) for line in f]
     print(f"Loaded {len(dataset)} examples from {dataset_path}")

-    prompts = [prompt_template.format(question=example["problem"]) for example in dataset]
+    enhanced_template = create_context_prompt(prompt_template, context_items)
+    prompts = [enhanced_template.format(question=example["problem"]) for example in dataset]
     answers = [example["answer"] for example in dataset]
     return prompts, answers

@@ -93,4 +134,42 @@
     if return_rewards:
         return batched_prompts, batched_answers, batched_responses, batched_rewards
     return results
-
+
+class MathContextBuffer:
+    """Buffer to maintain mathematical working memory across problems."""
+
+    def __init__(self, max_size=50):
+        self.buffer = []
+        self.max_size = max_size
+
+    def add_expressions(self, expressions):
+        """Add new mathematical expressions to the buffer."""
+        for expr in expressions:
+            if expr and len(expr.strip()) > 3:  # Filter out very short expressions
+                self.buffer.append(expr.strip())
+
+        # Keep only the most recent items
+        if len(self.buffer) > self.max_size:
+            self.buffer = self.buffer[-self.max_size:]
+
+    def get_relevant_context(self, query_text, max_items=10):
+        """Get the most relevant context items for a given query."""
+        if not self.buffer:
+            return []
+
+        # Simple relevance scoring based on shared mathematical terms
+        scored_items = []
+        query_lower = query_text.lower()
+
+        for item in self.buffer[-20:]:  # Consider recent items
+            score = sum(1 for word in item.lower().split() if word in query_lower)
+            if score > 0:
+                scored_items.append((score, item))
+
+        # Return top items sorted by relevance
+        scored_items.sort(reverse=True, key=lambda x: x[0])
+        return [item[1] for item in scored_items[:max_items]]
+
+    def clear(self):
+        """Clear the context buffer."""
+        self.buffer = []
--- repo_variants_full_pipeline_bsz50_fixed2_epoch7/idea_43/grpo.py
+++ repo_variants_full_pipeline_bsz50_fixed2_epoch7/idea_43/grpo.py
@@ -1,6 +1,6 @@
 from transformers import AutoModelForCausalLM, AutoTokenizer
 from utils import tokenize_prompt_and_output, get_response_log_probs
-from sample import load_prompt_template, load_dataset_with_prompt_template, sample_rollout
+from sample import load_prompt_template, load_dataset_with_prompt_template, sample_rollout, MathContextBuffer, extract_mathematical_expressions, create_context_prompt
 from drgrpo_grader import r1_zero_reward_fn_train
 from evaluate import r1_zero_reward_fn_eval, evaluate_vllm
 from grpo_utils import compute_group_normalized_rewards, grpo_microbatch_train_step
@@ -66,6 +66,9 @@ def evaluate_model(policy_model, vllm_model, eval_prompts, eval_answers, eval_s
 def train_loop(model, train_prompts, train_answers, learning_rate, grpo_steps, train_steps_per_rollout, output_dir, batch_size, gradient_accumulation_steps = 4, group_size = 2, rollout_subset_size = 256, device = "cuda", logging_steps = 20, saving_steps = 4000, eval_epochs = 5, eval_prompts = None, eval_answers = None, sampling_params = None, eval_vllm_model = None, cliprange = 0.2, loss_type = "reinforce_with_baseline"):
     model.to(device)
     training_steps = grpo_steps
+
+    # Initialize mathematical context buffer
+    context_buffer = MathContextBuffer(max_size=100)
     optimizer = torch.optim.AdamW(model.parameters(), lr=learning_rate, weight_decay=0.0, betas=(0.9, 0.95))
     global_step = 0  # Initialize global step counter

@@ -85,8 +88,31 @@ def train_loop(model, train_prompts, train_answers, learning_rate, grpo_steps,
         load_policy_into_vllm_instance(model, vllm_model)

         ## sample rollouts
+        # Get enhanced prompts with mathematical context for this epoch
         print ("Sampling rollouts for epoch: ", epoch)
-        rollout_prompts, rollout_answers, rollout_responses, rollout_rewards = sample_rollout(vllm_model, r1_zero_reward_fn_train, train_prompts, train_answers, G=group_size, eval_sampling_params=eval_sampling_params, subset_size=rollout_subset_size, return_rewards=True, batch_size=512)
+
+        # Create enhanced prompts with context for this batch
+        enhanced_prompts = []
+        base_template = load_prompt_template()
+
+        # Select subset of problems for this epoch
+        if rollout_subset_size is not None:
+            indices = random.sample(range(len(train_prompts)), rollout_subset_size)
+            epoch_prompts = [train_prompts[i] for i in indices]
+            epoch_answers = [train_answers[i] for i in indices]
+        else:
+            epoch_prompts = train_prompts
+            epoch_answers = train_answers
+
+        # Create context-enhanced prompts
+        for prompt in epoch_prompts:
+            relevant_context = context_buffer.get_relevant_context(prompt, max_items=8)
+            enhanced_prompt = create_context_prompt(base_template.format(question=prompt.split("Question: ")[-1]), relevant_context)
+            enhanced_prompts.append(enhanced_prompt)
+
+        # Sample with enhanced prompts (need to adapt sample_rollout for direct prompt input)
+        rollout_prompts, rollout_answers, rollout_responses, rollout_rewards = sample_rollout_with_prompts(vllm_model, r1_zero_reward_fn_train, enhanced_prompts, epoch_answers, G=group_size, eval_sampling_params=eval_sampling_params, return_rewards=True, batch_size=512)
+
         # Randomly sample 2 rollouts to print
         indices = random.sample(range(len(rollout_prompts)), 2)
         print ("Example rollouts:")
@@ -96,6 +122,15 @@ def train_loop(model, train_prompts, train_answers, learning_rate, grpo_steps,
             print(f"Response: {rollout_responses[idx]}")
             print(f"Reward: {rollout_rewards[idx]}")
             print(f"Ground truth: {rollout_answers[idx]}")
+
+        # Update context buffer with mathematical expressions from responses
+        print("Updating mathematical context buffer...")
+        for response in rollout_responses:
+            if rollout_rewards[rollout_responses.index(response)] > 0:  # Only add from correct responses
+                mathematical_expressions = extract_mathematical_expressions(response)
+                if mathematical_expressions:
+                    context_buffer.add_expressions(mathematical_expressions)
+
         rollout_tokenized = tokenize_prompt_and_output(rollout_prompts, rollout_responses, tokenizer)
         rollout_data_loader = create_data_loader(rollout_tokenized, batch_size=batch_size, shuffle=False)

@@ -196,6 +231,47 @@ def train_loop(model, train_prompts, train_answers, learning_rate, grpo_steps,
                     }, step=global_step)
                     # print ("Global Step: ", global_step, "Loss: ", loss.item(), "Entropy: ", entropy.mean().item(), "Clip fraction: ", metadata.get("clip_fraction", 0.0))

+def sample_rollout_with_prompts(
+    vllm_model,
+    reward_fn,
+    prompts,
+    answers,
+    G,
+    eval_sampling_params,
+    return_rewards=False,
+    batch_size=64
+):
+    """Sample rollouts using provided prompts directly."""
+    # Create batched prompts by repeating each prompt G times
+    batched_prompts = []
+    batched_answers = []
+    for prompt, answer in zip(prompts, answers):
+        batched_prompts.extend([prompt] * G)
+        batched_answers.extend([answer] * G)
+
+    # Process in batches to avoid OOM
+    all_outputs = []
+    for i in range(0, len(batched_prompts), batch_size):
+        batch_prompts = batched_prompts[i:i + batch_size]
+        batch_outputs = vllm_model.generate(batch_prompts, eval_sampling_params)
+        all_outputs.extend(batch_outputs)
+
+    # Process results
+    batched_responses = []
+    batched_rewards = []
+    total_rewards = 0
+
+    for output, answer in tqdm(zip(all_outputs, batched_answers)):
+        generated_answer = output.outputs[0].text
+        reward = reward_fn(generated_answer, answer)
+        batched_responses.append(generated_answer)
+        batched_rewards.append(reward["reward"])
+        total_rewards += reward["reward"]
+
+    print(f"Accuracy of sampled rollouts: {total_rewards}/{len(batched_prompts)} = {total_rewards / len(batched_prompts) * 100}%
+
+    return batched_prompts, batched_answers, batched_responses, batched_rewards
+
 if __name__ == "__main__":
     parser = argparse.ArgumentParser()
     parser.add_argument("--model_path", type=str, default="Qwen/Qwen2.5-Math-1.5B")
@@ -203,7 +279,7 @@ if __name__ == "__main__":
     parser.add_argument("--train_dataset_path", type=str, default="../MATH/train.jsonl")
     parser.add_argument("--eval_dataset_path", type=str, default="../MATH/test.jsonl")
     parser.add_argument("--output_dir", type=str, default="ckpts/")
-    parser.add_argument("--learning_rate", type=float, default=1e-5)
+    parser.add_argument("--learning_rate", type=float, default=2.5e-5)
     parser.add_argument("--grpo_steps", type=int, default=200)
     parser.add_argument("--group_size", type=int, default=8)
     parser.add_argument("--rollout_subset_size", type=int, default=256)
@@ -212,7 +288,7 @@ if __name__ == "__main__":
     parser.add_argument("--gradient_accumulation_steps", type=int, default=16)
     parser.add_argument("--batch_size", type=int, default=4)
     parser.add_argument("--cliprange", type=float, default=0.2)
-    parser.add_argument("--loss_type", type=str, default="grpo_clip")
+    parser.add_argument("--loss_type", type=str, default="reinforce_with_baseline")
     parser.add_argument("--wandb_project", type=str, default="grpo-math")
     parser.add_argument("--wandb_name", type=str, default="grpo_clip_1")
     args = parser.parse_args()
--- repo_variants_full_pipeline_bsz50_fixed2_epoch7/idea_43/run_job.sh
+++ repo_variants_full_pipeline_bsz50_fixed2_epoch7/idea_43/run_job.sh
@@ -21,7 +21,7 @@ timeout 2h uv run   \
     --index https://download.pytorch.org/whl/cu128 \
     --index-strategy unsafe-best-match \
     python grpo.py \
-        --learning_rate 1e-5 \
+        --learning_rate 2.5e-5 \
         --grpo_steps 20 \
         --group_size 8 \
         --rollout_subset_size 128 \
@@ -30,7 +30,7 @@ timeout 2h uv run   \
         --gradient_accumulation_steps 16 \
         --batch_size 4 \
         --cliprange 0.2 \
-        --loss_type grpo_clip \
+        --loss_type reinforce_with_baseline \
         --wandb_name $wandb_name

 echo "Experiment finished successfully!"
\end{CodeBlock}

\section{Reinforcement learning from execution reward}
\label{sec:rl-from-execution-reward}

Unlike evolutionary search, reinforcement learning shapes model behavior through gradient updates.
Despite recent success on verifiable domains like math and coding~\cite{r1}, RL's effectiveness on open-ended AI research remains unclear.
We explore whether the automated executor can serve as a reward function to directly finetune LLMs for more effective idea generation via RL.
We detail our implementation, experiment setup, and analysis of training dynamics.

\begin{figure}[!htbp]
  \centering
  \subfigure{\includegraphics[width=0.495\linewidth]{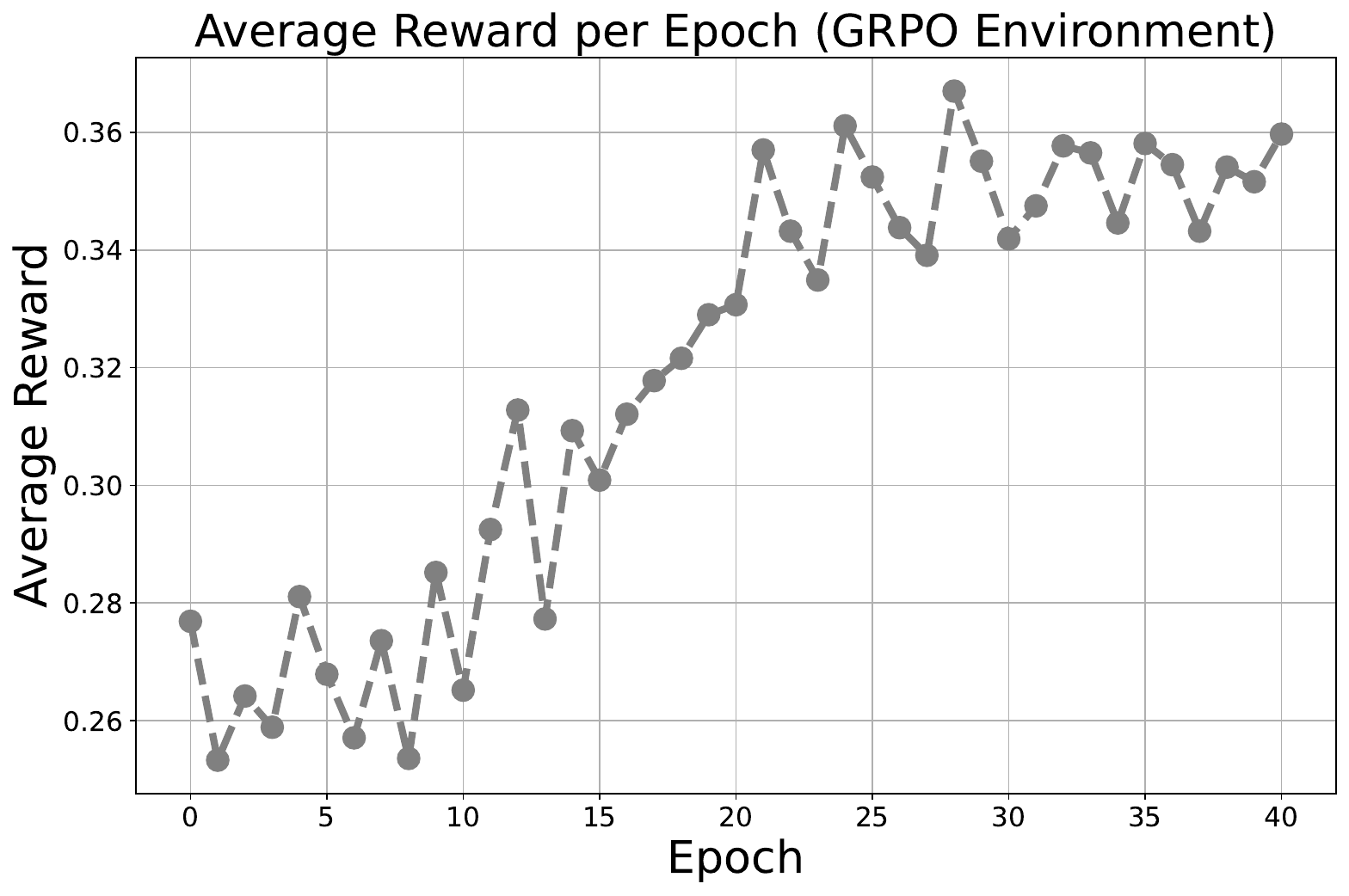}}
  \subfigure{\includegraphics[width=0.495\linewidth]{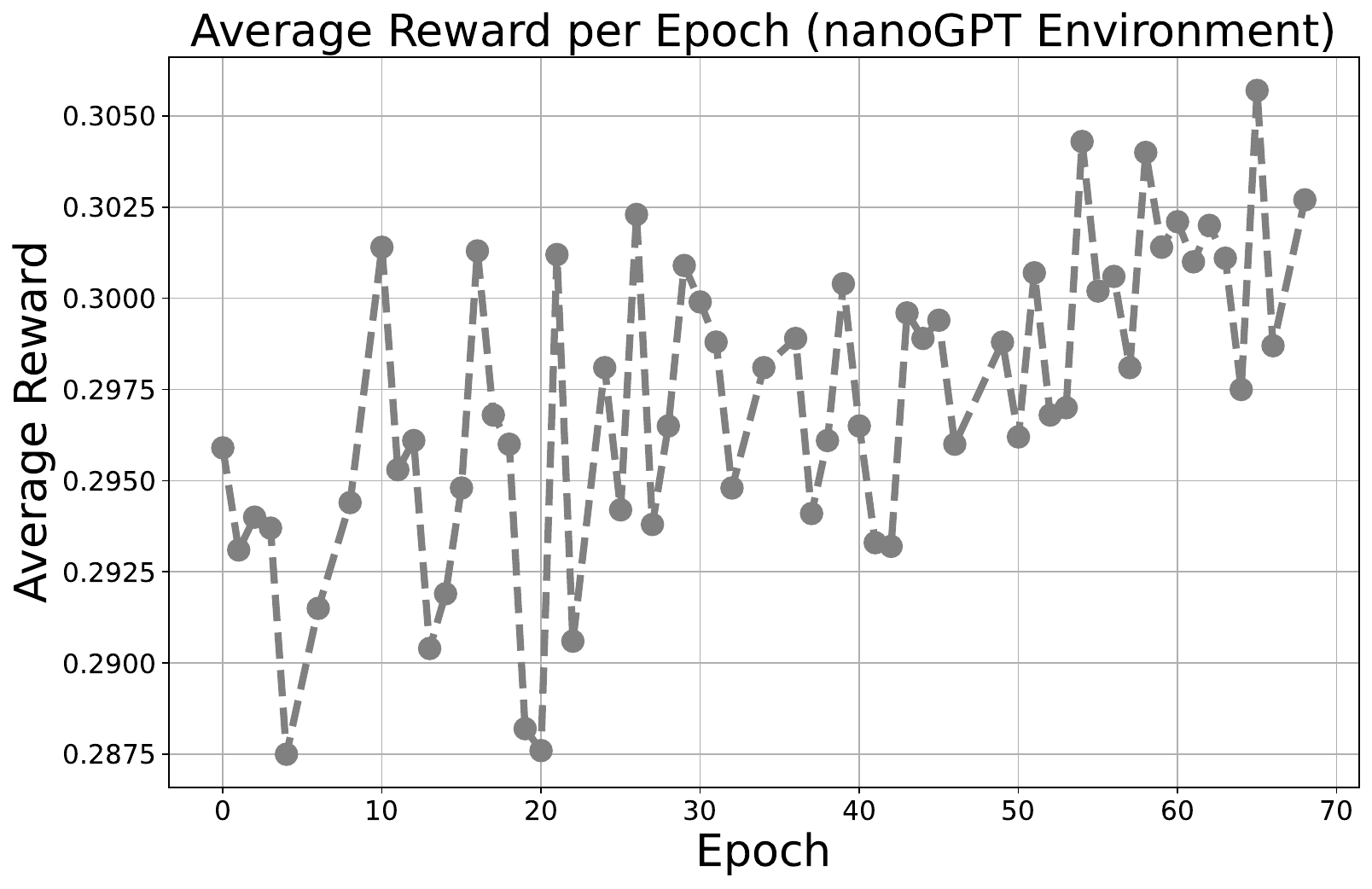}}

    \subfigure{\includegraphics[width=0.495\linewidth]{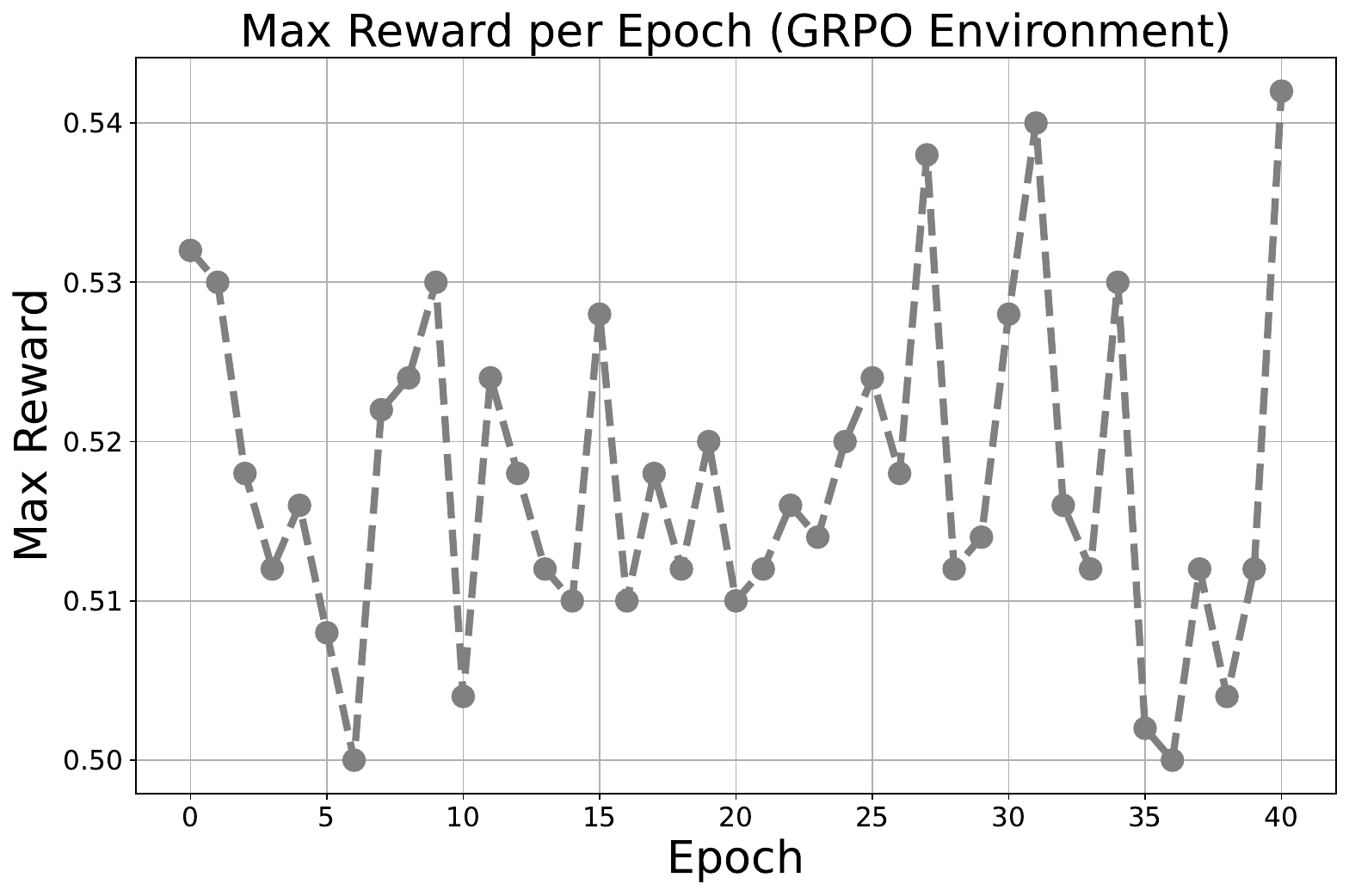}}
  \subfigure{\includegraphics[width=0.495\linewidth]{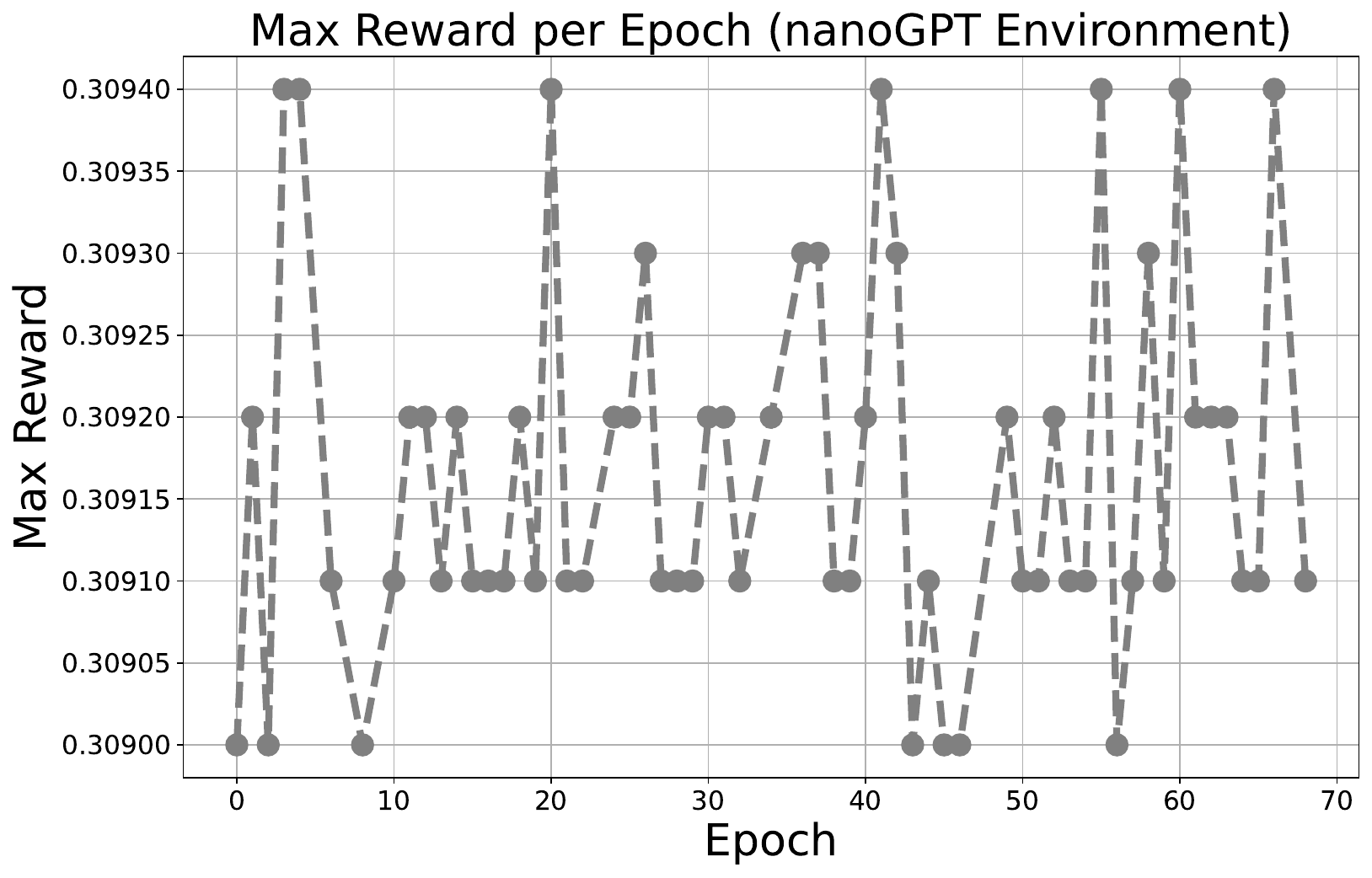}}

  \caption{Training curves of RL from execution reward. We plot the average reward per epoch in the upper row, and the max reward per epoch in the lower row. For the GRPO environment, the reward is the accuracy; for the nanoGPT environment, the reward is the reciprocal of the loss. The average reward increases, but not the max reward.}
  \label{fig:rl_results}
\end{figure}

\subsection{Reward design and experiment setup}

We use Qwen3-30B-A3B~\cite{Yang2025Qwen3TR} as the base model and finetune it with the standard GRPO algorithm~\cite{shao2024deepseekmathpushinglimitsmathematical}, motivated by its consistent empirical success on other verifiable domains.
Our prompt batch size is one because we have only one prompt per environment.
In the prompt, we provide the baseline codebase and ask the model to generate new ideas to improve the baseline (GRPO or nanoGPT).
This setup resembles prior work on RLVR from one training example~\cite{Wang2025ReinforcementLF}.

We use large group sizes to stabilize training: 256 for the post-training environment and 128 for the pre-training environment.
Because each GRPO idea runs on one GPU and each nanoGPT idea runs on 8 GPUs, these group sizes correspond to parallel execution on 256 GPUs (for GRPO) or 1024 GPUs (for nanoGPT) to obtain execution rewards for each batch of rollout ideas.
Each rollout consists of a thinking trace followed by the natural language idea.
We set a max output length of 8192 tokens for rollout sampling and feed only the extracted ideas to the automated executor, excluding the thinking trace.

For the post-training environment, we use the validation set accuracy of each rollout idea after execution as the reward. For ideas without a valid accuracy (i.e., failed execution due to code generation errors), we assign a reward of 0.
For the pre-training environment, we use the reciprocal of the validation loss ($\frac{1}{loss}$) as the reward and assign 0 to ideas with failed execution.
Our experiments use the Tinker API~\cite{Tinker}.

\subsection{Experiment results}

\paragraph{Positive training curves for average reward} The upper row of Figure~\ref{fig:rl_results} plots the average reward of all rollouts per training epoch.
We find that the average performance of generated ideas increases after sufficient training epochs on open-ended research environments.
The average accuracy on the GRPO environment increases from 0.253 to 0.343 after 40 training epochs (top left of Figure~\ref{fig:rl_results}); the average reward on the nanoGPT environment increases from 0.194 to 0.246 after 68 epochs (top right of Figure~\ref{fig:rl_results}), corresponding to a decrease in average validation loss from 5.150 to 4.066.
These training curves resemble prior findings on one-shot RLVR on other verifiable domains like math~\cite{Wang2025ReinforcementLF}.

\paragraph{The case of max reward} Despite reproducing the positive training curves observed in other domains, we argue that idea generation differs fundamentally from other verifiable tasks.
For scientific discovery, what matters is the upper bound of idea generation, not the average quality.
We want one breakthrough idea that dominates the baselines, not many safe ideas with a high average.
The lower row of Figure~\ref{fig:rl_results} plots the max reward of all rollouts at each training epoch.
The trend here is strikingly different---\textbf{the max reward fluctuates throughout RL training without a clear upward trend}.
This reveals a crucial limitation of standard GRPO for idea generation.
We next analyze why RL from execution reward improves average but not max reward.

\FloatBarrier

\section{Supplementary materials for test-time scaling}
\label{sec:s1-bf-appendix}

\subsection{Evaluation determinism}
\label{sec:s1-bf-eval-determinism}

We run our evaluations using vLLM~\citep{kwon2023efficientmemorymanagementlarge} because it is faster than alternatives.
However, even with identical random seeds and greedy sampling, evaluation scores can change across runs due to:
\begin{itemize}
\item different batch sizes,
\item continued generations, and
\item changes in tensor parallelism.
\end{itemize}

Because our model generates long reasoning traces before answering, small numeric changes can snowball into large differences.
We observe many generations that match exactly for thousands of tokens, then suddenly diverge on a single token, ultimately producing entirely different answers.
To mitigate this, we run our final evaluations in full precision unless otherwise indicated.

\FloatBarrier

\chapter{Supplementary materials for Chapter~\ref{chap:conclusion}}
\section{From Newton's law to Poisson's equation}
\label{sec:appendix-newton-poisson}

We derive Poisson's equation from Newton's law of gravitation~\eqref{eqn:newton-gravity}.
The first step toward a field-theoretic description is to pass from discrete masses to a continuous distribution.
Consider a test point mass $m$ at position $\mathbf{r}$ and a continuous matter distribution with mass density $\rho(\mathbf{r}')$ (kilograms per cubic meter).
An infinitesimal volume element $dV'$ at position $\mathbf{r}'$ contains mass $dM = \rho(\mathbf{r}')\,dV'$ and, by Newton's law~\eqref{eqn:newton-gravity}, exerts a force
\begin{equation}
\label{eqn:dF}
d\mathbf{F} = -\frac{Gm\,\rho(\mathbf{r}')\,dV'}{|\mathbf{r} - \mathbf{r}'|^3}\,(\mathbf{r} - \mathbf{r}')
\end{equation}
on $m$.
Integrating over all source matter and dividing by $m$ gives the \emph{gravitational field}, the force per unit test mass:
\begin{equation}
\label{eqn:grav-field}
\mathbf{g}(\mathbf{r}) = -G\!\int\!\frac{\rho(\mathbf{r}')\,(\mathbf{r} - \mathbf{r}')}{|\mathbf{r} - \mathbf{r}'|^3}\,dV'.
\end{equation}
This integral can be rewritten as the gradient of a scalar potential.
Define the \emph{gravitational potential} $\Phi(\mathbf{r})$---the potential energy per unit test mass (units of $\mathrm{m^2\,s^{-2}}$)---as
\begin{equation}
\label{eqn:newton-potential}
\Phi(\mathbf{r}) = -G\!\int\!\frac{\rho(\mathbf{r}')}{|\mathbf{r} - \mathbf{r}'|}\,dV'.
\end{equation}
For a single point mass $M$ at the origin, $\rho(\mathbf{r}') = M\,\delta^3(\mathbf{r}')$ and the integral reduces to $\Phi = -GM/r$.
Taking the gradient of~\eqref{eqn:newton-potential} with respect to $\mathbf{r}$ and passing it inside the integral:
\begin{equation}
-\nabla\Phi(\mathbf{r}) = G\!\int\!\rho(\mathbf{r}')\,\nabla\!\left(\frac{1}{|\mathbf{r} - \mathbf{r}'|}\right)dV' = -G\!\int\!\frac{\rho(\mathbf{r}')\,(\mathbf{r} - \mathbf{r}')}{|\mathbf{r} - \mathbf{r}'|^3}\,dV',
\end{equation}
where the second equality uses $\nabla(|\mathbf{r} - \mathbf{r}'|^{-1}) = -(\mathbf{r} - \mathbf{r}')/|\mathbf{r} - \mathbf{r}'|^3$.
The right-hand side is exactly the gravitational field~\eqref{eqn:grav-field}, so
\begin{equation}
\label{eqn:field-from-potential}
\mathbf{g} = -\nabla\Phi.
\end{equation}

The passage from the integral form~\eqref{eqn:grav-field} to a local differential equation is the crucial conceptual leap.
Applying the Laplacian $\nabla^2 = \nabla \cdot \nabla$ (with respect to $\mathbf{r}$) to both sides of~\eqref{eqn:newton-potential} and passing it inside the integral gives
\begin{equation}
\nabla^2 \Phi(\mathbf{r}) = -G\!\int\!\rho(\mathbf{r}')\,\nabla^2\!\left(\frac{1}{|\mathbf{r} - \mathbf{r}'|}\right)dV'.
\end{equation}
Using the identity $\nabla^2(1/|\mathbf{r} - \mathbf{r}'|) = -4\pi\,\delta^3(\mathbf{r} - \mathbf{r}')$, this becomes
\begin{equation}
\nabla^2 \Phi(\mathbf{r}) = 4\pi G\!\int\!\rho(\mathbf{r}')\,\delta^3(\mathbf{r} - \mathbf{r}')\,dV'.
\end{equation}
The Dirac delta sifts out the density at $\mathbf{r}$, collapsing the integral to yield
\begin{equation}
\nabla^2 \Phi = 4\pi G\,\rho,
\end{equation}
which is Poisson's equation~\eqref{eqn:poisson}.

\section{The metric tensor, stress-energy tensor, and spacetime curvature}
\label{sec:appendix-einstein-details}

We describe the tensorial ingredients of the Einstein field equations~\eqref{eqn:einstein-field} in detail.

\paragraph{The metric tensor $g_{\mu\nu}$.}
In Newtonian gravity, the potential $\Phi$ is a single number at each point in space.
Einstein replaces it with the \emph{metric tensor} $g_{\mu\nu}$, a $4 \times 4$ symmetric matrix-valued function of the spacetime coordinates $(t, x, y, z)$.
The metric tensor generalizes the Pythagorean theorem to curved spacetime: it defines the infinitesimal squared interval between two nearby events as
\begin{equation}
\label{eqn:metric-interval}
ds^2 = g_{\mu\nu}\,dx^\mu\,dx^\nu,
\end{equation}
where we use the Einstein summation convention (summing over repeated indices $\mu, \nu = 0,1,2,3$) throughout.
In flat spacetime with no gravity, the metric reduces to the Minkowski metric $\eta_{\mu\nu} = \mathrm{diag}(-1, 1, 1, 1)$, and the interval is $ds^2 = -c^2 dt^2 + dx^2 + dy^2 + dz^2$.
The presence of mass and energy warps the metric away from $\eta_{\mu\nu}$, and this warping is what we experience as gravity.

Because $g_{\mu\nu}$ is symmetric ($g_{\mu\nu} = g_{\nu\mu}$), it has 10 independent components---matching the 10 independent components of the field equation we arrive at below.
Its inverse is denoted $g^{\mu\nu}$ and satisfies $g_{\mu\alpha}\,g^{\alpha\nu} = \delta_\mu^\nu$, where $\delta_\mu^\nu$ is the Kronecker delta.

To see the metric tensor in action, consider the Schwarzschild solution---the geometry outside a spherical, non-rotating mass $M$---which in spherical coordinates $(t, r, \theta, \phi)$ takes the form
\begin{equation}
\label{eqn:schwarzschild}
ds^2 = -\left(1 - \frac{2GM}{rc^2}\right)c^2\,dt^2 + \left(1 - \frac{2GM}{rc^2}\right)^{-1}dr^2 + r^2\,d\theta^2 + r^2\sin^2\theta\,d\phi^2.
\end{equation}
The components of $g_{\mu\nu}$ are functions of the coordinates: the flow of time ($g_{00}$) and the measurement of radial distance ($g_{11}$) both depend on $r$, encoding gravitational time dilation and spatial curvature.

\paragraph{The stress-energy tensor $T_{\mu\nu}$.}
In Poisson's equation, the source of gravity is the scalar mass density $\rho$.
Einstein replaces it with the \emph{stress-energy tensor} $T_{\mu\nu}$, a $4 \times 4$ symmetric tensor encoding the matter and energy content at each point in spacetime: each component $T_{\mu\nu}$ is the flux of the $\mu$-th component of four-momentum through a surface of constant $x^\nu$.

The time-time component $T_{00}$ is the energy density.
Rest mass $m$ carries energy $mc^2$, so mass density $\rho$ contributes $\rho c^2$ to the energy density---but $T_{00}$ is the \emph{total} energy density, which also includes contributions from electromagnetic fields, thermal motion, and kinetic energy.
For ordinary matter at rest, $T_{00} = \rho c^2$, recovering the Newtonian source.

The mixed components $T_{0i} = T_{i0}$ encode momentum density and energy flux---two descriptions of the same physical quantity.
Moving energy carries momentum ($p = Ev/c^2$), so the flux of energy across a surface is identically the density of momentum in the corresponding direction.

The spatial block $T_{ij}$ encodes stress: the rate at which the $i$-th component of momentum is transported across a surface perpendicular to direction $j$.
For a perfect fluid, $T_{ij} = p\,\delta_{ij}$, where $p$ is the isotropic pressure; viscous fluids additionally have off-diagonal shear components.

In Newtonian gravity, only mass generates gravity.
In general relativity, energy, momentum, and pressure all curve spacetime.

\paragraph{From the metric to curvature: $\Gamma$, $R^{\alpha}{}_{\beta\mu\nu}$, $R_{\mu\nu}$, and $\mathcal{R}$.}
The left-hand side of Einstein's equation must express how spacetime is curved.
We build this curvature entirely from $g_{\mu\nu}$ and its derivatives, through a four-step pipeline.

\emph{Step 1: Christoffel symbols.}
The Christoffel symbols $\Gamma^\alpha_{\mu\nu}$ encode how the coordinate basis vectors change from point to point in curved spacetime.
They are computed from the first derivatives of the metric:
\begin{equation}
\label{eqn:christoffel}
\Gamma^\alpha_{\mu\nu} = \frac{1}{2}\,g^{\alpha\beta}\!\left(\partial_\mu g_{\nu\beta} + \partial_\nu g_{\beta\mu} - \partial_\beta g_{\mu\nu}\right),
\end{equation}
where $\partial_\mu \equiv \partial/\partial x^\mu$.
Christoffel symbols are not tensors---they can be made to vanish at any single point by choosing appropriate coordinates (the mathematical expression of the equivalence principle, which states that gravity is locally indistinguishable from acceleration).
However, their derivatives, combined appropriately, yield genuine tensors that measure intrinsic curvature.

\emph{Step 2: the Riemann curvature tensor.}
The Riemann tensor $R^{\alpha}{}_{\beta\mu\nu}$ is the fundamental measure of spacetime curvature, constructed from the Christoffel symbols and their first derivatives:
\begin{equation}
\label{eqn:riemann}
R^{\alpha}{}_{\beta\mu\nu} = \partial_\mu \Gamma^\alpha_{\nu\beta} - \partial_\nu \Gamma^\alpha_{\mu\beta} + \Gamma^\alpha_{\mu\gamma}\,\Gamma^\gamma_{\nu\beta} - \Gamma^\alpha_{\nu\gamma}\,\Gamma^\gamma_{\mu\beta}.
\end{equation}
Because the Christoffel symbols involve first derivatives of $g_{\mu\nu}$, the Riemann tensor involves \emph{second} derivatives of the metric.
Physically, it measures tidal forces: if a cloud of freely falling particles drifts through curved spacetime, the Riemann tensor determines how the cloud is stretched and squeezed.
The tensor has 256 components in four dimensions, but symmetries reduce the independent components to 20.

\emph{Step 3: the Ricci tensor.}
We obtain the Ricci tensor $R_{\mu\nu}$ by contracting (tracing over) one pair of indices of the Riemann tensor:
\begin{equation}
\label{eqn:ricci-tensor}
R_{\mu\nu} = R^{\alpha}{}_{\mu\alpha\nu}.
\end{equation}
This contraction distills the 20-component Riemann tensor into a 10-component symmetric tensor.
$R_{\mu\nu}$ isolates the volume-changing part of the curvature: it measures how the volume of a small ball of freely falling particles shrinks or grows as the ball moves through spacetime, discarding the shape-distorting tidal effects captured by the full Riemann tensor.

\emph{Step 4: the Ricci scalar.}
The Ricci scalar $\mathcal{R}$ is the trace of the Ricci tensor:
\begin{equation}
\label{eqn:ricci-scalar}
\mathcal{R} = g^{\mu\nu}\,R_{\mu\nu}.
\end{equation}
$\mathcal{R}$ compresses the entire curvature into a single number at each point---positive on sphere-like regions (where volumes are smaller than in flat space) and negative on saddle-like regions (where volumes are larger).

\paragraph{The Einstein tensor.}
The unique symmetric, divergence-free tensor constructible from $g_{\mu\nu}$ and its first and second derivatives is the Einstein tensor $G_{\mu\nu} \equiv R_{\mu\nu} - \frac{1}{2}\mathcal{R}g_{\mu\nu}$.
The notation $\nabla_\mu G^{\mu\nu}$ denotes the covariant divergence: the covariant derivative $\nabla_\mu$ is summed over the repeated index $\mu$, generalizing the vector divergence $\partial_i F^i$ to a rank-2 tensor in curved spacetime.
The identity $\nabla_\mu G^{\mu\nu} = 0$ (the Bianchi identity) ensures compatibility with the conservation of energy and momentum ($\nabla_\mu T^{\mu\nu} = 0$).

Stripping away the shorthand, the field equations~\eqref{eqn:einstein-field} are a system of 10 coupled, nonlinear, second-order partial differential equations for the 10 independent components of $g_{\mu\nu}$, sourced by $T_{\mu\nu}$.
Every term on the left-hand side---$R_{\mu\nu}$, $\mathcal{R}$, and $g_{\mu\nu}$ itself---is built from $g_{\mu\nu}$ and its derivatives.
The nonlinearity stems from the Christoffel symbols appearing both inside derivatives and multiplied against each other in the Riemann tensor~\eqref{eqn:riemann}, so gravitational perturbations do not simply superpose---the field's behavior depends on the field itself.

\section{The Newtonian limit}
\label{sec:appendix-newtonian-limit}

As a consistency check, Einstein's equations must reduce to Newton's in the appropriate limit.
We impose three physical assumptions that characterize everyday gravitational environments like the solar system.
\begin{enumerate}[leftmargin=16pt]
\item \emph{Weak field}: spacetime is nearly flat, so we can write $g_{\mu\nu} = \eta_{\mu\nu} + h_{\mu\nu}$ where $|h_{\mu\nu}| \ll 1$.
\item \emph{Slow motion}: all matter moves much slower than light, $v \ll c$.
\item \emph{Static field}: the gravitational field does not change in time, $\partial_0 h_{\mu\nu} = 0$.
\end{enumerate}

Under these assumptions, the stress-energy tensor is dominated by its time-time component $T_{00} \approx \rho c^2$, with all other components negligible.
Its trace is $T = g^{\mu\nu}T_{\mu\nu} \approx g^{00}T_{00} \approx -\rho c^2$.

On the geometric side, to first order in $h_{\mu\nu}$, the Christoffel symbols~\eqref{eqn:christoffel} reduce to $\Gamma^\alpha_{\mu\nu} \approx \frac{1}{2}\,\eta^{\alpha\beta}(\partial_\mu h_{\nu\beta} + \partial_\nu h_{\beta\mu} - \partial_\beta h_{\mu\nu})$.
Dropping all products of Christoffel symbols in the Riemann tensor~\eqref{eqn:riemann} and all time derivatives ($\partial_0 h_{\mu\nu} = 0$), the Ricci tensor's 00-component becomes
\begin{equation}
R_{00} \approx \partial_k \Gamma^k_{00} = \partial_k\!\left[-\tfrac{1}{2}\,\delta^{kl}\,\partial_l h_{00}\right] = -\tfrac{1}{2}\,\nabla^2 h_{00}.
\end{equation}
To obtain the source side, we trace-reverse the field equations~\eqref{eqn:einstein-field}.
Contracting both sides with $g^{\mu\nu}$ gives $-\mathcal{R} = \frac{8\pi G}{c^4}\,T$ (since $g^{\mu\nu}G_{\mu\nu} = \mathcal{R} - 2\mathcal{R} = -\mathcal{R}$), so $\mathcal{R} = -\frac{8\pi G}{c^4}\,T$.
Substituting back into~\eqref{eqn:einstein-field} yields the trace-reversed form $R_{\mu\nu} = \frac{8\pi G}{c^4}(T_{\mu\nu} - \frac{1}{2}\,T\,g_{\mu\nu})$.
Evaluating the 00-component with $T_{00} \approx \rho c^2$, $T \approx -\rho c^2$, and $g_{00} \approx \eta_{00} = -1$:
\begin{equation}
R_{00} = \frac{8\pi G}{c^4}\!\left(\rho c^2 - \tfrac{1}{2}(-\rho c^2)(-1)\right) = \frac{8\pi G}{c^4}\!\cdot\!\tfrac{1}{2}\,\rho c^2 = \frac{4\pi G\rho}{c^2}.
\end{equation}
Equating the two expressions for $R_{00}$:
\begin{equation}
\nabla^2 h_{00} = -\frac{8\pi G\rho}{c^2}.
\end{equation}
The final step is to identify $h_{00}$ with the Newtonian potential.
In general relativity, a freely falling particle follows a geodesic: $\frac{d^2 x^\mu}{d\tau^2} + \Gamma^\mu_{\alpha\beta}\,\frac{dx^\alpha}{d\tau}\,\frac{dx^\beta}{d\tau} = 0$.
For a slowly moving particle ($dx^i/d\tau \ll c\,dt/d\tau$), the sum over $\alpha,\beta$ is dominated by $\alpha = \beta = 0$, reducing the spatial components to
\begin{equation}
\frac{d^2 x^i}{dt^2} \approx -c^2\,\Gamma^i_{00} = -c^2\!\cdot\!\left(-\tfrac{1}{2}\,\partial_i h_{00}\right) = \frac{c^2}{2}\,\partial_i h_{00},
\end{equation}
where $\Gamma^i_{00} = -\frac{1}{2}\,\delta^{ij}\,\partial_j h_{00}$ follows from the linearized, static Christoffel symbol.
Comparing with Newton's $d^2 x^i/dt^2 = -\partial_i \Phi$ identifies $h_{00} = -2\Phi/c^2$.
Substituting:
\begin{equation}
\nabla^2\!\left(-\frac{2\Phi}{c^2}\right) = -\frac{8\pi G\rho}{c^2} \qquad\Longrightarrow\qquad \nabla^2 \Phi = 4\pi G\rho.
\end{equation}
This is exactly Poisson's equation~\eqref{eqn:poisson}: in the weak-field, slow-motion, static limit, the full machinery of curved spacetime collapses to the scalar potential theory of Newtonian gravity.

\section{Deriving the Friedmann equations}
\label{sec:appendix-friedmann-derivation}

We derive the Friedmann equations by substituting the FLRW metric~\eqref{eqn:flrw-metric} into the Einstein field equations~\eqref{eqn:einstein-field}.
We write $\dot{a} \equiv da/dt$ and $\ddot{a} \equiv d^2a/dt^2$ throughout.

\paragraph{Christoffel symbols.}
Since the metric~\eqref{eqn:flrw-metric} is diagonal, the Christoffel symbols~\eqref{eqn:christoffel} simplify: $g^{\alpha\sigma}$ is nonzero only when $\alpha = \sigma$, and since $g_{00} = -c^2$ is constant and $g_{0i} = 0$, the only nonzero symbols involve one time index and one or two spatial indices.

\emph{Case 1: two spatial lower indices.}
\begin{equation}
\Gamma^0_{ij} = \tfrac{1}{2}\,g^{00}\bigl(\partial_i g_{j0} + \partial_j g_{i0} - \partial_0 g_{ij}\bigr) = \tfrac{1}{2}\!\left(-\frac{1}{c^2}\right)\!\left(-2a\dot{a}\,\delta_{ij}\right) = \frac{a\dot{a}}{c^2}\,\delta_{ij},
\end{equation}
since the first two terms vanish ($g_{i0} = 0$) and $\partial_t(a^2\delta_{ij}) = 2a\dot{a}\,\delta_{ij}$.
Explicitly: $\Gamma^0_{11} = \Gamma^0_{22} = \Gamma^0_{33} = a\dot{a}/c^2$.

\emph{Case 2: one temporal and one spatial lower index.}
\begin{equation}
\Gamma^i_{0j} = \tfrac{1}{2}\,g^{ii}\bigl(\partial_0 g_{ji} + \partial_j g_{0i} - \partial_i g_{0j}\bigr) = \tfrac{1}{2}\!\left(\frac{1}{a^2}\right)\!\left(2a\dot{a}\,\delta_{ji}\right) = \frac{\dot{a}}{a}\,\delta^i{}_j.
\end{equation}
Explicitly: $\Gamma^1_{01} = \Gamma^2_{02} = \Gamma^3_{03} = \dot{a}/a$.
The quantity $H(t) \equiv \dot{a}/a$ is the \emph{Hubble parameter}.

\emph{All other Christoffel symbols vanish}: $\Gamma^\alpha_{00} = 0$ for all $\alpha$ (because $g_{00}$ is constant and $g_{0i} = 0$), and $\Gamma^i_{jk} = 0$ for all spatial $i,j,k$ (because $a^2\delta_{ij}$ has no spatial derivatives).

\paragraph{Ricci tensor.}
Expanding the contraction $R_{\mu\nu} = R^\alpha{}_{\mu\alpha\nu}$ via the Riemann tensor~\eqref{eqn:riemann} gives the working formula
\begin{equation}
R_{\mu\nu} = \partial_\alpha \Gamma^\alpha_{\mu\nu} - \partial_\nu \Gamma^\alpha_{\mu\alpha} + \Gamma^\alpha_{\alpha\beta}\Gamma^\beta_{\mu\nu} - \Gamma^\alpha_{\nu\beta}\Gamma^\beta_{\mu\alpha}.
\end{equation}
We compute $R_{00}$ and $R_{11}$; the remaining spatial components follow by isotropy ($R_{22} = R_{33} = R_{11}$).

\emph{Computing $R_{00}$.}
Term 1: $\partial_\alpha \Gamma^\alpha_{00} = 0$, since all $\Gamma^\alpha_{00}$ vanish.
Term 2: $-\partial_0 \Gamma^\alpha_{0\alpha}$.
The nonzero contributions come from spatial $\alpha$: $\sum_{\alpha=1}^{3}\Gamma^\alpha_{0\alpha} = 3\dot{a}/a$, so
\begin{equation}
-\frac{d}{dt}\!\left(\frac{3\dot{a}}{a}\right) = -3\!\left(\frac{\ddot{a}}{a} - \frac{\dot{a}^2}{a^2}\right) = -\frac{3\ddot{a}}{a} + \frac{3\dot{a}^2}{a^2}.
\end{equation}
Term 3: $\Gamma^\alpha_{\alpha\beta}\Gamma^\beta_{00} = 0$, since all $\Gamma^\beta_{00}$ vanish.
Term 4: $-\Gamma^\alpha_{0\beta}\Gamma^\beta_{0\alpha}$, summed over $\alpha,\beta$.
Both factors require spatial indices, each contributing $\dot{a}/a$ when $\alpha = \beta$:
\begin{equation}
-\sum_{\alpha=1}^{3}\!\left(\frac{\dot{a}}{a}\right)^{\!2} = -\frac{3\dot{a}^2}{a^2}.
\end{equation}
Combining all four terms:
\begin{equation}
\label{eqn:R00-flrw}
\boxed{\;R_{00} = -\frac{3\ddot{a}}{a}.\;}
\end{equation}

\emph{Computing $R_{11}$.}
Term 1: $\partial_\alpha \Gamma^\alpha_{11}$.
The only nonzero Christoffel symbol is $\Gamma^0_{11} = a\dot{a}/c^2$, giving
\begin{equation}
\frac{d}{dt}\!\left(\frac{a\dot{a}}{c^2}\right) = \frac{1}{c^2}\bigl(\dot{a}^2 + a\ddot{a}\bigr).
\end{equation}
Term 2: $-\partial_1 \Gamma^\alpha_{1\alpha} = 0$, since the Christoffel symbols depend only on $t$.
Term 3: $\Gamma^\alpha_{\alpha\beta}\Gamma^\beta_{11}$.
Since $\Gamma^\beta_{11}$ is nonzero only for $\beta = 0$, this becomes $\sum_\alpha \Gamma^\alpha_{\alpha 0}\cdot a\dot{a}/c^2$.
The spatial contributions are $\Gamma^1_{10} + \Gamma^2_{20} + \Gamma^3_{30} = 3\dot{a}/a$, giving $3\dot{a}^2/c^2$.
Term 4: $-\Gamma^\alpha_{1\beta}\Gamma^\beta_{1\alpha}$, summed over $\alpha,\beta$.
The two nonzero contributions are $(\alpha,\beta) = (0,1)$ and $(1,0)$, each giving $\dot{a}^2/c^2$:
\begin{equation}
-\frac{a\dot{a}}{c^2}\!\cdot\!\frac{\dot{a}}{a} - \frac{\dot{a}}{a}\!\cdot\!\frac{a\dot{a}}{c^2} = -\frac{2\dot{a}^2}{c^2}.
\end{equation}
Combining all four terms:
\begin{equation}
\label{eqn:R11-flrw}
\boxed{\;R_{11} = \frac{1}{c^2}\bigl(a\ddot{a} + 2\dot{a}^2\bigr).\;}
\end{equation}

\paragraph{Ricci scalar.}
Contracting with the inverse metric:
\begin{align}
\mathcal{R} &= g^{00}R_{00} + 3\,g^{11}R_{11} = \left(-\frac{1}{c^2}\right)\!\left(-\frac{3\ddot{a}}{a}\right) + 3\!\left(\frac{1}{a^2}\right)\!\cdot\!\frac{1}{c^2}\bigl(a\ddot{a} + 2\dot{a}^2\bigr) \notag\\
&= \frac{3\ddot{a}}{c^2 a} + \frac{3\ddot{a}}{c^2 a} + \frac{6\dot{a}^2}{c^2 a^2}.
\end{align}
\begin{equation}
\label{eqn:ricci-scalar-flrw}
\boxed{\;\mathcal{R} = \frac{6}{c^2}\!\left[\frac{\ddot{a}}{a} + \frac{\dot{a}^2}{a^2}\right].\;}
\end{equation}

\paragraph{Einstein tensor.}
The 00-component:
\begin{align}
G_{00} &= R_{00} - \tfrac{1}{2}\,g_{00}\,\mathcal{R} = -\frac{3\ddot{a}}{a} - \tfrac{1}{2}(-c^2)\!\cdot\!\frac{6}{c^2}\!\left[\frac{\ddot{a}}{a} + \frac{\dot{a}^2}{a^2}\right] \notag\\
&= -\frac{3\ddot{a}}{a} + 3\!\left[\frac{\ddot{a}}{a} + \frac{\dot{a}^2}{a^2}\right].
\end{align}
The $\ddot{a}$ terms cancel completely:
\begin{equation}
\label{eqn:G00-flrw}
\boxed{\;G_{00} = \frac{3\dot{a}^2}{a^2}.\;}
\end{equation}

The 11-component:
\begin{align}
G_{11} &= R_{11} - \tfrac{1}{2}\,g_{11}\,\mathcal{R} = \frac{1}{c^2}\bigl(a\ddot{a} + 2\dot{a}^2\bigr) - \tfrac{1}{2}\,a^2\!\cdot\!\frac{6}{c^2}\!\left[\frac{\ddot{a}}{a} + \frac{\dot{a}^2}{a^2}\right] \notag\\
&= \frac{1}{c^2}\bigl(a\ddot{a} + 2\dot{a}^2\bigr) - \frac{3}{c^2}\bigl(a\ddot{a} + \dot{a}^2\bigr).
\end{align}
\begin{equation}
\label{eqn:G11-flrw}
\boxed{\;G_{11} = -\frac{1}{c^2}\bigl(2a\ddot{a} + \dot{a}^2\bigr).\;}
\end{equation}
By isotropy, $G_{22} = G_{33} = G_{11}$.

\paragraph{Field equations.}
Substituting into the field equations~\eqref{eqn:einstein-field} with $T_{00} = \rho c^2$ and $T_{ij} = 0$:

The 00-equation gives
\begin{equation}
\frac{3\dot{a}^2}{a^2} = \frac{8\pi G\rho}{c^2},
\end{equation}
which is the first Friedmann equation~\eqref{eqn:friedmann-I}.

The 11-equation, with $T_{11} = 0$, gives $2a\ddot{a} + \dot{a}^2 = 0$, or equivalently
\begin{equation}
\frac{2\ddot{a}}{a} + \frac{\dot{a}^2}{a^2} = 0,
\end{equation}
which is the second Friedmann equation~\eqref{eqn:friedmann-II}.

\section{Solving the Friedmann equations}
\label{sec:appendix-friedmann-solution}

We solve the Friedmann equations~\eqref{eqn:friedmann-I} and~\eqref{eqn:friedmann-II} for the scale factor $a(t)$ and the matter density $\rho(t)$.

From the second Friedmann equation~\eqref{eqn:friedmann-II}, multiplying through by $a^2$:
\begin{equation}
2a\ddot{a} + \dot{a}^2 = 0.
\end{equation}
Substituting $p = \dot{a}$, so $\ddot{a} = p\,dp/da$, and dividing by $p \neq 0$:
\begin{equation}
2a\,\frac{dp}{da} + p = 0.
\end{equation}
Setting $u = p^2 = \dot{a}^2$ gives $a\,du/da + u = 0$, a separable equation with solution $u = C/a$ for some constant $C > 0$.
Therefore $\dot{a}^2 = C/a$, and separating variables:
\begin{equation}
\sqrt{a}\;da = \sqrt{C}\;dt \qquad\Longrightarrow\qquad \tfrac{2}{3}\,a^{3/2} = \sqrt{C}\,(t - t_0).
\end{equation}
\begin{equation}
\label{eqn:a-of-t}
\boxed{\;a(t) \propto t^{2/3}.\;}
\end{equation}
The universe begins at $a = 0$---a singularity where all distances vanish and the density is infinite---then expands forever, decelerating but never stopping.

From the first Friedmann equation~\eqref{eqn:friedmann-I} and $\dot{a}^2 = C/a$:
\begin{equation}
\label{eqn:rho-of-t}
\rho = \frac{3c^2 C}{8\pi G\,a^3} \;\propto\; \frac{1}{a^3} \;\propto\; \frac{1}{t^2}.
\end{equation}
The density decreases as the cube of the scale factor: the total mass in any comoving volume is conserved ($\rho\,a^3 = \text{const}$), while the volume grows as $a^3$.
This conservation law also follows independently from the Bianchi identity $\nabla_\mu G^{\mu\nu} = 0$, which implies the continuity equation $\dot{\rho} + 3(\dot{a}/a)\rho = 0$, or equivalently $d(\rho\,a^3)/dt = 0$.

\bibliographystyle{plainnat}
\bibliography{reference}
\end{document}